\documentclass{article}
\pdfoutput=1  
\usepackage{etoolbox}  

\usepackage{booktabs,color,xcolor}    
\newtoggle{colt}
\togglefalse{colt} 
\newtoggle{arxiv}
\toggletrue{arxiv} 
\newtoggle{nips}
\togglefalse{nips}

\usepackage{boxedminipage}
\usepackage{multirow,nicefrac}
\usepackage{makecell,upgreek}
\usepackage{footnote}
\usepackage{longtable}
\usepackage[shortlabels]{enumitem}
\usepackage{tablefootnote}
\usepackage[T1]{fontenc}
\usepackage{verbatim} 
\usepackage[utf8]{inputenc}

 \usepackage{stackengine}

\usepackage{booktabs}   
\usepackage{float}

\usepackage[rightcaption, raggedright]{sidecap}

\usepackage{xcolor}              
\usepackage{preamble/color-edits}
\addauthor{kz}{brown}
\addauthor{ms}{magenta}
\addauthor{ju}{cyan}
\addauthor{jp}{purple}

\iftoggle{arxiv}
{
\usepackage{amsthm}
\usepackage{subfig}

  \usepackage[breaklinks=true]{hyperref}
  \usepackage[margin=1in, headheight=10mm]{geometry}
  \usepackage{times}

}
{
  
}
  
\usepackage[noend]{algpseudocode}
  \usepackage{algorithm}

\usepackage{amssymb,mathrsfs}
\usepackage{mathtools}
\DeclareMathSymbol{\shortminus}{\mathbin}{AMSa}{"39}

	\usepackage{natbib}

\usepackage{prettyref}

\usepackage[capitalise,nameinlink]{cleveref}
\Crefname{equation}{Eq.}{Eqs.}
\Crefname{assumption}{Assumption}{Assumptions}
\Crefname{lem}{Lemma}{Lemmas}
\Crefname{condition}{Condition}{Conditions}
\Crefname{claim}{Claim}{Claims}

\makeatletter
\newcommand\incircbin
{%
  \mathpalette\@incircbin
}
\newcommand\@incircbin[2]
{%
  \mathbin%
  {%
    \ooalign{\hidewidth$#1#2$\hidewidth\crcr$#1\bigcirc$}%
  }%
}
\newcommand{\fauxtimes}{\raisebox{0.1345ex}{\scalebox{0.7421}{$\incircbin{\times}$}}}

\usepackage{breakcites,bm}

\hypersetup{colorlinks,citecolor=blue,linkcolor = blue}
\usepackage{latexsym}
\usepackage{relsize}
\usepackage{thm-restate}
\usepackage{appendix}

\usepackage{xcolor}
\usepackage{dsfont}

\newcommand{\N}{\mathbb{N}}

\newcommand{\R}{\mathbb{R}}

\numberwithin{equation}{section}
\newcommand\numberthis{\addtocounter{equation}{1}\tag{\theequation}}

\newcommand{\ftil}{\tilde{f}}

\newcommand{\rmd}{\mathrm{d}}

\def\bB{\mathbf{B}}

\def\bu{\mathbf{u}}

\def\bz{\mathbf{z}}

\def\bu{\mathbf{u}}
\def\bv{\mathbf{v}}

\def\bw{\mathbf{w}}
\def\bA{\mathbf{A}}
\def\bS{\mathbf{S}}

\def\bP{\mathbf{P}}
\def\bQ{\mathbf{Q}}
\def\bV{\mathbf{V}}

\DeclareFontFamily{U}{mathx}{\hyphenchar\font45}
\DeclareFontShape{U}{mathx}{m}{n}{
      <5> <6> <7> <8> <9> <10>
      <10.95> <12> <14.4> <17.28> <20.74> <24.88>
      mathx10
      }{}
\DeclareSymbolFont{mathx}{U}{mathx}{m}{n}
\DeclareFontSubstitution{U}{mathx}{m}{n}
\DeclareMathAccent{\widecheck}{0}{mathx}{"71}
\DeclareMathAccent{\wideparen}{0}{mathx}{"75}

\newcommand{\ignore}[1]{}

\DeclareMathOperator{\BigOm}{\mathcal{O}}

\newcommand{\BigOh}[1]{\BigOm\left({#1}\right)}

\newcommand{\matx}{\mathbf{x}}

\newcommand{\maty}{\mathbf{y}}

\newcommand{\iidsim}{\overset{\mathrm{i.i.d}}{\sim}}

\newcommand{\op}{\mathrm{op}}

\newcommand{\fro}{\mathrm{F}}

\iftoggle{colt}
{
  \newtheorem{lem}{Lemma}[section]
  \newtheorem{cor}{Corollary}[section]

  \newtheorem{prop}{Proposition}[section]
   \newtheorem{defn}{Definition}[section]
      \newtheorem{rem}{Remark}[section]

  \newtheorem{assumption}{Assumption}[section]
  \newtheorem{condition}{Condition}[section]
  \newtheorem{fact}{Fact}[section]
  \newtheorem{claim}{Claim}[section]
  \newtheorem*{theorem*}{Theorem}

}
{	
  \newtheorem{theorem}{Theorem}
	\newtheorem{lem}{Lemma}[section]
  \newtheorem{cor}{Corollary}[section]

  \newtheorem{prop}{Proposition}[section]
   \newtheorem{defn}{Definition}[section]
      \newtheorem{rem}{Remark}[section]

  \newtheorem{assumption}{Assumption}[section]
  \newtheorem{condition}{Condition}[section]
  \newtheorem{fact}{Fact}[section]
  \newtheorem{claim}{Claim}[section]
  \newtheorem*{theorem*}{Theorem}

\newtheoremstyle{named}{}{}{\itshape}{}{\bfseries}{}{.5em}{\Cref{#3} {\normalfont (informal)} }
{}
\theoremstyle{named}

\theoremstyle{plain}
}

\makeatletter
\newcommand{\neutralize}[1]{\expandafter\let\csname c@#1\endcsname\count@}
\makeatother

  \newenvironment{lemmod}[2]
  {\renewcommand{\thelem}{\ref*{#1}#2}%
   \neutralize{lem}\phantomsection
   \begin{lem}}
  {\end{lem}}

    \newenvironment{propmod}[2]
  {\renewcommand{\theprop}{\ref*{#1}#2}%
   \neutralize{prop}\phantomsection
   \begin{prop}}
  {\end{prop}}

   \newenvironment{asmmod}[2]
  {\renewcommand{\theassumption}{\ref*{#1}#2}%
   \neutralize{assumption}\phantomsection
   \begin{assumption}}
  {\end{assumption}}

\DeclareMathAlphabet{\mathbfsf}{\encodingdefault}{\sfdefault}{bx}{n}

\DeclareMathOperator*{\argmin}{arg\,min}

\let\Pr\relax
\DeclareMathOperator{\Pr}{\mathbb{P}}

\newcommand{\ceil}[1]{\lceil #1 \rceil}

\newcommand{\trace}{\mathrm{tr}}

\newcommand{\poly}{\mathrm{poly}}

\newcommand{\eps}{\varepsilon}

\newcommand{\half}{\frac{1}{2}}

\renewcommand{\leq}{~\le~}
\renewcommand{\geq}{~\ge~}

\let\oldtfrac\tfrac
\renewcommand{\tfrac}[2]{\smash{\oldtfrac{#1}{#2}}}

\let\nablaold\nabla
\renewcommand{\nabla}{\nablaold\mkern-2.5mu}

\newcommand{\Exp}{\mathbb{E}}

\usepackage[mathscr]{euscript}

\newcommand{\lesssimst}{\lesssim_{\star}}

\newcommand{\I}{\mathbb{I}}

\newcommand{\fronorm}[1]{\|#1\|_{\fro}}

\newcommand{\circnorm}[1]{\|#1\|_{\circ}}

\newcommand{\eye}{\mathbf{I}}

\newcommand{\pd}[1]{\mathbb{S}_{>}^{#1}}
\newcommand{\psd}[1]{\mathbb{S}_{\ge}^{#1}}

\newcommand{\sigst}{\sigma^{\star}}

\newcommand{\textnrm}[1]{\text{\normalfont #1}}

\newcommand{\bSigma}{\mathbf{\Sigma}}
\newcommand{\sigbar}{\bar{\sigma}}

    \newcommand{\btilSigma}{\tilde{\bSigma}}

   \newcommand{\balproj}{\mathsf{Proj}_{\mathrm{bal}}}

\newcommand{\bstSigma}{\mathbf{\Sigma}^{\star}}
\newcommand{\bDelta}{\mathbf{\Delta}}
 \newcommand{\bLambda}{\mathbf{\Lambda}}

\newcommand{\btilx}{\check{\mathbf{x}}}
\newcommand{\btily}{\check{\mathbf{y}}}
\newcommand{\bdelst}{\bmsf{\delta}^\star}
\newcommand{\epstil}{\tilde{\epsilon}}
\newcommand{\epsSig}{\epsilon_{\Sigma}}
\newcommand{\nnone}{n_1}
\newcommand{\nntwo}{n_2}
\newcommand{\hatSigfmu}{\hat{\bSigma}_{\tilde{f},\mu}}
\newcommand{\hatSiggmu}{\hat{\bSigma}_{\tilde{g},\mu}}
\newcommand{\Sigfmu}{{\bSigma}_{\tilde{f},\mu}}
\newcommand{\Siggmu}{{\bSigma}_{\tilde{g},\mu}}

\newcommand{\distx}{\cD_{\scrX}}
\newcommand{\disty}{\cD_{\scrY}}

\newcommand{\capac}{\mathscr{M}}
\newcommand{\gtil}{\tilde{g}}
\newcommand{\Psibal}{\Psi_{\mathrm{bal}}}
\newcommand{\seprank}{\mathsf{sep}\textnrm{-}\mathsf{rank}}
\newcommand{\covbal}{\mathsf{CovBal}}

\newcommand{\kapdens}{\kappa_{\mathrm{den}}}

\newcommand{\Sigst}{\bmsf{\Sigma}_{1\otimes 1}^\star}

\newcommand{\Sigf}{\bmsf{\Sigma}_{\fst}}
\newcommand{\Sigg}{\bmsf{\Sigma}_{\gst}}
\newcommand{\tailsf}{\bm{\mathsf{tail}}^\star}

\renewcommand{\sigst}{\sigma^\star}
\newcommand{\sigstk}{\sigma^\star_k}

\newcommand{\epstrain}{\epsilon_{\mathrm{trn}}}
\newcommand{\bsfeye}{\bmsf{I}}

\newcommand{\cdxone}{\cdx{1}}
\newcommand{\cdyone}{\cdy{1}}
\newcommand{\cdxtwo}{\cdx{2}}
\newcommand{\cdytwo}{\cdy{2}}

\newcommand{\kapcov}{\kappa_{\mathrm{cov}}}

\newcommand{\hilnorm}[1]{\|#1\|_{\hilspace}}
\newcommand{\Dtest}{\cD_{\mathrm{test}}}
\newcommand{\Dtrain}{\cD_{\mathrm{train}}}
\newcommand{\Unif}{\mathsf{Unif}}

\newcommand{\fstk}{\fst_k}
\newcommand{\gstk}{\gst_k}
\newcommand{\hst}{h^{\star}}
\newcommand{\Fclass}{\mathscr{F}}
\newcommand{\Gclass}{\mathscr{G}}
\newcommand{\fhat}{\hat{f}}
\newcommand{\ghat}{\hat{g}}

\newcommand{\gst}{g^{\star}}
\newcommand{\fst}{f^{\star}}
\newcommand{\ftilst}{\tilde{f}^{\star}}
\newcommand{\gtilst}{\tilde{g}^{\star}}

\newcommand{\epspred}{\epsilon_{\mathrm{pred}}}
\newcommand{\hilspace}{\mathcal{H}}
\newcommand{\rank}{\mathrm{rank}}

\newcommand{\opnorm}[1]{\|#1\|_{\op}}
\newcommand{\barsig}{\bar{\sigma}}

\newcommand{\xspac}{\mathscr{X}}
\newcommand{\yspace}{\mathscr{Y}}
\newcommand{\hilprod}[1]{\langle #1 \rangle_{\hilspace}}

\newcommand{\bmsf}[1]{\bm{\mathsf{#1}}}

\newcommand{\projopst}{\bmsf{P}^\star}

\newcommand{\cdx}[1]{\cD_{\xspac,#1}}
\newcommand{\cdy}[1]{\cD_{\yspace,#1}}
\newcommand{\cdyj}{\cdy{j}}
\newcommand{\cdxi}{\cdx{i}}
\newcommand{\Risk}{\cR}

\newcommand{\hilnormm}[1]{\|#1\|}
\newcommand{\hilprodd}[1]{\langle#1\rangle}
\newcommand{\ErrTerm}{\mathsf{\Delta}}

\newcommand{\sighat}{\hat{\sigma}}

\newcommand{\alphatrain}{\alpha_{\mathrm{trn}}}
\newcommand{\errbartrain}{\bar{\ErrTerm}_{\mathrm{train}}}
\newcommand{\lst}{{\ell}_{\star}}
\newcommand{\lstargs}{\lst(\epsilon,\kpick)}

\newcommand{\projopstk}{\projopst_k}
\newcommand{\range}{\mathrm{range}}
\newcommand{\fstgk}{\fst_{>k}}
\newcommand{\gstgk}{\gst_{>k}}
\newcommand{\disone}{\cD_{1\otimes 1}}
\newcommand{\distwo}{\cD_{2\otimes 2}}

\newcommand{\Fmaps}{\xspac \to \hilspace}
\newcommand{\Gmaps}{\yspace \to \hilspace}
\newcommand{\hstk}{\hst_k}

\newcommand{\eyeop}{\bmsf{I}}

\newcommand{\adj}{\mathsf{H}}
\newcommand{\omegaoff}{\omega_{\mathrm{min}}}
\newcommand{\ErrTermOff}{\ErrTerm_{\mathrm{off}}}
\newcommand{\disij}{\cD_{i\otimes j}}

\newcommand{\sfp}{\mathsf{p}}
\newcommand{\sfq}{\mathsf{q}}
\newcommand{\bhatSigma}{\hat{\bSigma}}
\newcommand{\bstAk}{\bstA_{[k]}}
\newcommand{\bstBk}{\bstB_{[k]}}
\newcommand{\bstPk}{\bstP_{[k]}}

\newcommand{\nullspace}{\mathrm{nullspace}}
\newcommand{\rowspace}{\mathrm{rowspace}}
\newcommand{\epstilop}{\tilde{\epsilon}_{\op}}
\newcommand{\epstilfro}{\tilde{\epsilon}_{\mathrm{fro}}}

\newcommand{\tail}{\mathsf{tail}}

\renewcommand{\projopstk}[1][k]{\bmsf{P}^\star_{#1}}
\newcommand{\delst}{\updelta^{\star}}
\newcommand{\delstk}{\delst_k}

\newcounter{relctr} 
\everydisplay\expandafter{\the\everydisplay\setcounter{relctr}{0}} 

\AtBeginDocument{} 

\def\ddefloop#1{\ifx\ddefloop#1\else\ddef{#1}\expandafter\ddefloop\fi}
\def\ddef#1{\expandafter\def\csname bb#1\endcsname{\ensuremath{\mathbb{#1}}}}
\ddefloop ABCDEFGHIJKLMNOPQRSTUVWXYZ\ddefloop

\def\ddefloop#1{\ifx\ddefloop#1\else\ddef{#1}\expandafter\ddefloop\fi}
\def\ddef#1{\expandafter\def\csname fr#1\endcsname{\ensuremath{\mathfrak{#1}}}}
\ddefloop ABCDEFGHIJKLMNOPQRSTUVWXYZ\ddefloop

\def\ddefloop#1{\ifx\ddefloop#1\else\ddef{#1}\expandafter\ddefloop\fi}
\def\ddef#1{\expandafter\def\csname eul#1\endcsname{\ensuremath{\EuScript{#1}}}}
\ddefloop ABCDEFGHIJKLMNOPQRSTUVWXYZ\ddefloop

\def\ddefloop#1{\ifx\ddefloop#1\else\ddef{#1}\expandafter\ddefloop\fi}
\def\ddef#1{\expandafter\def\csname scr#1\endcsname{\ensuremath{\mathscr{#1}}}}
\ddefloop ABCDEFGHIJKLMNOPQRSTUVWXYZ\ddefloop

\def\ddefloop#1{\ifx\ddefloop#1\else\ddef{#1}\expandafter\ddefloop\fi}
\def\ddef#1{\expandafter\def\csname b#1\endcsname{\ensuremath{\mathbf{#1}}}}
\ddefloop ABCDEFGHIJKLMNOPQRSTUVWXYZ\ddefloop

\def\ddefloop#1{\ifx\ddefloop#1\else\ddef{#1}\expandafter\ddefloop\fi}
\def\ddef#1{\expandafter\def\csname bhat#1\endcsname{\ensuremath{\hat{\mathbf{#1}}}}}
\ddefloop ABCDEFGHIJKLMNOPQRSTUVWXYZ\ddefloop

\def\ddefloop#1{\ifx\ddefloop#1\else\ddef{#1}\expandafter\ddefloop\fi}
\def\ddef#1{\expandafter\def\csname btil#1\endcsname{\ensuremath{\tilde{\mathbf{#1}}}}}
\ddefloop ABCDEFGHIJKLMNOPQRSTUVWXYZ\ddefloop

\def\ddefloop#1{\ifx\ddefloop#1\else\ddef{#1}\expandafter\ddefloop\fi}
\def\ddef#1{\expandafter\def\csname bbar#1\endcsname{\ensuremath{\bar{\mathbf{#1}}}}}
\ddefloop ABCDEFGHIJKLMNOPQRSTUVWXYZ\ddefloop

\def\ddefloop#1{\ifx\ddefloop#1\else\ddef{#1}\expandafter\ddefloop\fi}
\def\ddef#1{\expandafter\def\csname bst#1\endcsname{\ensuremath{\mathbf{#1}^\star}}}
\ddefloop ABCDEFGHIJKLMNOPQRSTUVWXYZ\ddefloop

\def\ddef#1{\expandafter\def\csname c#1\endcsname{\ensuremath{\mathcal{#1}}}}
\ddefloop ABCDEFGHIJKLMNOPQRSTUVWXYZ\ddefloop

\def\ddef#1{\expandafter\def\csname cbar#1\endcsname{\ensuremath{\bar{\mathcal{#1}}}}}
\ddefloop ABCDEFGHIJKLMNOPQRSTUVWXYZ\ddefloop

\def\ddef#1{\expandafter\def\csname cvec#1\endcsname{\ensuremath{\vec{\mathcal{#1}}}}}
\ddefloop ABCDEFGHIJKLMNOPQRSTUVWXYZ\ddefloop

\def\ddef#1{\expandafter\def\csname h#1\endcsname{\ensuremath{\widehat{#1}}}}
\ddefloop ABCDEFGHIJKLMNOPQRSTUVWXYZ\ddefloop
\def\ddef#1{\expandafter\def\csname hc#1\endcsname{\ensuremath{\widehat{\mathcal{#1}}}}}
\ddefloop ABCDEFGHIJKLMNOPQRSTUVWXYZ\ddefloop
\def\ddef#1{\expandafter\def\csname t#1\endcsname{\ensuremath{\widetilde{#1}}}}
\ddefloop ABCDEFGHIJKLMNOPQRSTUVWXYZ\ddefloop
\def\ddef#1{\expandafter\def\csname tc#1\endcsname{\ensuremath{\widetilde{\mathcal{#1}}}}}
\ddefloop ABCDEFGHIJKLMNOPQRSTUVWXYZ\ddefloop

\title{Tackling Combinatorial Distribution Shift: \\
A Matrix Completion Perspective}

\author{%
  Max Simchowitz\footnote{Massachusetts Institute of Technology, Cambridge, MA, 02139.  \texttt{msimchow@csail.mit.edu}}
  \and
  Abhishek Gupta\footnote{University of Washington, Seattle, WA, 98195. \texttt{abhgupta@cs.washington.edu}} 
  \and
  Kaiqing Zhang\footnote{University of Maryland,  College Park, MD, 20742. \texttt{kaiqing@umd.edu}}
}

\date{}

\begin{document}

\maketitle 

\begin{abstract}

Obtaining rigorous statistical guarantees for generalization under distribution shift  remains an open and active research area. We study a setting we call   \emph{combinatorial distribution shift}, where (a) under the test- and training-distributions,  the labels $z$ are determined by pairs of features $(x,y)$, (b) the training  distribution has coverage of certain \emph{marginal} distributions over $x$ and $y$ separately, but (c) the test distribution involves examples from a product distribution over $(x,y)$ that is \emph{not} covered by the training distribution. Focusing on the special case where the labels are given by \emph{bilinear embeddings}  into a Hilbert space $\mathcal H$: $\mathbb{E}[z \mid x,y ]=\langle f_{\star}(x),g_{\star}(y)\rangle_{\mathcal{H}}$, we aim to extrapolate to a test distribution domain that is {not} covered in training, i.e., achieving \emph{bilinear combinatorial extrapolation}. 

Our setting generalizes a special case of   matrix completion from missing-not-at-random  data,  for which all existing results require the ground-truth matrices to be either \emph{exactly low-rank}, or to exhibit very sharp spectral cutoffs.  In this work, we develop a series of theoretical results that enable bilinear combinatorial extrapolation under \emph{gradual} spectral decay as observed in typical high-dimensional data, including novel algorithms, generalization guarantees, and  linear-algebraic results. A key tool is a novel perturbation bound for the rank-$k$ singular value decomposition approximations between two matrices that depends on the \emph{relative} spectral gap rather than the \emph{absolute} spectral gap, a result that may be of broader independent interest.


\end{abstract}

\section{Introduction}\label{sec:introduction}

While statistical learning theory has classically studied \emph{out-of-sample generalization} from training data to test data drawn from the same distribution (e.g., \cite{bartlett2002rademacher,vapnik2006estimation}), in almost all practical settings,  one wishes to ensure strong performance on data which may be generated quite differently from the training data \citep{koh2021wilds,taori2020measuring}.
This paper studies formal guarantees for a type of \emph{out-of-distribution}  generalization we call  \emph{combinatorial distribution shift.}
Informally, we consider predictions from pairs of features $(x,y)$ such that: (a) the \emph{marginal} distributions of each of the features separately under the test data are covered by the training distribution, but (b) the \emph{joint} distribution of the features may not be covered. We refer to \emph{combinatorial extrapolation} as the process of generalization under combinatorial distribution shift.
Our setting may encompass a broad swath of applications including:  computer vision tasks which  extrapolate to novel  combinations of objects, backgrounds, and lighting conditions that have been seen individually \citep{liu2016dhsnet}; extrapolation to manipulating objects with novel combinations of masses, shapes, and sizes in robotic manipulation \citep{tremblay2018deep}; extrapolation to  predictions of the outcomes of medical intervention from one set of subgroups to others with novel combinations of salient traits \citep{gilhus2015myasthenia}.  See  \Cref{fig:illustration} for an   illustration. 



\begin{SCfigure}
    {\includegraphics[width=0.25\linewidth]{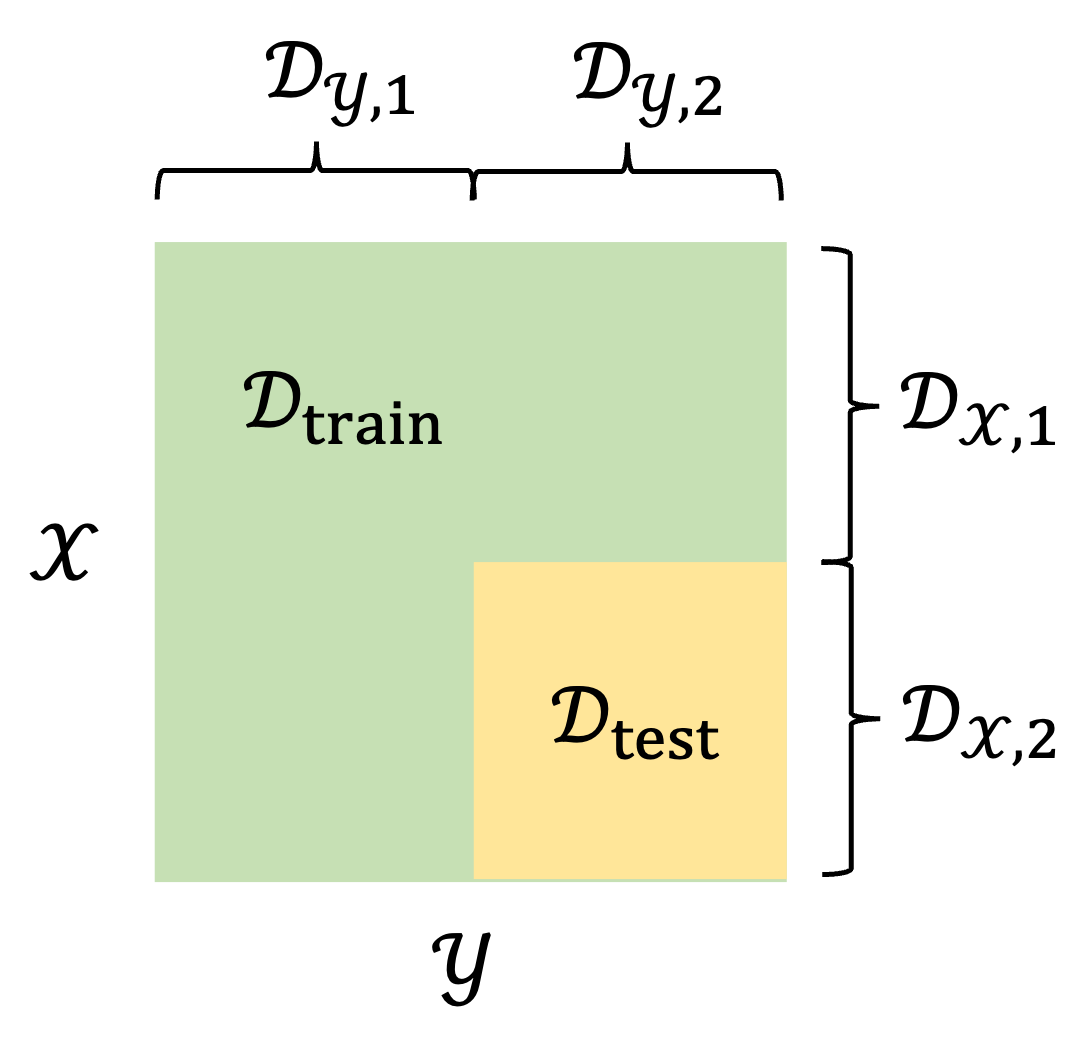}
     \hspace{15pt}     \includegraphics[width=0.20\linewidth]{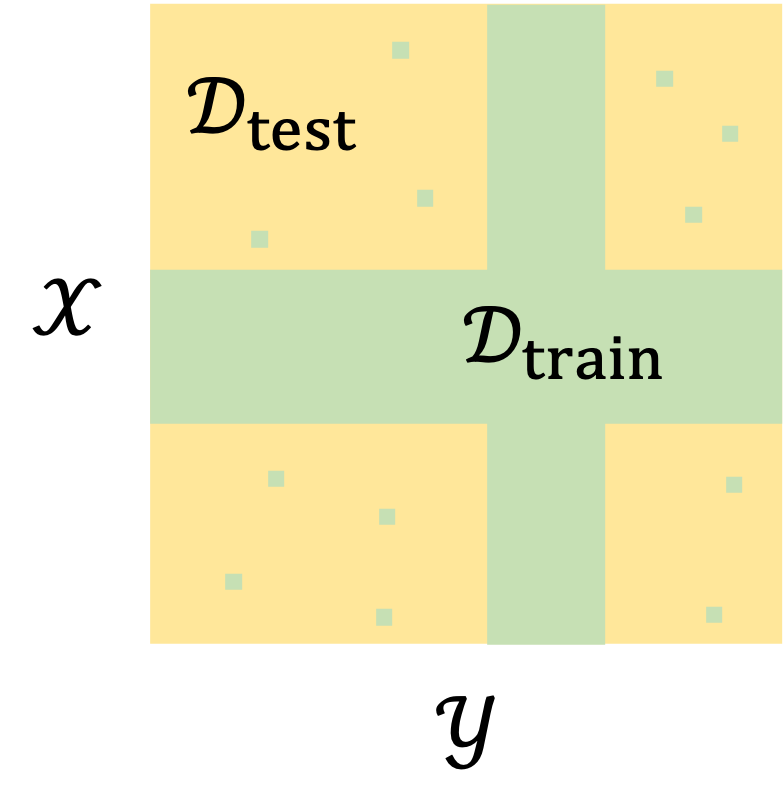}} \hspace{15pt} 
    \caption{Illustration of the combinatorial extrapolation settings where the marginal distributions of $\Dtest$ is covered by $\Dtrain$, while some {\it combinations} of them are not covered by $\Dtrain$.
    } 
    \label{fig:illustration}
\end{SCfigure}

\paragraph{Bilinearity, low-rank structure \& matrix completion.} A popular technique for compositional and combinatorial generalization is to \emph{embed} features into a semantic vector space \citep{mikolov2013efficient}. For example, CLIP \citep{radford2021learning}    learns embedding words and text into an \emph{inner-product space} in order to achieve zero-shot generalization to new image classes. In this work, we adopt a matrix-completion perspective to study the potential of these bilinear approaches. Indeed, if the features $(x,y)$ correspond to indices of a large data matrix, bilinear combinatorial extrapolation may be understood as \emph{matrix completion}: complete an entire matrix $\bM \in \R^{n \times m}$ from observing a subset $\Omega \subset [n]\times [m]$ of its entries. The estimation of accurate bilinear embeddings, then, corresponds to finding a low-rank approximate factorization of the data. We detail this connection in \Cref{sec:connection_matrix_completion}. Whereas classical results study the \emph{missing-at-random} (MAR) regime where $\Omega$ is drawn uniformly at random \citep{candes2012exact,recht2011simpler,hastie2015matrix}, the absence of  joint-distribution coverage makes our setting a special case of \emph{missing-not-at-random} (MNAR) recovery  (see, e.g.,  \cite{ma2019missing}). There is a rich literature on MNAR matrix recovery (see a detailed review in \Cref{sec:extended_related_work}). A common assumption in this literature of MNAR matrix recovery is that, the data matrix $\bM$ is either \emph{exactly low-rank}, or exhibits \emph{sharp drop-offs}  between adjacent singular values. This is in contrast to MAR matrix recovery, where it suffices that the singular values of $\bM$ are only summable \citep{koltchinskii2011nuclear}. While it is widely accepted that real data are approximately low-rank \citep{udell2019big}, they tend to exhibit the more \emph{gradual} singular value decay required by MAR matrix recovery, than the \emph{rapid} decay necessitated by the existing MNAR-case results. Indeed, the spectra of random data matrices have continuous limiting distributions \citep{bai2010spectral}, and thus their singular values \emph{do not} exhibit sharp cutoffs. 

 

\paragraph{Our contributions.} This paper demonstrates   conditions under which \emph{bilinear predictors} are statistically consistent under combinatorial  distribution shift. We assume real labels $z$ can be predicted from pairs of features $(x,y) \in \xspac \times \yspace$ via bilinear embeddings into a Hilbert space $\cH$: $\Exp[z \mid x,y] = \langle \fst(x),\gst(y) \rangle_{\hilspace}$. We then state structural assumptions, inspired by a canonical case of  matrix  completion with MNAR data  (see \Cref{fig:illustration_1} and \Cref{sec:connection_matrix_completion}), which facilitate  extrapolation from a training distribution $\Dtrain$ over pairs $(x,y)$ that has a full coverage of certain marginal distributions over $x$ and $y$ separately, to a test distribution $\Dtest$ containing samples from a product distribution over $(x,y)$ that is \emph{not}  covered by $\Dtrain$. 
In contrast to the MNAR matrix completion  literature described above, we analyze a setting more akin to the kernel least-squares literature \citep{bissantz2007convergence,mendelson2010regularization}, where  a suitably defined feature covariance matrix $\Sigst$ (\Cref{asm:bal}) may exhibit spectral decay as gradual as $\lambda_i(\Sigst) \le Ci^{-(1+\gamma)}$ for some $\gamma > 0$ (\Cref{asm:eigendecay}). 
Our contributions are as follows: 
\begin{itemize}
\item  Given finite-rank embeddings $\fhat: \cX \to \R^r$ and $\ghat: \cY \to \R^r$, we establish a meta-theorem, \Cref{thm:generic_risk_bound}, which establishes upper bounds for   the excess risk \iftoggle{arxiv}{
    \begin{align*}
    \Risk(\fhat,\ghat;\Dtest):= \Exp_{\Dtest}[(\langle \fhat,\ghat\rangle_{\cH} - \langle \fst,\gst\rangle_{\cH})^2] 
    \end{align*}
}{$\Risk(\fhat,\ghat;\Dtest):= \Exp_{\Dtest}[(\langle \fhat,\ghat\rangle_{\cH} - \langle \fst,\gst\rangle_{\cH})^2] $} on $\Dtest$ by the excess risk on $\Dtrain$, and the error on a sub-distribution $\cD_{1\otimes 1}$ of $\Dtrain$, which  corresponds to a dense  diagonal block matrix in  MNAR matrix completion.
\item Using the meta-theorem, we show in \Cref{thm:st_main} that if $(\fhat,\ghat)$ above are trained via a \emph{single stage} of supervised empirical risk minimization (ERM)  (from a suitably expressive function class), then whenever it happens that $(\fhat,\ghat)$ are well-conditioned (in a sense defined), $\Risk(\fhat,\ghat;\Dtest)$ scales with an inverse of some polynomials in the number of samples and in the rank $r$, provided that the exponent $\gamma$ in the polynomial decay satisfies $\gamma > 3$.
\item Finally, we introduce a \emph{double-stage} ERM procedure (\Cref{alg:main_alg}), which  produces final estimates $(\fhat,\ghat)$ of the embeddings that (with high probability)  are guaranteed to be well-conditioned, and have $\Risk(\fhat,\ghat;\Dtest) \to 0$ for any decay exponent $\gamma > 0$ (see \Cref{thm:main_dt_simple}).
\end{itemize}

\begin{figure}[!t]         
    {\centering  
     \hspace{-9pt}\includegraphics[width=0.278\linewidth]{figs/Fig_1}
     \hspace{-5pt}      \includegraphics[width=0.52\linewidth]{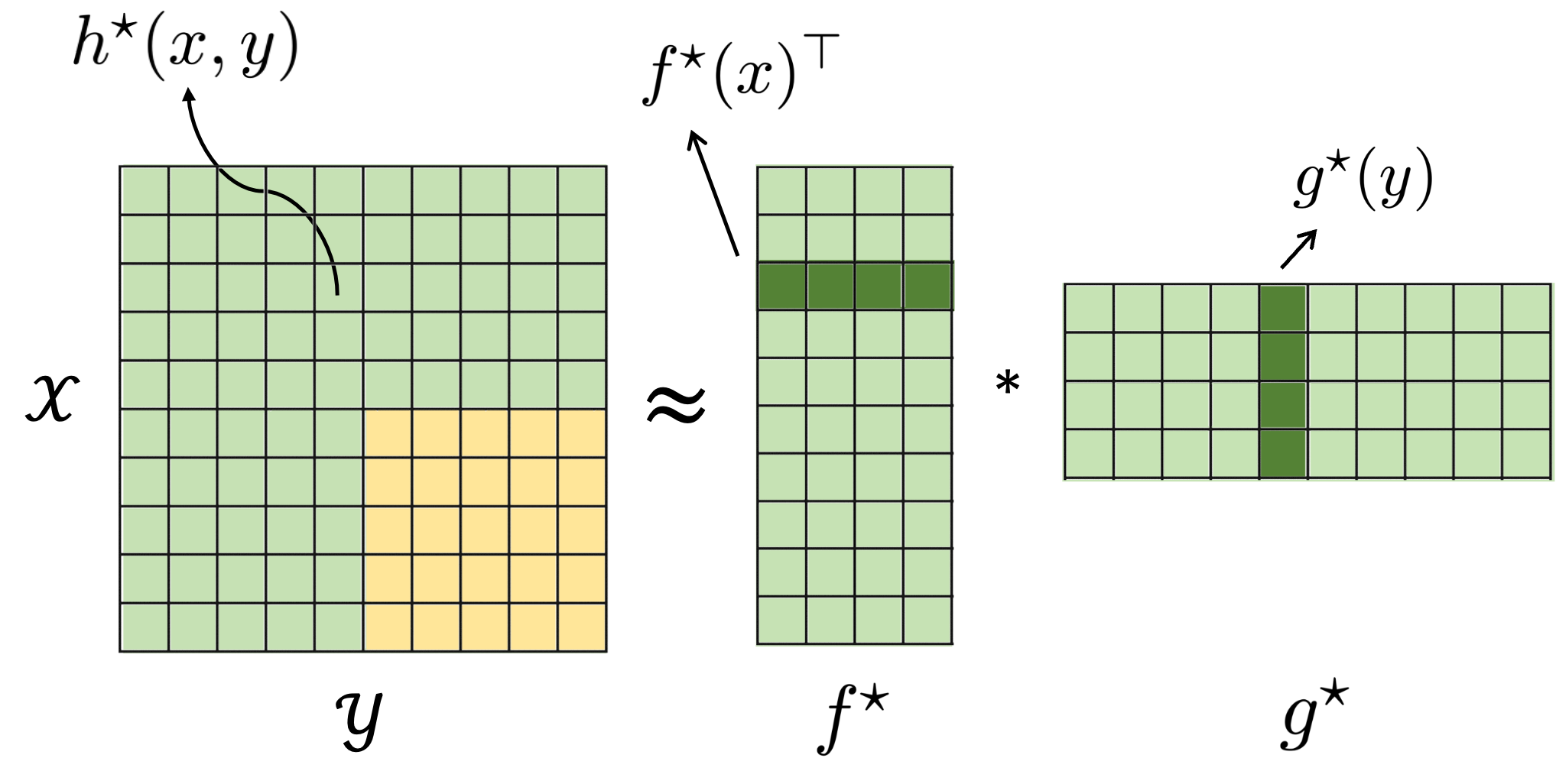}  
     \hspace{5pt}   \includegraphics[width=0.165\linewidth]{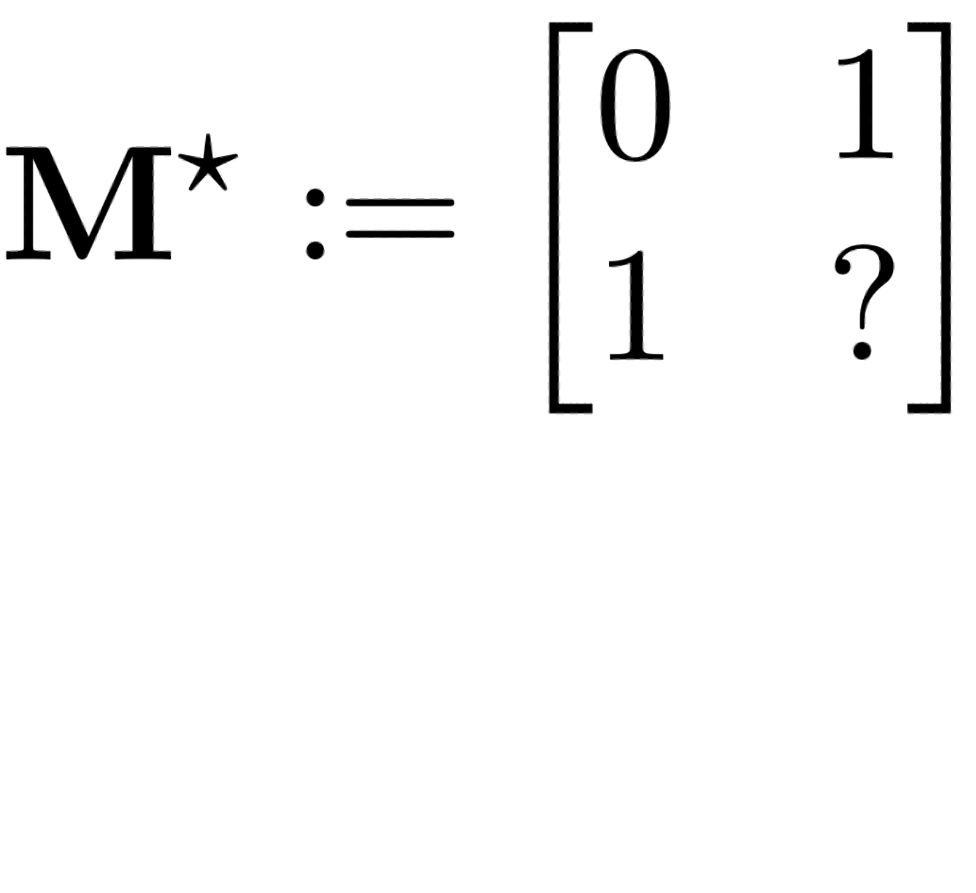}\\  
  \qquad\qquad\quad  
   \textbf{(a)}   ~~\qquad\qquad\qquad\qquad\qquad\qquad\qquad \qquad\qquad~~ \textbf{(b)}    ~~~~~\qquad\qquad \qquad\qquad\qquad\qquad~~~\textbf{(c)}} 
    \caption{\textbf{(a)} Bilinear combinatorial extrapolation that satisfies the $2\times 2$ block decomposition; \textbf{(b)} A basic case of \textbf{(a)} with discrete distributions, which can be viewed as  matrix completion with MNAR data, and the bilinear representation of the distribution naturally appears; \textbf{(c)}  An example of a matrix  that does not satisfy \Cref{asm:cov} and thus fails to be completed uniquely. 
    } 
    \label{fig:illustration_1}
\end{figure}

\subsection{Relative singular-gap perturbation bound for the SVD approximation}

Before describing our overall proof strategy, we highlight a key technical ingredient that we believe may be of more universal interest. Consider two real matrices $\bstM,\bhatM \in \R^{n \times m}$, and let $\sigma_k(\cdot)$ denote the $k$-th largest singular value. The celebrated Davis-Kahan Sine Theorem and its generalization, Wedin's Theorem (see, e.g.,  \cite{stewart1990matrix}), states that  the principal angles between their (left or right) singular spaces scale with $\|\bstM - \bhatM\|_{\fro}/\updelta^{\mathrm{abs}}_k(\bstM)$, where $\updelta^{\mathrm{abs}}_k(\bstM) := \sigma_k(\bstM) - \sigma_{k+1}(\bstM)$ denotes the \emph{absolute} singular gap. For the special case of multiplicative perturbations, $\bhatM = (\eye + \bDelta_1) \bstM (\eye + \bDelta_2)$ with matrices $\bDelta_1,\bDelta_2$ close to zero, the  perturbation scales with the (possibly much smaller) relative singular value gap \citep{li1998relative},
\begin{align} \updelta_{k}(\bstM):=\frac{\sigma_{k}(\bstM) - \sigma_{k +1}(\bstM)}{\sigma_{k}(\bstM)}\label{eq:updelta_k}.
\end{align}
So far, we have reviewed bounds on the deviation in the singular value \emph{subspaces} of the matrices $\bstM$ and $\bhatM$. But in many cases, we do not know about these subspaces, but instead, know about the differences in the rank-$k$ SVD approximations to these matrices. For this desideratum, we establish a perturbation bound which depends only on the \emph{relative gap} and which, unlike the singular subspace bound of \cite{li1998relative}, applies to \emph{generic, additive perturbations}. Our result is as follows.


\begin{restatable}[Perturbation of SVD Approximation with Relative Gap]{theorem}{svdpert}\label{thm:svd_pert} Let $\bstM,\bhatM \in \R^{n \times m}$. Fix a $k \le \min\{n,m\}$ for which $\sigma_k(\bstM) > 0$ and the relative spectral gap $\updelta_{k}(\bstM)$ (\Cref{eq:updelta_k})
is positive. Then, if $\|\bstM - \bhatM\|_{\op} \le \eta \sigma_{k}(\bstM)\updelta_{k}(\bstM)$ for some $\eta \in (0,1)$, we have that the rank-$k$ SVD approximations of  $\bstM$ and $\bhatM$, denoted as $\bstM_{[k]}$ and $\bhatM_{[k]}$, are unique, and satisfy
    \begin{align*}
    \big\|\bhatM_{[k]} - \bstM_{[k]}\big\|_{\fro} &\le  \frac{9\|\bhatM - \bstM\|_{\fro}}{\updelta_{k}(\bstM)(1-\eta)}. 
    \end{align*} 
    \end{restatable}
\Cref{thm:svd_pert} is proven in \Cref{sec:SVD_pert} via a careful peeling argument. By contrast, a more naive application of Wedin's theorem incurs a dependence on  absolute singular gap $\updelta_k^{\mathrm{abs}}(\bstM)$. Our bound  is significantly sharper: for example, consider $
\sigma_k(\bstM) \sim\Theta(2^{-k})$, then $\updelta_{k_i}^{\mathrm{abs}}=\sigma_{k_i}(\bstM) -  \sigma_{k_i+1}(\bstM)$ is of order $O(2^{-(k_i+1)})$, while $\updelta_{k_i}$ as defined in \Cref{eq:updelta_k} is of order $\Omega(1)$. Having highlighted this particular technical result, we now turn to an overview of the entire analysis.

\subsection{Overview of proof techniques and notation}  
Throughout, the key technical challenge, from a matrix completion perspective, is generalizing the case with \emph{sharp}  spectral cutoffs to that with a \emph{gradual} spectral decay. This challenge is  considerably  more difficult for \emph{bilinear factorizations} than that for \emph{linear} predictors studied in typical RKHS settings. 
Regarding the proof of our meta-theorem, \Cref{thm:generic_risk_bound}: when distributions on $(x,y)$ have finite support, the bilinear combinatorial extrapolation problem for discrete distributions can be reinterpreted as the completion of a  block matrix $\bM$ with blocks $\bM_{ij}$, given data from blocks $\{(1,1),(1,2),(2,1)\}$. With a careful error decomposition, we argue  that the extrapolation error is controlled by the recovery of a factorization of the top-left block $\bM_{11}$ (see \Cref{prop:final_error_decomp_simple}). More specifically, if we let $\bstM = \bM_{11}$ and let $\bhatM$ correspond to the estimates of a bilinear predictor $\langle \fhat,\ghat\rangle$ on the $(1,1)$-block, the key step is to show that if we can factor $\bstM = \bstA(\bstB)^\top$ and $\bhatM = \bhatA\bhatB^\top$, then 
$\bhatM \approx \bstM$ implies $\bhatA \approx \bstA_{[k]}$ and $\bhatB \approx \bstB_{[k]}$ in the sharpest possible sense, where $k$ is some target rank and $(\cdot)_{[k]}$ denotes rank-$k$ singular value decomposition (SVD)  approximation of the matrix. While factor recovery guarantees do exist (notably \citet[Lemma 5.14]{tu2016low}), all prior results require sharp spectral cutoffs. To this end, we provide a novel factor recovery guarantee (\Cref{thm:main_matrix_body}); this, in turn, relies on \Cref{thm:svd_pert} above, as well as a careful \emph{partition} of the singular values of a matrix we call the \emph{well-tempered partition} (see \Cref{sec:main_mat_body}). Limiting arguments pass from the matrix/discrete-distribution case to arbitrary distributions  (\Cref{sec:proof_from_matrix_to_feature}). 

Given \Cref{thm:generic_risk_bound}, the instantiation to a single stage of ERM (\Cref{thm:st_main}) is straightforward. Analyzing our double-stage ERM procedure (\Cref{alg:main_alg}) requires more care. Notably, the analysis depends on a careful characterization of what we term as the \emph{balancing operator} -- a linear algebraic operator which determines the change-of-basis in which the positive-definite covariance  matrices   are equal. Discussion of the algorithm and a proof sketch are given in \Cref{sec:ermds}, with a complete proof deferred to \Cref{sec:analysis_training_alg}; properties of the balancing operator are studied in  \Cref{sec:balancing_operator}.

\paragraph{Notation.}  
For two probability measures  $\cD,\cD'$, we let $\cD\otimes \cD'$  denote the product measure, and $\frac{\rmd \cD}{\rmd \cD'}$ the Radon–Nikodym derivative of $\cD$ with respect to $\cD'$. 
Upper case bold letters $\bA,\bB,\bM$ denote matrices, lower case bold letters $\bv,\bw$ denote vectors.  Operators and  elements of the Hilbert space $\hilspace$ are denoted by  bold serafs as $\bmsf \Sigma$ and $\bmsf v$, respectively. Adjoints and transposes are \emph{both} denoted with $(\cdot)^\top$; e.g., $\bmsf v^\top $ and $\bv^\top$ for $\bmsf v \in \cH$, $\bv \in \R^d$. 
 The $i$-th entry of a vector $\bv$ is denoted by $\bv[i]$, the $i$-th row of a matrix $\bA$ by $\bA[i,:]$, and the $(i,j)$-th entry by $\bA[i,j]$. The space of symmetric (resp. positive semi-definite, resp. positive definite) $d$-by-$d$ matrices are denoted as $\mathbb{S}^d$, (resp. $\mathbb{S}^{d}_+$, resp. $\mathbb{S}^{d}_{++}$). For $\bM \in \R^{d \times d}$, $\sigma_i(\bM)\geq 0$  denotes its $i$-th largest singular value; for symmetric $\bM$, $\lambda_i(\bM)$ denotes its $i$-th largest eigenvalue, and if $\bM \succeq 0$, $\bM^{1/2}$ its matrix square-root; similar notation applies to operators $\bmsf{\Sigma}$ on $\hilspace$.  For $n \in \N$,  $[n]$ denotes the set $\{1,\cdots,n\}$, and for finite sets $\cS$, $|\cS|$ denotes its cardinality. For any Hilbert space $\cV$, we use $\langle x,y\rangle_\cV$ to denote the inner product of $x,y\in\cV$, and $\|x\|_\cH$ to denote the Hilbert norm defined by the product. When $\cV$ is omitted, they mean the inner-product and vector norm in the Euclidean space. We use $\log$ to denote the   base-$e$ logarithm.






\newcommand{\bsigst}{\bm{\sigma}^\star}

\newcommand{\kpick}{s}
\newcommand{\cdone}{\cD_{1\otimes 1}}

\newcommand{\kaptraintil}{\tilde{\kappa}_{\mathrm{trn}}}
\newcommand{\mnar}{\textsc{mnar}}
\newcommand{\kaptrain}{\kappa_{\mathrm{trn}}}
\newcommand{\kaptest}{\kappa_{\mathrm{tst}}}

    \newcommand{\sigcut}{\sigma_{\mathrm{cut}}}
    \newcommand{\rcut}{r_{\mathrm{cut}}}
\newcommand{\dimred}{\textsc{DimReduce}}
\newcommand{\kapapx}{\kappa_{\mathrm{apx}}} 

\section{Problem Formulation}\label{sec:formulation}

In the \emph{bilinear combinatorial extrapolation} problem, covariates $(x,y) \in \xspac \times \yspace$ are regressed to real labels $z \in \R$. We are given access to a training distribution $\Dtrain$ and a test  distribution $\Dtest$ on $\xspac \times \yspace \times \R$. We assume that the Bayes optimal predictor is identical between the two distributions, and is given by the inner product of bilinear embeddings defined below. 

\begin{assumption}[Bilinear Representation]\label{asm:bilinear} There is a Hilbert space $(\hilspace,\hilprod{\cdot,\cdot})$ and two embeddings $\fst: \xspac \to \hilspace$ and $\gst: \yspace \to \hilspace$ satisfying  that $\hst(x,y) := \hilprod{\fst(x),\gst(y)}$ is the Bayes optimal predictor on $\Dtrain$ and $\Dtest$, i.e., $\Exp_{\Dtrain}[z \mid x,y] = \Exp_{\Dtest}[z \mid x,y] = \hst(x,y)$. Also,  $\Exp_{\Dtrain}[\langle\fst(x),\gst(y)\rangle^2] < \infty $. 
\end{assumption}

\paragraph{Assumptions that  facilitate extrapolation.}  

The  bilinear structure of $\hst$ is insufficient for general  combinatorial extrapolation; otherwise, in the finite-dimensional case, a matrix would have been completable from a single entry.  We, therefore,  assume that our training distribution can be decomposed into four blocks, such that the first three blocks, i.e.,  the blocks $(1,1),(1,2),(2,1)$, are ``covered'' under $\Dtrain$, but the fourth block, i.e., the block $(2,2)$, may only be covered under $\Dtest$. It is formally introduced in the following assumption. 

\begin{assumption}[Coverage Decomposition]\label{asm:density} There exist constants $\kaptrain,\kaptest \ge 1$ and \emph{marginal  distributions}  $\cdx{1},\cdx{2}$ over $\xspac$, and $\cdy{1},\cdy{2}$  over $\yspace$, with their product measures $\cD_{i\otimes j} := \cdx{i} \otimes \cdyj$, such that the following is true for all $(x,y) \in \xspac \times \yspace$: 
\iftoggle{arxiv}{
    \begin{itemize}
        \item[(a)]\textbf{Training Coverage:} for pairs $(i,j) \in \{(1,1),(1,2),(2,1)\}$, $\frac{\rmd \cD_{i \otimes j}(x,y)}{\rmd \Dtrain(x,y)} \le \kaptrain$.
        \item[(b)] \textbf{Test Coverage:} $\frac{\rmd \Dtest(x,y)}{\sum_{i,j}  \rmd\cD_{i\otimes j}(x,y)} \le \kaptest$.
    \end{itemize}
}{(a) \textbf{Training Coverage:} for pairs $(i,j) \in \{(1,1),(1,2),(2,1)\}$, $\frac{\rmd \cD_{i \otimes j}(x,y)}{\rmd \Dtrain(x,y)} \le \kaptrain$, and (b) \textbf{Test Coverage:} $\frac{\rmd \Dtest(x,y)}{\sum_{i,j}  \rmd\cD_{i\otimes j}(x,y)} \le \kaptest$.}
\end{assumption}



The above condition means that the only part of $\Dtest$ not covered by $\Dtrain$ is the samples $(x,y)$ from $\cD_{2 \otimes 2}$. Thus, bilinear combinatorial extrapolation amounts to the \emph{generalization}  problem  on these pairs. This condition represents the simplest case of the \emph{Missing-Not-At-Random} (\mnar) matrix completion; see \Cref{fig:illustration_1} (a \& b) and \Cref{sec:connection_matrix_completion} for an illustration and further discussions.  As illustrated in \Cref{fig:illustration_1} (c), a unique completion requires that the top block has a rank equal to the other three blocks. Intuitively, we require an assumption that ensures that every feature which ``appears'' in $\cD_{2\otimes 2}$ also ``appears'' in $\cD_{1 \otimes 1}$.
 We formalize this in the following assumption. 
\begin{assumption}[Change of Covariance]\label{asm:cov} There exists $\kapcov \ge 1$ such that 
\iftoggle{arxiv}
{
\begin{align*}
&\Exp_{x\sim \cdx{2}}[\fst(x)\fst(x)^\top] \preceq \kapcov\cdot\Exp_{x\sim \cdx{1}}[\fst(x)\fst(x)^\top] \quad \text{ and } \\
&\Exp_{y\sim \cdy{2}}[\gst(y)\gst(y)^\top] \preceq \kapcov\cdot\Exp_{y\sim \cdy{1}}[\gst(y)\gst(y)^\top].
\end{align*}
}
{$\Exp_{x\sim \cdx{2}}[\fst(x)\fst(x)^\top] \preceq \kapcov\cdot\Exp_{x\sim \cdx{1}}[\fst(x)\fst(x)^\top]$ and $\Exp_{y\sim \cdy{2}}[\gst(y)\gst(y)^\top] \preceq \kapcov\cdot\Exp_{y\sim \cdy{1}}[\gst(y)\gst(y)^\top].$}
\end{assumption}

\paragraph{Spectral assumptions.} In addition to the above conditions, we require some control on the \emph{richness} of the embeddings $\fst,\gst$. We assume that the covariances $\Sigf := \Exp_{\cdxone}[\fst(\fst)^\top]$ and $\Sigg := \Exp_{\cdyone}[\gst(\gst)^\top]$ are trace-class operators on $\cH$. We assume that we are in a basis of $\hilspace$ for which $(\fst,\gst)$ are \emph{balanced} in the following sense.

 \begin{assumption}[Balanced Basis]\label{asm:bal} The ground truth embeddings $\fst$ and $\gst$ are in an appropriate basis such that $\Sigf = \Sigg =: \Sigst$ are trace-class. For simplicity, we also assume that $\lambda_1(\Sigst) > 0$. 
\end{assumption}
The assumption $\Sigf = \Sigg$ may seem restrictive, but is achievable more-or-less without loss of generality by a change of basis (see \Cref{sec:balance_other_rank}). 
Trace-class operators necessarily exhibit spectral decay. Hence, a key object throughout is the low-rank projections of our embeddings.

\begin{defn}[Low-Rank Approximations] Under \Cref{asm:bal}, let $\projopst_{k}$ denote the \iftoggle{nips}{any valid choice of}{}  projection onto the top-$k$ eigenspace of $\Sigst$\iftoggle{nips}{}{\footnote{When $\lambda_k(\Sigst) = \lambda_{k+1}(\Sigst)$, $\projopst_k$ is non-unique; in this case, assumptions stated in terms of $\fstk,\gstk$ can be chosen to hold for \emph{any}  valid choice of $\projopst_k$.}},  
 $\fstk := \projopst_{k}\fst$, $\gstk := \projopst_k \gst$, and $\hstk(x,y) = \langle \fstk(x),\gstk(y)\rangle_{\hilspace}$. 
\end{defn}
To take advantage of spectral decay, we shall reason extensively about the low-rank approximations  $\fstk,\gstk$ to the ground-truth embeddings $\fst,\gst$. Our final condition ensures that low-rank approximations  to $\hst$ perform well on all the training data. 

\begin{assumption}\label{asm:training_approximation} For all $k\in\N$, $\Exp_{\Dtrain}[(\langle\fstk,\gstk \rangle_\cH - \hst)^2] \le \kapapx\cdot\Exp_{\cD_{1\otimes 1}}[(\langle\fstk,\gstk \rangle_\cH - \hst)^2]$. 
\end{assumption}

We remark a sufficient (but strictly weaker) assumption which implies \Cref{asm:training_approximation} is that $\Dtrain$ is covered by the four-factor distributions in the sense that if $\rmd \Dtrain /(\sum_{i,j = 1}^2\rmd \cD_{i \otimes j}) \le \kaptraintil$; then one can check that \Cref{asm:training_approximation} holds with $\kapapx = 4\kaptraintil\kapcov^2$ if \Cref{asm:cov} holds. Note that such a case is easily satisfied by the standard matrix completion case, i.e., when the embeddings here are finite-dimensional. 
To make our results more concrete, we focus our attention on two classical regimes of spectral decay:

\begin{assumption}[Spectral Decay] \label{asm:eigendecay}  There exist $C,\gamma > 0$ such that \textbf{either} (a) $\lambda_i(\Sigst) \le Ci^{-(1+\gamma)}$ (the ``{polynomial decay regime}'') \textbf{or} (b) $\lambda_i(\Sigst) \le Ce^{-\gamma i}$ (the ``{exponential decay regime}''). 
\end{assumption}
Notice that, for any $\gamma > 0$, the decay $\lambda_i(\Sigst) \le Ci^{-(1+\gamma)}$ does indeed ensure $\Sigst$ is trace-class. 


\paragraph{Function approximation.} As the spaces $\cX,\cY$ are arbitrary, we require control of the statistical complexity of the embeddings $\fstk,\gstk$.  We opt for the simplest possible  assumption: for each $k\in\N$, the low-rank embeddings $\fstk,\gstk$ are captured by finite, uniformly bounded function classes.

\begin{assumption}\label{asm:function_apx} Let $B$ be the upper bound in \Cref{asm:bal}. By inflating $B$ if necessary,  we assume that, for each $k \in \N$,  there exist finite-cardinality function classes $\cF_k \subseteq \{\xspac \to \R^k\}$ and $\cG_k \subseteq \{\yspace \to \R^k\}$ mapping into $\R^k$, such that (a) $\sup_{f \in \cF_k} \sup_{x\in\xspac} \|f(x)\|_2 \le B$  and $\sup_{g \in \cG_k} \sup_{y\in\yspace}\|g(y)\|_2 \le B$, and (b) There exist some $(f,g) \in \cF_k \times \cG_k$ such that $\langle f(x), g(y) \rangle = \langle \fstk(x), \gstk(y) \rangle_{\cH}$  for all $(x,y)\in\xspac\times\yspace$. 
We define $\capac_k := \log |\cF_k||\cG_k|$, and assume without loss of generality that $\capac_k$ are non-decreasing as a function of $k\in\N$. Lastly, we also assume that  for some $B > 0$, \iftoggle{arxiv}{
    \begin{align}
    \sup_{x\in\xspac,y \in \yspace}|\langle\fst(x),\gst(y)\rangle_{\cH}| \le B^2 \quad \text{and}\quad  \Pr_{(x,y,z) \sim \Dtrain}[|z|\le B^2] = 1.
    \end{align}
}{$\sup_{x\in\xspac,y \in \yspace}|\langle\fst(x),\gst(y)\rangle_{\cH}| \le B^2$ and $\Pr_{(x,y,z) \sim \Dtrain}[|z|\le B^2] = 1$.}
\end{assumption}

\Cref{asm:function_apx} can  be easily relaxed to accommodate infinite  function  classes with bounded covering numbers and bounded Rademacher  complexities \citep{bartlett2002rademacher}, classes that satisfy more general tail conditions, and classes that  only capture $\fstk,\gstk$ up to some error.  As our bounds end up being polynomial in the log-cardinality of $\capac_k$, we find \Cref{asm:function_apx} to be sufficient in capturing the essence of the function approximation setting.
 
\begin{rem}Notice that all assumptions, with the exception of the function-approximation conditions in \Cref{asm:function_apx}, apply only to either the (a) distribution of the data $(x,y)$ under $\Dtrain,\Dtest$, or (b) to the structure of the ground truth predictors $\fst,\gst$ under these distributions. Thus, our results do not demand strong structural conditions of the class of possible approximators. As described in \Cref{sec:illustrative_examples}, these rather weak conditions preclude any guarantees for vanilla ERM in the worst-case, and necessitate our more sophisticated double-training procedure to ensure consistent estimation. In particular, the fact that \Cref{asm:cov} need only hold for the \emph{ground truth} $\fst,\gst$ makes the proof considerably more challenging. We refer the reader to  \Cref{prop:final_error_decomp_simple} and its proof for deriving guarantees from only this rather weak condition. 
\end{rem}


\iftoggle{arxiv}
{\subsection{Specialization to MNAR matrix completion}
}
{
\section{Connection to Matrix Completion}
}
\label{sec:connection_matrix_completion}


We now specialize bilinear combinatorial extrapolation problem to the problem of matrix completion of a block-diagonal matrix with MNAR data, and explain how our assumptions concretize to this special setting. Consider a bilinear combinatorial extrapolation setting where the support sets $\xspac$ and $\yspace$  have  \emph{finite} cardinalities, with elements $\{x_1,\dots,x_n\}$ and $\{y_1,\dots,y_m\}$. For $i,j \in \{1,2\}$, define the probabilities $\sfp_{i,\ell} = \Pr_{\cdxi}[x = x_{\ell}]$ and $\sfq_{j,k} = \Pr_{\cdyj}[y = y_{k}]$. Because of the finite support of the distributions,  we can regard any $\hilspace$-embeddings $(f,g)$ (including $(\fst,\gst)$)  as embeddings into $\R^d$,  $d = \max\{n,m\}$,  appending zeros if necessary.  We can then define matrices  $\bA_i(f) \in \R^{n \times d}$ and $\bB_j(g) \in \R^{m \times d}$ by assigning the rows to the scaled values of the embeddings 
\begin{align*}\bA_i(f)[\ell,:] = \sqrt{\sfp_{i,\ell}}f(x_\ell)^\top,\quad \bB_j(g)[k,:] = \sqrt{\sfq_{j,k}}g(y_k)^\top,
\end{align*}
and  the matrices 
\begin{align*}\bM_{i \otimes j}(f,g) = \bA_i(f)\bB_j(g)^\top.
\end{align*} 
	Each matrix $\bM_{i\otimes j}(f,g)$ can be thought of as a look-up table, where $\bM_{i\otimes j}(f,g)[\ell,k] = \sqrt{\sfp_{i,\ell}\sfq_{j,k}}\langle f(x_\ell),g(y_k)\rangle$ is the prediction of $\langle f, g\rangle$, scaled by the square root probability of $x_\ell$ and $y_k$. Then, one can see that the risk of $f,g$ is precisely equal to the Frobenius-norm error difference between the matrices  $\bM_{i\otimes j}(f,g)$ and $\bM_{i\otimes j}(\fst,\gst)$. For  simplicity, we write them as $\bM$ and $\bstM$ for short, respectively.  

\newcommand{\Itrain}{\cI_{\mathrm{train}}}
\newcommand{\Itest}{\cI_{\mathrm{test}}}

Consider the bilinear  combinatorial extrapolation setting where we can sample from the matrix $\bstM$ in the top three blocks, i.e., the block $\{(1,1),(1,2),(2,1)\}$, where for convenience we partition $\bstM$ as 
\begin{align*}
\bstM =  \begin{bmatrix}\bstM_{11} & \bstM_{12}\\
\bstM_{21} & \bstM_{22}\end{bmatrix},
\end{align*} with $\bstM_{11} \in \R^{\alpha n \times \beta m}$, $\bstM_{12} \in \R^{\alpha n \times (1-\beta)m}$, $\bstM_{21} \in \R^{(1-\alpha) n \times \beta m}$, and $\bstM_{22} \in \R^{(1-\alpha) n \times (1-\beta) m}$. Here we assume that  $\alpha,\beta\in(0,1)$ are chosen such that the dimensions of these sub-matrices are positive integers. Our goal is to use data from blocks $\{\bstM_{11},\bstM_{12},\bstM_{21}\}$ to predict and generalize to the uniform distribution supported on the bottom block $\bstM_{22}$. Moreover, we define $\bstA = \bA(\fst),\bstB = \bB(\gst)$. Thus, $\bstM = \bstA(\bstB)^\top$, where the factors have block-decomposition.
\begin{align*}\bstA = \begin{bmatrix} \bstA_{1}\\
\bstA_{2}\end{bmatrix}, \quad \bstB = \begin{bmatrix} \bstB_{1}\\
\bstB_{2}\end{bmatrix}
\end{align*} 
Similarly, for estimates $\fhat,\ghat$, we can define $\bhatM_{ij} = \bM_{ij}(\fhat,\ghat)$, $\bhatA_i = \bA_i(\fhat)$, $\bhatB_j = \bB_j(\ghat)$, giving the block-decompositions:
\begin{align*}
\bhatM = \begin{bmatrix}\bhatM_{11} & \bhatM_{12}\\
\bhatM_{21} & \bhatM_{22}\end{bmatrix}, \quad \bhatA = \begin{bmatrix} \bhatA_{1}\\
\bhatA_{2}\end{bmatrix}, \quad \bhatB = \begin{bmatrix} \bhatB_{1}\\
\bhatB_{2}\end{bmatrix}
\end{align*}

\paragraph{Data distribution and representation of train and test risks.} For this example, we suppose that in $\Dtrain$, data is sampled uniformly  sampled from the entries of $\{\bstM_{11},\bstM_{12},\bstM_{12}\}$, and data from $\Dtest$ is uniform on $\bstM_{22}$. Then,
\begin{align*}
\Risk(\fhat,\ghat;\Dtest) &= \frac{1}{|\Itrain|}\sum_{x,y \in \Itrain}(\langle\fhat(x),\ghat(y)\rangle - \langle\fst(x),\gst(y)\rangle)^2\\
\Risk(\fhat,\ghat;\Dtest) &= \frac{1}{|\Itest|}\sum_{x,y \in \Itest}(\langle\fhat(x),\ghat(y)\rangle - \langle\fst(x),\gst(y)\rangle)^2,
\end{align*}
where $\Itrain = \{1 \le x \le \alpha n, 1 \le y \le \beta m\} \cup \{\alpha n < x \le n, 1 \le y \le \beta m\} \cup \{1  < x \le \alpha n, \beta m < y \le  m\}$, and  $\Itest = \{\alpha n < x \le n, \beta m < y \le  m\}$. Then, up to constants polynomial in $\alpha,\beta,1-\alpha,1-\beta$, and their inverses,
\begin{align}
\Risk(\fhat,\ghat;\Dtest) &\propto \frac{1}{nm}\sum_{i,j \in \{1,1\},\{1,2\},\{2,1\}}\|\bstM_{ij} - \bhatM_{ij}\|^2 \nonumber\\
&= \frac{1}{nm}\sum_{i,j \in \{1,1\},\{1,2\},\{2,1\}}\|\bstA_i(\bstB_j)^\top - \bhatA_i(\bhatB_j)\|^2 \label{eq:training_risk}\\
\Risk(\fhat,\ghat;\Dtest) &\propto \frac{1}{nm}\|\bstM_{22} - \bhatM_{22}\|^2 = \frac{1}{nm}\|\bstA_2(\bstB_2)^\top - \bhatA_2(\bhatB_2)\|^2
\end{align}

  \paragraph{Understanding \Cref{asm:bilinear}.} The factorization  inherently introduces the bilinear embedding form as defined in \Cref{asm:bilinear}, and is also  illustrated figuratively in \Cref{fig:illustration_1}.

  \paragraph{Understanding \Cref{asm:density}.}  We let $\cD_{i\otimes j}$ be uniform on the entries of $\bstM_{ij}$ (i.e. $\cdxone$ uniform on $[\alpha n]$, $\cdyone$ uniform on $[\beta m]$, and so on). We can compute,
\begin{align}
&\frac{\rmd \cD_{1\otimes 1}}{\rmd \Dtrain} = \frac{(1-\alpha)\beta + (1-\beta)\alpha + \alpha\beta}{\alpha \beta},~~~\frac{\rmd \cD_{1\otimes 2}}{\rmd \Dtrain} = \frac{(1-\alpha)\beta + (1-\beta)\alpha + \alpha\beta}{(1-\beta)\alpha}\label{equ:uniform_derivative_1}\\
&\qquad\qquad\frac{\rmd \cD_{2\otimes 1}}{\rmd \Dtrain} = \frac{(1-\alpha)\beta + (1-\beta)\alpha + \alpha\beta}{(1-\alpha)\beta},~~~\frac{\rmd \Dtest}{\sum_{i,j}  \rmd\cD_{i\otimes j}} = \alpha\beta, \label{equ:uniform_derivative_2}
\end{align} 
where we write $\frac{d\cD_1(x,y)}{d\cD_2(x,y)}$ as $\frac{d\cD_1}{d\cD_2}$ for short since they are identical on the support with uniform distributions. Note that \Cref{equ:uniform_derivative_1,equ:uniform_derivative_2} instantiate the constants $\kaptrain=\frac{(1-\alpha)\beta + (1-\beta)\alpha + \alpha\beta}{\min\{(1-\beta)\alpha,(1-\alpha)\beta,\alpha\beta\}}$ and $\kaptest=\alpha\beta$ in \Cref{asm:density}. 

\paragraph{Understanding \Cref{asm:cov}.} In our notation, \Cref{asm:cov} is equivalent to the condition that
\begin{align}
\frac{1}{1-\alpha}(\bstA_{2})^\top \bstA_{2}  \preceq \frac{\kapcov}{\alpha}(\bstA_{1})^\top \bstA_{1} , \quad  \frac{1}{1-\beta}(\bstB_{2})^\top \bstB_{2}  \preceq \frac{\kapcov}{\beta}(\bstB_{1})^\top \bstB_{1} \label{eq:change_of_cov}
\end{align}
This assumption gives a quantitative version of the follow qualitative statement that all the ``features'' of $\bstA_2$ (resp. $\bstB_2$) are ``contained in'' $\bstA_1$ (resp. $\bstB_2$). That is, 
\begin{align*}
\rowspace(\bstA_2) \subset \rowspace(\bstA_1) = \rowspace(\bstA), \quad \rowspace(\bstB_2) \subset \rowspace(\bstB_1) = \rowspace(\bstB). 
\end{align*}
\begin{rem}[A heuristic explanation of how we apply \Cref{eq:change_of_cov} ]\label{rem:heuristic_derivation} Essentially, \Cref{eq:change_of_cov}  implies the following decomposition:
\begin{align*}
nm\cdot \Risk(\fhat,\ghat;\Dtest) &\propto \|\bstA_2(\bstB_2)^\top - \bhatA_2(\bhatB_2)\|^2 \\
&= \|\bstA_2(\bstB_2 - \bhatB_2) + (\bstA_2 - \bhatA_2)(\bstB_2)^\top  + (\bstA_2 - \bhatA_2)(\bstB_2 - \bhatB_2)^\top\|_{\fro}^2\\
&\lesssim \|\bstA_2(\bstB_2 - \bhatB_2)\|_{\fro}^2 + \|(\bstA_2 - \bhatA_2)(\bstB_2)^\top\|_{\fro}^2  + \|(\bstA_2 - \bhatA_2)(\bstB_2 - \bhatB_2)^\top\|_{\fro}^2,
\end{align*}
where above, we more precisely view $(\bhatA,\bhatB)$ as being representative of an equivalence class of $(\bhatA \bS,\bhatB \bS^{-\top})$ for an invertible transformation $\bS$. 
Assuming heuristically that $\|(\bstA_2 - \bhatA_2)(\bstB_2 - \bhatB_2)^\top\|_{\fro}^2$ is lower-order (that this is true in the approximate low-rank setting is \emph{far} from obvious), the remaining terms are then
\begin{align}
\|\bstA_2(\bstB_2 - \bhatB_2)\|_{\fro}^2 + \|(\bstA_2 - \bhatA_2)(\bstB_2)^\top\|_{\fro}^2, \label{eq:off_term_one}
\end{align}
which, applying \Cref{asm:cov} as re-stated in \Cref{eq:change_of_cov}, yields a bound of
\begin{align}
\cO(1) \cdot \left(\|\bstA_1(\bstB_2 - \bhatB_2)\|_{\fro}^2 + \|(\bstA_2 - \bhatA_2)(\bstB_1)^\top\|_{\fro}^2\right). \label{eq:off_term_two}
\end{align}
When the training risk is small, \Cref{eq:training_risk} gives that $\|\bstA_1(\bstB_1)^\top - \bhatA_1\bhatB_1^\top\|_{\fro}^2$ is small. Making another heuristic leap that, up to a similarity transform, this implies that $\|\bhatA_1 - \bhatA_1\|$ and $\bstB - \bhatB_1\|$ are small in appropriate norms, we can control \Cref{eq:off_term_two} as soon as we can bound
\begin{align*}
\left(\|\bstA_1(\bstB_2 - \bhatB_2)\|_{\fro}^2 + \|(\bstA_2 - \bhatA_2)(\bstB_1)^\top\|_{\fro}^2\right) \approx \left(\|\bhatA_1(\bstB_2 - \bhatB_2)\|_{\fro}^2 + \|(\bstA_2 - \bhatA_2)(\bhatB_1)^\top\|_{\fro}^2\right),
\end{align*}
Due to \Cref{eq:training_risk} , this term is controlled as soon as the training risk is small. Importantly, in this heuristic derivation, we only apply \Cref{asm:cov} to translate \Cref{eq:off_term_one} to \Cref{eq:off_term_two}, which shows that we only apply the covariance relation on the star-matrices $\bstA_i,\bstB_j$, and not the hat-matrices $\bhatA_i,\bhatB_j$. 
\end{rem}
\paragraph{Understanding \Cref{asm:bal}.} This assumptions means that we select a basis for which 
\begin{align*}\frac{1}{\alpha n}(\bstA_{1})^\top \bstA_{1}=\frac{1}{\beta m}(\bstB_{1})^\top \bstB_{1} := \Sigst.
\end{align*} 
Again, it is argued in \Cref{sec:balance_other_rank} that we can select such a change of basis without loss of generality.

\paragraph{Understanding \Cref{asm:training_approximation}}
 Unnder this   setting, we also have 
\begin{align*}
\frac{\rmd \Dtrain}{\sum_{i,j}  \rmd\cD_{i\otimes j}} = \frac{\alpha\beta(1-\alpha)(1-\beta)}{(1-\alpha)\beta + (1-\beta)\alpha + \alpha\beta}=:\kaptraintil, 
\end{align*} 
for some $\kaptraintil$. Then, together with \Cref{asm:cov} with $\kapcov=1/\kappa_1$, we know that \Cref{asm:training_approximation} is satisfied with 
$$
\kapapx = 4\kaptraintil\kapcov^2=\frac{4\alpha\beta(1-\alpha)(1-\beta)}{\kappa_1^2\cdot[(1-\alpha)\beta + (1-\beta)\alpha + \alpha\beta]}. 
$$

\paragraph{Understanding \Cref{asm:eigendecay}}. This is precisely the eigendecay of the matrix
In addition,  \Cref{asm:eigendecay} now becomes the spectral decay assumption on the matrix 
\begin{align}
\Sigst := \frac{1}{\alpha n}(\bstA_{1})^\top \bstA_{1}=\frac{1}{\beta m}(\bstB_{1})^\top \bstB_{1}. 
\end{align}

%
%

\newcommand{\rhat}{\hat{r}}
\newcommand{\Riskr}[1][r]{\cR_{[#1]}}

\renewcommand{\epstrain}{\epsilon_{\mathrm{trn}}}
\newcommand{\epsone}{\epsilon_{1\otimes 1}}

\section{Algorithms and Main Results}\label{sec:algorithms}

\paragraph{Additional notation.} For any inner-product space $\cV$ (e.g., $\hilspace$ or $\R^r$ for $r \in \N$), we  say $(f,g)$ are $\cV$-embeddings if $f:\xspac \to \cV$, $g:\yspace \to \cV$; we say they are 
\emph{isodimensional embeddings} if $(f,g)$ are $\cV$-embeddings for some $\cV$. Given a probability distribution $\cD$ on $(x,y)\in\xspac\times\yspace$ pairs, we define the excess \emph{risk} of the  isodimensional $\cV$-embeddings $(\fhat,\ghat)$ as $\Risk(\fhat,\ghat;\cD) := \Exp_{(x,y)\sim \cD}[(\langle \fhat(x),\ghat(y) \rangle_{\cV} - \hst(x,y))^2]$.
We often omit function dependence on $(x,y)$ in expectations, i.e., writing  it as  $\Risk(f,g;\cD):=\Exp_{\cD}[(\langle g,f \rangle_\cV - \hst)^2]$ for short. We further define the
\begin{align}
\sigma_i(f,g) := \sigma_i\left(\Exp_{\cdxone}[ff^\top]^{\frac{1}{2}}\cdot \Exp_{\cdyone}[gg^\top]^{\frac{1}{2}}\right), \label{eq:sigi_def}
\end{align}
and if $(f,g)$ are $\R^r$-embeddings, we say $(f,g)$ are \textbf{\emph{full-rank}} if $\sigma_r(f,g) > 0$. We adopt the shorthand \iftoggle{arxiv}
{
\begin{align}
\bsigst_i := \lambda_{i}(\Sigst), \quad \tailsf_q(k) := \sum_{i > k} \lambda_i(\Sigst)^q = \sum_{i > k} (\bsigst_i)^q, ~ q \ge 1.
\end{align}
}
{$\bsigst_i := \lambda_{i}(\Sigst)$, and  for $q \ge 1$, $\tailsf_q(k) := \sum_{i > k} \lambda_i(\Sigst)^q = \sum_{i > k} (\bsigst_i)^q$.} We use $a\lesssim b$ to denote $a\leq c\cdot b$ for some absolute constant $c$; we use $a\lesssimst b$ to denote $a\leq c \cdot b$ for some $c$ that is at most  polynomial  in the problem constants $\kapcov,\kaptrain,\kaptest,\kapapx$ in \Cref{asm:cov,asm:density,asm:training_approximation}.

\subsection{A meta-theorem for bilinear combinatorial extrapolation}

We now provide a meta-theorem on the risk bound for bilinear combinatorial extrapolation.
The bound depends on an upper bound $\epstrain$ on the risk of the learned embedding $(\fhat,\ghat)$ on the training distribution $\Dtrain$, on $\epsone$ that upper-bounds the risk on the top-block distribution $\cdone$, as well as on  $\sigma_r(\fhat,\ghat)$ defined in \Cref{eq:sigi_def}.

\begin{defn}[$\alpha$-Conditioned \& $(\epstrain,\epsone)$-Accurate Embeddings]\label{defn:acc_embed} Given $\alpha \ge 1$ and $\epstrain,\epsone > 0$, we say $\R^r$-embeddings  $(\fhat,\ghat)$ are $\alpha$-conditioned if $\sigma_r(\fhat,\ghat)^2 \ge (\bsigst_r)^2/\alpha$ and \emph{$(\epstrain,\epsone)$-accurate} if $\Risk(\fhat,\ghat;\Dtrain) \le \epstrain^2$ and  $\inf_{r' \ge r}\Riskr[r'](\fhat,\ghat;\cdone) \le \epsone^2$\footnote{Because $\hst_{r'}$ converges to $\hst$ in $\cL_2(\cdone)$ as $r' \to \infty$,  $\inf_{r' \ge r}\Riskr[r'](\fhat,\ghat;\cdone) \le \Risk(\fhat,\ghat;\cdone)$.}, where $\Riskr[s]$ is the excess risk relative to $\hst_s = \langle \fst_s,\gst_s\rangle$, evaluated on $\cdone$:
 \begin{align}
\Riskr[s](\fhat,\ghat;\cdone) := \Exp_{(x,y)\sim \cdone}[(\langle \fhat(x),\ghat(y) \rangle - \hst_s(x,y))^2]. \label{eq:Riskr}
\end{align}
\end{defn}
\begin{theorem}[Main Risk Bound]\label{thm:generic_risk_bound}Given $\alpha \ge 1$, suppose  $(\fhat,\ghat)$ are $\alpha$-conditioned and $(\epstrain,\epsone)$-accurate $\R^r$-embeddings, where $r \le \sigst_1/(40\epsone)$. Then under \Cref{asm:bilinear,asm:cov,asm:density,asm:bal}, and if $\bsigst_r > 0$, 
\begin{align}
\Risk(\fhat,\ghat;\Dtest) \lesssim_{\star}    \left(r^4\epsone^2 + \alpha r^2 (\bsigst_{r+1})^2  +  \tailsf_1(r)^2\right)  + \alpha \left(\frac{r^6 \epsone^4 + \epstrain^4 +  \tailsf_2(r)^2}{(\bsigst_r)^2}\right)
\end{align}
Moreover, the condition $\epsone^2 \le (1-\alpha^{-1})(\bsigst_r)^2$ ensures that $\sigma_r(\fhat,\ghat)^2 \ge (\bsigst_r)^2/\alpha$. 
\end{theorem}

\Cref{thm:generic_risk_bound} is proved in \Cref{sec:proof_overview}; its implications are best understood through its instantiation below.  Here, we note an important point that the dependence on the ``top-block'' error $\epsone^2$ is scaled up by polynomial factors of $r$. It is from this  fact that the benefits of double-stage ERM derive.

 
\subsection{Single-stage empirical risk minimization}
\newcommand{\fhatst}{\fhat_{\textsc{ss}}}
\newcommand{\ghatst}{\ghat_{\textsc{ss}}}
\newcommand{\fhatdt}{\fhat_{\textsc{ds}}}
\newcommand{\ghatdt}{\ghat_{\textsc{ds}}}
\newcommand{\epsst}{\epsilon_{\mathrm{ss}}}
\newcommand{\specerrst}{\textsc{ApxErr}_{\textsc{ss}}}
\newcommand{\staterrst}{\textsc{StatErr}_{\textsc{ss}}}

\newcommand{\errsingt}{\textsc{Err}_{\textsc{ss}}}

A natural algorithm is to fix a target rank $r \in \N$ and compute a \emph{single-stage empirical risk minimizer},  i.e., to find $(\fhatst,\ghatst) \in \argmin_{f \in \cF_r,g \in \cG_r} \sum_{i=1}^n ( \langle f(x_i),g(y_i)\rangle - z_i)^2$, where we draw $(x_i,y_i,z_i) \iidsim \Dtrain$ and function classes $\cF_r,\cG_r$ are as given in \Cref{asm:function_apx}. By combining \Cref{thm:generic_risk_bound}, the fact that $\epsone \le \kaptrain\epstrain$ by \Cref{asm:density}, and standard statistical learning  arguments to bound $\epstrain$,   we can obtain the following guarantee (whose proof is given in \Cref{sec:thm:st_main}). 

\begin{theorem}\label{thm:st_main} 
Fix $\delta \in (0,1)$, $\alpha \ge 1$. Fix $r \in \N$, suppose and if $\bsigst_r > 0$, and consider the rank-$r$ ERM predictors above.

Under \Cref{asm:bilinear,asm:cov,asm:density,asm:bal,asm:function_apx,asm:training_approximation},  with probability at least $1-\delta$, {\textbf{if}} $(\fhatst,\ghatst)$ are $\alpha$-conditioned,  then $\Risk(\fhatst,\ghatst;\Dtest) \lesssim_{\star}    \errsingt(r,n,\delta)$ with 
\begin{align*}
\errsingt(r,n,\delta) &:= \alpha \specerrst(r) + r^4\staterrst(r,n,\delta) + \tfrac{\alpha r^6}{(\sigst_r)^2}\staterrst(r,n,\delta)^2,
\end{align*}
where $\staterrst(r,n,\delta) :=  \frac{B^4 (\capac_r + \log(1/\delta))}{n}$ captures the statistical error, and where $\specerrst(r) := r^4 \tailsf_2(r)  + \tailsf_1(r)^2 + r^2 (\bsigst_{r+1})^2 + \frac{r^6\cdot \tailsf_2(r)^2}{(\sigst_r)^2}$. Moreover, under \Cref{asm:eigendecay},
\begin{align}
\specerrst(r) &\lesssimst \begin{cases} C^2(1+\gamma^{-1})^2 r^{6-2\gamma} & \text{(polynomial decay)}\\
C^2r^6(\gamma^{-1} + r)^2e^{-2\gamma r} & \text{(exponential decay)}. \label{eq:specerrst_bound}
\end{cases}
\end{align}
\end{theorem}

To the best of our knowledge, \Cref{thm:st_main}  is the first result that establishes bilinear combinatorial extrapolation for (sufficiently fast) polynomial decay, $\gamma > 3$. However, the theorem has two weaknesses:  first, our upper bound on $\specerrst(r)$ does not decay to zero under polynomial decay with $\gamma \le 3$. Second, $\alpha$ depends on the ratio of $\sigma_r(\fhatst,\ghatst)$ to $\bsigst_r$, and we do not (yet) know a way to control this quantity, except in the special case  when $(\bsigst_r)^2 > 2\kapapx\kaptrain\tailsf_2(r)$ (see \Cref{rem:suff_eig_cond}). To see the culprit, consider the (somewhat trivializing) case where $\Dtrain = \cdone$. Then $\epsone^2 = \epstrain^2$, and by the Eckhart-Young theorem, $\epsone^2 \ge \Risk(\fst_r,\gst_r;\cdone) = \tailsf_2(r)$. In this case, we have (a) the upper bound in \Cref{thm:generic_risk_bound} is no better than $\frac{r^6 \tailsf_2(r)^2}{(\bsigst_r)^2}$, which scales like $r^{6 - 2\gamma}$ for polynomial spectral decay, and (b) unless $\tailsf_2(r) < (\bsigst_r)^2$, we can not use \Cref{thm:generic_risk_bound} to ensure a lower bound on $\alpha$. These issues exactly  arise from our consideration of the  \emph{modest} spectral decay case, and would not cause trouble in a standard MNAR matrix completion case with a \emph{sharp} spectral cutoff. In the next section, we present a more involved algorithm to circumvent these limitations. 
\newcommand{\hred}{\hat{h}_{\textsc{red}}}

\newcommand{\ermds}{\textsc{ErmDS}}

\subsection{Double-stage empirical risk minimization (\ermds)}\label{sec:ermds} Given a desired rank cutoff $\rcut$, we also develop a \emph{Double-Stage ERM} ($\ermds$) algorithm, which learns $\R^{\rhat}$-embeddings $(\fhatdt,\ghatdt)$ for a data-dependent $\rhat$ such that  $\epsone \ll \rcut^3\tailsf_2(\rcut)$, for which   $\tailsf_q(\rhat)$ is not much larger than $\tailsf_q(\rcut)$. Hence, we can instantiate \Cref{thm:generic_risk_bound} with $r = \rcut$, but \emph{without suffering from the prefactor powers of $\rcut$ premultiplying $\epsone$}. Our procedure relies on a slightly stronger oracle:
\begin{assumption}[Unlabeled $\cdone$-Oracle]\label{asm:unlabeled_cdone_orac} In addition to being able to sample i.i.d. data  $(x,y,z) \sim \Dtrain$, we can also sample \emph{unlabeled} i.i.d.  data $(x,y) \sim \cdone$.
\end{assumption}
Moreover generally, \Cref{sec:gen:asm:unlabeled_cdone_orac} shows that $\cdone$ can be replaced with any product distribution on $\xspac \times \yspace$ with bounded density with respect to  $\cdone$. 

We summarize the details of \ermds{} in \Cref{alg:main_alg}.
The algorithm has three spectral parameters: an overparametrized rank $p$, a spectral cutoff $\sigcut$, and a rank cutoff $\rcut$. We first train high-dimensional $\R^p$-embeddings $(\tilde f,\tilde g)$, where ideally $p \gg \rcut$ is sufficiently large so that $\tailsf_2(p) \ll \rcut^6 \tailsf_2(\rcut)^2/(\sigst_{\rcut})^2$. We then perform an SVD-approximation  of $(\tilde f,\tilde g)$, first by estimating their covariance matrices, and then using these matrices to perform dimension reduction (the routine $\dimred$ in \Cref{alg:dim_red}). The dimension reduction routine   reduces to a rank-at-most-$\rhat \le \rcut$ predictor $\hred$, where $\rhat$ is determined by the estimated covariances matrices and spectral cutoff $\sigcut$. In a final distillation phase, we learn $\R^{\rhat}$-embeddings $(\fhatdt,\ghatdt)$ by regularizing the supervised training error on labeled samples from $\Dtrain$ with empirical risk on samples $(x',y',\hred(x',y'))$, where $(x',y')$  are drawn from $\cdone$ and labeled by $\hred$. This is similar to the process of distillation in \cite{hinton2015distilling}, where a larger deep network is used to supervise the learning of a smaller one.  
\Cref{alg:main_alg} enjoys the following guarantee, the detailed version of which is given in \Cref{app:DS_ERM_detailed_results} and  proved in \Cref{sec:analysis_training_alg}.
\begin{theorem}\label{thm:main_dt_simple}  For any $\rcut \gtrsim_{\star}\mathrm{poly}(C/\bsigst_1,\gamma^{-1})$ and $\epsilon > 0$ and $\delta > 0$, there exists a choice of $\sigcut > 0$, $p \lesssim_{\star} (\rcut)^{c}$ for some  universal  $c > 0$, and sample sizes $n_1,n_2,n_3,n_4 \lesssim_{\star} \poly(p,\capac_p,\log(1/\delta),B,\epsilon^{-2})$, such that, \Cref{alg:main_alg} with $\lambda = \rcut^4$ and $\mu = B^2/n_1$ satisfies that with  probability at least $1-\delta$,
\begin{align*}
\Risk(\fhatdt,\ghatdt;\Dtest) \lesssimst \epsilon^2 + C^2(1+\gamma^{-2})\begin{cases}   \rcut^{-2\gamma} & \text{(polynomial decay)}\\   e^{-2\gamma \rcut} & \text{(exponential decay)}
\end{cases}.
\end{align*}
\end{theorem}
\emph{Proof Sketch of \Cref{thm:main_dt_simple}.}  We first show, by analogy to \Cref{thm:st_main}, that $\Risk(\ftil,\gtil;\cdone) \lesssim_{\star} \tailsf_2(p) + o(n_1)$. We then learn a data-dependent $\rhat$, chosen by the $\dimred$ procedure, so as to satisfy $\bsigst_{\rhat} \gtrsim \sigcut$, and to have lower bounded relative singular-value gap $(\bsigst_{\rhat} - \bsigst_{\rhat+1})/\bsigst_{\rhat} \gtrsim 1/\rcut$. We then argue that $\hred$ constructed in \Cref{line:dim_red_line}  is the correct  analogue  rank-$\rhat$ SVD approximation of $\langle \ftil, \gtil \rangle$, just as  $\hst_{\rhat}$ is the best rank-$\rhat$ approximation of $\hst$ on $\cdone$. We then use our novel relative-gap SVD perturbation bound (\Cref{thm:svd_pert}) and  limiting arguments to show that our bound $\Risk(\ftil,\gtil;\cdone) = \Exp_{\cdone}[(\langle \ftil,\gtil \rangle - \hst)^2]$ implies  $\Exp_{\cdone}[(\hred - \hst_{\rhat})^2] \lesssim_{\star} \rcut^2\left(\tailsf_2(p) + o(n_1)\right)$. The factor of $\rcut^2$ arises from the relative singular-value gap at $\rhat$ mentioned above. In addition, we argue that $\dimred$ chooses $\rhat$ large enough such  that the tails $\tailsf_q(\rhat)$ and $\tailsf_q(\rcut)$ are close. Finally, we show that the distillation   step with a large $\lambda$ forces $\langle \fhatdt,\ghatdt\rangle$ to be close to $\hred \approx \hst_{\rhat}$ on $\cdone$; this ensures that we can invoke \Cref{thm:generic_risk_bound} with $\epsone^2 = \Riskr[\rhat]( \fhatdt,\ghatdt;\cdone) \approx  \Exp_{\cdone}[(\hred - \hst_{\rhat})^2] \lesssim_{\star} \rcut^2\left(\tailsf_2(p) + o(n_1)\right)$. In particular, by making $p \gg \rcut \ge \rhat$, we can ensure $\rhat^3 \epsone^2 \le \rcut^3\epsone^2 \ll \tailsf_2(\rcut)$, as desired. \hfill $\blacksquare$

\iftoggle{nips}
{}
{\begin{algorithm}[!t]
    \begin{algorithmic}[1]
    \State{}\textbf{Input:} Sample sizes $n_1,\dots,n_4$; over-parameterized rank $p$, under-parameterized cutoff $\rcut$, parameter $\sigcut$, 
    regularization parameters $\mu,\lambda > 0$.
    \State{}\textbf{Overparametrized Training.} Sample $n_1$ labeled triples $\{(x_{1,i},y_{1,i},z_{1,i}\}_{i \in [n_1]}$ i.i.d. from $\Dtrain$, and set
    $$(\ftil,\tilde g) \in \argmin_{(f,g) \in \cF_{p}\times \cG_{p}} \frac{1}{n_1}\sum_{i=1}^{n_1} (\langle f(x_{1,i}),g(y_{1,i})\rangle - z_{1,i})^2.$$ 
    \State{}\textbf{Covariance Estimation.}  Sample $n_2$ \emph{unlabeled} examples $\{(x_{2,i},y_{2,i})\}_{i \in [n_2]} \sim \cD_{1\otimes 1}$, and define  covariance  matrices $
    \hat{\bSigma}_{\ftil} := \frac{1}{n_2}\sum_{i=1}^{n_2}\ftil(x_{2,i})\ftil(x_{2,i})^\top$, $\hat{\bSigma}_{\gtil} := \frac{1}{n_2}\sum_{i=1}^{n_2}\gtil(y_{2,i})\gtil(y_{2,i})^\top$.
    \State{}\textbf{Dimension Reduction.} 
    $$(\rhat,\bhatQ_{\rhat}) \gets \dimred(\hat{\bSigma}_{\ftil} +  \mu \eye_p, \hat{\bSigma}_{\gtil} +  \mu \eye_p,\rcut,\sigcut),$$ and $\hred(x,y) := \langle \ftil(x), \bhatQ_r \cdot \gtil(y) \rangle$. \label{line:dim_red_line}
    \State{}\textbf{Distillation. }  Sample $n_3$ \emph{labeled} examples $\{(x_{3,i},y_{3,i},z_{3,i})\}_{i\in [n_3]} \sim \Dtrain$ and $n_4$ \emph{unlabeled} samples $\{(x_{4,i},y_{4,i})\}_{i\in [n_4]} \sim \cD_{1\otimes 1}$. Define the losses $\hat{L}_{(3)}(f,g) = \frac{1}{n_3}\sum_{i=1}^{n_3} (\langle f(x_{3,i}), g(y_{3,i}) \rangle - z_{3,i})^2 $ and $
  \hat{L}_{(4)}(f,g) =  \frac{1}{n_4}\sum_{i=1}^{n_4} (\langle f(x_{4,i}), g(y_{4,i}) \rangle - \hred(x_{4,i},y_{4,i}))^2$,  and select 
  $$(\fhatdt,\ghatdt) \in \argmin_{(f,g) \in \cF_{\rhat} \times \cG_{\rhat}} \hat{L}_{(3)}(f,g) + \lambda \hat{L}_{(4)}(f,g).$$
    \end{algorithmic}
      \caption{Double-Stage ERM (\ermds)}
      \label{alg:main_alg}
    \end{algorithm}

    \begin{algorithm}[!t]
    \begin{algorithmic}[1]
    \State{}\textbf{Input:} $\bX,\bY \succ 0$, $r_0 \in \N$, $\sigma_0$. 
    \State{} Compute $\bW := \bX^{\half}(\bX^{\half}\bY \bX^{\half})^{-\half}\bX^{\half}$. 
    \State{} Compute $\bSigma := \bW^{\half} \bY \bW^{\half}$; set 
    \iftoggle{colt}
    {$r \gets  \max\left\{r \in [r_0]: \sigma_r(\bSigma) \ge \sigma_0, \,\sigma_r(\bSigma) - \sigma_{r+1}(\bSigma) \ge \frac{\sigma_r(\bSigma)}{r_0}\right\}$.}
    {
    \begin{align*}
    r \gets  \max\left\{r \in [r_0]: \sigma_r(\bSigma) \ge \sigma_0, \,\sigma_r(\bSigma) - \sigma_{r+1}(\bSigma) \ge \frac{\sigma_r(\bSigma)}{r_0}\right\}
    \end{align*}
    }
    \State{} Let $\bP_r$ denote the projection onto the top $r$ eigenvectors of $\bSigma$.
    \State{} Return $(r,\bQ_r)$, where $\bQ_r \gets \bW^{-\half}\bP_r \bW^{\half}$.
    \end{algorithmic}
      \caption{$\dimred(\bX, \bY, r_0, \sigma_0)$}
      \label{alg:dim_red}
    \end{algorithm}}

\subsection{Illustrative Examples}\label{sec:illustrative_examples}
We now present some illustrative examples which demonstrate conditions under which the various theorems apply, and when they do not. In all that follows, we take $\cX = \cY = [2n]$, where $n \in \N$ is some integer. We let $\cdxone,\cdyone$ both denote the uniform distributions on $[n]$, and $\cdxtwo,\cdytwo$ the uniform distribution on $ [n+1:2n]:=\{n+1,n+2,\dots,2n\}$. We then let $\Dtrain$ denote the uniform distribution on the set $[n]\times[n] \cup [n] \times [n+1:2n] \cup [n+1:2n] \times [n]$, and $\Dtest$ be uniform on $[2n]\times [2n]$.

\paragraph{Example 1: A simple victory for ERM.} Consider the setting where $\cH = \R$, and 
\begin{align}
\fst(x) = \gst(y) = 1 \in \R, \quad \forall x,y \in [2n].
\end{align} 
Then, we must have that 
\begin{align}
\Risk(\fhat,\ghat;\Dtrain) = \frac{1}{3n^2}\sum_{i,j=1}^n(\fhat(i)\ghat(j) - 1)^2 + (\fhat(i+n)\ghat(j) - 1)^2 + (\fhat(i)\ghat(j+n) - 1)^2.
\end{align}
Let's for simplicity assume consider minimizes $\fhat,\ghat$ with zero risk: $\Risk(\fhat,\ghat;\Dtrain)  = 0$. Then, for all $j,j' \in [2n]$, 
\begin{align}
\ghat(j)\fhat(i) = 1 = \ghat(j')\fhat(i),
\end{align}
so that $\ghat(\cdot) = 1/\fhat(i)$ is constant. Similarly, $\fhat(\cdot)$ is a constant. It then follows that any zero-loss minimizer is determined by 
\begin{align}
\fhat(\cdot) = c, \quad \ghat(\cdot) = \frac{1}{c}, \quad c \ne 0.
\end{align}
This argument can be generalized to the standard argument for completion of a block-diagonal matrix, e.g. in \cite{shah2020sample}.

\paragraph{Example 2: Necessity of \Cref{asm:cov}} It is straightforward to show that \Cref{asm:cov} is necessary. Suppose that 
\begin{align}
\fst(x) = (1,\gamma \I(x > n)), \gst(y) = (1,\gamma \I(y > n)), \quad x,y \in [2n].
\end{align}
Then, 
\begin{align}
\langle \fst(x),\gst(y)\rangle = 1 + \gamma^2 \I(x,y > n).
\end{align}
We then see that 
\begin{align}
\fhat(x) = \ghat(y) \equiv (1,0)
\end{align}
has $\Risk(\fhat,\ghat;\Dtrain) = 0$, despite $\Risk(\fhat,\ghat;\Dtest) = \gamma^2$.

\paragraph{Example 3: ERM can be rank deficient, and the necessity of double training.} This example demonstrates that pure ERM does not guarantee consistent recovery, thereby motivating for double-training algorithm. Let us consider the case where $\hilspace = \R^2$, i.e. the embeddings embed into $\R^2$. We suppose that the ground truth embeddings are
\begin{align*}
\fst(x) = \gst(y) = (1,0) \in \R^2, \quad \forall x,y \in [2n].
\end{align*}
Thus, $\langle \fst(x),\gst(y)\rangle \equiv 1$. 
Consider estimators
\begin{align}\label{eq:fhat_ghat_exmp}
\fhat(x) = (1,\I(x > n)), \quad \ghat(y) = (1,\I(y > n)) \in \R^2. 
\end{align}
It is then clear that, for all $x,y$ such that either $x \le n$ or $y \le n$, $\langle \fhat(x),\ghat(y)\rangle = 1 = \langle \fst(x),\gst(y)\rangle $; hence,
\begin{align}
\Risk(\fhat,\ghat;\Dtrain) = 0. \label{eq:Risk_train_exmp}
\end{align}
Still, on $\Dtest$, which is supported on $[n+1:2n]\times [n+1:2n]$, $\langle \fhat(x),\ghat(y)\rangle = 1 + \I(x,y > n) = 2$. Thus, 
\begin{align*}
\Risk(\fhat,\ghat;\Dtest) = 1.
\end{align*}
This may seem like a contradiction of our results, as all the assumption in \Cref{sec:formulation} are satisfied. Indeed, all these assumptions pertain to either (a) the data distribution, or (b) the ground-truth embeddings $\fst,\gst$. Moreover, we can imagine our ERM classes $\Fclass_{r},\Gclass_r$for $r = 2$ to consist of $\Fclass_2 = \{\fst,\fhat\}$ and $\Gclass_2 = \{\gst,\ghat\}$; even though these classes are well-specified (meeting \Cref{asm:function_apx}), \Cref{eq:Risk_train_exmp} ensures that $(\fhat,\ghat)$ is a valid ERM on $\Dtrain$.

To reconcile this seeming contradiction, observe that, in our example, 
\begin{align*}
\Sigst = \begin{bmatrix}1 & 0 \\ 0 & 1\end{bmatrix}
\end{align*} Hence, $\bsigst_2 = 0$. As $\fhat,\ghat$ are embeddings with $r = 2$, we see that the condition $\bsigst_r > 0$ required for \Cref{thm:generic_risk_bound,thm:st_main} to be non-vacuous is violated. 

Our guarantee for double-training, \Cref{thm:main_dt_simple}, is \emph{not} violated. Indeed, imagine that $(\ftil,\gtil) = (\fhat,\ghat)$ are the solution to the first-stage of ERM in our double-training productive, namely Line 2 in \Cref{alg:main_alg}, with over-parametrized dimnsion $p = 2$. When \Cref{alg:main_alg} calls the dimension-reduction subroutine, \Cref{alg:dim_red}, it selects a rank $r$ for which an appropriate covariance matrix $\bSigma$ has $\sigma_r(\bSigma) \ge \sigcut$, where $\sigcut > 0$ is some cutoff. One can check that, if $(\ftil,\gtil) = (\fhat,\ghat)$ for $\fhat,\ghat$ as in \Cref{eq:fhat_ghat_exmp}, then we compute the matrix
\begin{align}
\bSigma = \begin{bmatrix}1 & 0 \\ 0 & 1\end{bmatrix}
\end{align}
Hence, the dimension-reduction subroutine  will select a rank $r = 1$ instead of rank-$2$, and train therefore ERM with a rank equal to the true rank of the ground-truth $\fst,\gst$.

\section{Proof Overview of the Meta-Theorem -- \Cref{thm:generic_risk_bound}}\label{sec:proof_overview}

In this section, we provide an overview of the key techniques in our proof of the main result \Cref{thm:generic_risk_bound}, which is completed in \Cref{proof:generic_risk_bound}. As noted above, the proofs of \Cref{thm:main_dt_simple,thm:st_main} are given in  \Cref{sec:analysis_training_alg}.

\subsection{Reformulation as matrix completion}\label{sec:mat_for_corr}

To explain the intuition behind our proofs, \iftoggle{arxiv}{let us recall the finite-support case described in \Cref{sec:connection_matrix_completion}}{it helps to consider the case},  where  $|\xspac|$ and $|\yspace|$ are finite, with elements $\{x_1,\dots,x_n\}$ and $\{y_1,\dots,y_m\}$. For $i,j \in \{1,2\}$, we define the probabilities $\sfp_{i,\ell} = \Pr_{\cdxi}[x = x_{\ell}]$ and $\sfq_{j,k} = \Pr_{\cdyj}[y = y_{k}]$. Because of the finite support of the distributions,  we can regard any $\hilspace$-embeddings $(f,g)$ (including $(\fst,\gst)$)  as embeddings into $\hilspace  = \R^d$,  $d = \max\{n,m\}$,  appending zeros if necessary. Define matrices  $\bA_i(f) \in \R^{n \times d}$ and $\bB_j(g) \in \R^{m \times d}$ by assigning the rows to the scaled values of the embeddings \iftoggle{arxiv}{
	\begin{align*}
	\bA_i(f)[\ell,:] = \sqrt{\sfp_{i,\ell}}f(x_\ell)^\top, \quad \bB_j(g)[k,:] = \sqrt{\sfq_{j,k}}g(y_k)^\top, 
	\end{align*}
}{ $\bA_i(f)[\ell,:] = \sqrt{\sfp_{i,\ell}}f(x_\ell)^\top, \quad \bB_j(g)[k,:] = \sqrt{\sfq_{j,k}}g(y_k)^\top$,}
and define $\bM_{i \otimes j}(f,g) = \bA_i(f)\bB_j(g)^\top$. 
	Each matrix $\bM_{i\otimes j}(f,g)$ can be thought of as a look-up table, where $\bM_{i\otimes j}(f,g)[\ell,k] = \sqrt{\sfp_{i,\ell}\sfq_{j,k}}\langle f(x_\ell),g(y_k)\rangle$ is the prediction of $\langle f, g\rangle$, scaled by the square root probability of $x_\ell$ and $y_k$. This reformulation yields the following equivalences, verified  in \Cref{lem:matrix_discrete_correspondence}.

	
	\begin{lem} The following identities hold: (a) $\Risk(f,g;\cD_{i \otimes j}) = \|\bM_{i\otimes j}(f,g) - \bM_{i\otimes j}(\fst,\gst)\|_{\fro}^2$ and (b) $\Exp_{\cdxi}[ff^\top] = \bA_i(f)^\top\bA_i(f)$, and similarly for $\Exp_{\cdyj}[gg^\top] = \bB_j(g)^\top\bB_j(g)$.
	\end{lem}
	Most of our technical results are easiest to establish for the matrix completion  formulation, and then are generalized to accommodate  arbitrary distributions via some careful limiting arguments.

\subsection{Balancing and singular value decomposition} 

	Note that for any  isodimensional embedding $(f,g)$, any embedding $(f',g')= (\bT^{-\top} f, \bT g)$ for some invertible operator $\bT$ satisfies $\langle f', g' \rangle \equiv \langle f,g \rangle$. We thus focus on \textbf{balanced} embeddings.
	
	\begin{defn}[Balanced Embeddings]\label{defn:balanced} \label{defn:balanced_matrix} We say any isodimensional embeddings $(f,g)$ are \emph{balanced} if the  covariance $\Exp_{\cdxone}[ff^\top] = \Exp_{\cdyone}[gg^\top]$; given $\bM \in \R^{n \times m}$, we say $(\bA,\bB) \in \R^{n \times d}\times \R^{m \times d}$ is a balanced factorization of $\bM$ if $\bM = \bA\bB^\top$ and $\bA^\top\bA = \bB^\top \bB$. 
	\end{defn}
	
	Balancing is \emph{orthogonally invariant}: for any  orthogonal transformation $\bU$ (of appropriate dimension), $(f,g)$ are balanced if and only if $(\bU f, \bU g)$ are. Similarly, $(\bA,\bB)$ is a balanced factorization of $\bM$ if and only if $(\bA \bU, \bB \bU)$ is. Moreover, when distributions are discrete, $(f,g)$ are  balanced if and only if $(\bA_1(f),\bB_1(g))$ is a balanced factorization of $\bM_{1\otimes 1}(f,g)$. The matrix-completion interpretation reveals many useful properties of balanced embeddings/factorizations.
	
	\begin{lem}\label{lemma:property_balanced_fac}
	 Suppose  $(f,g)$ are balanced $\hilspace$-embeddings, and $\xspac,\yspace$ are finite spaces. Let $\bP_{[r]}$ denote the orthogonal projection onto the top-$r$ eigenvectors of $\Exp_{\cdxone}[ff^\top] = \Exp_{\cdyone}[gg^\top]$. Then, (a) $\bA_{1}(\bP_{[r]} f)$ is equal to the rank-$r$ SVD  approximation of $\bA_1(f)$, and similarly for $\bB_1(\bP_{[r]} g)$ and $\bB_1(g)$; (b) $\bM_{1 \otimes 1}(\bP_{[r]} f,\bP_{[r]} g)$ is equal to the rank-$r$ SVD approximation of $\bM_{1 \otimes 1}(f,g)$; and (c) For any $i \ge 1$, $\sigma_i(\bM_{1 \otimes 1}(f,g)) = \sigma_i(\bA_1(f))^2 = \sigma_i(\bB_1(g))^2$.
	\end{lem}
	This lemma is a partial statement of a more complete result,  \Cref{lem:matrix_discrete_correspondence}, given in the appendix.
	Importantly, the appropriate SVD approximation for balanced embeddings can be computed by projecting onto the top eigenvectors  of the covariance matrix of $f$ (or equivalently, of $g$). Via limiting arguments in \Cref{sec:proof_from_matrix_to_feature}, this characterization can be extended to  the case where spaces $\xspac,\yspace$ are continuous, and where the  covariances can be computed from samples.  One can also construct a balanced embedding from a non-balanced one. This is most succinctly stated as finite-dimensional \emph{full-rank} embeddings; a more extensive statement and its proof are given in \Cref{sec:balance_other_rank}.

	\begin{lem}\label{lem:balancing_simple}  For full-rank $\R^r$-embeddings $(\fhat,\ghat)$, there exists a unique $\bT \in \pd{r}$ for which $(\ftil,\gtil) = (\bT^{-1} \fhat,\bT \ghat)$ is balanced; moreover, $\sigma_r(\Exp_{\cdxone}[\ftil\ftil^\top]) = \sigma_r(\Exp_{\cdyone}[\gtil\gtil^\top]) = \sigma_r(\fhat,\ghat)$.

	\end{lem}



\subsection{Error decomposition}

	We now specify our error decomposition result.  First, we describe embeddings $(f,g)$ into $\hilspace$ which are consistent with the learned embedding $(\fhat,\ghat)$,  but are balanced,  and are aligned with the top-$k$ eigenspace of $\Sigst$. This allows us to reason about the differences between $f - \fst$ and $g - \gst$.

	\begin{defn}[Aligned Proxies]\label{defn:valid_proxies_cont} We say $\iota_r:\R^r \to \hilspace$ is  an \emph{isometric inclusion} if it preserves inner products, i.e., $\langle v,w \rangle = \langle \iota_r(v),\iota_r(w)\rangle_{\hilspace}$. Fix a dimension $r \in \N$, and some  $k \in \N$, and let $\fhat:\xspac \to \R^r$ and $\ghat:\yspace \to \R^r$ be full-rank. We say $(f,g)$ are \emph{aligned $k$-proxies} for $(\fhat,\ghat)$ if: (a) $f = (\iota_r \circ \bT^{-1})\fhat$, $g = (\iota_r \circ \bT) \ghat$, where $\iota_r:\R^r \to \hilspace$ is an isometric inclusion, and $\bT$ is the balancing operator of \Cref{lem:balancing_simple}, and (b) for\footnote{In case of non-uniqueness, any choice of projection works.}  $\projopstk$ being the projection onto the top $k$-eigenvectors of $\Sigst$, we have 
	\begin{align}\label{eq:alignment_main}
	\range(\projopstk) \subseteq  \range(\Exp_{\cdxone}[ff^\top]).
	\end{align}
	\end{defn}	
	


\newcommand{\ErrTermHil}{\bmsf{\Delta}}
	\newcommand{\errtrain}{\ErrTermHil_{\mathrm{train}}}
	\newcommand{\errapx}{\ErrTermHil_{\mathrm{apx}}}

	\begin{defn}[Key Error Terms]\label{defn:key_err_terms_simple} Given aligned $k$-proxies $(f,g)$ of $(\fhat,\ghat)$, we define
	\begin{align*}
	\ErrTermHil_0(f,g,k) &:= \max\left\{\Exp_{\cD_{1\otimes 1}}\hilprodd{\fstk, \gstk - g}^2, \,\Exp_{\cD_{1\otimes 1}}\hilprodd{\fstk - f , \gstk}^2\right\} \tag{weighted error}\\
	\ErrTermHil_1(f,g,k) &:= \max\left\{\Exp_{\cdx{1}}\hilnormm{\fstk - f}^2,\,\Exp_{\cdy{1}}\hilnormm{\gstk - g}^2\right\} \tag{unweighted error}\\
	\errtrain &:= \Risk(f,g;\Dtrain).  \tag{training error}
	\end{align*}
	\end{defn}

	
	\begin{prop}[Main Error Decomposition]\label{prop:final_error_decomp_simple} Suppose \Cref{asm:bilinear,asm:cov,asm:density,asm:bal}  hold. Fix $r \ge k>0$, let $(f,g)$ be aligned $k$-proxies for full-rank $\R^r$-embeddings $(\fhat,\ghat)$. Define the parameter $\sigma^2 := \min\{ \sigma_r(\fhat,\ghat)^2, \,\tailsf_2(k) + \ErrTermHil_0(f,g,k) + \errtrain\}$. Then, 
	\begin{align*}
	\Risk(\fhat,\ghat;\Dtest) = \Risk(f,g;\Dtest) \lesssimst   (\ErrTermHil_1(f,g,k))^2  + \frac{1}{\sigma^2}(\tailsf_2(k) + \ErrTermHil_0(f,g,k) + \errtrain)^2. 
	\end{align*}
	\end{prop}
	
	The unweighted error, $\ErrTermHil_1(f,g,k)$, measures how close  the aligned proxies $(f,g)$ track the best rank-$k$ approximation $(\fstk,\gstk)$. The weighted error, $\ErrTermHil_0(f,g,k)$, does the same, but only along the directions of $\fstk$ and $\gstk$ which have spectral decay.  Thus, one can expect the weighted errors to be considerably smaller. This is important, because we pay for $ \frac{1}{\sigma^2}(\tailsf_2(k) + \ErrTermHil_0(f,g,k) + \errtrain)^2$, so we need to ensure that  $\ErrTermHil_0(f,g,k)^2 \ll \sigma^2$ in order to achieve consistent recovery. \Cref{prop:final_error_decomp_simple} is proved, along with a more general statement, in \Cref{sec:proof_main_results_error_decomp}. The idea behind the derivation follows along the lines of the heuristic derivation in \Cref{rem:heuristic_derivation}; the main challenge are verifying the various heuristically claimed steps, which we show incur additional error terms depending on the objects introduced in \Cref{defn:key_err_terms_simple}.

\subsection{From error terms to factor recovery, and concluding the proof of \Cref{thm:generic_risk_bound} }\label{label:err_term_factor_recovery}

We now aim for upper bounding  $\ErrTermHil_i(f,g,k),i \in \{0,1\}$ in terms of the parameter $\epsone$ in \Cref{defn:acc_embed}. In this section, we expose how to obtain the bound for distributions with finite support. This result is equivalent to a guarantee for \emph{factor-recovery} in matrix completion.  
In the sequel, we adopt the finite-support setting, so that $\hilspace = \R^d$. Define $\bstA := \bA_1(\fst), \bstB := \bB_1(\gst)$ so that $\bstM :=  \bM_{1\otimes 1}(\fst,\gst) = \bstA (\bstB)^\top$, and similarly set $\bhatA = \bA_1(f), \bhatB = \bB_1(g),\bhatM = \bM_{1\otimes 1}(f,g) = \bhatA\bhatB^\top$. We further let $\bstA_{[k]},\bstB_{[k]}$ denote the rank-$k$ approximation of $\bstA,\bstB$, defined formally in \Cref{eq:rank_k_svd}. Lastly,  for an orthogonal matrix  $\bR \in \bbO(d)$, we define the following error terms
    \begin{align}
    \ErrTerm_0(\bR,k) &= \|(\bstAk - \bhatA \bR)(\bstBk)^\top\|_{\fro}^2 \vee \|\bstAk(\bstBk - \bhatB \bR)^\top\|_{\fro}^2 \label{eq:ErrTerm_zero_body}\\
    \ErrTerm_1(\bR,k) &= \|\bstAk - \bhatA\bR\|_{\fro}^2 \vee \|\bstBk - \bhatB\bR \|_{\fro}^2.  \label{eq:ErrTerm_one_body}
    \end{align}
    One can then check (see \Cref{sec:simple_func}) that for the matrices defined above and $i \in \{0,1\}$, we have $\ErrTerm_i(\bR,k) = \ErrTermHil_i(\bR^\top f,\bR^\top g,k)$. Here, the matrix $\bR$ allows us to rotate embeddings $(f,g)$ to minimize the factor error. 
 In sum, we have shown that the error terms in \Cref{prop:final_error_decomp_simple} are corresponding to the recovery of factors in matrix completion. 

    The error terms in \Cref{eq:ErrTerm_zero_body,eq:ErrTerm_one_body} are related to the Procrustes Problem studied by \cite{tu2016low}. In the case where $\bstM$ is exactly low rank and exhibits a sharp spectral cutoff,  \citet[Lemma 5.14]{tu2016low} provides a bound $\inf_{\bR \in \bbO(d)}\ErrTerm_1(\bR,k)$. What we need is something considerably stronger:
    \begin{itemize}
    	\item We require a bound on the weighted error term, $\ErrTerm_0(\bR,k)$, which is considerably sharper than the bound given in \citet[Lemma 5.14]{tu2016low}.
    	\item We require error bounds that hold even if $\bstM$ is not exactly low rank. 
    \end{itemize}

   We now establish an error bound on these factory-recovery terms, which is the main technical effort of this paper. 

    \begin{theorem}\label{thm:main_matrix_body} Let $\bstA,\bhatA \in \R^{n \times d}$, $\bstB,\bhatB \in \R^{m \times d}$,  and suppose $(\bstA,\bstB)$ and $(\bhatA,\bhatB)$ are balanced factorizations of $\bstM = \bstA(\bstB)^\top$, and $\bhatM = \bhatA\bhatB^\top$. Let $r = \rank(\bhatM)$. Fix $\epsilon > 0$ and $\kpick \in \N$ such that $\kpick>1$,   $\epsilon\ge \|\bhatM - \bstM\|_{\fro}$, and $\epsilon \le \frac{\|\bstM\|_{\op}}{40 \kpick}$.  Also, for $q \ge 1$, let $\tail_q(\bM;k) := \sum_{i>k}\sigma_i(\bM)^q$. Then, there exists an index $k \in [\min\{r,\kpick-1\}]$ and an  orthogonal matrix $\bR \in \bbO(p)$
    such that  
    \begin{align*}
    \ErrTerm_0(\bR,k) +  \tail_2(\bstM;k) &\lesssim s^3\epsilon^2 + \kpick ( \sigma_{\kpick}(\bstM))^2 +  \tail_2(\bstM;\kpick)\\
    \ErrTerm_1(\bR,k) &\lesssim (\sqrt{r}+\kpick^2)\epsilon + \kpick \sigma_{\kpick}(\bstM)  +  \tail_1(\bstM;\kpick), 
    \end{align*}
    and moreover, $\range((\bhatA\bR)^\top\bhatA\bR) \supset \range ((\bstA_{[k]})^\top (\bstA_{[k]}))$. 
    \end{theorem}
    The above theorem is a specialization of a more extensive guarantee, \Cref{thm:main_matrix}, stated and proved in \Cref{sec:matrix_main_results_details}. There are a number of important points to make. First, the theorem requires specifying a target rank $s$, but the guarantee applies to a smaller rank $k$; this is explained in the proof sketch below. Still, care is ensured to guarantee that the upper bounds on  $\ErrTerm_i(\bR,k)$ depend only on the tail-decay at $s$, but not $k$. Second, we observe that when instantiated with $\bstM = \bM_{1\otimes 1}(\fst,\gst)$ as above, $\tail_q(\bstM;\kpick) = \tailsf_q(\kpick)$, i.e., it is the tail of the spectrum of $\Sigst$. Third, the guarantee applies to an orthogonal transformation $\bR$, and the guarantee  of $\range((\bhatA\bR)^\top\bhatA\bR) \supset \range ((\bstA_{[k]})^\top (\bstA_{[k]}))$ ensures that, for $\bhatA = \bA_1(f)$ as instantiated  above, the transformed embeddings $(\bR^\top f,\bR^\top g)$ are aligned-$k$ proxies. Lastly, observe that the weighted error is asymptotically \emph{quadratically smaller} in $\epsilon$ than the unweighted one; this is also explained in the proof sketch below.

To conclude the proof of \Cref{thm:generic_risk_bound}, we first extend, via limiting arguments, to the setting of bilinear embeddings with arbitrary distributions; this result, \Cref{thm:one_block}, and its proof, are given in \Cref{sec:proof_from_matrix_to_feature}. This provides an upper bound on $\ErrTermHil_0(f,g,k),\ErrTermHil_1(f,g,k)$ in terms of the term $\epsone^2$ in \Cref{defn:acc_embed}. Finally, we conclude the proof of \Cref{thm:generic_risk_bound} in \Cref{proof:generic_risk_bound} by plugging \Cref{thm:one_block} into \Cref{prop:final_error_decomp_simple} and  substituting $\errtrain = \Risk(f,g;\Dtrain) \gets \epstrain$ as in \Cref{defn:acc_embed}.
\subsection{Proof sketch of \Cref{thm:main_matrix_body}}\label{sec:main_mat_body} 

The proof of \Cref{thm:main_matrix_body} is our most technically innovative result; we sketch some of these techniques here, deferring the formal proof to \Cref{sec:matrix_main_results_details}. Though previous bounds for matrix  recovery exist (notably  \citet[Lemma 5.14]{tu2016low} as restated in \Cref{lem:proc}), these results assume matrices to either have \emph{exactly low-rankness}, or have \emph{sufficiently large spectral gap}. Addressing more gradual spectral decay requires a far more subtle treatment.  


\paragraph{Technical novelty \#1: Relative singular-gap SVD perturbation.} The first technical ingredient is the perturbation for the rank-$k$ SVD approximation, \Cref{thm:svd_pert}, highlighted in the introduction, which replaces a dependence on \emph{absolute} singular-gap with one on \emph{relative} singular-gap. 

\paragraph{Technique novelty \#2: ``Well-tempered''  partition.} Motivated by the advantages of considering a relative (as opposed to absolute) singular gap, we construct a certain partition of the spectrum of $\bstM$, which we call a ``well-tempered partition'' (\Cref{defn:well_tempered_partition}). This partition splits the indices of the top-$s$ singular values of $\bstM$ into intervals where: (a) the relative-singular gap separation between the \emph{intervals} is sufficiently large, and (b)  all singular values are of similar magnitude.


Specifically, we denote the subsets in this  partition as {$\cK_i = \{k_{i}+1,k_{i}+2,\dots,k_{i+1}\}$};  we call $k_i$ the \emph{pivot} and each $\cK_i$ a \emph{block}. We show that  the partition can be constructed so as to ensure that the \emph{relative} spectral gap  $
\updelta_{k_i}(\bstM)$, where for any $k$, is at least $\Omega(1/s)$. Here again,  $s$ is the target rank in \Cref{thm:main_matrix_body}. As noted above, the \emph{absolute} singular gaps can be arbitrarily smaller.

Given this partition, we decompose the factor matrices $\bstA_{[k]},\bstB_{[k]},\bhatA,\bhatB$ into a sum over block-zero-masked matrices $\bstA_{\cK_i},\bstB_{\cK_i},\bhatA_{\cK_i},\bhatB_{\cK_i}$, with each block corresponding to one element  $\cK_i$ of the well-tempered partition.  We let $\bstM_{\cK_i} = \bstA_{\cK_i}(\bstB_{\cK_i})^\top$, with $\bhatM_{\cK_i}$ being defined similarly.  We use the triangle inequality to relate $\|\bhatM_{\cK_i} - \bstM_{\cK_i}\|_{\fro}$ to $\max_{j \in \{k_i,k_{i+1}\}}\{\|\bhatM_{[j]} - \bstM_{[j]}\|_{\fro}\}$, and bound the latter two using our SVD perturbation result (\Cref{thm:svd_pert}). This is to our advantage, since our choice of well-tempered partition guarantees that  $\updelta_{j}(\bstM) = \Omega(1/s)$ for $j \in \{k_i,k_{i+1}\}$, and implies via \Cref{thm:svd_pert} that $\|\bhatM_{\cK_i} - \bstM_{\cK_i}\|_{\fro}^2 \lesssim s^2 \epsilon^2$.
We then apply an  existing matrix factorization lemma, \citet[Lemma 5.14]{tu2016low} to these blocks. The rotation matrix $\bR$ aligns the block-masked factor matrices to minimize factor error. Though \Cref{thm:svd_pert} depends on \emph{relative} gaps, the factor recovery error in block $i$ in \citet[Lemma 5.14]{tu2016low} depends on \emph{absolute} ones, scaling with 
\begin{align}
\frac{\|\bhatM_{\cK_i} - \bstM_{\cK_i}\|_{\fro}^2}{\sigma_{k_i}(\bstM)} \lesssim \frac{s^2 \epsilon^2}{\sigma_{k_i}(\bstM)}. \label{eq:tu_term}
\end{align} 
 For the unweighted error, we select the rank cutoff $k$ to ensure $\sigma_{k}(\bstM)$ is  large enough; trading-off the tails $\tail_q(k;\bstM)$ with $\sigma_{k}(\bstM)$ leads to the unweighted error $\ErrTerm_1(\bR,k)$ scaling with $\epsilon$, rather than $\epsilon^2$. For the weighted error $\ErrTerm_0(\bR,k)$, we can weight the factor recovery errors in the $i$-th block by $\sigma_{k_{i-1}+1} = \max\{\sigma_j(\bstM): j \in \cK_i\}$. We then use the second property of the well-tempered partition: all singular values indexed in $\cK_i$ are of roughly constant magnitude; thus, weighting by $\sigma_{k_{i-1}+1}(\bstM)$  cancels out the denominator of $\sigma_{k_i}(\bstM)$ in \Cref{eq:tu_term}, yielding a sharper estimate. 

\section{Conclusion}
In sum, this paper explored the connection between combinatorial distribution shift and matrix completion, developing fundamental and novel technical tools along the way. Whether our results can be extended to more general coverage assumptions than those depicted in \Cref{fig:illustration} remains an exciting direction for future research. 

\section*{Acknowledgements}
The authors would like to thank the anonymous reviewers for the 2023 Conference on Learning Theory for their helpful comments. K.Z.  acknowledges support from Simons-Berkeley Research Fellowship. M.S. acknowledges support from an Amazon.com Services LLC; PO\# 2D-06310236 grant. We also acknowledge Devavrat Shah, Aviv Netanyahu, and Jason Lee for helpful discussion and feedback.

\bibliographystyle{plainnat}
\bibliography{refs}
\newpage

\appendix

\centerline{{\fontsize{14}{14}\selectfont \textbf{Supplementary Materials for }}}

\vspace{10pt}
\centerline{\fontsize{13}{13}\selectfont \textbf{
	``Tackling Combinatorial Distribution Shift: A Matrix Completion Perspective''}}


\vspace{8pt}

\tableofcontents
\newpage 

\part{Organization, Related Work, Further Discussion, and the SVD Perturbation bound}\label{sec:part_1}
\section{Organization of the Appendix}\label{sec:app_org}

We detail the organization of our appendix as follows. 
\Cref{sec:part_1} provides the overall organization of the appendix in \Cref{sec:app_org}, a detailed related work in \Cref{sec:extended_related_work}, especially on matrix completion\iftoggle{arxiv}{.}{, and an elaboration on the connection between bilinear combinatorial extrapolation and matrix completion in \Cref{sec:connection_matrix_completion}.} In \Cref{sec:SVD_pert}, we prove \Cref{thm:svd_pert}, our main SVD perturbation bound. 
 
\Cref{part:supp_sing_double} provides supplementary material regarding our guarantees for the single-stage and double-stage ERM procedures. \Cref{sec:ERM_addenda} provides the high-level proof of our guarantee for single-stage ERM (\Cref{thm:st_main}), and a more detailed guarantee for double-stage ERM (\Cref{thm:main_dt_risk_bound_thing}), deriving \Cref{thm:main_dt_simple} from that more granular result. \Cref{sec:analysis_training_alg} provides the proof of \Cref{thm:main_dt_risk_bound_thing}, which in turn contains as the single-stage ERM guarantee used by \Cref{thm:st_main}. These proofs in turn rely on some general (though quite standard) learning-theoretic bounds, which are supplied in \Cref{app:learning_theory}. Finally, \Cref{sec:rate_instant} performs the computations which instantiates out single- and double-stage ERM guarantees for the spectral decay regimes prescribed by \Cref{asm:eigendecay}.

\Cref{part:supp_meta_theorem} contains the supplementary results needed for the  proof of the meta-theorem (\Cref{thm:generic_risk_bound}), as well as general-purpose linear algebraic results. \Cref{sec:matrix_main_results_details} contains the proof of our main technical endeavor - a bound on the error of factor recovery in low-rank matrix approximation. \Cref{sec:proof_from_matrix_to_feature} extends the matrix factorization guarantee to its natural generalization to bilinear embeddings, applying suitable limiting arguments to accomodate distributions with infinite/uncountable support. Most supporting linear algebraic results/proofs  are deferred to \Cref{sec:lin_alg_proof}; notable, these include the proof of our relative singular-value gap perturbation bound (\Cref{thm:svd_general_thing}). Results pertaining to balancing (of both matrices and embeddings) are given in \Cref{sec:balancing_operator}. Finally, \Cref{sec:proof_main_results_error_decomp} provides the proof of the error decomposition (\Cref{prop:final_error_decomp_simple}), as well as the derivation of \Cref{thm:generic_risk_bound} from  \Cref{prop:final_error_decomp_simple} and \Cref{thm:one_block}.

\newcommand{\unseen}{\textbf{?}}

\section{Detailed Related Work} \label{sec:extended_related_work}

This subsection  provides a more detailed summary of related work, to the best of our knowledge.  

\subsection{Matrix completion}
To facilitate  comparison, we consider a ground-truth matrix $\bstM \in \R^{M \times N}$ as the matrix to be completed. $\bM \in \R^{M \times N}$ is a noisy realization of $\bstM$ with $\Exp[\bM] = \bstM$, and we assume that we are given observed matrix $\btilM \in (\R \cup \{\unseen\})^{M \times N}$, where `$\unseen$' denotes an unseen entry, such that $\btilM_{[ab]}= \bM_{[ab]}$ unless $\btilM_{[ab]} =\,\unseen$. We let $\bD \in \{0,1\}^{M \times N}$ denote the masking matrix of $\btilM$: $\bD_{[ab]} = \I\{\btilM_{[ab]} \ne \,\unseen\}$.   

\paragraph{Missing-completely-at-random (MAR) matrix completion.} The MAR setting assumes  that the entries of $\bD$ are i.i.d. Bernoulli random variables with positive probability $p > 0$ and independent of $\bM$; see e.g., \cite{candes2010power,recht2011simpler,hastie2015matrix,mazumder2010spectral,koltchinskii2011nuclear}. More recent works study settings where $\bstM_{[ab]}$ is generated by the bivariate function $\hst(x_a,y_b) = \langle \fst(x_a),\gst(y_b)\rangle$ of features $x_a,y_b$; in  \citep{xu2018rates}, this encodes graphon structure, whereas in \cite{song2016blind,li2019nearest}, $\hst(x,y)$ is a globally Lipschitz function, which admits learning via matrix completion by considering linearizing expansions. \cite{yu2021nonparametric} considers an extension to the ``one-sided'' covariate setting that is more  challenging, where only the first argument of $\hst$ is observed. A ``one-bit'' sensing model has also been studied in  \cite{davenport20141}, and refined under a latent variable model for features $x_a,y_b$ \citep{borgs2017thy}. All aforementioned works consider the MAR setting. 

\paragraph{Missing at random.} In the missing-at-random setting, it is assumed that there exists a set of observed covariates $\cO$ such that $\bM \perp \bD \mid \cO$, 
and that $\bD_{[ab]} \mid \cO$ are independent Bernoulli random-variables with possibly different probabilities $p_{ab}$ uniformly bounded below. See e.g.,  \cite{schnabel2016recommendations,wang2018modeling,liang2016modeling}.

\paragraph{Missing-not-at-random (MNAR) matrix completion.} 
Many works consider generative models, relating missingness of entries to either ground-truth or realized values of the matrix via logistic expressions \citep{sportisse2020imputation,yang2021tenips}. Guarantees obtained from this strategy typically depend on a lower bound on the minimal probability that an entry is revealed \citep{ma2019missing}, dependence on which is also incurred in an alternative approach due to \cite{bhattacharya2022matrix}. Note that in our setting, we allow the entries of $\bstM_{2,2}$ to be \emph{entirely} omitted from $\btilM_{2,2}$, so these guarantees are vacuous here. Another approach due to \cite{foucart2020weighted} studies reconstruction from MNAR data under weighting matrices that are suitably calibrated to the pattern of missing entries. Again, in our setting, these results become vacuous. 

Two more recent works establish recovery for entries that are indeed missing with  probability one. \cite{shah2020sample} considers almost precisely our  setting, where, motivated by reinforcement learning, one attempt to recover $\bstM_{2,2}$ by observing the other blocks $\bstM_{1,2},\bstM_{2,1},\bstM_{2,2}$. However, their results require that either (a) $\bstM$ is an exactly low rank, or (b) that $\bstM$ is an approximately low rank, but that the error between $\bstM$ and its rank-$r$ SVD is very small entry-wise. This  precludes the much more gradual polynomial decay allowed by our main results. A second work, \cite{agarwal2021causal}, considers far more general patterns of missing entries than we do in this work. However, this comes at the cost of requiring even stronger assumptions on the spectrum \citep[Assumption 6]{agarwal2021causal}, which again precludes approximately low-rank matrices with spectral decay.

\subsection{Learning under distribution shift}

In contrast to the well-established  statistical learning theory \citep{bartlett2002rademacher,vapnik2006estimation}, our theoretical understanding of distribution shift is considerably  sparser. Notably, recent work has given precise characterizations of the effects of covariate shift for certain specific function classes, notably kernels \citep{ma2022optimally} and H\"older smooth classes \citep{pathak2022new}; still, these works focus on the regimes where the test-distribution has bounded density with respect to the train distribution; in our bilinear combinatorial extrapolation setting, however, this is no longer the case.
Resilience to distribution shift has received considerable empirical attention lately, see  \cite{miller2021accuracy,taori2020measuring,santurkar2020breeds,koh2021wilds,zhou2022domain} for example. 


\section{Relative-Gap Perturbations of SVD Approximation}\label{sec:SVD_pert}

\svdpert*

\begin{proof}[Proof of \Cref{thm:svd_pert}]   We begin by expanding the Frobenius error:
\begin{align*}
\|\bhatM - \bstM\|_{\fro}^2 &= \|\bhatM_{[k]} - \bstM_{[k]} + (\bhatM_{> k} - \bstM_{>k})\|_{\fro}^2 \\
&= \|\bhatM_{[k]} - \bstM_{[k]}\|_{\fro}^2 + \|\bhatM_{>k} - \bstM_{>k}\|_{\fro}^2 + 2 \langle \bhatM_{[k]} - \bstM_{[k]},  \bhatM_{>k} - \bstM_{>k} \rangle. 
\end{align*}
Hence,
\begin{align*} 
\|\bhatM_{[k]} - \bstM_{[k]}\|_{\fro}^2 &\le \|\bhatM - \bstM\|_{\fro}^2 + 2| \langle \bhatM_{[k]} - \bstM_{[k]},  \bhatM_{>k} - \bstM_{>k} \rangle|\\
&\overset{(i)}{=} \|\bhatM - \bstM\|_{\fro}^2 + 2| \langle \bhatM_{[k]} ,  \bstM_{>k} \rangle - \langle \bstM_{[k]}, \bhatM_{>k} \rangle|\\ 
&\overset{(ii)}{\le} \|\bhatM - \bstM\|_{\fro}^2 + 2| \langle \bhatM_{[k]},  \bstM_{>k} \rangle| + 2|\langle \bstM_{[k]}, \bhatM_{>k} \rangle|, \numberthis \label{eq:penultimate_svd}
\end{align*}
where above $(i)$ uses that the range of the rank-$k$ SVD of a matrix and its complement are orthogonal, and $(ii)$ is just the triangle inequality. The following claim bounds the cross terms:

\begin{claim}\label{claim:cross_term}Suppose $\sigma_{k}(\bhatM) > \sigma_{k+1}(\bstM)$. Then, 
\begin{align*}
| \langle \bhatM_{[k]},  \bstM_{>k} \rangle| &\le 4\|\bhatM - \bstM\|_{\fro}^2\cdot \left( \left(1 - \frac{\sigma_{k+1}(\bstM)}{\sigma_{k}(\bhatM)}\right)^{-2} + 4\right). 
\end{align*}
Similarly, if $\sigma_{k}(\bstM) > \sigma_{k+1}(\bhatM)$. Then, 
\begin{align*}
| \langle \bstM_{[k]},  \bhatM_{>k}\rangle| &\le 4\|\bhatM - \bstM\|_{\fro}^2\cdot \left( \left(1 - \frac{\sigma_{k+1}(\bhatM)}{\sigma_{k}(\bstM)}\right)^{-2} + 4\right).
\end{align*}
\end{claim}
The proof of \Cref{claim:cross_term} uses a careful peeling argument, and is deferred to the end. The key idea is to parition the singular values of $\bstM$ into blocks whose singular values are all within a constant factor, and into one final block such corresponding to singular values $j > k$ of $\bstM$. We then apply a standard variant of Wedin's theorem (\Cref{lem:gap_free_Wedin}) to each block. The form of the matrix inner product allows us to  weight the contribution of each block by its associated singular value. The upshot is that this leads to gap-free bounds for all but the last-block (as all singular values in these blocks are within a constant of eachother), and a similar argument leaves us only with dependence on the relative singular gap for the final block.  

We now specialize the above upper bound when $\|\bhatM - \bstM\|_{\op}$ is sufficiently small.
\begin{claim} Suppose $\|\bhatM - \bstM\|_{\op} \le \eta \delstk \sigma_{k}(\bstM)$. Then, 
\begin{align*}
| \langle \bhatM_{[k]},  \bstM_{>k} \rangle| \vee  | \langle \bstM_{[k]},  \bhatM_{>k} \rangle| &\le 4\|\bhatM - \bstM\|_{\fro}^2\cdot \left( (\delstk)^{-2}(1 - \eta)^{-2} + 4\right). 
\end{align*}
\end{claim}
\begin{proof} If $\|\bhatM - \bstM\|_{\op} \le \eta \delstk \sigma_{k}(\bstM)$, 
\begin{align*}
1 - \frac{\sigma_{k+1}(\bstM)}{\sigma_{k}(\bhatM)} &\ge 1 - (1-\eta\delstk)^{-1}\frac{\sigma_{k+1}(\bstM)}{\sigma_{k}(\bstM)}\\
&= 1 - (1-\eta\delstk)^{-1}(1-\delstk)\\
&= \frac{1-\eta\delstk - (1-\delstk)}{1-\eta\delstk} = \frac{\delstk(1-\eta)}{1-\eta\delstk} \ge \delstk(1-\eta),
\end{align*}
and 
\begin{align*}
1 - \frac{\sigma_{k+1}(\bhatM)}{\sigma_{k}(\bstM)} &\ge 1 - \frac{\eta \delstk \sigma_k(\bstM) + \sigma_{k+1}(\bstM)}{\sigma_k(\bstM)}\\
&= 1 - \eta \delstk - \frac{ \sigma_{k+1}(\bstM)}{\sigma_k(\bstM)} = (1-\eta \delstk) - (1-\delstk) = \delstk(1-\eta).
\end{align*}
Hence, in both cases, \Cref{claim:cross_term} yields.
\begin{align*}
| \langle \bhatM_{[k]},  \bstM_{>k} \rangle| \vee  | \langle \bhatM_{[k]},  \bstM_{>k} \rangle| &\le 4\|\bhatM - \bstM\|_{\fro}^2\cdot \left( (\delstk)^{-2}(1 - \eta)^{-2} + 4\right), 
\end{align*}
which completes the proof. 
\end{proof}
To conclude, we recall \Cref{eq:penultimate_svd} and apply the previous claim
\begin{align*}
\|\bhatM_{[k]} - \bstM_{[k]}\|_{\fro}^2 &\le \|\bhatM - \bstM\|_{\fro}^2 + 2| \langle \bhatM_{[k]},  \bstM_{>k} \rangle| + 2|\langle \bstM_{[k]}, \bhatM_{>k} \rangle|\\
&\le \|\bhatM - \bstM\|_{\fro}^2 + 4 (| \langle \bhatM_{[k]},  \bstM_{>k} \rangle| \vee|\langle \bstM_{[k]}, \bhatM_{>k} \rangle|)\\
&\le \|\bhatM - \bstM\|_{\fro}^2 +  \|\bhatM - \bstM\|_{\fro}^2\cdot \left( 16(\delstk)^{-2}(1 - \eta)^{-2} + 64\right)\\
&=   \|\bhatM - \bstM\|_{\fro}^2\cdot \left( 16(\delstk)^{-2}(1 - \eta)^{-2} + 65\right)\\
&\le  81 \|\bhatM - \bstM\|_{\fro}^2\cdot \left( (\delstk)^{-2}(1 - \eta)^{-2} \right). 
\end{align*}
The bound follows.
\end{proof}

\begin{proof}[Proof of \Cref{claim:cross_term}] We prove the first statement of the claim; the second is analogous. Consider a sequence of indices $k_0 > k_1 > \dots k_{\ell} = 0$ as follows (For convenience, $k_i$ are decreasing, unlike the pivots $k_i$ in the definition of the well-tempered parition \Cref{defn:well_tempered_partition}).
\begin{itemize}
    \item $k_0 = k$.
    \item Given $k_i$, set $k_{i+1} = \max \{j \ge 1: \sigma_{j}(\bhatM) \ge 2\sigma_{k_i}(\bhatM)\}$. If no such $j$ exists, set $i+1 = \ell$ and $k_{\ell} = 0$. 
\end{itemize}
We also define the index sets and corresponding SVD of $\bhatM$ as 
\begin{align*}
\cI_{i} := \{j: k_i \ge j > k_{i+1}\}, \quad \bhatM_{\cI_i} := \bhatU_{\cI_i}\bhatSigma_{\cI_i} \bhatV_{\cI_i}^\top,
\end{align*}
where $\bhatU_{\cI_i} \in \R^{n \times |\cI_i|},\bhatSigma_{\cI_i}\in \R^{|\cI_i|\times |\cI_i|}, \bhatV_{\cI_i}\in \R^{m \times |\cI_i|}$ denote a compact SVD of $\bhatM_{\cI_i}$ corresponding to singular values/vectors with indices in $\cI_i$ (i.e. to the rows of $\bhatU$ corresponding to entries $j \in \cI_i$, and similarly for $\bhatSigma_{\cI_i} \bhatV_{\cI_i}^\top$). We then have 
\begin{align*}
\sum_{i=0}^{\ell-1}\bhatM_{\cI_i} = \bhatM_{[k]}. 
\end{align*}
Using this decomposition, we write
\begin{align*}
| \langle \bhatM_{[k]},  \bstM_{>k} \rangle| &= \Big|\langle \sum_{i=0}^{\ell-1}\bhatM_{\cI_i},\bstM_{>k}\rangle\Big|\\
&\le \sum_{i=0}^{\ell-1}|\langle \bhatM_{\cI_i},\bstM_{>k}\rangle|\\
&= \sum_{i=0}^{\ell-1}|\langle \bhatU_{\cI_i}\bhatSigma_{\cI_i} \bhatV_{\cI_i}^\top,\bstU_{>k}\bstSigma_{>k}(\bstV_{>k})^\top\rangle|\\
&= \sum_{i=0}^{\ell-1}|\trace( \bhatV_{\cI_i} \bhatSigma_{\cI_i}\bhatU_{\cI_i}^\top\bstU_{>k}\bstSigma_{>k}(\bstV_{>k})^\top)|\\
&= \sum_{i=0}^{\ell-1}|\trace( \bhatSigma_{\cI_i}\bhatU_{\cI_i}^\top\bstU_{>k}\bstSigma_{>k}(\bstV_{>k})^\top \bhatV_{\cI_i} )|\\
&\le \sum_{i=0}^{\ell-1}\|\bhatSigma_{\cI_i} \bhatU_{\cI_i}^\top\bstU_{>k}\|_{\fro}\|\bstSigma_{>k}(\bstV_{>k})^\top \bhatV_{\cI_i} \|_{\fro}\\
&\le \sum_{i=0}^{\ell-1}\|\bhatSigma_{\cI_i}\|_{\op}\|\bstSigma_{>k}\|_{\op} \|\bhatU_{\cI_i}^\top\bstU_{>k}\|_{\fro}\|(\bstV_{>k})^\top \bhatV_{\cI_i} \|_{\fro}. \numberthis\label{eq:svd_perp_last_line}
\end{align*}
Since $\cI_i \subseteq [k_i]$, we can bound
\begin{align*}
\|\bhatU_{\cI_i}^\top\bstU_{>k}\|_{\fro}\|(\bstV_{>k})^\top \bhatV_{\cI_i} \|_{\fro} &\le \|\bhatU_{[k_i]}^\top\bstU_{>k}\|_{\fro}\|(\bstV_{>k})^\top \bhatV_{[k_i]} \|_{\fro}\\
&\le \frac{1}{2}\left(\|\bhatU_{[k_i]}^\top\bstU_{>k}\|_{\fro}^2 + \|(\bstV_{>k})^\top \bhatV_{[k_i]} \|_{\fro}^2\right).  
\end{align*}
In particular, since $k_i \le k$, we see that as long as $\sigma_{k}(\bhatM) > \sigma_{k+1}(\bstM)$, then by a standard variant of Wedin's theorem, \Cref{lem:gap_free_Wedin},
\begin{align*}
\|\bhatU_{\cI_i}^\top\bstU_{>k}\|_{\fro}\|(\bstV_{>k})^\top \bhatV_{\cI_i} \|_{\fro} \le \frac{2}{(\sigma_{k_i}(\bhatM) - \sigma_{k+1}(\bstM))^2}\|\bhatM - \bstM\|_{\fro}^2. 
\end{align*}
We furthe observe that $\|\bstSigma_{>k}\|_{\op} = \sigma_{k+1}(\bstM)$, and $\|\bhatSigma_{\cI_i}\|_{\op} = \sigma_{k_{i+1}-1}(\bhatM) \le 2\sigma_{k_i}(\bhatM)$. Thus, picking up from \Cref{eq:svd_perp_last_line}
\begin{align*}
| \langle \bhatM_{[k]},  \bstM_{>k} \rangle| 
&\le \|\bhatM - \bstM\|_{\fro}^2\cdot\sum_{i=0}^{\ell-1} \frac{4\sigma_{k+1}(\bstM)\sigma_{k_i}(\bhatM) }{(\sigma_{k_i}(\bhatM) - \sigma_{k+1}(\bstM))^2}\\
&= \|\bhatM - \bstM\|_{\fro}^2\cdot \left(\frac{4\sigma_{k+1}(\bstM)\sigma_{k}(\bhatM) }{(\sigma_{k}(\bhatM) - \sigma_{k+1}(\bstM))^2} + \sum_{i=1}^{\ell-1} \frac{4\sigma_{k+1}(\bstM)\sigma_{k_i}(\bhatM) }{(\sigma_{k_i}(\bhatM) - \sigma_{k+1}(\bstM))^2}\right)\\
&\overset{(i)}{\le} \|\bhatM - \bstM\|_{\fro}^2\cdot \left(\frac{4\sigma_{k+1}(\bstM)\sigma_{k}(\bhatM) }{(\sigma_{k}(\bhatM) - \sigma_{k+1}(\bstM))^2} + \sum_{i=1}^{\ell-1} \frac{4\sigma_{k}(\bhatM)\sigma_{k_i}(\bhatM) }{(\sigma_{k_i}(\bhatM) - \sigma_{k}(\bhatM))^2}\right), 
\end{align*}
where in $(i)$ we use that $\sigma_k(\bhatM) \ge \sigma_{k+1}(\bstM)$. Using that $\sigma_{k_i}(\bhatM) \ge 2\sigma_{k_{i-1}}(\bhatM) \ge \dots 2^{i} \sigma_{k_0}(\bhatM) = 2^i \sigma_k(\bhatM)$,  we find
\begin{align*}
\sum_{i = 1}^{\ell-1}\frac{4\sigma_{k}(\bhatM)\sigma_{k_i}(\bhatM) }{(\sigma_{k_i}(\bhatM) - \sigma_{k}(\bhatM))^2} &= \sum_{i = 1}^{\ell-1}\frac{4\sigma_{k}(\bhatM) }{(\sigma_{k_i}(\bhatM) - \sigma_{k}(\bhatM))(1 - \sigma_{k}(\bhatM)/\sigma_{k_i}(\bhatM))}\\
&\le \sum_{i = 1}^{\ell-1}\frac{4\sigma_{k}(\bhatM) }{(2^i - 1)\sigma_{k}(\bhatM)(1 - 2^{-i})} = \sum_{i = 1}^{\ell-1}\frac{4}{(2^i - 1)(1-2^{-i})}\\
&\le \sum_{i \ge 1}\frac{4}{(2^i - 1)(1-2^{-i})} \le 16. 
\end{align*}
Hence, we conclude 
\begin{align*}
| \langle \bhatM_{[k]},  \bstM_{>k} \rangle| 
&\le \|\bhatM - \bstM\|_{\fro}^2\cdot \left(\frac{4\sigma_{k}(\bhatM)^2 }{(\sigma_{k}(\bhatM) - \sigma_{k+1}(\bstM))^2} + 16\right)\\
&= 4\|\bhatM - \bstM\|_{\fro}^2\cdot \left( \left(1 - \frac{\sigma_{k+1}(\bstM)}{\sigma_{k}(\bhatM)}\right)^{-2} + 4\right), 
\end{align*}
completing the proof.

\subsection{Useful variants of Wedin's Theorem}

    \begin{lem}[``Gap-Free'' Davis Kahan, Lemma B.3, \cite{allen2016lazysvd}]\label{lem:gap_free_DK} Let $\circnorm{\cdot}$ denote any Schatten $p$-norm. Fix $\epsilon > 0$, and suppose that $\bX,\btilX$ are symmetric matrices with $\circnorm{\bX - \btilX} \le \epsilon$. Given $\mu \ge 0$ and $\tau \ge 0$, let $\bU_{0}$ be an orthonormal matrix with columns being the eigenvectors of $\bX$, whose corresponding eigenvalues  have absolutely value $\le \mu$, and $\btilU_{1}$ be an orthonormal matrix with columns being the eigenvectors of $\btilX$, whose corresponding eigenvalues  have absolutely value   $\ge \mu + \tau$. Then, $\circnorm{\bU_{0}^\top \btilU_{1}} \le \frac{\tau}{\epsilon}$.
    \end{lem}

    \begin{proof} We follow the proof of Lemma B.3, \cite{allen2016lazysvd}, originally stated in the operator norm (and for positive semidefinite matrices), to accommodate the Frobenius norm and absolute value eigenvalue magnitudes. Next, write out compact diagonalizations
    \begin{align*}
    \bX &= \bU_{0}\bSigma_{0}(\bU_0)^\top + \bU_{1}\bSigma_{1}(\bU_1)^\top\\
    \btilX &= \btilU_{0}\btilSigma_{0}(\btilU_0)^\top + \btilU_{1}\btilSigma_{1}(\btilU_1)^\top,
    \end{align*}
    where all entries of $\bSigma_{0}$ lie in $[-\mu,\mu]$, and entries of $\bSigma_{1}$ lie in $(-\infty,\mu) \cup (\mu,\infty)$, all entries of $\btilSigma_{0}$ lie in $(-(\mu+\tau),\mu+\tau)$, and entries of $\btilSigma_{1}$ are in $(-\infty,-(\mu+\tau)] \cup [\mu+\tau,\infty)$. Consider the residual $\bDelta := \bX - \btilX$, we find that 
    \begin{align*}
    \bSigma_0\bU_0^\top &= \bU_0^\top \bX = \bU_0^\top \btilX + \bU_0^\top \bDelta\\
    \text{implying }\bSigma_0\bU_0^\top\btilU_1 &= \bU_0^\top  \btilX  \btilU_1 + \bU_0^\top \bDelta \btilU_1 \\
    &= \bU_0^\top \btilU_1 \btilSigma_1 + \bU_0^\top \bDelta \btilU_1.
    \end{align*}
    Taking norms and applying the triangle inequality
    \begin{align*}
    \circnorm{\bSigma_0(\bU_0)^\top \btilU_1} \ge \circnorm{(\bU_0)^\top \btilU_1 \btilSigma_1} - \circnorm{(\bU_0)^\top \bDelta \btilU_1}. 
    \end{align*}
    Since $(\bSigma_0)^\top (\bSigma_0) \preceq \mu^2 \eye$, and $\btilSigma_1^\top \btilSigma_1 \succeq (\mu+\tau)^2 \eye$, and since $\bU_0,\btilU_1$ are orthogonal, we estimate $\circnorm{\bSigma_0(\bU_0)^\top\btilU_1} \le \mu \circnorm{\bU_0^\top\btilU_1}$, that $\circnorm{(\bU_0)^\top \btilU_1 \btilSigma_1} \ge (\mu+\tau) \circnorm{\bU_0^\top\btilU_1}$, and $\circnorm{(\bU_0)^\top \bDelta \btilU_1} \le \circnorm{\bDelta}$. Thus
    \begin{align*}
    \mu \|\bU_0^\top\btilU_1\|_{\circ} \ge (\mu+\tau)\|\bU_0^\top\btilU_1\|_{\circ}  - \|\bDelta\|_{\circ}. 
    \end{align*} 
    Rearranging concludes the proof. 
    \end{proof}
    \begin{lem}[Variant of Wedin's Theorem]\label{lem:gap_free_Wedin} Suppose that $\bM,\btilM \in \R^{m \times n}$. Given $\mu \ge 0$ and $\tau \ge 0$, let $\bU_{0},\bV_0$ be an orthonormal basis for left (resp. right) singular vectors of $\bM$ whose corresponding singular values are $\le \mu$, and let $\btilU_{1},\btilV_1$ be the same for singular vectors of $\btilM$ whose corresponding singular values are $\ge \mu + \tau$. Then, 
    \begin{align*}
    \left(\|\bU_0^\top \btilU_1\|_{\fro}^2 + \|\bV_0^\top \btilV_1\|_{\fro}^2\right)^{\half} \le \frac{2\fronorm{\bM - \btilM}}{\tau}. 
    \end{align*}
    The same is true when the Frobenius norm is replaced by the operator norm.\footnote{A similar bound can be established for arbitrary Schatten $p$-norms, albeit with a slightly worse constant.}
    \end{lem}
    \begin{proof} 
        Consider the matrices
        \begin{align*}
        \bX = \begin{bmatrix} 0 & \bM\\
        \bM^\top & 0
        \end{bmatrix}, \quad \btilX = \begin{bmatrix} 0 & \btilM\\
        (\btilM)^\top & 0
        \end{bmatrix}. 
        \end{align*}
        Letting $\bM = \bU \bSigma \bV^\top$ and $\btilM = \btilU \btilSigma \btilV^\top$, we observe that we can write
        \begin{align*}
        \bX = \bW\bLambda \bW^\top, \quad \bW :=  \frac{1}{\sqrt{2}}\begin{bmatrix} \bU & \bU \\
        \bV & - \bV,
        \end{bmatrix}, \quad \bLambda := \begin{bmatrix} \bSigma & 0 \\
        0 & -\bSigma
        \end{bmatrix}
        \end{align*}
        and analogously  for $\btilX$. Letting $\bM = \bU_0 \bSigma_0 \bV_0^\top + \bU_1 \bSigma_1 \bV_1^\top$ decompose into singular values $\le \mu$ and those $> \mu$, we can write 
        \begin{align*}
        & \bX = \bW_0\bLambda_0 \bW_0^\top + \bW_1\bLambda_1 \bW_1^\top, \\\
        & \bW_0 =  \frac{1}{\sqrt{2}}\begin{bmatrix} \bU_0 & 0 & \bU_0 & 0 \\
        \bV_0 & 0 & - \bV_0 & 0 
        \end{bmatrix}, \quad \bW_1 = \frac{1}{\sqrt{2}}\begin{bmatrix} 0 & \bU_1  & 0 & \bU_1 \\
        0 & \bV_1 & 0 & - \bV_1
        \end{bmatrix},
        \end{align*}
        where $\bLambda_0$ has eigenvalues with absolute value $\le \mu$, and $\bLambda_1$ eigenvalues with absolute value $> \mu$. Applying a similar decomposition to $\btilX$, we find that \Cref{lem:gap_free_DK} yields that, for $\|\cdot\|_{\circ}$ representing either the operator norm or Frobenius norm, 
        \begin{align}
        \left\|\bW_0^\top \bW_1\right\|_{\circ} \le \frac{\fronorm{\bX - \btilX}}{\tau} = \frac{\sqrt{2}\circnorm{\bM - \btilM}}{\tau}.  \label{eq:Wbound}
        \end{align}
        On the other hand, we expand
        \begin{align*}
        \left\|\bW_0^\top \bW_1\right\|_{\circ} &= \frac{1}{2}\left\|\begin{bmatrix} \bU_0 & 0 & \bU_0 & 0 \\
        \bV_0 & 0 & - \bV_0 & 0 
        \end{bmatrix}^\top \begin{bmatrix} 0 & \btilU_1  & 0 & \btilU_1 \\
        0 & \btilV_1 & 0 & - \btilV_1
        \end{bmatrix} \right\|_{\circ} \\
        &= \frac{1}{2}\left\|\begin{bmatrix} 0 &  \bU_0^\top \btilU_1 + \bV_0^\top \btilV_1  & 0 &  \bU_0^\top \btilU_1 - \bV_0^\top \btilV_1
        \end{bmatrix} \right\|_{\circ} \\
        &= \frac{1}{2}\left\|\begin{bmatrix}  \bA + \bB  &  \bA - \bB 
        \end{bmatrix} \right\|_{\circ}, \quad \bA := \bU_0^\top \btilU_1, ~~\bB := \bV_0^\top \btilV_1.  
        \end{align*}
        When $\circ$ denotes the Frobenius norm, we use
        \begin{align*}
        \left\|\begin{bmatrix}  \bA + \bB  &  \bA - \bB 
        \end{bmatrix} \right\|_{\fro}^2 &= \langle \bA + \bB, \bA + \bB \rangle + \langle \bA - \bB, \bA -  \bB \rangle \\
        &= 2\langle \bA, \bA \rangle + 2 \langle \bB, \bB \rangle = 2\left(\|\bA\|_{\fro}^2 + \|\bB\|_{\fro}^2\right). 
        \end{align*}
        Similarly, when $\circ$ denotes the operator norm, 
        \begin{align*}
        \left\|\begin{bmatrix}  \bA + \bB  &  \bA - \bB 
        \end{bmatrix} \right\|_{\op}^2 &= \max_{\bv:\|\bv\| = 1}\|\bv^\top(\bA+\bB)\|_2^2 +\|\bv^\top(\bA-\bB)\|_2^2  \\
        &= \max_{\bv:\|\bv\| = 1}2\bv^\top \bA\bA^\top\bv + 2\bv^\top \bB\bB^\top\bv + 2\bv^\top \bA\bA^\top\bv - 2\bv^\top \bA\bA^\top\bv\\
        &= \max_{\bv:\|\bv\| = 1}2\bv^\top \bA\bA^\top\bv + 2\bv^\top \bB\bB^\top\bv\\
        &\le 2\left(\|\bA\|_{\op}^2 + \|\bB\|_{\op}^2\right).
        \end{align*}

        Thus, 
        \begin{align*}
        \left\|\bW_0^\top \bW_1\right\|_{\circ} \le\frac{1}{\sqrt{2}}\left(\|\bU_0^\top \btilU_1\|_{\circ}^2 + \|\bV_0^\top \btilV_1\|_{\circ}^2\right)^{\half}.
        \end{align*}
        Plugging this into \Cref{eq:Wbound} concludes.
    \end{proof}

\newpage
\part{Supplement for Single- and Double-Stage ERM}\label{part:supp_sing_double}

\newcommand{\dterrzero}{\mathsf{ApxErr}_{\textsc{dt},0}}
\newcommand{\dterrone}{\mathsf{ApxErr}_{\textsc{dt},1}}

\newcommand{\rst}{r_{\star}}
\newcommand{\rbarst}{\bar{r}_{\star}}
\newcommand{\rbar}{\bar{r}}
\newcommand{\staterrdt}{\mathsf{StatErr}_{\textsc{dt}}}

\newcommand{\Elucky}{\cE_{\mathrm{luck}}}
\section{Addenda for Single- and Double-Stage ERM (\Cref{thm:st_main})}
\label{sec:ERM_addenda}
\subsection{Single-stage ERM}\label{sec:thm:st_main}

\begin{proof}[Proof of \Cref{thm:st_main}] The first part of \Cref{thm:st_main} follows directly from combining  \Cref{thm:generic_risk_bound} and using a standard statistical training guarantee, \Cref{lem:single_erm}, to bound $\epsone$ and $\epstrain$; \Cref{eq:specerrst_bound} follows from a computation performed in \Cref{lem:st_instantiation}, below, and whose proof appears in \Cref{sec:rate_instant}.
\end{proof}

\begin{lem}[Single Training Bound]\label{lem:st_instantiation}
Under \Cref{asm:eigendecay}, we have
\begin{align*}
\specerrst(r) &\lesssim \begin{cases} C^2(1+\gamma^{-1})^2 r^{6-2\gamma} & \text{(polynomial decay)}\\
C^2r^6(\gamma^{-1} + r)^2e^{-2\gamma r} & \text{(exponential decay)}
\end{cases}.
\end{align*}
\end{lem}

\begin{rem}[Sufficient Spectral Decay for $\alpha$-Conditioning]\label{rem:suff_eig_cond}. For sufficiently rapid spectral decay, it \emph{is} possible to ensure $(\fhatst,\ghatst)$ are well-conditioned. From \Cref{lem:single_erm}, we have that with probability at least $1-\delta$, 
\begin{align*}
\Risk(\fhatst,\ghatst;\cdone) &\le \kaptrain(2\kapapx\tailsf_2(r) + \frac{352 B^4 (\capac_r + \log \frac{2}{\delta}) }{n}).
\end{align*}
In particular, if for a given $\alpha \ge 1$ it holds that 
\begin{align}
2\kaptrain\kapapx\tailsf_2(r) \le (1-\alpha^{-1})(\bsigst_r)^2, \label{eq:tail_suffice_cond}
\end{align} then, by letting $n \ge {352 \alpha B^4 (\capac_r + \log \frac{2}{\delta})}{\bsigst_r}$, we can take  $\epsone^2 = \Risk(\fhatst,\ghatst;\cdone) \le (1-(2\alpha)^{-1})(\bsigst_r)^2$. By \Cref{thm:generic_risk_bound}, this implies that $(\fhatst,\ghatst)$ are $2\alpha$-conditioned.  Thus, when the tail of the spectrum at $r$ is considerably smaller than $(\bsigst_r)^2$, we can ensure that $(\fhatst,\ghatst)$ are well-conditioned.

\Cref{eq:tail_suffice_cond} requires rather rapid spectral decay, and will not hold for polynomially decaying singular values (e.g. $\bsigst_r = r^{-(1+\gamma)})$. Under the exponential decay regime of \Cref{asm:eigendecay} (for all $n$, $\bsigst_n \le Ce^{-\gamma n}$),  \Cref{lem:tailsf_bounds} implies that  $\tailsf_2(r) \le C^2(1+\gamma^{-1})e^{-2\gamma(r + 1)}$ (which is more-or-less tight in the worst case). Thus, \Cref{eq:tail_suffice_cond}
holds as soon as 
\begin{align}
2\kaptrain\kapapx(1+\gamma^{-1})e^{-2\gamma} \le (1-\alpha^{-1})e^{2\gamma r}\left(\frac{\sigst}{C}\right)^2. \label{eq:tail_suffice_cond_two}
\end{align}
Now assume that a lower bound for spectral decay also holds: for some other constant $c$, we have $\bsigst_r \ge ce^{-r\gamma}$. Then, \Cref{eq:tail_suffice_cond_two} holds as soon as 
\begin{align}
2\kaptrain\kapapx(1+\gamma^{-1})e^{-2\gamma} \le (1-\alpha^{-1})\left(\frac{c}{C}\right)^2,
\end{align}
which is true once $\gamma > \log(\frac{2\sqrt{2}C\kaptrain\kapapx}{c(1-\alpha^{-1})})$. In summary, we can ensure well-conditioned $(\fhatst,\ghatst)$ when (a) there is rapid, exponential spectral decay and (b) a lower bound on the spectral decay as well.

\end{rem}

\subsection{Double-stage ERM (\Cref{thm:main_dt_simple})}\label{app:DS_ERM_detailed_results}

Here, we present \Cref{thm:main_dt_risk_bound_thing},  a more detailed version of \Cref{thm:main_dt_simple} which specifies the necessary setting of algorithm parameters. We then specialize  \Cref{thm:main_dt_risk_bound_thing} to  \Cref{thm:main_dt_simple} at the end of the section. These two aforementioned conditions are specified in the following two conditions.

\newcommand{\pointnum}[1]{{\normalfont (#1)}}
\begin{condition}[Algorithm Parameters]\label{cond:dt_apx_conds} Let $c_1$ be some unspecified parameter satisfying $1 \le c_1 \lesssimst 1$. We stipulate that the algorithm parameters $(\sigcut,\rcut,p)$ satisfy
\begin{itemize}
    \item[(a)] $\rcut \ge c_1$ and $\tailsf_2(\rcut) \le \frac{1}{c_1}\rcut^2(\sigcut)^2$; 
    \item[(b)] $\tailsf_2(p) \le \frac{1}{c_1}\frac{\sigcut^2}{\rcut^5}$;
    \item[(c)] $\sigcut \in [2\bsigst_{\rcut},\frac{2}{3e}\bsigst_1]$.
\end{itemize}
\end{condition}
\begin{condition}[Sample Size Conditions]\label{cond:dt_sample_conds}Let $c_2$ be some unspecified  parameter satisfying $c_2 \lesssimst 1$.  We stipulate that, given $\delta \in (0,1)$,
    \begin{itemize}
    \item The supervised sample sizes of $n_1,n_3$ satisfy
    \begin{align*}
    n_1 & \ge p + B^4 c_2(\capac_{p} + \log \frac{1}{\delta})  \rcut^4 \sigcut^{-2}, \quad
    n_3 \ge c_2 B^4(\capac_{\rcut} + \log \frac{1}{\delta})\sigcut^{-2}
    \end{align*}
    \item The unsupervised sample sizes $n_2,n_4$ satisfy
    \begin{align*}
    n_2 \ge 722 \rcut^2 n_1^9\log(24p/\delta), \quad n_4 \ge \rcut^4 n_1 n_3.
    \end{align*}
\end{itemize}
 Note that when $\rcut \le p$ (and hence $\capac_{\rcut} \le \capac_p$), it suffices take $n_i \lesssimst \mathrm{poly}(p,\capac_p,\log(1/\delta),B,\sigcut^{-2})$.
\end{condition}
\newcommand{\errdt}{\textsc{Err}_{\textsc{dt}}}
Our main detailed theorem is as follows, and its proof is given in \Cref{sec:analysis_training_alg}.
\begin{theorem}\label{thm:main_dt_risk_bound_thing} Suppose \Cref{alg:main_alg} is run with parameters $\sigcut,\rcut,p$, sample sizes $n_1,\dots,n_4$, and  $\lambda = \rcut^4$, $\mu = B^2/n_1$ and fix a probability of error $\delta \in (0,1)$. Then, as long $\sigcut,\rcut,p$ satisfy \Cref{cond:dt_apx_conds} and $n_{1:4}$ satisfy \Cref{cond:dt_sample_conds}, it holds with probability at least $1 - \delta$,
\begin{align*}
\Risk(\fhat,\ghat;\Dtest) \lesssimst \errdt(\rcut,\sigcut) := \rcut^2 \sigcut^2  +\tailsf_1(\rcut)^2   + \frac{\tailsf_2(\rcut)^2 }{(\sigcut)^2}.
\end{align*}
\end{theorem}

In \Cref{sec:lem:dt_training_decay}, we prove the following lemma. It gives an upper bound $\errdt(\rcut,\sigcut)$, as well as sufficient conditions for \Cref{cond:dt_apx_conds}, under the spectral decay assumption in \Cref{asm:eigendecay}.
\begin{lem}[Double Training Decay Bounds]\label{lem:dt_training_decay}  
Suppose \Cref{asm:eigendecay} holds, and that the algorithm parameters $\sigcut,\rcut,p$ satisfy $\sigcut \le \frac{2}{3e}\bsigst_1$, and the following (feasible) constraints
\begin{align*}
\rcut \ge c_1 \vee \begin{cases}\frac{3e C}{\bsigst_1}\\
\sqrt{c_1(1+\frac{1}{\gamma})} \vee \frac{1}{\gamma} \log(\frac{3e C}{\bsigst_1})\end{cases} ~~ p \ge \begin{cases}c_1^{-\frac{1}{1+2\gamma}}\rcut^{\frac{7+5\gamma}{1+2\gamma}}\\
2\rcut \vee \frac{1}{\gamma}\log(\rcut^5 c_1)\end{cases}~~ \sigcut \ge \begin{cases} 2C\rcut^{-(1+\gamma)}\\
2Ce^{-\gamma \rcut}
\end{cases},
\end{align*}
where the top-case correponds to the polynomial-decay regime, and the bottom to exponential-decay. Then, \Cref{cond:dt_apx_conds} holds and 
\begin{align}\label{eq:err_dt_spec}
\errdt(\rcut,\sigcut) &\lesssim \sigcut^2 \rcut^2 + C^2(1+\gamma^{-2})\begin{cases}   \rcut^{-2\gamma} & \text{(polynomial decay)}\\   e^{-2\gamma \rcut} & \text{(exponential decay)}
\end{cases}.
\end{align}
\end{lem}

\begin{proof}[Proof of \Cref{thm:main_dt_simple}] For the target accuracy $\epsilon$, set $\sigcut = \max\{2C\rcut^{-(1+\gamma)},\epsilon/\rcut\}$ under polynomial spectral decay, and $\sigcut = \max\{2Ce^{-\gamma \rcut},\epsilon/\rcut\}$. From \Cref{eq:err_dt_spec} and absorbing absolute constants into $\lesssim$, it then follows that
\begin{align}\label{eq:err_dt_spec}
\errdt(\rcut,\sigcut) &\lesssim \epsilon^2 + C^2(1+\gamma^{-2})\begin{cases}   \rcut^{-2\gamma} & \text{(polynomial decay)}\\   e^{-2\gamma \rcut} & \text{(exponential decay)}
\end{cases}.
\end{align}
From \Cref{lem:dt_training_decay}, \Cref{cond:dt_apx_conds} holds as soon as $\rcut \gtrsim_{\star}\mathrm{poly}(C/\bsigst_1,\gamma^{-1})$ and $p \lesssim_{\star} (\rcut)^{c}$ for a universal $c > 0$ (note that, in the constraint on $p$ in polynomial case, the ratio $\frac{7+5\gamma}{1+2\gamma}$ is at most $7$). Moreover,  there exist sample sizes $n_1,n_2,n_3,n_4 \lesssim_{\star} \poly(p,\capac_p,\log(1/\delta),B,\epsilon^{-2})$ which ensure \Cref{cond:dt_sample_conds}. The result now follows from \Cref{thm:main_dt_risk_bound_thing} above.
\end{proof}

\subsection{Generalizing unsupervised access to $\cdone$ (\Cref{asm:unlabeled_cdone_orac})}\label{sec:gen:asm:unlabeled_cdone_orac}

\newcommand{\cdonetil}{\tilde{\cD}_{1\otimes 1}}
\newcommand{\cdxonetil}{\tilde{\cD}_{\cX,1}}
\newcommand{\cdyonetil}{\tilde{\cD}_{\cY,1}}
\newcommand{\kaptil}{\tilde{\kappa}}
In this section, we argue that if we replace $\cdxone,\cdyone$ with any other distribution $\cdxonetil,\cdyonetil$ satisfying for some $\kaptil \ge 1$ the inequalities
\begin{align*}
\kaptil^{-1} \le \frac{\rmd \cdxonetil(x)}{\rmd \cdxone(x)} \le \kaptil, \quad \kaptil^{-1} \le \frac{\rmd \cdyonetil(y)}{\rmd \cdyone(y)} \le \kaptil, \numberthis \label{eq:kaptil_ratio}
\end{align*}
and if  the function classes $\cF_k,\cG_k$ are sufficiently expressive,
then all of our problem assumptions remain true, up to multiplicative constants in $\kaptil$. In particular, this means that, for any target distributions $\cdxone,\cdyone$, we can replace the oracle in \Cref{asm:unlabeled_cdone_orac} with the  one that samples from $\cdonetil := \cdxonetil \otimes \cdyonetil$. We now go through each assumption in sequence.

\newcommand{\Sigftil}{\tilde{\bmsf{\Sigma}}_{f^\star}}
\newcommand{\Siggtil}{\tilde{\bmsf{\Sigma}}_{g^\star}}

\begin{itemize}
    \item  First, \Cref{asm:bilinear} is unaffected.
    \item Second, let us consider the covariance $\Sigf=\Exp_{x\sim \cdx{1}}\big[\fst(x) \fst(x)^\top\big]$ and $\Sigg=\Exp_{y\sim \cdy{1}}\big[\gst(y)\gst(y)^\top\big]$. Uder assumption \Cref{asm:bal}, $\Sigf = \Sigg$, and $\|\fst(x)\|_{\hilspace} \vee \|\gst(x)\|_{\hilspace} \le B$. Introduce as well $\Sigftil=\Exp_{x\sim \cdxonetil}\big[\fst(x) \fst(x)^\top\big]$ and $\Siggtil=\Exp_{y\sim \cdyonetil}\big[\gst(y)\gst(y)^\top\big]$. Then, \Cref{eq:kaptil_ratio} implies that 
    \begin{align} \label{eq:Sigf_Sigftil_comp}
    \kaptil^{-1}\Sigf \preceq \Sigftil \preceq \kaptil \Sigf, \quad \kaptil^{-1}\Sigg \preceq \Siggtil \preceq \kaptil \Sigg.
    \end{align}
    Using $\Sigf = \Sigg$, we have
    \begin{align*}
    \kaptil^{-2}\Sigftil \preceq \Siggtil \preceq \kaptil^2\Sigftil
    \end{align*}
    By generalizing \Cref{lem:bal_properties}(i\&iv) to linear operators, we can construct a transformation an invertible $\bmsf{W}$ such that $\kaptil^{-1} \bmsf{I} \preceq \bmsf{W} \preceq \kaptil \bmsf{I}$ and 
    \begin{align*}
    \bmsf{W}\Sigftil\bmsf{W} = \Siggtil
    \end{align*}
    Hence, if we define the operator $\bmsf{T} =  \bmsf{W}^{1/2}$ and set
    \begin{align*}
    \ftilst := \bmsf{T}\fst, \quad \gtilst := \bmsf{T}\gst,
    \end{align*}
    then $(\ftilst,\gtilst)$ are balanced:
    \begin{align*}
    \tilde{\bmsf{\Sigma}}_{1\otimes 1} := \Exp_{x\sim \cdxonetil}\big[\fst(x) \fst(x)^\top\big] = \Exp_{y\sim \cdyonetil}\big[\gtilst(y)\gtilst(y)^\top\big]. 
    \end{align*}
    Moreover, as $\kaptil^{-1} \bmsf{I} \preceq \bmsf{W} \preceq \kaptil \bmsf{I}$, $\kaptil^{-1/2} \bmsf{I} \preceq \bmsf{T} \preceq \kaptil^{1/2} \bmsf{I}$, so that
    \begin{align*}
    \sup_{x,y} \|\ftilst(x)\|_{\hilspace} \vee \|\gtilst(x)\|_{\hilspace} \le \sqrt{\kaptil} \sup_{x,y} \|\fst(x)\|_{\hilspace} \vee \|\gst(x)\|_{\hilspace}  \le \sqrt{\kaptil} B;
    \end{align*}
    that is, \Cref{asm:bal} holds with upper bound $\tilde B = \sqrt{\kaptil} B$. 
    \item One can directly check from \Cref{eq:kaptil_ratio} that replacing $\cdxone \gets \cdxonetil$ and $\cdyone \gets \cdyonetil$ ensures \Cref{asm:density} holds with $\tilde{\kappa}_{\mathrm{trn}} = \tilde{\kappa}^2 \kaptrain$ and $\tilde{\kappa}_{\mathrm{tst}} = \tilde \kappa^2 \kaptest$. 
    \item Similarly, one can check that \Cref{asm:cov} with $\tilde{\kappa}_{\mathrm{cov}} \gets \tilde\kappa \kapcov $. 
    \item The construction of $\tilde{\bmsf{\Sigma}}_{1\otimes 1}$ and \Cref{lem:bal_properties} (vii) imply
    \begin{align*}
    \lambda_i(\tilde{\bmsf{\Sigma}}_{1\otimes 1}) \le \sigma_i(\Sigftil^{1/2}\Siggtil^{1/2} ) \le \sqrt{\lambda_i(\Sigftil)\lambda_i(\Siggtil)}.
    \end{align*}
    Using \Cref{eq:Sigf_Sigftil_comp} to bound $\lambda_i(\Sigftil) \le \kaptil \lambda_i(\Sigf) = \lambda_i(\Sigst)$ and similarly for $\Siggtil$, we find
     \begin{align*}
    \lambda_i(\tilde{\bmsf{\Sigma}}_{1\otimes 1}) \le \kaptil \lambda_i(\Sigst).
    \end{align*}
    Thus, \Cref{asm:eigendecay} holds after inflacting the constant $C$ by a factor of $\kaptil$.
    \item In can be directly checked that \Cref{asm:training_approximation} holds after replacing $\kapapx$ with $\tilde{\kappa}_{\mathrm{apx}} := \tilde \kappa^2 \kapapx$.
    \item The last assumption, \Cref{asm:function_apx} needs to be modified so as to ensure the function classes $\cF_k,\cG_k$ are rich enough to express the rank-$k$ projections $\ftilst_k,\gtilst_k$ (the analogues of $\fst_k,\gst_k$ defined in \Cref{sec:formulation}).
\end{itemize}

\newcommand{\hatSigmubal}{\hat{\bSigma}_{\mathrm{bal},\mu}}
    \newcommand{\Sigmubal}{{\bSigma}_{\mathrm{bal},\mu}}
    \newcommand{\hatWmubal}{\hat{\bW}_{\mathrm{bal},\mu}}
    \newcommand{\Wmubal}{{\bW}_{\mathrm{bal},\mu}}

    \newcommand{\ftilmu}{\tilde{f}_{\mu}}
    \newcommand{\gtilmu}{\tilde{g}_{\mu}}
    \newcommand{\ftilmubal}{\tilde{f}_{\mu,\mathrm{bal}}}
    \newcommand{\gtilmubal}{\tilde{g}_{\mu,\mathrm{bal}}}
    \newcommand{\htilmur}{\tilde{h}_{\mu,r}}

    \newcommand{\cdbarone}{\bar{\cD}_{1\otimes 1}}
    \newcommand{\cdbarxone}{\bar{\cD}_{\cX,1}}
    \newcommand{\cdbaryone}{\bar{\cD}_{\cY,1}}
    \newcommand{\distequals}{\overset{\mathrm{dist}}{=}}
\section{Analysis of the  Algorithms}\label{sec:analysis_training_alg}

In this section, we provide analyses for the training algorithms we proposed. \Cref{sec:single_train} gives guarantees for a single stage of supervised ERM. \Cref{sec:thm:main_dt_risk_bound_thing} establishes our main guarantee for \Cref{alg:main_alg}, \Cref{thm:main_dt_risk_bound_thing}, via a technical proposition \Cref{thm:double_training_analysis}, whose proof is divided between the subsequent three sections.

\subsection{Statistical guarantee for single-stage ERM}\label{sec:single_train}

We present an analysis of a single phase of empirical risk minimization, which we use both to analyze the single-stage ERM, and to serve as the first step in our analysis of double-stage ERM. The following is proved in \Cref{sec:single_erm_proof}, using a standard analysis of empirical risk minimization with the squared loss. 
\begin{lem}\label{lem:single_erm} Let $(\tilde f,\tilde g) \in \Fclass_p \times \Gclass_p$ be empirical risk minimizers on $\nnone$ i.i.d. samples $(x_i,y_i,z_i) \sim \Dtrain$. Then,  for any $\delta \in (0,1)$, the followings hold  with probability  at least $1 - \delta$: 
\begin{align*}
\Risk(\tilde f,\tilde g;\Dtrain) &\le 2\Risk(\fst_p,\gst_p;\Dtrain) + \frac{352 B^4 (\capac_p + \log \frac{2}{\delta}) }{\nnone}\\
\Risk(\tilde f,\tilde g;\Dtrain) &\le 2\kapapx\tailsf_2(p) + \frac{352 B^4 (\capac_p + \log \frac{2}{\delta}) }{\nnone}\\
\Risk(\tilde f,\tilde g;\cdone) &\le \kaptrain(2\kapapx\tailsf_2(p) + \frac{352 B^4 (\capac_p + \log \frac{2}{\delta}) }{\nnone}).
\end{align*}
\end{lem}

\subsection{Proof overview of \Cref{thm:main_dt_risk_bound_thing}}\label{sec:thm:main_dt_risk_bound_thing}

\newcommand{\Espec}{\cE_{\mathrm{spec}}}
To prove \Cref{thm:main_dt_risk_bound_thing}, we first demonstrate that a certain technical proposition \Cref{thm:double_training_analysis} which shows that (a) a good spectral event $\Espec$ holds, under which the rank $\rhat$ chosen by \Cref{alg:dim_red} satisfies various convenient spectral properties, and (b) that the regularized risk optimized in the last line of \Cref{alg:main_alg} is small. More precisely, we define: 
\begin{defn}[Good Spectral Event]\label{defn:espec}
For parameters $(\sigcut,\rcut)$ used in  \Cref{alg:main_alg}, we define $\Espec(\rhat,\sigcut,\rcut)$ as the event that the following  inequalities hold: 
    \begin{align*}
    &\bsigst_{\rhat} \ge \frac{3}{4}\sigcut, \quad \bsigst_{\rhat+1} \le 3\sigcut, \quad \bsigst_{\rhat} - \bsigst_{\rhat+1} \ge \frac{\bsigst_{\rhat}}{3\rcut}\\
    &\tailsf_2(\rhat) \le \tailsf_2(\rcut) + 9\sigcut^2 \rcut, \quad \tailsf_1(\rhat)^2 \le 18 \rcut^2 \sigcut^2 + 2\tailsf_1(\rcut)^2. 
    \end{align*}
    \end{defn}
Our technical proposition is as follows.
\begin{prop}\label{thm:double_training_analysis} Suppose that the parameters in \Cref{alg:main_alg} are chosen as $\mu = B^2/\nnone$, and other parameters $(p,\sigcut,\rcut)$, the sample sizes $n_1,\dots,n_4$, and $\lambda > 0$ satisfy that for some $C \lesssimst 1$, 
\begin{itemize}
    \item $\sigcut \in [2\bsigst_{\rcut},\frac{2}{3e}\bsigst_1]$,  $\tailsf_2(p) \le \frac{\sigcut^2}{C\rcut^2}$,  and $p \ge 2$;
    \item $\nnone \ge p+ C\sigcut^{-2}\rcut^2\max\{1,B^4\} (\capac_p + \log \frac{1}{\delta})$,  $\nntwo \ge 722 \rcut^2 \nnone^9\log(24p/\delta)$, $n_4 \ge \lambda n_1 n_3 $.
\end{itemize}
Then, with probability at least $1 -\delta$, the event $\Espec(\rhat,\sigcut,\rcut)$ holds and
\begin{align*}
&\Risk(\fhatdt,\fhatdt;\Dtrain) + \lambda \Risk_{[r]}(\fhatdt,\fhatdt;\cdone)\\
&\quad \lesssimst \tailsf_2(\rcut) + \rcut \sigcut^2 + \lambda \rcut^2 \tailsf_2(p) + \frac{B^4(\capac_{\rcut} + \log(1/\delta))}{n_3} + \frac{\lambda \rcut^2 B^4 (\capac_p + \log(1/\delta))}{n_1}.
\end{align*} 
\end{prop}

We will prove \Cref{thm:double_training_analysis} in \Cref{sec:double_train_first_proof,sec:double_train_second_proof}, addressing the first and second phases of training in \Cref{alg:main_alg} respectively. Using this result, we prove \Cref{thm:main_dt_risk_bound_thing}.
\begin{proof}[Proof of \Cref{thm:main_dt_risk_bound_thing}]
Recall the statement of \Cref{thm:generic_risk_bound}. It states that if  $(\fhatdt,\fhatdt)$ are $(\epstrain,\epsone)$-accurate,  that $\epsone \le \min\{\bsigst_1/40 \rhat, \bsigst_{\rhat}/4\}$, then we can bound $\alpha \le 2$ and therefore bound
\begin{align*}
\Risk(\fhatdt,\fhatdt;\Dtest) &\lesssim_{\star}    \left\{\rhat^4 \epsone^2  +  \tailsf_1(\rhat)^2 + \rhat^2 (\bsigst_{\rhat+1})^2 \right\}  +   \left\{\frac{(\rhat^3 \epsone^2 + \epstrain^2  +  \tailsf_2(\rhat))^2 }{(\bsigst_{\rhat})^2}     \right\}.  
\end{align*} 
In particular, recall we select $\rhat \le \rcut$ and $\lambda = \rcut^4$. Then it suffices that $\epsone \le \min\{\bsigst_1/40 \rcut, \bsigst_{\rcut}/4\}$ to ensure
\begin{align*}
\Risk(\fhatdt,\fhatdt;\Dtest) &\lesssim_{\star}    \left\{\lambda\epsone^2  +  \tailsf_1(\rhat)^2 + \rhat^2 (\bsigst_{\rhat+1})^2 \right\}  +   \left\{\frac{(\lambda \epsone^2 + \epstrain^2  +  \tailsf_2(\rhat))^2 }{(\bsigst_{\rhat})^2}     \right\}.  
\end{align*}
On the event $\Espec(r,\sigcut,\rcut)$, we can then bound 
\begin{align*}
\Risk(\fhatdt,\fhatdt;\Dtest) &\lesssim_{\star}    \left\{\lambda\epsone^2   +  \rcut^2 \sigcut^2  +\tailsf_1(\rcut)^2  + \rhat^2 \sigcut^2 \right\}  \\
&\qquad+   \left\{\frac{(\lambda \epsone^2 + \epstrain^2  +  \rcut (\sigcut)^2 +\tailsf_2(\rcut))^2 }{(\sigcut)^2}     \right\}\\
&\lesssim    \left\{\lambda\epsone^2 +  \rcut^2 \sigcut^2  +\tailsf_1(\rcut)^2   \right\}  \tag{$\rhat \le \rcut$}\\
&\qquad+   \left\{\frac{(\lambda \epsone^2 + \epstrain^2  +  \rcut (\sigcut)^2 +\tailsf_2(\rcut))^2 }{(\sigcut)^2}     \right\}.  
\end{align*}
Next, we set $\epstrain^2:= \Risk(\fhatdt,\fhatdt;\Dtrain)$ and $\epsone^2:= \Risk_{[r]}(\fhatdt,\fhatdt;\cdone)$. Then, on the event of the conclusion of \Cref{thm:double_training_analysis}, 
and using $\lambda = \rcut^4$, we have
\begin{align*}
&\lambda \epsone^2 \le \lambda \epsone^2 + \epstrain^2  \\
&\lesssimst \tailsf_2(\rcut) + \rcut \sigcut^2 + \rcut^6 \tailsf_2(p) + \underbrace{\tfrac{B^4(\capac_{\rcut} + \log(1/\delta))}{n_3} + \tfrac{ \rcut^6 B^4 (\capac_p + \log(1/\delta))}{n_1}}_{\le 2\sigcut^2 \text{ under \Cref{cond:dt_sample_conds}}}\\
&\lesssim \tailsf_2(\rcut) + \rcut \sigcut^2 + \rcut^6 \tailsf_2(p).
\end{align*}
Plugging the former display into the one before it, and suppressing constants, we have 
\begin{align*}
\Risk(\fhatdt,\fhatdt;\Dtest) &\lesssim_{\star}   \left\{\rcut^6 \tailsf_2(p)  +  \tailsf_2(\rcut) + \rcut \sigcut^2 +  \rcut^2 \sigcut^2  +\tailsf_1(\rcut)^2   \right\} \\
&\qquad+   \left\{\frac{(\rcut^6 \tailsf_2(p)  +  \rcut (\sigcut)^2 +\tailsf_2(\rcut))^2 }{(\sigcut)^2}     \right\}.  
\end{align*}
In particular, if in addition it holds that 
\begin{align}
 \tailsf_2(p) \le \frac{(\sigcut)^2 }{r^5}, \label{eq:last_tailsfp}
\end{align}
then 
\begin{align*}
\Risk(\fhatdt,\fhatdt;\Dtest) &\lesssim_{\star}    \tailsf_2(\rcut) + \rcut \sigcut^2 +  \rcut^2 \sigcut^2  +\tailsf_1(\rcut)^2   \frac{(  \rcut (\sigcut)^2 +\tailsf_2(\rcut))^2 }{(\sigcut)^2}   \\
&\lesssim     \tailsf_2(\rcut)  +\rcut^2 \sigcut^2  +\tailsf_1(\rcut)^2   + \frac{\tailsf_2(\rcut)^2 }{(\sigcut)^2}\\
&\lesssim     \rcut^2 \sigcut^2  +\tailsf_1(\rcut)^2   + \frac{\tailsf_2(\rcut)^2 }{(\sigcut)^2},
\end{align*}
where in the last step, we use $\tailsf_1(\rcut)^2 = (\sum_{i > \rcut} \bsigst_i)^2 \ge \sum_{i >\rcut} (\bsigst_i)^2 = \tailsf_2(\rcut)$.  so that, if the conditions on $n_{1:4}$ and $p$ of \Cref{thm:double_training_analysis} are met, and if $\epsone \le \min\{\bsigst_1/40 \rhat, \bsigst_{\rhat}/4\}$, and if  \Cref{eq:last_tailsfp} holds, then with probability at least $1 - \delta$,
\begin{align*}
\Risk(\fhatdt,\fhatdt;\Dtest) &\lesssim_{\star}   \rcut^2 \sigcut^2  +\tailsf_1(\rcut)^2   + \frac{\tailsf_2(\rcut)^2 }{(\sigcut)^2}.
\end{align*}

\paragraph{Checking the appropriate conditions.} For  \Cref{thm:double_training_analysis} to hold with $\lambda = \rcut^4$, and for  we need that for some $c_0 \lesssimst 1$,
\begin{align}
\nnone \ge p+ c_0\sigcut^{-2}\rcut^2B^4 (\capac_p + \log \frac{1}{\delta}), \quad p \ge 2, \quad \tailsf_2(p) \le \frac{\sigcut^2}{c_0\rcut^5}, \label{eq:cond_first}
\end{align}
as well as
\begin{align}
  \nntwo \ge 722 \rcut^2 \nnone^9\log(24p/\delta), \quad n_4 \ge \rcut^4 n_1 n_3,\qquad \mu = B^2/\nnone. \label{eq:cond_second}
\end{align}
All these conditions are ensured by \Cref{cond:dt_sample_conds,cond:dt_apx_conds}.

Let us conclude by making explicit conditions under which
$\epsone \le \min\{\bsigst_1/40 \rhat, \bsigst_{\rhat}/4\}$ holds, provided the high-probablity event of \Cref{thm:double_training_analysis} holds. As $\rhat \le \rcut$, on the $\Espec(\rhat,\sigcut,\rcut)$, it is enough that, for some small universal constant $c$,
\begin{align}
\epsone^2 \le c\min\left\{\frac{(\bsigst_1)^2}{ \rcut^2} , \sigcut^2\right\}. \label{eq:epsone_in_alg}
\end{align}
On the event of \Cref{thm:double_training_analysis}, we would like to have 
\begin{align*}
\epsone^2 \lesssimst  \frac{\tailsf_2(\rcut)}{\rcut^4} + \frac{\sigcut^2}{\rcut^3} + \rcut^2 \tailsf_2(p) + \frac{B^4(\capac_{\rcut} + \log(1/\delta))}{n_3 \rcut^4} + \frac{\rcut^2 B^4 (\capac_p + \log(1/\delta))}{n_1 }. 
\end{align*}
By modifying $c_0 \lesssimst 1$ below  if necessary, it suffices that for \Cref{eq:epsone_in_alg} that
\begin{align*}
\max\{\tfrac{\sigcut^2}{\rcut^3}, \tfrac{\tailsf_2(\rcut)}{\rcut^4}, \rcut^2 \tailsf_2(p), \tfrac{B^4(\capac_{\rcut} + \log(1/\delta))}{n_3 \rcut^4}, \tfrac{\rcut^2 B^4 (\capac_p + \log(1/\delta))}{n_1 }\} \le \frac{1}{c_0}\min\left\{\frac{(\bsigst_1)^2}{ \rcut^2} , \sigcut^2\right\}.
\end{align*}
We handle each term in sequence,
\begin{enumerate}
    \item As $\sigcut \le \bsigst_1$, we have $\tfrac{\sigcut^2}{\rcut^3} \le \frac{1}{c_0}\min\left\{\frac{(\bsigst_1)^2}{ \rcut^2} , \sigcut^2\right\}$ as soon as $\rcut \ge c_0$. 
    \item The term $\tfrac{\tailsf_2(\rcut)}{\rcut^4}$ is appropriately bounded as soon as $\tailsf_2(\rcut) \le \frac{1}{c_0}\min\left\{\rcut^2(\bsigst_1)^2 , \rcut^4\sigcut^2\right\}$. Under the condition that $\sigcut \le \bsigst_1$, it suffices that $\tailsf_2(\rcut) \le \frac{1}{c_0}\rcut^2\sigcut^2$.
    \item The term $\rcut^2 \tailsf_2(p)$ is appropriately bounded as soon as $\tailsf_2(p) \le \frac{1}{c_0}\min\left\{\frac{(\bsigst_1)^2}{ \rcut^4} , \sigcut^2/\rcut^2\right\}$. As $\sigcut \le \bsigst_1$, this holds when $\tailsf_2(p) \le \frac{\sigcut^2}{c_0\rcut^5}$.
    \item The term $\tfrac{B^4(\capac_{\rcut} + \log(1/\delta))}{n_3 \rcut^4}$ is appropriately bounded as soon as 
    \begin{align*}
    n_3 \ge c_0 B^4(\capac_{\rcut} + \log(1/\delta)) \left\{\frac{1}{(\bsigst_1)^2\rcut^2} + \frac{1}{\sigcut^2\rcut^4}\right\}.
    \end{align*}
     \item Similarly, term $\tfrac{\rcut^2 B^4 (\capac_p + \log(1/\delta))}{n_1 }$ is appropriately bounded as soon as (adding an additive $p$ for convenience),
    \begin{align*}
    n_1  \ge p + B^4 c_0(\capac_{p} + \log(1/\delta))  \left\{\frac{\rcut^4}{(\bsigst_1)^2} + \frac{\rcut^2}{\sigcut^2}\right\}.
    \end{align*}
    For which, using $\sigcut \le \bsigst_1$, it suffices that 
    \begin{align*}
    n_1  \ge p + B^4 c_0(\capac_{p} + \log(1/\delta))  \frac{\rcut^4}{\sigcut^2}.
    \end{align*}
\end{enumerate}
All such bounds hold under \Cref{cond:dt_apx_conds,cond:dt_sample_conds}. 
This completes the proof of \Cref{thm:main_dt_risk_bound_thing}. 
\end{proof}

\subsection{Analysis of the first phase of double-stage ERM}\label{sec:double_train_first_proof}

    We begin with a precise analysis of the first phase of the double-stage ERM \Cref{alg:main_alg}. Recall that $(\ftil,\gtil)$ are the empirical risk minimizers on $\nnone$ i.i.d. samples $(x_i,y_i,z_i) \sim \Dtrain$, and $\bhatQ_{\rhat}$ is the balancing projection on the top $r$ eigenvectors of $\hat{\bSigma}_{\gtil}$. We define the following effective error term. 
    \begin{align}
    \epstil(p,\nnone,\delta)^2:= \kaptrain\left(2\kapapx\tailsf_2(p) + \frac{354 B^4 (\capac_p + \log \frac{6}{\delta}) }{\nnone}\right).\label{defn:epstil}
    \end{align}
    We first show that $(\ftil,\hat\bQ_{\rhat} \gtil)$ has small risk on the top block. 
    \begin{prop}[Guarantee for Double-Training, First-Phase]
    \label{thm:double_training_phase_1} Suppose $\sigcut \in [2\bsigst_{\rcut},\frac{2}{3e}\bsigst_1]$,  $\nnone \ge p \ge 2$, $\mu = B^2/\nnone$, and both $\nnone \ge B^2/\sigcut^2$ and $\nntwo \ge 722 \rcut^2\nnone^9\log(24p/\delta)$. Further, suppose
    \begin{align}
    \epstil(p,\nnone,\delta)^2 \le \sigcut^2/(64\rcut^2). \label{eq:epstil_small}
    \end{align}
    Then,  with probability at least $1 - \frac{2}{3}\delta$, we have 
    \begin{align*}
    \Riskr[\rhat](\ftil,\hat \bQ_{\rhat} \gtil;\cdone) \le  3000\rcut^2\epstil(p,\nnone,\delta)^2.
    \end{align*}
    Moreover, on this same event, both $\sup_{x,y}|\langle \ftil(x), \bhatQ_{\rhat} \gtil(y) \rangle| \le \sqrt{2n_1}B^2$ and  $\Espec(\rhat,\sigcut,\rcut)$, defined in \Cref{defn:espec}, holds.  
    \end{prop}

    \newcommand{\ftilbal}{\tilde{f}_{\mathrm{bal}}}
    \newcommand{\gtilbal}{\tilde{g}_{\mathrm{bal}}}

    \subsubsection{Proof overview}\label{sssec:doubtrain_proof_overview}

    Our first step is to verify the performance of the overparametrized $(\tilde f, \tilde g)$ on the nominal distribution $\disone$. For convenience, we upper bound a slightly augmented quantity which absorbs errors from regularizing the balancing covariances. 
    
    \begin{lem}\label{lem:tilde_funs_erm} Recall $\epstil(\cdot)$ defined in  \Cref{defn:epstil}. With probability at least $1 - \frac{1}{3}\delta$, it holds that  $\Risk(\tilde f,\tilde g;\disone) + \frac{2B^4}{\nnone} \le \epstil(p,\nnone,\delta)^2$.
    \end{lem}
    The above lemma is a direct consequence of the last line of \Cref{lem:single_erm}.


    Our next goal is to find a good rank-$\rhat$ projection of the functions $(\ftil,\gtil)$ which enjoys good performance on $\cdone$. This projection is best computed in a coordinate system in which $\ftil,\gtil$ are balanced in the sense of \Cref{defn:balanced_matrix}: that is, under a transformation $\bT$ such that $f = \bT^{-\top}\ftil$ and $g = \bT \gtil$, it holds that $\Exp_{\cdxone}[ff^\top] = \Exp_{\cdyone}[gg^\top]$. To compute this transformation, we first introduce sample and population covariance matrices. 
    \begin{defn}[Covariance Matrices] Let  $\{(x_{2,i},x_{2,i})\}_{i=1}^{\nntwo} \iidsim \cdone$, we define the population covariance matrices $\bSigma_{\ftil} = \Exp_{\cdxone}[\ftil\ftil^\top], \quad \bSigma_{\gtil} = \Exp_{\cdyone}[\gtil\gtil^\top]$, and their finite sample analogues using the $n_2$ samples: 
    \begin{align*}
    \hat{\bSigma}_{\ftil} &= \frac{1}{\nntwo} \sum_{i=1}^{\nntwo} \ftil(x_{2,i})\ftil(x_{2,i})^\top, \quad \hat{\bSigma}_{\gtil} = \frac{1}{\nntwo}\sum_{i=1}^{\nntwo} \gtil(x_{2,i})\gtil(x_{2,i})^\top. 
    \end{align*}
    \end{defn} 

    Balancing then finds a transformation $\bT$ for which $\bT^{-\top}\bSigma_{\ftil}\bT^{-1} = \bT\bSigma_{\gtil}\bT^{\top}$. It is challenging to establish a lower bound on $\lambda_{\min}(\bSigma_{\ftil})$ and $\lambda_{\min}(\bSigma_{\gtil})$, say when $\Sigst$ has rapid spectral decay. The matter only becomes worse when solving for $\bT$ using the finite sample covariance matrices $\hat{\bSigma}_{\ftil}$ and $\hat{\bSigma}_{\gtil}$. As a consequence, we instead consider regularized covariance matrices, defined as follows:
    \begin{defn}[Regularized Covariance Matrices] Let $\mu > 0$.  Define 
    \begin{align*}
    \Sigfmu = \bSigma_{\tilde f} + \mu \eye_p \quad \Siggmu := \bSigma_{\tilde g} + \mu \eye_p  \quad \hatSigfmu = \hat\bSigma_{\tilde f} + \mu \eye_p \quad \hatSiggmu = \hat\bSigma_{\tilde g} + \mu \eye_p.
    \end{align*}
    \end{defn}
    Leveraging standard finite sample concentration inequality of matrices (see \Cref{lem:matrix_Hoeffding} in the appendix), we ensure that the empirical and population covariance matrices concentrate. 
    \begin{lem}\label{lem:cov_conc} Let $\{(x_{2,i},y_{2,i})\}_{i=1}^{\nntwo} \iidsim \cdone$, and define the following empirical and population  covariance operators. 
    Then, with probability at least $1 - \frac{1}{3}\delta$, we have
    \begin{align*}
    \max\left\{ \|\Siggmu - \hatSiggmu\|_{\op}, \|\Sigfmu - \hatSigfmu\|_{\op}\right\} \le \epsSig(\nntwo,\delta) :=  B^2 \sqrt{\frac{2\log(24p/\delta)}{\nntwo}}.
    \end{align*}
    Moreover, for any $\bSigma \in \{\Siggmu, \hatSiggmu,\Sigfmu,\hatSigfmu\}$, we have $\mu \eye_p \preceq \bSigma \preceq B^2 + \mu \eye_p$.
    \end{lem}
    The above bound is proved for the non-regularized covariances, and follows by adding and subtracting $\mu \eye_p$. 
    The remainder of the proof has three components, each of which we give its own subsection below.
    \begin{itemize}
        \item[(a)] We first show that the regularized covariance matrices can be thought of as \emph{unregularized} covariance matrices corresponding to convolving the embeddings $(\ftil,\gtil)$ with isotropic noise. We argue that the excess risk of these noisy embeddings, denoted by $(\ftil_{\mu},\gtil_{\mu})$, is $\BigOh{\mu}$, and always upper bounds the risk of the noise-free embeddings. Hence, we can analyze balancing and projecting the noisy-embeddings as a proxy for the noise-free ones.  
        \item[(b)] We then analyze the performance of a balanced projection of  the embeddings $(\ftil,\gtil)$, and that of the projections of noisy embeddings $(\ftil_{\mu},\gtil_{\mu})$.  
        \item[(c)] We analyze the empirical balancing operator obtained via samples, and conclude the proof of \Cref{thm:double_training_phase_1} by combining the above results. 
    \end{itemize}

    \subsubsection{Interpreting regularization as convolution with noise}

    In this part of the proof, we illustrate how the \emph{regularized} covariance matrices of $(\ftil,\gtil)$ correspond to \emph{unregularized} covariance matrices obtained by convolving $(\ftil,\gtil)$ with noise. Let  $\cK_p := \{-1,1\}^p$  denote the $p$-dimensional (boolean, centered) hypercube. We can augment $\cdxone$ and $\cdyone$ to form distributions $\cdbarxone$ and $\cdbaryone$ over $\cY \times \cK_p$ and $\yspace \times \cK_p$, where 
    \begin{equation}\label{eq:bar_dists}
    \begin{aligned}
    \bar{x} &= (x,\btilx)\sim \cdbarxone \distequals x \sim \cdxone ~ \perp ~ \btilx \sim \Unif[\cK_p]\\
    \bar{y} &= (y,\btily) \sim \cdbaryone \distequals y \sim \cdyone~ \perp ~\btily\sim \Unif[\cK_p].
    \end{aligned}
    \end{equation}
    On these augmented distributions, we define
    \begin{align*}
    \ftil_{\mu}(x,\btilx) := \ftil(x) + \sqrt{\mu} \btilx, \quad \gtil_{\mu}(y,\btily) = \gtil(y) + \sqrt{\mu} \btily.
    \end{align*}
    We can readily verify that
    \begin{align*}
    \Sigfmu &:= \Exp_{\cdbarxone}[\ftil_{\mu}\ftil_{\mu}^\top] = \bSigma_{\ftil} + \mu \eye_p, \quad 
    \Siggmu := \Exp_{\cdbaryone}[\gtil_{\mu}\gtil_{\mu}^\top] = \bSigma_{\gtil} + \mu \eye_p.
    \end{align*}

    Two other observations are useful. In both, let $\cdbarone := \cdbarxone \otimes \cdbaryone$ (by analogy to $\cdone$), so that $\Risk(\ftil_\mu,\gtil_\mu;\cdbarone) = \Exp[(\langle \ftil_\mu(\bar x), \gtil_\mu(\bar{y}) \rangle - \hst(x,y))^2]$. Then, the following bounds the excess risk of the regularized functions $\ftil_{\mu},\gtil_{\mu}$ in terms of that of $\ftil,\gtil$:
    \begin{lem}\label{lem:ftil_mu_risk_ub} The following holds for any $B$-bounded $\ftil,\gtil$ and associated $\ftil_{\mu},\gtil_{\mu}$:
    \begin{align*}
    \Risk(\ftil_\mu,\gtil_\mu;\cdbarone) \le \epstil_{\mu}^2:= p\mu^2 + \mu B^2 + \Risk(\ftil,\gtil;\cdone).
    \end{align*}
    In particular, if $\nnone \ge p$, then for $\mu \le B^2/\nnone$, the functions $(\ftil,\gtil)$ as in \Cref{lem:tilde_funs_erm} satisfy
    \begin{align*}
    \Risk(\ftil_\mu,\gtil_\mu;\cdbarone) \le \epstil_{\mu}^2 \le \epstil(p,n_1,\delta)^2,
    \end{align*}
    with probability at least $1-\delta/3$. 
    \end{lem}

    \begin{proof}[Proof of \Cref{lem:ftil_mu_risk_ub}]
    Using independence of $x,y,\btilx,\btily$ under $\cdbarone$, and and $\Exp[\btilx\btilx^\top] = \Exp[\btily\btily^\top] = \eye_p$, we have 
    \begin{align*}
    &\Exp_{\cdbarone}[(\langle \ftil_\mu(\bar x), \gtil_\mu(\bar{y}) \rangle - \hst(x,y))^2] \\
    &= \Exp_{\cdbarone}[(\mu \langle \btilx,\btily\rangle + \sqrt{\mu}\langle \btilx, \gtil(y) \rangle +  \sqrt{\mu} \langle \ftil(x), \btily\rangle +  \langle\ftil(x), \gtil(y) \rangle - \hst(x,y))^2]\\
    &= \mu^2 \Exp_{\cdbarone}[\langle \btilx,\btily\rangle^2] + \mu\Exp_{\cdbarone}[\langle \btilx,\gtil(y)\rangle^2 + \langle \btily,\ftil(x)\rangle^2] +  \Exp_{\cdone}[(\langle \ftil(x), \gtil(y) \rangle - \hst(x,y))^2]\\
    &=  \trace(\eye_p)\mu^2  + \mu\Exp_{\cdone}[\|f(x)\|^2 + \|g(y)\|^2] +  \Exp_{\cdone}[(\langle \ftil(x), \gtil(y) \rangle - \hst(x,y))^2] \\
    &=  p\mu^2  + \mu\Exp_{\cdone}[\|f(x)\|^2 + \|g(y)\|^2] +  \Exp_{\cdone}[(\langle \ftil(x), \gtil(y) \rangle - \hst(x,y))^2] \\
    &\le p\mu^2 + \mu B^2 +\underbrace{ \Exp[(\langle \ftil(x), \gtil(y) \rangle - \hst(x,y))^2]}_{\Risk(\ftil,\gtil;\cdone)}. 
    \end{align*}
    The second statement of the lemma follows from selecting $\mu \leq  B^2/n_1$, using the assumption that $n_1 \ge p$, and invoking \Cref{lem:tilde_funs_erm}. 
    \end{proof}
    The second fact shows that weighted inner products involving the regularized functions are always worse predictors than the corresponding unregularized functions: 
    \begin{lem}\label{lem:ftil_mu_jensen} The following inequality holds for any  $\ftil,\gtil$ and associated $\ftil_{\mu},\gtil_{\mu}$, matrix $\bA \in \R^{p \times p}$, and $h: \cX \times \cY \to \R$:
    \begin{align*}
    \Exp_{\cdbarone}[(\langle \ftil_{\mu}(\bar x),\bA \gtil_{\mu}(\bar y) \rangle - h(x,y))^2]  \ge \Exp_{\cdone}[(\langle \ftil(x),\bA \gtil(y) \rangle - h(x,y))^2].
    \end{align*}
    \end{lem}
    The lemma is a direct consequence of Jensen's inequality, and the fact that $\Exp_{\cdbarone}[\langle \ftil_{\mu}(\bar x),\bA \gtil_{\mu}(\bar y) \rangle\mid x,y] = \langle \ftil(x),\bA\gtil(y) \rangle$ for any $\bA \in \R^{p \times p}$.

    \subsubsection{Analysis under an exact balanced projection}

    We now analyze the performance of an idealized balanced projection of $(\ftil,\gtil)$, and as a corollary, state a guarantee for deviations from this idealized projection. We accomplish this by analyzing the performance of the projections of noisy embeddings $(\ftil_{\mu},\gtil_{\mu})$ as a proxy, and then applying \Cref{lem:ftil_mu_jensen} to return to the noise-free embeddings.

    It is useful for us to formalize balancing as a general operation on matrices.  
    \begin{restatable}[Balancing Operator]{defn}{balop}\label{defn:balop}
        Let $\bX,\bY \in \pd{p}$. We define the balancing operator 
        \begin{align*}
        \Psibal(\bY; \bX) := \bX^{\half}(\bX^{\half}\bY \bX^{\half})^{-\half}\bX^{\half} \in \pd{p}.
        \end{align*} 
    \end{restatable} 
    It is shown in \Cref{lem:bal_properties} that $\bW = \Psibal(\bY;\bX)$ is the unique positive definite operator satisfying $\bX = \bW \bY \bW$. As a consequence, given $(\ftilmu,\gtilmu)$, the functions $(\ftilmubal,\gtilmubal)$ defined as
    \begin{align*}
    \ftilmubal &= \Wmubal^{-\half} \ftilmu, \quad \gtilmubal = \Wmubal^{\half}\gtilmu, \quad \Wmubal := \Psibal(\Siggmu;\Sigfmu)
    \end{align*}
    satisfy (using $\Wmubal = \Wmubal^\top$)
    \begin{align}
    \Exp_{\cdbarxone}[\ftilmubal (\ftilmubal)^{\top}] = \underbrace{\Wmubal^{-\half} \Sigfmu\Wmubal^{-\half}  = \Wmubal^{\half}\Siggmu\Wmubal^{\half}}_{:=\Sigmubal} = \Exp_{\cdbaryone}[\gtilmubal (\gtilmubal)^{\top}], \label{eq:tilde_balance_demo}
    \end{align}
    as well as trivially $\langle \ftilmubal, \gtilmubal \rangle \equiv \langle \ftilmu,\gtilmu\rangle$.
    That is, the transformation
    \begin{align*}
    (\ftilmu,\gtilmu) \mapsto (\Wmubal^{-\half}\ftilmu,\Wmubal^{\half}\gtilmu) 
    \end{align*}
    balances $(\ftilmu,\gtilmu)$. We now introduce an operator expressing the covariance matrix of the balanced functions (in our case, $\Sigmubal$ above).

    \begin{restatable}[Balanced Covariance]{defn}{balcov}\label{defn:balanced_cov} Given $\bX,\bY \in \pd{p}$, we define
    \begin{align*}
    \covbal(\bX,\bY) = \Psibal(\bY;\bX)^{\frac{1}{2}}\cdot\bY \cdot\Psibal(\bY;\bX)^{\frac{1}{2}}.
    \end{align*}
    \end{restatable}
    We remark that $\covbal(\bX,\bY) = \covbal(\bY,\bX)$, as illustrated in \Cref{eq:tilde_balance_demo}. In particular,
    \begin{align*}
    \Sigmubal = \covbal(\Siggmu,\Sigfmu) = \covbal(\Sigfmu,\Siggmu). 
    \end{align*}

    We can now define our main object of interest:  the operator which performs a singular value decomposition of the factorization $\ftilmu,\gtilmu$ in the coordinate system in which they are balanced.
    \begin{restatable}[Balancing Projection]{defn}{projbaldefn} Given $\bX,\bY \in \pd{p}$, for any $r\in[p]$, we define 
    \begin{align*}
    \balproj(r,\bX,\bY) := \bW^{-\half} \bP_{r} \bW^{\half},
    \end{align*}
    where $\bW = \Psibal(\bY;\bX)$, and $\bP_{\rhat}$ is the orthogonal projection onto the top-$r$ eigenvectors of $ \covbal(\bY,\bX) = \bW^{\half}\bY \bW^{\half}$. We say that $\bQ_{\rhat} = \balproj(r,\bX,\bY)$ is \emph{unique} if the aforementioned projection $\bP_{\rhat}$ is unique, that is, if $\sigma_{r}(\covbal(\bY,\bX)) > \sigma_{r+1}(\covbal(\bY,\bX))$. Note that when $r=p$, this projection is trivially  unique. 
    \end{restatable}

    In particular, suppose we consider $\bQ_{\rhat} := \balproj(\rhat,\Sigfmu,\Siggmu)$. This performs a rank-$\rhat$ projection in the coordinates in which $\ftilmu,\gtilmu$ are balanced, and transforming $\langle \ftilmu,\gtilmu\rangle$ to $\langle \ftilmu, \bQ_{\rhat} \gtilmu \rangle$ is equivalent to computing a rank-$\rhat$ SVD of the matrices.
    Thus, the error between $\langle \ftilmu, \bQ_{\rhat} \gtilmu \rangle$ and $\langle \fst_{\rhat}, \gst_{\rhat}\rangle$ can be analyzed in terms of the error between the rank-$\rhat$ SVD approximations of two matrices which are close by.  We use this insight to prove a perturbation bound, which we describe below.

    The following lemma establishes three useful bounds: (a) an $\ell_2$-deviation bound between the spectrum of $\Sigmubal$ and the spectrum of $\Sigst$; (b) a suboptimality guarantee for applying the \emph{exact} balanced projection $\bQ_{\rhat} = \balproj(r,\Sigfmu,\Siggmu)$  to $(\ftilmu,\gtilmu)$, where $(\ftilmu,\gtilmu)$ are the noise-convolved functions defined in the previous section; and (c) a perturbation inequality for applying an approximation $\bQ'$ of $\bQ_{\rhat}$ to $(\ftilmu,\gtilmu)$, and the subsequent guarantee when  applying this projection to the original (non-noisy) functions $(\ftil,\gtil)$.

    \newcommand{\epstilmu}{\epstil_{\mu}}
    \begin{lem}[Accuracy of Balancing Projections]\label{lem:acc_balanced projection} Recall the definition of $\epstilmu^2$ from \Cref{lem:ftil_mu_risk_ub}. Then,
    \begin{itemize}
    	\item[(a)] It holds that $\sum_{i \ge 1}(\sigma_i(\Sigmubal) - \bsigst_i)^2 \le \epstilmu^2$.
    	\item[(b)] Given a given $\rhat \in \N$ for which $\bsigst_{\rhat} > 0$,  define $\bdelst_{\rhat} := 1 - \frac{\bsigst_{\rhat+1}}{\bsigst_{\rhat}}$. If $\epstilmu \le \eta \bsigst_{\rhat}\bdelst_{\rhat}$ for a given $\eta \in [0,1)$, then
            \begin{align*}
            \Riskr[\rhat](\ftilmu, \bQ_{\rhat} \cdot \gtilmu;\cdbarone) \le  \frac{81\epstilmu^2}{(\bdelst_{\rhat}(1-\eta))^2},
            \end{align*}
    	where we define $\bQ_{\rhat} = \balproj(\rhat,\Sigfmu,\Siggmu)$.
    	\item[(c)] Under the assumptions of (b), if $\bhatQ \in \R^{p \times p}$ is any other matrix, then, assuming $\mu \le B^2/p$, 
    	\begin{align*}
    	\Riskr[\rhat](\ftil,\bhatQ \cdot  \gtil;\cdone )  &\le \Riskr[\rhat](\ftilmu,\bhatQ \cdot  \gtilmu;\cdbarone ) \le  8B^2\|\bhatQ - \bQ_{\rhat}\|_{\op} + \frac{162\epstilmu^2}{(\bdelst_{\rhat}(1-\eta))^2}.
    	\end{align*}
        \end{itemize}
    \end{lem}
\begin{proof}[Proof of \Cref{lem:acc_balanced projection}] The functions $\ftilmubal,\gtilmubal$ are balanced under $\cdbarone$: $\Exp_{\cdbarxone}[\ftilmubal(\ftilmubal)^\top] = \Exp_{\cdbaryone}[\gtilmubal(\gtilmubal)^\top] = \Sigmubal$. Moreover, by \Cref{lem:ftil_mu_risk_ub}, 
\begin{align*}
\Risk(\ftilmubal,\gtilmubal;\cdbarone) = \Risk(\ftilmu,\gtilmu;\cdbarone) \le \epstil_{\mu}^2.
\end{align*} 
Further, we have
\begin{align*}
\langle \ftilmu, \bQ_{\rhat} \gtilmu \rangle = \langle \ftilmubal, \bP_{\rhat} \gtilmubal \rangle, 
\end{align*}
where $\bP_{\rhat}$ is the projection onto the top $\rhat$ eigenvectors of $\Sigmubal$. Hence, we can invoke\footnote{We note that while \Cref{thm:svd_general_thing} is stated in terms of the non-augmented distribution $\cdone$, it holds for $\cdbarone$ as well, as the augmented distribution preserves the covariance and balancing of the ground truth embeddings. }
 \Cref{thm:svd_general_thing} to find both (a) $\sum_{i \ge 1}(\sigma_i(\Sigmubal) - \bsigst_i)^2 \le \epstilmu^2$ and (b) $\Exp_{\cdbarone}[(\langle\ftilmu, \bQ_{\rhat} \cdot \gtilmu\rangle -  \langle \fst_{\rhat},\gst_{\rhat}\rangle)^2]  \le  \frac{81\epstilmu^2}{(\bdelst_{\rhat}(1-\eta))^2}$. For part (c), the first inequality is a special case of \Cref{lem:ftil_mu_jensen}. Moreover, 
 \begin{align*}
\Exp_{\cdbarone}[(\langle\ftilmu, \bhatQ \cdot \gtilmu\rangle -  \langle \fst_{\rhat},\gst_{\rhat}\rangle)^2]  &= \Exp_{\cdbarone}[(\langle\ftilmu, (\bhatQ - \bQ_{\rhat}) \cdot \gtilmu\rangle + \langle\ftilmu, \bQ_{\rhat} \cdot \gtilmu\rangle -  \langle \fst_{\rhat},\gst_{\rhat}\rangle)^2]\\
&\le 2\Exp_{\cdbarone}[\langle\ftilmu, (\bhatQ - \bQ_{\rhat}) \cdot \gtilmu\rangle^2] + 2\Exp_{\cdbarone}(\langle\ftilmu, \bQ_{\rhat} \cdot \gtilmu\rangle -  \langle \fst_{\rhat},\gst_{\rhat}\rangle)^2]. 
 \end{align*}
 As the second term above is controlled by part (b) of the lemma, it remains to bound $\Exp_{\cdbarone}[\langle\ftilmu, (\bhatQ - \bQ_{\rhat}) \cdot \gtilmu\rangle^2]$. Using independence of $x,y,\btilx,\btily$ under $\cdbarone$, and and $\Exp[\btilx\btilx^\top] = \Exp[\btily\btily^\top] = \eye_p$,
 \begin{align*}
 &\Exp_{\cdbarone}[\langle\ftilmu, (\bhatQ - \bQ_{\rhat}) \cdot \gtilmu\rangle^2] \\
 &=\Exp_{\cdbarone}[\langle \ftil(x) + \sqrt{\mu}\btilx, (\bhatQ - \bQ_{\rhat}) (\gtil(y) + \sqrt{\mu}\btily)\rangle^2]\\
 &=  \trace(\Exp_{\cdbarone}[(\ftil(x)+\sqrt{\mu}\btilx)(f(x)+\sqrt{\mu}\btilx)^\top (\bhatQ - \bQ_{\rhat}) (g(x)+\sqrt{\mu}\btily)(\gtil(y)+\sqrt{\mu}\btily)^\top (\bhatQ - \bQ_{\rhat})^\top]\\
 &=  \trace((\Exp_{\cdxone}[\ftil(x)\ftil(x)^\top] + \mu \eye_p) (\bhatQ - \bQ_{\rhat}) (\Exp_{\cdyone}[\gtil(y)\gtil(y)^\top] + \mu \eye_p) (\bhatQ - \bQ_{\rhat})^\top]\\
 &= \trace(\Exp_{\cdxone}[\ftil(x)\ftil(x)^\top] (\bhatQ - \bQ_{\rhat})\Exp_{\cdyone}[\gtil(y)\gtil(y)^\top] (\bhatQ - \bQ_{\rhat}))\\
 &\quad+ \mu \trace((\Exp_{\cdxone}[\ftil(x)\ftil(x)^\top] + \Exp_{\cdyone}[\gtil(y)\gtil(y)^\top]) (\bhatQ - \bQ_{\rhat}) (\bhatQ - \bQ_{\rhat})^\top)\\
 &\quad+ \mu^2 \trace( (\bhatQ - \bQ_{\rhat}) (\bhatQ - \bQ_{\rhat})^\top). 
 \end{align*}
 Using $\trace(\Exp_{\cdxone}[\ftil(x)\ftil(x)^\top]) \vee \trace(\Exp_{\cdyone}[\gtil(y)\gtil(y)^\top])\le B^2 $ due to $\ftil \in \cF_p, \gtil \in \cG_p$ and \Cref{asm:function_apx}, (and using various standard trace inequalities), the above is atmost
 \begin{align*}
 &B^4 \|\bhatQ - \bQ_{\rhat}\|^2_{\op} + 2\mu^2B^2 \|\bhatQ - \bQ_{\rhat}\|^2_{\op} + \mu^4 \trace( (\bhatQ - \bQ_{\rhat}) (\bhatQ - \bQ_{\rhat})^\top)\\
 &\le (B^4 + 2 \mu B^2 + \mu^2 p)\|\bhatQ - \bQ_{\rhat}\|^2_{\op} \le 4B^2\|\bhatQ - \bQ_{\rhat}\|^2_{\op}
 \end{align*}
where the last inequality takes $\mu \le B^2/p$. \end{proof}

    \subsubsection{Analysis of empirical balancing operator}

    \begin{restatable}{defn}{seprankdef}\label{defn:seprank} Given $\bSigma \in \psd{p}$, $r_0 \in [p]$, $\sigma > 0$, the \emph{separated-rank} at $(r_0,\sigma)$ (if it exists) is  
    \begin{align}
    \seprank(r_0,\sigma;\bSigma) := \max\left\{r \in [r_0]: \sigma_{r}(\bSigma) \ge \sigma, \sigma_{r}(\bSigma) - \sigma_{r+1}(\bSigma) \ge \frac{\sigma_{r}(\bSigma)}{r_0}\right\}. 
    \end{align}
    We say the separated-rank is \emph{well-defined} if the above maximum exists.
    \end{restatable}
    
    We next provide the result on the perturbation of the balancing projections, whose proof is  deferred to \Cref{sec:proof_prop_balproj}. 

    \begin{restatable}[Perturbation of Balancing Projections]{prop}{propbalproj}\label{prop:balproj} Let $r_0 \in \N$, matrices  $\bX,\bX',\bY,\bY' \in \pd{p}$, and positive numbers $\sigma > 0$ and $(\sigbar_i)_{i \in [r_0+1]}$ satisfy the following conditions:
    \begin{itemize}
        \item[(a)] For any $\bA \in \{\bX,\bX',\bY,\bY' \}$, $\mu \eye_p \preceq \bA \preceq M \eye_p$. 
        \item[(b)] $\max\{\opnorm{\bX - \bX'}, \opnorm{\bY - \bY'} \} \le \Delta$, where $\Delta \le \frac{\mu}{32r_0} (\mu/M)^2$.
        \item[(c)] $\max_{i \in [r_0+1]}|\sigbar_{i} - \sigma_{i}(\bSigma)| \le \sigma/8r_0$, where $\bSigma = \covbal(\bX,\bY)$.
        \item[(d)] $\sigma\in [\max\{\mu,2 \sigbar_{\rhat_0}\}, \frac{2}{3e}\sigbar_{1}]$. 
        \end{itemize}
    Define $\bSigma' = \covbal(\bX',\bY')$,  $r = \seprank(r_0,\sigma; \bSigma')$, $\bQ = \balproj(r;\bX,\bY)$ and $\bQ' = \balproj(r;\bX',\bY')$. Then, $r$ is well defined, $\bQ$ and $\bQ'$  are unique, and the following bounds hold:
    \begin{align*}
    \|\bQ' - \bQ\|_{\op} &\le \frac{19r_0 (M/\mu)^{5/2} \Delta}{\mu}, \quad
    \max\{\|\bQ\|_{\op},\|\bQ'\|_{\op}\} \le \sqrt{M/\mu}.
    \end{align*}
    Moreover,
    $\sigbar_{r}\ge \frac{3}{4}\sigma$, $\sigbar_{r+1} \le 3\sigma$, and $\sigbar_{r} - \sigbar_{r+1} \ge \frac{\sigbar_{\rhat}}{3r_0} $.
    \end{restatable}

    \subsubsection{Concluding the proof of \Cref{thm:double_training_phase_1}}
   
    \begin{proof}[Proof of \Cref{thm:double_training_phase_1}] Throughout suppose that the high probability events of \Cref{lem:cov_conc} and \Cref{lem:ftil_mu_risk_ub} hold, with have a total failure probability of $2\delta/3$. We instantiate \Cref{prop:balproj} with 
    \begin{itemize}
        \item[(1)] $\bX = \Sigfmu$, $\bX' = \hatSigfmu$, $\bY = \Siggmu$, and $\bY' = \hatSiggmu$.
        \item[(2)] $r_0 \gets \rcut$, $\sigma \gets \sigcut$, $\rhat \gets \seprank(\rcut,\sigcut;\hatSigmubal)$, and  
        \begin{align*}
        \bQ' \gets \bhatQ_{\rhat} := \balproj\left(r; \hatSigfmu,\hatSiggmu\right), \quad \bQ \gets \bQ_{\rhat} :=  \balproj\left(r; \Sigfmu,\Siggmu\right).  
        \end{align*}
        \item[(3)] $\mu \gets B^2/\nnone$ and $M \gets 2B^2$. By assumption, $\nnone \ge p$, so $\mu = B^2/p$ satisfies the conditions of \Cref{lem:acc_balanced projection}.
        \item[(4)] On the event of \Cref{lem:cov_conc}, we have $\max\{\opnorm{\bX - \bX'}, \opnorm{\bY - \bY'} \} \le \Delta$ for  $ \Delta = \epsSig(\nntwo) =   B^2 \sqrt{2\frac{\log(24p/\delta)}{\nntwo}}$. This holds with probability at least $1 - \delta/3$.
        \item[(5)] $\barsig_i \gets \bsigst_i$, and $\sigma_i = \sigma_i(\Sigmubal) = \sigma_i(\covbal(\Sigfmu,\Siggmu))$. 
    \end{itemize}
    We now check that the conditions (a)-(d) of \Cref{prop:balproj} are met.
    \begin{itemize}
        \item[(a)] The PSD inequality holds by \Cref{lem:cov_conc} .
        \item[(b)] $\Delta \le \frac{\mu}{32r_0}(\mu/M)^2$ holds for $\nntwo \ge \rcut^2 2^{11}(\nnone)^6\log(24p/\delta)$, on the event of \Cref{lem:cov_conc}. 
        \item[(c)] By \Cref{lem:acc_balanced projection}(a), it is enough that $\epstilmu^2 \le \frac{\sigcut^2}{64\rcut^2}$. On the event of \Cref{lem:ftil_mu_risk_ub}, it is enough that $\epstil(p,n_1,\delta)^2 \le  \frac{\sigcut^2}{64\rcut^2}$. 
        \item[(d)] Substituting in $\mu = B^2/\nnone$, $\sigma \gets \sigcut$ and $\sigbar_i \gets \bsigst_i$ the condition $\sigma\in [\max\{\mu,2 \sigbar_{\rhat_0}\}, \frac{2}{3e}\sigbar_{1}]$ holds for $\nnone \ge B^2/\sigcut^2$ and $\sigcut \in [2\bsigst_{\rcut},\frac{2}{3e}\bsigst_1]$.
    \end{itemize}
    Note that the suffcient conditions in (b)-(d) are all guaranteed by  \Cref{thm:double_training_phase_1}.
    With the above substitutions, we achieve
    \begin{itemize}
        \item[(i)] $\bsigst_{\rhat+1} \le 3\sigcut$, $\bsigst_{\rhat} \ge 3\sigcut/4$, and $\bdelst_{\rhat} = \frac{\bsigst_{\rhat} - \bsigst_{\rhat+1}}{\bsigst_{\rhat}} \ge \frac{1}{3\rcut}$, and thus $\bdelst_{\rhat}\bsigst_{\rhat} \ge \sigcut/(4\rcut)$.
        \item[(ii)] The upper bound on $\|\bhatQ_{\rhat} - \bQ_{\rhat}\|_{\op}$ is given by 
        \begin{align*}
        \|\bhatQ_{\rhat} - \bQ_{\rhat}\|_{\op} &\le \frac{19\rcut (M/\mu)^{5/2} \Delta}{\mu} = 19\sqrt{2\log(24p/\delta)} \rcut \cdot \sqrt{\nnone^7/\nntwo} \le \frac{1}{\nnone}
        \end{align*}
        for $\nntwo \ge 722 \rcut^2\nnone^9\log(24p/\delta)$ (achieved under the proposition).
    \end{itemize}

    From \Cref{lem:acc_balanced projection} with $\eta = 1/8$, we have that as long as $\epstilmu \le  \sigcut/(16\rcut) \le \bsigst_{\rhat}\bdelst_{\rhat}/4 $,
    \begin{align*}
    \Riskr[\rhat](\ftil,\bhatQ_{\rhat}  \cdot  \gtil;\cdone )  &\le   4B^2\|\bhatQ_{\rhat}  - \bQ_{\rhat}\|_{\op} + \frac{324\epstilmu^2}{(\bdelst_{\rhat})^2}\\
    &\le   \frac{4B^2}{\nnone} + 2898\rcut^2\epstilmu^2.
    \end{align*}
    where the last line follows by invoking items $(i)$ and $(ii)$ above. On the event of \Cref{lem:ftil_mu_risk_ub}, we may upper bound $\epstilmu^2$ by $\epstil(p,n_1,\delta)^2$, as in \Cref{lem:tilde_funs_erm}, giving 
    \begin{align*}
    \Riskr[\rhat](\ftil,\bhatQ_{\rhat}  \cdot  \gtil;\cdone ) \le \frac{4B^2}{\nnone} + 2898\rcut^2\epstil(p,n_1,\delta)^2\le 3000\rcut^2\epstil(p,n_1,\delta)^2
    \end{align*}

    We conclude by checking the two statements in the last line of \Cref{thm:double_training_phase_1}. To show the first, we note that, due to \Cref{prop:balproj}, we find $\|\bhatQ_{\rhat}\|_{\op} \le \sqrt{M/\mu} = \sqrt{2n_1}$. Using  \Cref{asm:function_apx}  and the fact that $\ftil \in \cF_p$ and $\gtil \in \cG_p$ concludes that  
    \begin{align*}
    |\langle \tilde f(x), \bhatQ_{\rhat} \tilde g(y) \rangle| \le B^2 \cdot \|\bhatQ_{\rhat}\|_{\op} \le \sqrt{2n_1}B^2. 
    \end{align*} 
    To show the second, we note that, due to \Cref{prop:balproj}, $\bsigst_{\rhat+1} \le 3\sigcut$, from which the inequalities $\tailsf_2(\rhat) \le \tailsf_2(\rcut) + 9\sigcut^2 \rcut$ and $\tailsf_1(\rhat)^2 \le 18 \rcut^2 \sigcut^2 + 2\tailsf_1(\rcut)^2$ are straightforward to verify. Together with \Cref{prop:balproj}, these verify that the event  $\Espec(r,\sigcut,\rcut)$, defined in \Cref{defn:espec}, holds.  
    \end{proof}

\subsection{Analysis of the second stage of double-stage ERM}\label{sec:double_train_second_proof}

The following lemma, which is established in \Cref{sec:double_training_phase_two_proof}, handles the error on the second phase of double-stage ERM in terms of the first. Recall that we choose
\begin{align*}
(\fhatdt,\fhatdt) &\in \argmin_{(f,g) \in \cF_{\rhat} \times \cG_{\rhat}} \hat{L}_{(3)}(f,g) + \lambda \hat{L}_{(4)}(f,g)\\
\hat{L}_{(3)}(f,g) &= \frac{1}{n_3}\sum_{i=1}^{n_3} (\langle f(x_{3,i}), g(y_{3,i}) \rangle - z_{3,i})^2 \\
\hat{L}_{(4)}(f,g) &=  \frac{1}{n_4}\sum_{i=1}^{n_4} (\langle f(x_{4,i}), g(y_{4,i}) \rangle - \langle \ftil(x_{4,i}), \bhatQ_{\rhat} \cdot \gtil(y_{4,i}) \rangle)^2.
\end{align*}

\begin{lem}\label{lem:double_training_phase_two} Suppose it holds that $\|\bhatQ_{\rhat}\|_{\op} \le \sqrt{2n_1}$, as in the proof  of \Cref{thm:double_training_phase_1}. Then, with probability at least $1-\delta/3$ over the samples collected in Line 5   of \Cref{alg:main_alg}, 
\begin{align*}
&\Risk(\fhatdt,\fhatdt;\Dtrain) + \frac{\lambda}{2} \Risk_{[r]}(\fhatdt,\fhatdt;\cdone) \\
&\quad\le 2\kapapx \tailsf_2(\rhat) + 3\lambda \Risk_{[r]}(\ftil, \bhatQ_{\rhat} \cdot \gtil;\cdone) + 352\left(1+\frac{\lambda n_1 n_3}{n_4}\right)\frac{B^4(\capac_{\rhat} + \log(12/\delta))}{n_3}.
\end{align*}
\end{lem}

We can now conclude the proof of our main theorem for double-stage ERM as follows. 

\begin{proof}[Proof of \Cref{thm:double_training_analysis}]
First, we bound the regularized risk $\Risk(\fhatdt,\fhatdt;\Dtrain) + \lambda \Risk_{[r]}(\fhatdt,\fhatdt;\cdone) $.
Using \Cref{thm:double_training_phase_1} in \Cref{lem:double_training_phase_two}, we have 
\begin{align*}
&\Risk(\fhatdt,\fhatdt;\Dtrain) + \lambda \Risk_{[r]}(\fhatdt,\fhatdt;\cdone) \\
&\lesssim  \kapapx(\tailsf_2(\rhat) + \lambda \kaptrain \rcut^2 \tailsf_2(p)) + \left(1+\frac{\lambda \kaptrain n_1 n_3}{n_4}\right)\frac{B^4(\capac_{\rhat} + \log(1/\delta))}{n_3} + \frac{\lambda\kaptrain\rcut^2 B^4 (\capac_p + \log(1/\delta))}{n_1}\\
&\lesssimst  \tailsf_2(\rcut) + \rcut \sigcut^2 + \lambda \rcut^2 \tailsf_2(p) + \left(1+\frac{\lambda n_1 n_3}{n_4}\right)\frac{B^4(\capac_{\rhat} + \log(1/\delta))}{n_3} + \frac{\lambda \rcut^2 B^4 (\capac_p + \log(1/\delta))}{n_1}\\
&\le  \tailsf_2(\rhat) + \rcut \sigcut^2 + \lambda \rcut^2 \tailsf_2(p) + \left(1+\frac{\lambda n_1 n_3}{n_4}\right)\frac{B^4(\capac_{\rcut} + \log(1/\delta))}{n_3} + \frac{\lambda\rcut^2 B^4 (\capac_p + \log(1/\delta))}{n_1},
\end{align*}
where in the second to last line, we use $\lesssimst$ to suprress polynomials in problem dependend constants, and in the last line, we use the assumptions that $K \mapsto \capac_K$ is non-decreasing (see \Cref{asm:function_apx}).  For our choice of $n_4 \ge \lambda n_1 n_3$, the above simplifies further to 
\begin{align*}
&\Risk(\fhatdt,\fhatdt;\Dtrain) + \lambda \Risk_{[r]}(\fhatdt,\fhatdt;\cdone) \\
&\quad \lesssimst \tailsf_2(\rcut) +\rcut \sigcut^2 + \lambda \rcut^2 \tailsf_2(p) + \frac{B^4(\capac_{\rcut} + \log(1/\delta))}{n_3} + \frac{\lambda \rcut^2 B^4 (\capac_p + \log(1/\delta))}{n_1}. 
\end{align*}

That the good spectral event $\Espec$ holds also follows from \Cref{thm:double_training_phase_1}. Lastly, we gather the necessary conditions in order for the conclusion of \Cref{thm:double_training_phase_1} to hold, $\mu = B^2/\nnone$, $\nnone \ge \max\{p,B^2/\sigcut^2\}$ $\nntwo \ge 722 \rcut^2 \nnone^9\log(24p/\delta)$, and finally, we require \Cref{eq:epstil_small}. Stated succinctly, this last condition stipulates that for some constant $C \lesssimst 1$,
    \begin{align*}
    \tailsf_2(p) + \frac{B^4 (\capac_p + \log \frac{1}{\delta}) }{\nnone} \le \frac{\sigcut^2}{C\rcut^2}. 
    \end{align*}
    Doubling $C$ by a factor 2, it is enough that $\tailsf_2(p) \le \frac{\sigcut^2}{C\rcut^2}$ and $\nnone \ge \frac{\rcut^2 B^4 (\capac_p + \log \frac{1}{\delta}) }{\sigcut^2}$.  The bound follows.
\end{proof}

\section{Learning Theory and Proofs in \Cref{sec:analysis_training_alg}}\label{app:learning_theory}


In this section, we review some fundamental while important results from learning theory, and related proofs   in \Cref{sec:analysis_training_alg}. 

\subsection{Concentration inequalities} 
We begin with Bernstein's inequality (see e.g.,  \cite[Chapter 2]{boucheron2005theory}). 

\begin{lem}[Bernstein Inequality]\label{lem:Bernstein} Let $Z_1,\dots,Z_n\in\R$ be i.i.d. random variables with $|Z_i | \le M$ and $\mathrm{Var}[Z_i] \le \sigma^2$. Then, with probability at least $1 - \delta$,
\begin{align*}
\left|\frac{1}{n}\sum_{i=1}^n Z_i - \Exp[Z_i]\right| \le \sqrt{\frac{2\sigma^2 \log(1/\delta)}{n}} + \frac{M\log(1/\delta)}{3n}.
\end{align*}
\end{lem}

The following is a simplification of \cite[Corollary 4.2]{mackey2014matrix}. 
\begin{lem}[Matrix Hoeffding]\label{lem:matrix_Hoeffding} Let $\bY_1,\dots,\bY_n \in \R^{d\times d}$ be i.i.d. symmetric matrices with $\Exp[\bY_i] = 0$ and $\|\bY\|^2_{\op} \le M$. Then, with probability at least $1 - \delta$,
\begin{align*}
\left\|\frac{1}{n}\sum_{i=1}^n\bY_i\right\|_{\op} \le M\sqrt{\frac{2\log(2d/\delta)}{n}}.
\end{align*} 
\end{lem}

\subsection{Learning with finite function classes}

\newcommand{\phist}{\phi^\star}
 \begin{lem}\label{lem:finite_class_fast_rate} Let $\Phi$ be a finite class of functions $\phi:\scrW \to \R^k$, and let $\phi_\star(w)$ be a nominal function, possibly not in $\Phi$. Let $M > 0$ be a constant such that $\sup_{w\in\scrW}\max_{\phi \in \Phi}\|(\phi - \phist)(w)\|_2 \le M$, and let $\cD$ be a distribution over pairs $(w,\bz) \in \cW \times \R^k$ such that $\|\bz - \phist(w)\|_2 \le M$ and $\Exp[\bz\mid w] = \phist(w)$. Define $R(\phi) := \Exp_{w \sim \cD}[\|\phi(w) - \phist(w)\|^2]$, $\hat L_n(\phi) := \frac{1}{n}\sum_{i=1}^n \|\phi(w_i) - \phist(w_i)\|^2$, 
 and set $\hat R_n(\phi) = \hat L_n(\phi) - \hat L_n(\phi_\star)$. 
 Then, for any $\delta \in (0,1)$, with probability at least $1 - \delta$: 
 \begin{itemize}
    \item The following guarantee holds simultaneously for all $\phi \in \Phi$ and all $\alpha > 0$:
    \begin{align*}
    |R(\phi) - \hat{R}_n(\phi)| \le \frac{\alpha R(\phi)}{2} + \left(\frac{9}{\alpha}+1\right)\cdot\frac{M^2 \log(2|\Phi|/\delta)}{n}.
    \end{align*}
    \item All empirical risk minimizers $\hat \phi \in \argmin_{\phi \in \Phi} \hat L_n(\phi) = \argmin_{\phi \in \Phi} \hat R_n(\phi)$  satisfy
    \begin{align*}
    R(\hat \phi) \le 2\inf_{\phi' \in \Phi} \Exp_{\cD}[(\phi'(w) - \phist(w))^2] + \frac{78 M^2 \log(2|\Phi|/\delta)}{n}.
    \end{align*}
 \end{itemize}

    \end{lem}
    \begin{proof} Throughout, all expectations are taken under spaces from $\cD$. We expand
    \begin{align*}
    \hat{R}_n(\phi)= \frac{1}{n}\sum_{i=1}^n Z_{i}(\phi), \quad Z_{i}(\phi) := \|\phi(w_i) - \bz_i\|^2 - \|\phist(w_i)-\bz_i\|^2. 
    \end{align*}
     By expanding
    \begin{align*}
    Z_i(\phi) := \|(\phi - \phist)(w_i)\|^2 + 2\langle(\phi - \phist)(w_i), (\phist(w_i) - \bz_i)\rangle,
    \end{align*}
    we see that
    \begin{align*}
    \forall \phi, \quad \Exp[Z_i(\phi)] = R(\phi), \quad \text{w.p. $1$,} \quad |Z_i(\phi)| \le 3M^2.
    \end{align*}
    Furthermore, for all $\phi$, 
    \begin{align*} 
    \Exp[Z_i(\phi)^2] &= \Exp[(\|(\phi - \phist)(w_i)\|^2 + 2\langle(\phi - \phist)(w_i), (\phist(w_i) - \bz_i)\rangle)^2]\\
    &\le \Exp[(\|(\phi - \phist)(w_i)\|^2 + 2\|(\phi - \phist)(w_i)\|\|\phist(w_i) - \bz_i\| )^2]\\
    &\le \Exp[(3M\|(\phi - \phist)(w_i)\|)^2] = 9M^2R(\phi). 
    \end{align*}
    Thus, by Bernstein's inequality (\Cref{lem:Bernstein}) and a union bound over all $\phi \in \Phi$, the following holds with probability at least $1 - \delta$:
    \begin{align*}
    \forall \phi \in \Phi, \quad |R(\phi) - \hat{R}_n(\phi)| \le \sqrt{\frac{18M^2 R(\phi) \log(2|\Phi|/\delta)}{n}} + \frac{M^2 \log(2|\Phi|/\delta)}{n}.
    \end{align*}
    Therefore, by AM-GM inequality, the following holds for all fixed $\alpha > 0$:
    \begin{align*}
    \forall \phi \in \Phi,~~~|R(\phi) - \hat{R}_n(\phi) - \hat{R}_n(\phi_{\star})| \le \frac{\alpha R(\phi)}{2} + \left(\frac{9}{\alpha}+1\right)\cdot\frac{M^2 \log(2|\Phi|/\delta)}{n}.
    \end{align*}
    This establishes the first statement of the lemma.

    To prove the second statement, let $\tilde \phi \in \argmin_{\phi \in \Phi}R(\phi)$. Then, we have that on the event of the previous  display, 
    \begin{align*}
    R(\hat \phi) - R(\tilde{\phi}) &=  R(\hat{\phi}) - R_n(\hat{\phi})  + \underbrace{\hat R_n(\hat \phi) -  \hat R_n(\tilde \phi)}_{\le 0} + \hat R_n(\tilde \phi) - R(\tilde \phi)\\
    &\le \frac{\alpha}{2} (R(\hat\phi) + R(\tilde \phi)) + 2\left(\frac{9}{\alpha}+1\right)\cdot\frac{M^2 \log(2|\Phi|/\delta)}{n}\\
    &\le \alpha R(\hat\phi) + 2\left(\frac{9}{\alpha}+1\right)\cdot\frac{M^2 \log(2|\Phi|/\delta)}{n}. 
    \end{align*}
    Selecting $\alpha = 1/2$ and rearranging
    \begin{align*}
    \frac{1}{2} R(\hat \phi) \le R(\tilde{\phi}) + 2(18+1)\cdot\frac{M^2 \log(2|\Phi|/\delta)}{n}.
    \end{align*}
    The bound follows.
    \end{proof}
\subsection{Proof of \Cref{lem:single_erm}\label{sec:single_erm_proof}}

The first inequality is a direct consequence of \Cref{lem:finite_class_fast_rate}. Here, we take the function class $\Phi = \Fclass_p \times \Gclass_p$, so $\log |\Phi| = \capac_p$. Moreover, by \Cref{asm:function_apx}, we can take 
\begin{align*}
M &= \sup_{\Fclass_p \in \cF,g \in \Gclass_p}\sup_{x,y}(\langle f(x),g(y)\rangle  - \langle \fst(x),\gst(y)\rangle ) \le 2B^2.
\end{align*} The second inequality uses \Cref{asm:training_approximation}   to bound $\Risk(\fst_p,\gst_p;\Dtrain) \le \kapapx \Risk(\fst_p,\gst_p;\disone)$, and noting  the fact that $\Risk(\fst_p,\gst_p;\disone) = \tailsf_2(p)$ by \Cref{lem:tailsf_two}. The third inequality uses \Cref{asm:density}, incurring an addition factor of $\kaptrain$. 
\hfill $\blacksquare$

\subsection{Proof of \Cref{lem:double_training_phase_two}}\label{sec:double_training_phase_two_proof}

Let $\Phi := \{(x,y) \mapsto \langle f(x), g(y)\rangle, (f,g) \in \cF_r \times \cG_r\}$. Further, define 
\begin{align*}\
\phi_{3,\star} &:= \langle \fst(x), \gst(y) \rangle, \quad \phi_{4,\star} := \langle \ftil(x),\bhatQ_r \gtil(y)\rangle. 
\end{align*}
We define $\cD_3$ as the distribution of $(x,y,z) \sim \Dtrain$, and $\cD_4$ as the distribution of $(x',y',z')$, where $(x',y') \sim \cdone$ and $z' = \phi_{4,\star}(x',y')$. We compute that, using \Cref{asm:bal,asm:function_apx},  and the last  statement of \Cref{thm:double_training_phase_1},
\begin{align*}
\sup_{x,y}\max_{\phi}\|\phi(x,y) -  \phi_{3,\star}(x,y)\| &\le 2B^2\\
\sup_{x,y}\max_{\phi}\|\phi(x,y) -  \phi_{4,\star}(x,y)\| &\le (1+\sqrt{2n_1})B^2,
\end{align*}
and
 \begin{align*}
 \log|\Phi| = \log|\cF_r||\cG_r| = \capac_r.
 \end{align*}

For $i \in \{3,4\}$, let  $R_i$ and $\hat{L}_{i,n_i},\hat{R}_{i,n_i}$ denote the corresponding excess risks as in \Cref{lem:finite_class_fast_rate}, the following holds with probability at least  $1 - \delta/3$ for all $\phi \in \Phi$
\begin{align*}
|R_3(\phi) - \hat R_{3,n_3}(\phi)| \le \frac{1}{4}R_3(\phi) + (19\cdot 4)\frac{B^4(\capac_r + \log(12/\delta))}{n_3}\\
|R_4(\phi) - \hat R_{4,n_4}(\phi)| \le \frac{1}{4}R_4(\phi) + (19\cdot (2+2n_1))\frac{B^4(\capac_r + \log(12/\delta))}{n_4},
\end{align*}
where we set $\alpha=1/2$ in the first statement of \Cref{lem:finite_class_fast_rate}. 
Set $R_\lambda(\phi) = R_3(\phi) + \lambda R_4(\phi)$. Then if $\hat \phi \in \argmin_{\phi \in \Phi} \hat{L}_{3,n_3}(\phi) + \lambda \hat{L}_{4,n_4}(\phi) = \argmin_{\phi \in \Phi} \hat{R}_{3,n_3}(\phi) + \lambda \hat R_{4,n_4}(\phi)$, we see that for any other $\tilde \phi \in \argmin_{\phi \in \Phi} R_\lambda(\phi)$,
\begin{align*}
R_\lambda(\hat \phi) - R_\lambda(\tilde \phi) &\le \frac{1}{4}(R_\lambda(\hat \phi) + R_\lambda(\tilde \phi)) + 2(19\cdot 4)\frac{B^4(\capac_r + \log(12/\delta))}{n_3} + 2\lambda(19\cdot (2+2n_1))\frac{B^4(\capac_r + \log(12/\delta))}{n_4}\\
&\le \frac{1}{2}R_\lambda(\hat \phi) + 176\left(1+\frac{\lambda n_1 n_3}{n_4}\right)\frac{B^4(\capac_r + \log(12/\delta))}{n_3}.
\end{align*}
Rearranging,  
\begin{align*}
R_\lambda(\hat \phi) \le 2 R_\lambda(\tilde \phi) + 352\left(1+\frac{\lambda n_1 n_3}{n_4}\right)\frac{B^4(\capac_r + \log(12/\delta))}{n_3}.
\end{align*}
To conclude, we handle the terms $R_\lambda (\hat \phi)$ and $R_\lambda (\tilde \phi)$. First, 
\begin{align*}
R_{\lambda}(\tilde \phi) &= \inf_{\phi \in \Phi}R_{\lambda}(\phi)\\
&= \inf_{(f,g) \in \cF_r\times \cG_r} \Risk(f,g;\Dtrain) + \lambda \Exp_{\cdone}[(\langle f, g \rangle - \langle \ftil, \bhatQ_r \cdot \gtil \rangle)^2]\\
&\le  \Risk(\fst_r,\gst_r;\Dtrain) + \lambda \Exp_{\cdone}[(\langle \fst_r, \gst_r \rangle - \langle \ftil, \bhatQ_r \cdot \gtil \rangle)^2] \tag{$(\fst_r,\gst_r) \in \cF_r \times \cG_r$}\\
&\le \kapapx \Risk(\fst_r,\gst_r;\cdone) + \lambda \Exp_{\cdone}[(\langle \fst_r, \gst_r \rangle - \langle \ftil, \bhatQ_r \cdot \gtil \rangle)^2] \tag{\Cref{asm:training_approximation}}\\
&= \kapapx \tailsf_2(r) + \lambda \Risk_{[r]}(\ftil, \bhatQ_r \cdot \gtil;\cdone).
\end{align*}
Second,
\begin{align*}
&\Risk(\fhat,\ghat;\Dtrain) + \frac{\lambda}{2} \Risk_{[r]}(\fhat,\ghat;\cdone)  \\
&=\Risk(\fhat,\ghat;\Dtrain) + \frac{\lambda}{2} \Exp_{\cdone}[(\langle \fhat,\ghat \rangle - \langle \fst_r,\gst_r\rangle)^2]\\
&\le
\underbrace{\Risk(\fhat,\ghat;\Dtrain)}_{=R_3(\hat \phi)} + \lambda\underbrace{\Exp_{\cdone}[(\langle \ftil,\bhatQ_r \cdot \gtil \rangle - \langle \fhat,\ghat \rangle)^2]}_{R_4(\hat \phi)} + \lambda \underbrace{\Exp_{\cdone}[(\langle \ftil,\bhatQ_r \cdot \gtil \rangle - \langle \fst_r,\gst_r\rangle)^2]}_{=\Risk_{[r]}(\ftil,\bhatQ_r \cdot\gtil;\cdone)}\\
&= R_{\lambda}(\hat \phi) + \lambda \Risk_{[r]}(\ftil, \bhatQ_r \cdot \gtil;\cdone).
\end{align*}
In sum, we conclude
\begin{align*}
&\Risk(\fhat,\ghat;\Dtrain) + \frac{\lambda}{2} \Risk_{[r]}(\fhat,\ghat;\cdone) \\
&\quad\le 2\kapapx \tailsf_2(r) + 3\lambda \Risk_{[r]}(\ftil, \bhatQ_r \cdot \gtil;\cdone) + 352\left(1+\frac{\lambda n_1 n_3}{n_4}\right)\frac{B^4(\capac_r + \log(12/\delta))}{n_3},
\end{align*}
which completes the proof.\hfill $\blacksquare$





\section{Proof of Rate Instantiations}\label{sec:rate_instant}
This section gives the proofs of \Cref{lem:st_instantiation,lem:dt_training_decay}, the instantiations of our error bounds under the spectral decay assumptions stipulated in \Cref{asm:eigendecay}. 
We begin by establishing the following two spectral decay bounds.
\begin{lem}[$\tailsf_q$ Bounds]\label{lem:tailsf_bounds} Suppose \Cref{asm:eigendecay} holds. Then,
\begin{align*}
\tailsf_1(r) &\le \begin{cases} C(1+\gamma^{-1})(r+1)^{-\gamma} & \text{(polynomial decay)}\\
C(1+\gamma^{-1})e^{-\gamma(r + 1)} & \text{(exponential decay)}
\end{cases}\\
\tailsf_2(r) &\le \begin{cases} 2C^2(r+1)^{-1 - 2\gamma} & \text{(polynomial decay)}\\
\tailsf_2(r) \le C^2(1+\gamma^{-1})e^{-2\gamma(r + 1)} & \text{(exponential decay)}
\end{cases}
\end{align*}
\end{lem}

\begin{lem}\label{lem:ratio_decay_lb} Suppose \Cref{asm:eigendecay} holds, and $\bsigst_r > 0$. Then, 
\begin{align*}
\frac{\tailsf_2(r)^2}{(\bsigst_r)^2}  \le
\begin{cases}
3C^2 r^{-2\gamma} &\text{(polynomial decay)} \\
C^2(1+\gamma^{-1} + r)^2e^{-2\gamma r} & \text{(exponential decay)}
\end{cases}.
\end{align*}
\end{lem}
The above lemmas are proved in \Cref{lem:tailsf_bounds,sec:lem:ratio_decay_lb} respectively. We give the proof of \Cref{lem:st_instantiation,lem:dt_training_decay} in the following two sections.

\subsection{Proof of \Cref{lem:st_instantiation}}\label{sec:lem:st_instantiation}
In both decay regimes, we apply \Cref{lem:ratio_decay_lb,lem:tailsf_bounds}.
Under the polynomial decay, we have
\begin{align*}
\specerrst(r) &:= r^4\cdot \tailsf_2(r)  + \tailsf_1(r)^2 + r^2 (\bsigst_{r+1})^2 + \frac{r^6\cdot \tailsf_2(r)^2}{(\sigst_r)^2}\\
&\le 2C^2 r^4 r^{-1-2\gamma}  + C^2(1+\gamma^{-1})^2 r^{-2\gamma} + C^2 r^{2 - 2(1+\gamma)}  + 3C^2 r^{6-2\gamma} \\
&\lesssim  C^2(1+\gamma^{-1})^2 r^{6-2\gamma}.
\end{align*}
In the exponential case, a similar argument applies.

\subsection{Proof of \Cref{lem:dt_training_decay}}\label{sec:lem:dt_training_decay}

 Again, let $\psi(r)$ be equal to $\psi(r) = Cr^{-(1+\gamma)}$ for polynomial decay, and $\psi(r) = Ce^{-r\gamma}$ for exponential decay; thus, under \Cref{asm:eigendecay}, $ \psi(r) \ge \bsigst_r$.

 \begin{claim}\label{claim:cond_c_hold} Suppose we take $\sigcut \ge 2\psi(\rcut)$. Then, if $\psi(\rcut ) \le \frac{1}{3e}\bsigst$, \Cref{cond:dt_apx_conds}\pointnum{c} holds.
 \end{claim}
 \begin{proof}[Proof of \Cref{claim:cond_c_hold}]
 Observe that, if we select $\sigcut = 2\psi(\rcut)$, then if $\psi(\rcut ) \le \frac{1}{3e}\bsigst$, thn
\begin{align*}
2\bsigst_{\rcut} \le 2\psi(\rcut) = \sigcut, \quad \sigcut = 2\psi(\rcut) \le \frac{2}{3e}\bsigst.
\end{align*}
Therefore, \Cref{cond:dt_apx_conds}\pointnum{c} holds. 
\end{proof}
\textbf{Polynomial Decay.} From \Cref{lem:tailsf_bounds}, we have $\tailsf_1(r) \le C(1+\gamma^{-1})(\rcut+1)^{-\gamma}$ and $\tailsf_2(r) \le 2C^2(\rcut+1)^{-(1 + 2\gamma)}$, and by definition of $\psi$, $\sigcut \ge 2\psi(\rcut) = 2C\rcut^{-(1+\gamma)}$. We then bound
\begin{align*}
\errdt(\rcut,\sigcut) &:= \rcut^2 \sigcut^2  +\tailsf_1(\rcut)^2   + \frac{\tailsf_2(\rcut)^2 }{(\sigcut)^2}\\
&\le \rcut^2 \sigcut^2 +  C^2(1+\gamma^{-1})^2 \rcut^{-2\gamma}  + \frac{4 C^2 \rcut^{-2(1+2\gamma)}}{4 C^2 \rcut^{-2(1+\gamma)}}\\
&\le \rcut^2 \sigcut^2 + C^2(1 + (1+\gamma^{-1})^2) \rcut^{-2\gamma}  \lesssim \rcut^2 \sigcut^2 + C^2(1+\gamma^{-2}) \rcut^{-2\gamma}.
\end{align*}

Let us check each of the conditions of \Cref{cond:dt_apx_conds}. 
\begin{claim}\label{claim:dt_apx_poly} Suppose that $\rcut \ge \max\{c_1,\frac{3e C}{\bsigst_1}\}$ and $p \ge c_1^{-\frac{1}{1+2\gamma}}\rcut^{\frac{7+5\gamma}{1+2\gamma}}$. Then, \Cref{cond:dt_apx_conds} holds, and the interval $[2C\rcut^{-(1+\gamma)}, \frac{2}{3e}\bsigst]$ is nonempty.
\end{claim}
\begin{proof}[Proof of \Cref{claim:dt_apx_poly}]
For \Cref{cond:dt_apx_conds}\pointnum{a}, we need $\rcut \ge c_1$, and $\tailsf_2(\rcut) \le \frac{1}{c_1}\rcut^2\sigcut^2 = \frac{4}{c_1}\rcut^2(\sigcut^2)$. It suffices that $2C^2(\rcut+1)^{-(1 + 2\gamma)} \le \frac{4C^2}{c_1}\rcut^{2-2(1+\gamma)} = \frac{4C^2}{c_1}\rcut^{-2\gamma}$. As $\rcut + 1 \ge \rcut$, it is enough that $1 \le (2/c_1)\rcut^{-2\gamma + 1 + 2\gamma } = (2/c_t)\rcut$, which holds for $\rcut \ge c_1$. For \Cref{cond:dt_apx_conds}\pointnum{b}, we need $\tailsf_2(p) \le \frac{1}{c_1}\frac{\sigcut^2}{\rcut^5}$. We have $\tailsf_2(p) \le 2C^2 p^{-(1 + 2\gamma)}$, and $\frac{1}{c_1} \cdot \frac{\sigcut^2}{\rcut^5} = \frac{1}{c_1} \cdot 4 C^2 \rcut^{-2(1+\gamma) -5}  = \frac{1}{c_1}\cdot 4 C^2 \rcut^{-(7 + 5\gamma)}$. Hence, it is enough that $p^{-(1+2\gamma)} \le \rcut^{-(7 + 5\gamma)}$, i.e. $p \ge c_1^{-\frac{1}{1+2\gamma}}\rcut^{\frac{7+5\gamma}{1+2\gamma}}$. For \Cref{cond:dt_apx_conds}\pointnum{c}, \Cref{claim:cond_c_hold} requires the choice of $\psi(\rcut ) \le \frac{1}{3e}\bsigst_1$, i.e. $C\rcut^{-(1+\gamma)} \le \frac{1}{3e}\bsigst_1$. For this, it is enough that $\rcut \ge \frac{3e C}{\bsigst_1}$.
\end{proof}

\paragraph{Exponential Decay.}  From \Cref{lem:tailsf_bounds}, we have $\tailsf_1(r)  \le C(1+\gamma^{-1})e^{-\gamma r}$, $\tailsf_2(r) \le C^2(1+\gamma^{-1})e^{-2\gamma r}$, and by definition of $\psi$, $\sigcut \ge 2\psi(\rcut) = 2Ce^{-\gamma r}$. Then, 
\begin{align*}
\errdt(\rcut,\sigcut) &:= \rcut^2 \sigcut^2  +\tailsf_1(\rcut)^2   + \frac{\tailsf_2(\rcut)^2 }{(\sigcut)^2}\\
&\le  \sigcut^2 \rcut^2 + C^2(1+\gamma^{-1})^2e^{-2\gamma \rcut} + \frac{C^4(1+\gamma^{-1})^2e^{-4\gamma \rcut}}{4C^2e^{-2\gamma \rcut}} \\
&\le  \sigcut^2 \rcut^2 + C^2(1+\gamma^{-1})^2e^{-2\gamma \rcut} + \frac{C^4(1+\gamma^{-1})^2e^{-4\gamma \rcut}}{4C^2e^{-2\gamma \rcut}} \\
&\lesssim \sigcut^2 \rcut^2 + C^2(1+\gamma^{-2}) e^{-2\gamma \rcut}
\end{align*}
We conclude by checking \Cref{cond:dt_apx_conds}

\begin{claim}\label{claim:cond_apx_exp} Suppose that $\rcut \ge \max\{c_1, \sqrt{c_1(1+\gamma^{-1})}, \frac{1}{\gamma} \log(\frac{3e C}{\bsigst_1})\}$ and $p \ge \max\{2\rcut,\frac{1}{\gamma}\log(\rcut^5 c_1)\}$. Then, \Cref{cond:dt_apx_conds} holds, and the interval $[2Ce^{-\gamma \rcut},\frac{2}{3e}\bsigst]$ is nonempty.
\end{claim}
\begin{proof}[Proof of \Cref{claim:cond_apx_exp}]
For \Cref{cond:dt_apx_conds}\pointnum{a}, we need $\rcut \ge c_1$, and $\tailsf_2(\rcut) \le \frac{1}{c_1}\rcut^2\sigcut^2 = \frac{4}{c_1}\rcut^2(\sigcut^2)$. It suffices that $C^2(1+\gamma^{-1})e^{-2\gamma\rcut} \le \frac{4C^2}{c_1}\rcut^{2}e^{-2\gamma\rcut}$. For this, it suffices that $\rcut \ge \sqrt{c_1(1+\gamma^{-1}}$. 

For \Cref{cond:dt_apx_conds} \pointnum{b}, we need $\tailsf_2(p) \le \frac{1}{c_1}\frac{\sigcut^2}{\rcut^5}$.  We have $\tailsf_2(p) \le (1+\gamma^{-1})C^2e^{-2\gamma p}$, 
 and $\frac{1}{c_1} \cdot \frac{\sigcut^2}{\rcut^5} = \frac{1}{\rcut^5 c_1} \cdot 4 C^2 e^{-2\gamma \rcut }$. Hence, it is enough that $e^{-2\gamma (p-\rcut)} \le \frac{4}{\rcut^5 c_1}$. For $p \ge 2\rcut$, it is enough that $e^{-\gamma p} \le  \frac{4}{\rcut^5 c_1}$. Thus, it suffices that $p \ge \max\{2\rcut, \frac{1}{\gamma}\log(\rcut^5 c_1)\}$. 
 For \Cref{cond:dt_apx_conds} \pointnum{c}, \Cref{claim:cond_c_hold} requires the choice of $\psi(\rcut ) \le \frac{1}{3e}\bsigst_1$, i.e. $Ce^{-\gamma r} \le \frac{1}{3e}\bsigst_1$. For this, it is enough that $\rcut \ge \frac{1}{\gamma} \log(\frac{3e C}{\bsigst_1})$.
\end{proof}
This concludes the proof. \hfill $\blacksquare$

\subsection{Proof of \Cref{lem:tailsf_bounds}}\label{sec:lem:tailsf_bounds}

We begin with the polynomial decay case, where $\bsigst_r \le Cr^{-(1+\gamma)}$. We compute
\begin{align*}
\tailsf_1(r) &= C \sum_{n > r}^{\infty}  n^{-(1+\gamma)} \\
&\le C (r + 1)^{-(1+\gamma)} + \int_{x = r+1}^{\infty}x^{-(1+\gamma)} \rmd x \\
&\le C(1+\gamma^{-1})(r+1)^{-\gamma}.
\end{align*}
and
\begin{align*}
\tailsf_2(r) &= C^2 \sum_{n > r}^{\infty}  n^{-2(1+\gamma)} \\
&\le C^2 (r + 1)^{-2(1+\gamma)} + C^2\int_{x = r+1}^{\infty}x^{-2(1+\gamma)} \rmd x \\
&\le C^2(1+\frac{1}{1 + 2\gamma})(r+1)^{-1 - 2\gamma} \le 2C^2(r+1)^{-1 - 2\gamma}.
\end{align*}
We now turn to the exponential decay case, where $\bsigst_r \le C\exp(-\gamma r)$. We have
\begin{align*}
\tailsf_1(r) &= C \sum_{n > r}^{\infty}  e^{-\gamma n} \\
&\le C e^{-\gamma(r + 1)} + C\int_{x = r+1}^{\infty}e^{-\gamma x} \rmd x \\
&\le C(1+\gamma^{-1})e^{-\gamma(r + 1)}.
\end{align*}
and 
\begin{align*}
\tailsf_2(r) &= C^2 \sum_{n > r}^{\infty}  e^{-2\gamma n} \\
&\le (C e^{-\gamma(r + 1)})^2 + C\int_{x = r+1}^{\infty}e^{-2\gamma x} \rmd x \\
&\le C^2(1+\frac{1}{2}\gamma^{-1})(e^{-\gamma(r + 1)})^2 \\
&\le C^2(1+\gamma^{-1})(e^{-\gamma(r + 1)})^2.
\end{align*}
\hfill $\blacksquare$

\subsection{Proof of \Cref{lem:ratio_decay_lb}}\label{sec:lem:ratio_decay_lb}
Let $\psi(r) := Cr^{-(1+\gamma)}$ under polynomial decay, and $\psi(r) = Ce^{-\gamma r}$ under exponential decay.  We start with a useful claim, and then turn to the polynomial and exponential decay regimes in sequence.
Going forward, set $\Delta = \bsigst_r$, and let $\rbar := \inf\{i \in \N: \psi(i)\le \Delta\}$. 

\begin{claim}\label{claim:r_to_rbar}  $\tailsf_2(r) \le \rbar \Delta^2  + \tailsf_2(\rbar)$ and $\tailsf_1(r) \le (\rbar-1) \Delta  + \tailsf_1(\rbar)$. 
\end{claim}
\begin{proof}[Proof of \Cref{claim:r_to_rbar}] $\psi(r) \ge \bsigst_r = \Delta$ implies $r \le \rbar$.
\begin{align*}
\tailsf_2(r) &= \sum_{n > r} (\sigst_{r})^2 = \sum_{n = r+1}^{\rbar} (\sigst_{r})^2  + \sum_{n > \rbar}(\sigst_{r})^2\\
&\le (\rbar - 1) (\sigst_{r+1})^2  + \sum_{n > \rbar+1}(\sigst_{r})^2 = \rbar(\sigst_r)^2 + \tailsf_2(\rbar)\\
&\le (\rbar-1) \Delta^2  + \tailsf_2(\rbar),
\end{align*}
where we use that $\sigst_{r+1} \le \Delta$.
\end{proof}

\paragraph{Polynomial decay.} For polynomial decay, we consider $\psi(i) = Ci^{-(1+\gamma)}$. Then $\rbar +1 = 1. \inf\{i : C(i)^{-(1+\gamma)} \le \Delta\} = 1+ \inf\{i: i  \ge  (\Delta/C)^{ - \frac{1}{1+\gamma}}\}$. Hence, $\rbar \ge (\Delta/C)^{ - \frac{1}{1+\gamma}}$ and $\rbar + 1\le (\Delta/C)^{ - \frac{1}{1+\gamma}}$. By \Cref{lem:tailsf_bounds}, we have
\begin{align*}
\tailsf_2(\rbar+1) \le 2C^2(1+\frac{1}{1 + 2\gamma})(\rbar+1)^{-1 - 2\gamma} \le 2C^2 (\Delta/C)^{\frac{1+2\gamma}{1+\gamma}}. \numberthis \label{eq:tailsf_one_poly_second}
\end{align*}

Thus, by \Cref{claim:r_to_rbar}, the above display, and the bound $\rbar \le (\Delta/C)^{ - \frac{1}{1+\gamma}}$,
\begin{align*}
\tailsf_2(r) &\le (\rbar-1) \Delta^2 + 2C^2 (\Delta/C)^{\frac{1+2\gamma}{1+\gamma}}\\
&\le (\Delta/C)^{ - \frac{1}{1+\gamma}} \Delta^2 + 2C^2 (\Delta/C)^{\frac{1+2\gamma}{1+\gamma}}\\
&= \Delta^{\frac{1+2\gamma}{1+\gamma}} C^{\frac{1}{1+\gamma}} + + 2C^2 (\Delta/C)^{\frac{1+2\gamma}{1+\gamma}}\\
&= 3C^2 (\Delta/C)^{\frac{1+2\gamma}{1+\gamma}} := 3C^2 (\bsigst_r/C)^{\frac{1+2\gamma}{1+\gamma}}
\end{align*}
Thus, using the above display and $\bsigst_r \le Cr^{-(1+\gamma)}$. 
\begin{align*}
\frac{\tailsf_2(r)^2}{(\bsigst_r)^2} \le 3C^2 (\bsigst_r/C)^{\frac{2(1+2\gamma)}{1+\gamma} - 2} = 3C^2 (\bsigst_r/C)^{\frac{2\gamma}{1+\gamma} } \le 3C^2 r^{-2\gamma}
\end{align*}

\paragraph{Exponential decay.}  For polynomial decay, we consider $\psi(i) = Ce^{-\gamma i}$. Then $\rbar = \inf\{i : Ce^{-\gamma i} \le \Delta\} = \inf\{i: i \ge  \gamma^{-1}\log \frac{C}{\Delta}\}$. Hence, 
\begin{align*}
\rbar \ge \gamma^{-1}\log \frac{C}{\Delta}, \quad \rbar - 1 \le \gamma^{-1}\log \frac{C}{\Delta}. 
\end{align*}
Then, by \Cref{lem:tailsf_bounds},
\begin{align*}
\tailsf_2(\rbar) 
&\le C^2(1+\gamma^{-1})(e^{-\gamma(\rbar + 1)})^2 \le (1+\gamma^{-1})\Delta^2 \le C^2(1+\gamma^{-1})\Delta^2. \numberthis \label{eq:tailsf_one_exp_second}
\end{align*}
Thus by \Cref{claim:r_to_rbar}, 
\begin{align*}
\tailsf_2(r) \le C^2(1+\gamma^{-1} + \log \frac{C}{\Delta})\Delta^2 = C^2(1+\gamma^{-1} + \gamma^{-1}\log \frac{C}{\bsigst_r})(\bsigst_r)^2
\end{align*}
Hence, 
\begin{align*}
\frac{\tailsf_2(r)^2}{(\bsigst_r)^2}\le C^2 (1+\gamma^{-1} + \gamma^{-1}\log \frac{C}{\bsigst_r})^2(\bsigst_r)^2
\end{align*}
As $x \log(1/x)$ is increasing in $x$, and as $\bsigst_r \le Ce^{-\gamma r}$, the above is at most
\begin{align*}
\frac{\tailsf_2(r)^2}{(\bsigst_r)^2} \le C^2(1+\gamma^{-1} +  r)^2e^{-2\gamma r}.
\end{align*}
\hfill $\blacksquare$

\newpage
\part{Supplement for the Meta-Theorem}\label{part:supp_meta_theorem}

\newcommand{\barl}{\ell_{\sigma,\kpick}}

\newcommand{\sfsigst}{\upnu^{\star}}

\newcommand{\delspace}{\delta}
\newcommand{\muspace}{M_{\mathrm{space}}}
\newcommand{\mucond}{\mu}
\newcommand{\muspec}{M_{\mathrm{spec}}}
\newcommand{\delstki}{\delst_{k_i}}
\newcommand{\delstkipl}{\delst_{k_{i+1}}}
\newcommand{\temparams}{(\delspace,\mucond)}
\newcommand{\partshort}{(\cK_i)_{i=1}^{\ell}}

\newcommand{\bstApr}{\bA^{\star\prime}}
\newcommand{\bstBpr}{\bB^{\star\prime}}
\newcommand{\bstBkpr}{\bB_{[k]}^{\star\prime}}
\newcommand{\bstAkpr}{\bA_{[k]}^{\star\prime}}
\newcommand{\bhatApr}{\bhatA^{\prime}}
\newcommand{\bhatBpr}{\bhatB^{\prime}}
\newcommand{\bstAkl}{\bstA_{[k_\ell]}}
\newcommand{\bstBkl}{\bstB_{[k_\ell]}}

\newcommand{\kst}{k_{\star}}
\newcommand{\bk}{\mathbf{k}}
 
\newcommand{\bSigmak}{\mathbf{\Sigma}_{[k]}}

\newcommand{\bstSigmak}{\mathbf{\Sigma}^{\star}_{[k]}}
\newcommand{\cproc}{c_0}
\newcommand{\epstot}{\epsilon_{\mathrm{tot}}}
\newcommand{\rnot}{r_{0}}
\newcommand{\epsrisk}{\eps_{\mathrm{rsk}}}
\newcommand{\support}{\mathrm{supp}}
\newcommand{\fstr}{\fst_r}
\newcommand{\gstr}{\gst_r}

\section{Factor Recovery for Matrix Factorization}\label{sec:matrix_main_results_details}

We recall the setup for matrix factor recovery; its relation to the bilinear embeddings is described in \Cref{label:err_term_factor_recovery}.
For matrices $\bstA,\bhatA \in \R^{n \times d}$, $\bstB,\bhatB \in \R^{m \times d}$, and matrices $\bstM,\bhatM$, and for orthogonal matrices $\bR \in \bbO(d)$,  consider the error terms
    \begin{align}
    \ErrTerm_0(\bR,k) &= \|(\bstAk - \bhatA \bR)(\bstBk)^\top\|_{\fro}^2 \vee \|\bstAk(\bstBk - \bhatB \bR)^\top\|_{\fro}^2 \label{eq:ErrTerm_zero}\\
    \ErrTerm_1(\bR,k) &= \|\bstAk - \bhatA\bR\|_{\fro}^2 \vee \|\bstBk - \bhatB\bR \|_{\fro}^2,  \label{eq:ErrTerm_one}
    \end{align}
    where $\bstAk$ and $\bstBk$ are the rank-$k$ approximations of $\bstA$ and $\bstB$;  formally\footnote{While $\bstPk$ is non-unique in general, \Cref{thm:main_matrix} ensures that there is a spectral gap at rank $k$, ensuring $\bstPk$ is indeed unique.} 
    \begin{align} 
    \bstAk = \bstA\bstPk, \quad \bstBk = \bstB\bstPk, \quad \bstPk \in \text{projection  on top-$k$ eigenspace of $(\bstA)^\top \bstA = (\bstB)^\top\bstB$}.  \label{eq:rank_k_svd}
    \end{align} 
    Our guarantee for controlling these matrix error terms is perhaps the most challenging technical ingredient of the paper. We state the following theorem, of which \Cref{thm:main_matrix_body} is a specialization.
    
    \begin{restatable}{theorem}{thmmatrix}\label{thm:main_matrix} Let $\bstA,\bhatA \in \R^{n \times d}$, $\bstB,\bhatB \in \R^{m \times d}$,  and suppose $(\bstA,\bstB)$ and $(\bhatA,\bhatB)$ are balanced factorizations of $\bstM = \bstA(\bstB)^\top$, and $\bhatM = \bhatA\bhatB^\top$. Let $r = \rank(\bhatM)$. Fix $\epsilon > 0$ and $\kpick \in \N$ such that $\kpick>1$,   $\epsilon\ge \|\bhatM - \bstM\|_{\fro}$, and $\epsilon \le \frac{\|\bstM\|_{\op}}{40 \kpick}$.  Also, for $q \ge 1$, let $\tail_q(\bM;k) := \sum_{i>k}\sigma_i(\bM)^q$. Then, 
    \begin{itemize}
    \item[(a)] There exists an index $k \in [\min\{r,\kpick-1\}]$ and an  orthogonal matrix $\bR \in \bbO(d)$
    such that 
    \begin{subequations}
    \begin{align}
    \mathrm{(weighted~error)}~~~~~\ErrTerm_0(\bR,k) &\lesssim \epsilon^2 \cdot \kpick^2\lstargs\\
    \mathrm{(unweighted~error)}~~~~~ \ErrTerm_1(\bR,k) &\lesssim (\sqrt{r}+\kpick^2)\epsilon + \kpick \sigma_{\kpick}(\bstM)  +  \tail_1(\bstM;\kpick), 
    \end{align}
    \end{subequations}
    where we define $\lstargs := \min\left\{1 + \log \frac{\|\bstM\|_{\op}}{40\kpick \epsilon},\,\kpick\right\}$. 
    \item[(b)] Moreover, the index $k$ satisfies
    \begin{align*}
    \tail_2(\bstM;k) \lesssim \kpick^3 \epsilon^2 + \kpick ( \sigma_{\kpick}(\bstM))^2 +  \tail_2(\bstM;\kpick).
    \end{align*}
    \item[(c)] The matrix $\bR$ and $k$ satisfy $(\bhatA\bR)^\top\bhatA\bR \succeq 39\epsilon\bstPk$ and $\sigma_k(\bstM) - \sigma_{k+1}(\bstM) \ge 40\epsilon/\kpick$.
    \end{itemize}
    \end{restatable}

    \paragraph{Explanation of \Cref{thm:main_matrix}.} There are a few essential points to the theorem, which we outline below.
    \begin{itemize}
    \item The parameter $\epsilon^2$ upper bounds $\|\bstM - \bhatM\|_{\fro}^2$. When  instantiated as in \Cref{label:err_term_factor_recovery}, $\epsilon^2$ bounds $\|\bM_{1\otimes 1}(f,g) - \bM_{1\otimes 1}(\fst,\gst)\|_{\fro}^2 = \Risk(f,g;\cdone)$. Because $\Dtrain$ covers $\cdone$, this ensures that we can choose $\epsilon$ sufficiently small for non-vacuous  bounds.
    \item The theorem guarantees the existence of some index $ k \in [s-1]$ for which the error terms with respect to the rank-$k$ approximation is small. It may not be the case that $k = s-1$, and indeed the construction of this index $k$ can be subtle. Fortunately, this index $k$ is only important for the analysis, and need not be known  by the algorithm.  Such an index $k$ leads to a partition of the singular values that enable us to better control the   relative spectral gap (see formal definition in \Cref{sec:proof_roadmap_matrix}). This is the key to obtaining our improved  bounds compared to the literature. 
    \item Part (a) of the theorem bounds $\ErrTerm_0$ and $\ErrTerm_1$. Our bound on $\ErrTerm_0$ is \emph{much smaller} than that on $\ErrTerm_1$, scaling quadratically in $\epsilon$ instead of linearly. This emphasizes the importance of weighting by the co-factors $\bstAk$ and $\bstBk$ in \Cref{eq:ErrTerm_zero}, or equivalently (via the discussion in \Cref{label:err_term_factor_recovery}), by the rank-reduced embeddings $\fstk,\gstk$ in \Cref{defn:key_err_terms_simple}.
    \item Part (b)  stipulates that truncating the spectrum of $\bstM$ at the index $k$ is not much worse than truncating the spectrum at the stipulated index $s$. Note that $\tail_2(\bstM;k)$ corresponds to $\tailsf_2(k)$  (in  \Cref{prop:final_error_decomp_simple}) under the choices in \Cref{label:err_term_factor_recovery}. Hence, this is useful for handling the term  $\tailsf_2(k)$ that emerges in the risk decomposition therein. 
    \item  Finally, the statement ensures that, even though $\bstM$ \emph{may not} have a spectral gap at its $s$-singular value, the stipulated index $k$ does ensure $\sigma_k(\bstM) - \sigma_{k+1}(\bstM) \ge 40\epsilon/\kpick$. Moreover, it also ensures that, after the rotation $\bR$, the column-space  of $\bhatA\bR$ contains the column space of $\bstA$; this corresponds to \Cref{eq:alignment_main}, and ensures that the chosen rotation $\bR$ makes $(\bR^\top f,\bR^\top g)$ \emph{aligned proxies}. This latter statement also gives some quantitative  wiggle room when applying limiting arguments for continuous distributions (see more details in \Cref{sec:proof_main_results_error_decomp}). 
    \end{itemize}

\subsection{Proof roadmap} \label{sec:proof_roadmap_matrix}

\paragraph{Technical challenges.} 
The key challenge throughout the proof of \Cref{thm:main_matrix} is that many classical matrix perturbation bounds  (e.g. Wedin's theorem) require some form of \emph{separation} (i.e. \emph{gaps})  among the singular values  of the matrix to which it is being applied. In sharp contrast, we assume no such condition on gaps in the spectrum of $\bstM$. 


Specifically, we appeal to a lemma due to \cite{tu2016low} (see the restatement in \Cref{lem:proc}), which controls (up to a rotation) the Frobenius   error of the  factors $\bX_1-\bR\bX_2,~\bY_1-\bR\bY_2$ in terms of the Frobenius error between their outer products $\bZ_1 = \bX_1\bY_1^\top$ and $\bZ_2 = \bX_2\bY_2^\top$. When applied directly to $\bhatM = \bhatA\bhatB^\top$ and $\bstM = \bstA(\bstB)^\top$, \Cref{lem:proc} has numerous limitations:  (a) it requires the factorization of $\bhatM$ and $\bstM$ to have the same rank $k$; (b) it requires a sufficiently large spectral gap on $\sigma_k(\bstM)-\sigma_{k+1}(\bstM)$; (c) the error bounds scale with the inverse of this gap, which can be very loose when  $\sigma_k(\bstM)$ becomes small.

\paragraph{Our techniques.} 
Instead of applying \Cref{lem:proc} directly, we construct a certain partition of the spectrum of $\bstM$, what we call a ``well-tempered partition'' (\Cref{defn:well_tempered_partition}), which partitions the indices $[s]$ of the top-$s$ singular values of $\bstM$ into intervals where (a) all singular values are of similar magnitude, and (b) the separation between the intervals is sufficiently large. Condition (b) is necessary for applying gap-dependent perturbation bounds, but condition (a) allows us to refine these bounds tremendously. 

Specifically, we denote the subsets in this  partition as {$\cK_i = \{k_{i}+1,k_{i}+2,\dots,k_{i+1}\}$};  we call $k_i$ the \emph{pivot}. We show that  the partition ensures that the \emph{relative gap} 
\begin{align}\label{equ:relative_gap_def}
\updelta_{k_i} = \frac{\sigma_{k_i}(\bstM) - \sigma_{k_i +1}(\bstM)}{\sigma_{k_i}(\bstM)}	
\end{align}
is at least $\Omega(1/s)$. By contrast, note that with exponentially decaying singular values, the \emph{absolute gap}  $\sigma_{k_i}(\bstM) -  \sigma_{k_i+1}(\bstM)$ can be exponentially small.

 With a careful change-of-basis,   the above spectral partition induces a   decomposition of $\bstAk,\bstBk$ and $\bhatA,\bhatB$ into blocks according to the indices in the  set $\cK_i$. We then apply \Cref{lem:proc} \emph{separately along each block}, arguing that  the factorization error is small \emph{block-wise}. This significantly sharpens our control over $\ErrTerm_0$ (recall the definition  in  \Cref{eq:ErrTerm_zero}) because we weight the error in block $i$ by the largest singular value in that block. Working through the algebra, we end up only paying for the \emph{relative gap}, which as noted above is $\Omega(1/s)$.

 For both $\ErrTerm_0$ and $\ErrTerm_1$, the above partition also has the advantage (indeed, necessity) that, by restricting to each set $\cK_i$ in the partition, we only need  to consider the factorizations of the same rank, and lower-bounded relative spectral gap. Recall that \Cref{lem:proc} establishes error bounds on factors in terms of error bounds on their outer-product.
By decomposing our matrices into their restriction to the singular values index by $\cK_i$, we therefore need some way of controlling the following: \emph{Denote by $\bstM_{\cK_i},\bhatM_{\cK_i}$ the SVDs of $\bstM,\bhatM$ containing only singular values indexed by $j \in \cK_i$. How large is $\|\bstM_{\cK_i} - \bhatM_{\cK_i}\|_{\fro}$, in terms of $\|\bstM- \bhatM\|_{\fro}$?}  

 Again, our control over relative spectral gaps come to the rescue.  Here, we invoke \Cref{thm:svd_pert}, which shows that the error in the SVDs between these objects grows only with the \emph{relative} spectral gap, which as we have stressed, is well-controlled. This again reduces the dependence on the small singular values of $\bstM$, which improves our bounds. 

 The formal proof is quite involved. Hence, we begin with an extensive setup of preliminaries, simplification and useful notation before diving into the main arguments. But to summarize, the key tools are: (a) \Cref{lem:proc} due to \cite{tu2016low}, (b) our novel construction of the ``well-tempered partition'' of the spectrum of $\bstM$, and (c) our novel relative-error perturbation bound.

\subsection{Proof preliminaries}  

\paragraph{Singular value notation.} We introduce the  following notation for the singular values of $\bstM$ and their relative gaps:

\begin{align*}
\sigstk := \sigma_k(\bstM), \quad \sigst_0 = +\infty, \quad \delst_k := 1 - \frac{\sigma_{k+1}(\bstM)}{\sigma_k(\bstM)}, \quad \delst_0 = 1.
\end{align*}

\paragraph{Explicit factorization.} We argue that, without loss of generality, we can pick factors $\bhatA,\bhatB,\bstA,\bstB$ of a canonical form. Construct the SVDs of $\bhatM$ and $\bstM$ as
\begin{align*}
\bhatM = \bhatU \bhatSigma \bhatV^\top , \quad \bstM = \bstU \bstSigma (\bstV)^\top,
\end{align*}
where $\bhatU,\bstU \in \R^{n \times d}$, $\bhatV,\bstV \in \R^{m \times d}$, and $\bhatSigma,\bstSigma \in \R^{p\times d}$ are diagonal matrices with non-negative entries arranged in (non-strictly) descending order. Note that $p\leq \min\{n,m\}$.  We now argue that we may assume the factor matrices take the following form, without loss of generality:
\begin{align}
\bhatA = \bhatU \bhatSigma^{\half}, \quad \bhatB = \bhatV \bhatSigma^{\half}, \quad \bstA = \bstU (\bstSigma)^{\half}, \quad \bstB = \bstV (\bstSigma)^{\half}.  \label{eq:Aform}
\end{align}
One can check that a valid choice of rank-$k$ SVD for $\bstM$ is  given by $\bstAk(\bstBk)^\top$, where 
\begin{align*}
 \bstAk = \bstU (\bstSigmak)^{\half}, \quad \bstBk = \bstV (\bstSigmak)^{\half},
\end{align*} 
and $\bstSigmak$ zeroes out all but the first $k$ entries of $\bstSigma$.  The assumption that the matrices take the above form is justified by the following lemma, which shows that any bounds on $\ErrTerm_0,\ErrTerm_1$ hold for the factorization in \Cref{eq:Aform}.
\begin{lem}\label{lem:simple_form_wlog} Assume that $\bhatA,\bhatB,\bstA,\bstB$ take the form \Cref{eq:Aform}. Let $(\bhatA',\bhatB')$ and $(\bstApr,\bstBpr)$ be any other rank-$d$ balanced factorizations of the matrices $\bhatM$ and $\bstM$, respectively. Then, for any $\bR' \in \bbO(d)$, there exists a $\bR \in \bbO(d)$ such that
\begin{align*}
\|(\bstAk - \bhatA \bR)(\bstBk)^\top\|_{\fro}^2&= \|(\bstAkpr - \bhatA' \bR')(\bstBkpr)^\top\|_{\fro}^2,~~
\|\bstAk(\bstBk - \bhatB \bR)^\top\|_{\fro}^2 = \|\bstAkpr(\bstBkpr - \bhatB' \bR')^\top\|_{\fro}^2\\
 \|\bstAk - \bhatA\bR\|_{\fro}^2  &=  \|\bstAkpr - \bhatApr\bR'\|_{\fro}^2,\qquad\qquad\qquad\quad~~ 
 \|\bstBk - \bhatB\bR \|_{\fro}^2  = \|\bstBkpr - \bhatBpr\bR' \|_{\fro}^2.   
\end{align*}
Moreover, if 
\begin{align*}
\rowspace\left((\bhatA_{[k]} \bR)^\top (\bhatA_{[k]} \bR)\right) \supseteq \rowspace\left((\bstAk)^\top (\bstAk)\right),
\end{align*}
then 
\begin{align*}
\rowspace\left((\bhatApr_{[k]} \bR')^\top (\bhatApr_{[k]} \bR')\right) \supseteq \rowspace\left((\bstApr_{[k]})^\top (\bstApr_{[k]})\right).
\end{align*} 
\end{lem}
The above lemma is proved  with the following fact.

\begin{lem} \label{lem:balance_construction} Let $(\bA,\bB)$ be a rank-at-most-$d$ balanced factorization of a matrix $\bM \in \R^{n \times m}$. Denote a SVD $\bM = \bU \bSigma \bV^\top$ (with $\bSigma \in \R^{d \times d}$).  Then, there exists  a rotation  matrix $\bR \in \bbO(d)$ such that $\bA= \bU \bSigma^{1/2}\bR $ and $\bB = \bV\bSigma^{1/2}\bR $. Moreover,  this $\bR$ satisfies $\bA_{[k]} = \bU\bSigmak^{1/2}\bR$, where $\bSigmak$ masks the all but the  first $k$ entries of $\bSigma$,  where $\bA_{[k]}$ is consistent with the definition of SVD as \Cref{eq:rank_k_svd}. We similarly define $\bB_{[k]}$.   
\end{lem} 
The proofs of \Cref{lem:balance_construction,lem:simple_form_wlog} are given in \Cref{sec:bal_wlog}.

\paragraph{Masking. } It will be convenient to have compact notation for masking entries of matrices. Recall that the index $d$ refers to the ``inner dimension'', e.g. $\bhatA \in \R^{n \times d}$.

To begin, we define masking for \emph{square matrices}. For matrices $\bX \in \R^{p \times d}$, and $\cK \subseteq  [d]$, define the matrix $\bX_{\cK} \in \R^{p\times d}$ by masking $\bX$'s entries in $\cK$:
\begin{align*}
(\bX_{\cK})_{ij} = \I\{i\in \cK \text{ and } j \in \cK\}\cdot\bX_{ij}. 
\end{align*}
We also define the shorthand notation 
\begin{align*}
\bX_{>k}:= \bX_{[d]\setminus [k]},
\end{align*}
with the convention that $\bX_{[0]} = 0$. 

Next, we define masking for \emph{factors matrices}. Given a matrix of the form $\bA = \bU \bSigma^{\half}$, where $\bU \in \R^{n\times d }$ and $\bSigma \in \R^{p\times d}$ is diagonal, we  define
\begin{align*}
&\bA_{\cK}:=\bU \bSigma_{\cK}^{\half}, \quad \bA_{[k]} = \bU \bSigmak^{\half}, \quad \bA_{>k} = \bU\bSigma_{>k}^{\half}, 
&(\bA,\bU,\bSigma) \in \{(\bstA,\bstU,\bstSigma),(\bhatA,\bhatU,\bhatSigma)\}. 
\end{align*}
We define analogous notation for $\bhatB,\bstB$. Finally, we define 
\begin{align*}
&\bM_{\cK} = \bA_{\cK}\bB_{\cK}^\top, \quad \bM_{[k]} = \bA_{[k]}\bB_{[k]}^\top, \quad \bM_{>k} = \bA_{>k}\bB_{>k}^\top\\
&(\bA,\bB,\bM) \in \{(\bstA,\bstB,\bstM),(\bhatA,\bhatB,\bhatM)\}. 
\end{align*}
In particular, $\bM_{[k]}$ is the rank-$k$ approximation of $\bM$, for $\bM \in \{\bhatM,\bstM\}$.

\paragraph{Partitions \&  Compatibility.} 
Importantly, we consider the set  $\cK$ which partitions the inner dimension $d$ into disjoint intervals. We call these sets  monotone partitions.

\begin{defn}[Monotone Partitions]\label{def:monotone_part} We say that $\partshort$ is a partition of $[d]$ if $\bigcup_{i=1}^{\ell}\cK_i = [d]$, and the sets $\{\cK_i\}_{i=1}^{\ell}$ are pairwise disjoint. We say that  it is a \emph{monotone} partition if there exists integers $0 = k_1 < \dots < k_{\ell} < k_{\ell+1} = p$ such that $\cK_i = \{k_i+1,\dots,k_{i+1}\}$. In  particular, this means $\cK_\ell = \{k_{\ell}+1,\dots,p\}$. We
call the entries $k_i$ the \emph{pivots} of the monotone partition, and  call the entry $k_{\ell}$ the  \emph{final pivot}.
\end{defn}

\begin{defn}[Compatibility]\label{def:compatibility}
 We say a matrix $\bX \in \R^{d \times d}$ is \emph{compatible} with a partition $\partshort$ of $[d]$ if $\bX = \sum_{i=1}^{\ell} \bX_{\cK_i}$.
\end{defn}
In particular,  if $\partshort$ is a monotone partition, then compatibility means that $\bX$ is a block-diagonal matrix whose blocks corresponding to the  indices in the sets $\cK_i$ for all $i=1,\cdots,\ell$. 

\subsection{Key error decomposition results} 
The following lemma decomposes the error across the partitions:
\newcommand{\emphcolor}[1]{{\color{red} #1}}
\begin{lem}\label{lem:decomp_by_partition} Let $\partshort$ be a monotone partition, and let $\bR \in \R^{p \times d}$ be compatible with $\partshort$. Then  
\begin{subequations}
\begin{align} 
(\bstA_{\emphcolor{[k_\ell]}} - \bhatA \bR)(\bstB_{\emphcolor{[k_\ell]}})^\top  &= \sum_{i=1}^{\emphcolor{\ell-1}} \left(\bstA_{\cK_i}  - \bhatA_{\cK_i}\bR_{\cK_i}\right)(\bstB_{\cK_i})^\top, \label{eq:A_decomp_weighted}\\
(\bstA_{\emphcolor{[k_\ell]}} - \bhatA \bR) &= \sum_{i=1}^{\emphcolor{\ell-1}} \left(\bstA_{\cK_i}  - \bhatA_{\cK_i}\bR_{\cK_i}\right) - \bhatA_{\emphcolor{>k_\ell}}\bR_{\emphcolor{>k_\ell}}, \label{eq:A_decomp_unweighted}
\end{align}
\end{subequations}
with the analogous composition being true for $(\bstBkl - \bhatB \bR)$ and $\bstAkl(\bstBkl - \bhatB \bR)^\top$. 
\end{lem}
In what follows, we choose the matrix $\bR$ above to be orthogonal. However,  orthogonality of $\bR$ is not strictly necessary the decomposition in \Cref{lem:decomp_by_partition}. Conveniently, the decomposition of $(\bstAkl - \bhatA \bR)(\bstBkl)^\top $ does not incur a dependence on the tail $\bhatA_{>k_\ell}\bR_{>k_\ell}$ of $\bhatA$, thereby avoiding all singular values after the final pivot. This is one of the reasons why the weighted error $\ErrTerm_0$ ends up being smaller than the unweighted $\ErrTerm_1$. The proof of \Cref{lem:decomp_by_partition} is given in Subsubsection \ref{sec:proof:decomp_by_partition}. 

In our analysis, we consider partitions of $[d]$ that enjoy favorable spectral properties: first, every pivot $k_i$ has large relative spectral gap $\delst_{k_i}$ (for $\bstM$), and second, the largest singular value in each partition is at most a constant times  that of the pivot of the next partition. 

\begin{defn}[Well-Tempered Partition]\label{defn:well_tempered_partition} We say a partition $\partshort$ is \emph{$\temparams$-well-tempered} if it is monotone, and for all $i \in [\ell]$, the corresponding pivots $k_i = \min_k\{k-1:k\in\cK_i\}$  satisfy 
\begin{itemize}
    \item[(a)] $\delst_{k_i} \ge \delspace$; 
    \item[(b)]$\max\{\sigst_{k'} : k' \in \cK_i\} \le \mucond \cdot \sigst_{k_{i+1}}$. 
\end{itemize}

For such a partition, we define  the constants 
$$
\muspace := \sum_{i=1}^{\ell} (\delst_{k_i})^{-2},~~\text{and}~~\muspec := \sum_{i=1}^{\ell} (\sigst_{k_i})^{-1}. 
$$
\end{defn}
At the end of the proof, we show that a well-tempered partition always exists with $\mucond = \BigOh{1}$ and $\delspace = 1/k$, and where $\muspace$ and $\muspec$ are well-behaved. For now, let us carry the analysis out in terms of the properties of the supposed partition. The key object in the analysis is the normalized error:

\begin{defn}[Normalized  Factored Error]\label{eq:normalized_error} Let $\partshort$ be a monotone partition. We define the \emph{normalized error term} for $i\in[\ell]$ as  
\begin{align*}
E_i(\bR) :=  \left\{(\delst_{k_i} \wedge \delst_{k_{i+1}})^2 (\sigst_{k_{i+1}})\right\}\cdot\max\left\{\|\bstA_{\cK_i}  - \bhatA_{\cK_i}\bR_{\cK_i}\|_{\fro}^2,\,\|\bstB_{\cK_i}  - \bhatB_{\cK_i}\bR_{\cK_i}\|_{\fro}^2\right\}. 
\end{align*}
\end{defn}
We now bound our given error terms in terms of the $E_i$-quantities. The proof of the following lemma is given in Subsubsection \ref{sec:lem:well_tempered_inter}.
 
\begin{lem}\label{lem:well_tempered_intermediate} Let $\bR \in \bbO(d)$ be compatible with a $\temparams$-well-tempered partition  $\partshort$ with pivots $\{k_i\}_{i\in[\ell]}$. Then
\begin{subequations}
\begin{align} 
\ErrTerm_0(\bR,k_\ell) &\le 2\mucond\cdot\muspace  \cdot \max_{i \in [\ell-1]} E_i(\bR)\\
\ErrTerm_1(\bR,k_\ell) &\le \frac{\muspec}{\delspace^2}\cdot\max_{i \in [\ell-1]}E_i(\bR) + \sum_{i > k_\ell} \sigma_i(\bhatM).
\end{align}
\end{subequations}
\end{lem}

\subsection{Controlling the normalized errors}

As shown in \Cref{lem:well_tempered_intermediate}, bounds on both $\ErrTerm_0$ and $\ErrTerm_1$ amount to bounding (suitably rotated) errors between the factorizations  of the ground-truth and estimated matrix. To do so, we  invoke the following factorization lemma due to \cite{tu2016low}.

\begin{lem}[\citep{tu2016low}, Lemma 5.14]\label{lem:proc} Let $(\bA,\bB)$ and $(\bA',\bB')$ be rank-$r$ balanced factorization of matrices $\bM$ and $\bM'$, respectively. Suppose that $\|\bM - \bM'\|_{\op} \le \frac{1}{2}\sigma_r(\bM)$. Then, there exists an orthogonal matrix $\bO \in \bbO(r)$ such that
\begin{align*}
\|\bA - \bA' \bO\|_{\fro}^2 + \|\bB - \bB'\bO\|_{\fro}^2 \le \cproc \frac{\|\bM - \bM'\|_{\fro}^2}{ \sigma_r(\bM)}, 
\end{align*} 
where $\cproc =\frac{2}{\sqrt{2}-1}$. 
\end{lem}

As a consequence, we can deduce the following bound on the normalized error terms. 

\begin{lem}\label{lem:application_of_procrustes} Let $\partshort$ be a  $\temparams$-well-tempered partition with pivots $\{k_i\}_{i\in[\ell]}$. Define  
\begin{align}
\epstilop := \max_{i \in [\ell+1]}\delst_{k_i}\|\bhatM_{[k_i]} - \bstM_{[k_i]}\|_{\op}, \quad \epstilfro := \max_{i \in [\ell+1]}\delst_{k_i}\|\bhatM_{[k_i]} - \bstM_{[k_i]}\|_{\fro}. 
\end{align}  
Then, if $\epstilop \le \frac{\delspace\sigst_{k_\ell}}{4}$, 
there  exists a $\bR \in\bbO(d)$ which is compatible with $\partshort$ such that 
\begin{align*}
\max_{i \in[\ell]} E_i(\bR)   &\le 4\cproc \epstilfro^2 \lesssim \epstilfro^2,
\end{align*}
{where $\cproc =\frac{2}{\sqrt{2}-1}$.} 
\end{lem}

To prove the above lemma, we invoke \Cref{lem:proc} to bound $\|\bstA_{\cK_i}  - \bhatA_{\cK_i}\bR_{\cK_i}\|_{\fro}$ and $\|\bstB_{\cK_i}  - \bhatB_{\cK_i}\bR_{\cK_i}\|_{\fro}$ in terms of $\|\bstM_{\cK_i} - \bhatM_{\cK_i}\|_{\fro}$. Then we notice that $\bhatM_{\cK_i} =  \bhatM_{[k_{i+1}]}-\bhatM_{[k_i]}$ and  $\bstM_{\cK_i} = \bstM_{[k_{i+1}]}-\bstM_{[k_i]}$,  so we can derive a bound directly in terms of the differences between rank-$k_i$ SVDs. The proof is given in Subsubsection \ref{sec:lem:app_procrustes}. In applying \Cref{lem:application_of_procrustes}, we crudely bound  $\epstilop \le \epstilfro$, but the above lemma is stated so that a more refined analysis may be possible.  

We now bound $\epstilfro$ using a generic bound for the SVD decomposition,   which we state below. 
\svdpert*

Taking $\eta = 1/10$, and noting that $\delst_{k_i} \ge \delspace$ and $\sigst_{k_i} \ge  \sigst_{k_\ell}$ for all pivots in a $\temparams$-well-tempered partition, we have the following corollary. 
\begin{cor}\label{cor:one_svd} Suppose $\partshort$ is $\temparams$-well-tempered, and $\|\bhatM - \bstM\|_{\op} \le \sigst_{k_\ell}\delspace/10$.  
Then, $\epstilfro \le 10 \|\bhatM - \bstM\|_{\fro}$. 
\end{cor}

Combining  with \Cref{lem:application_of_procrustes}, and then \Cref{lem:well_tempered_intermediate}, we obtain the following guarantee.

\begin{prop}\label{prop:almost_final_bound} Suppose $\partshort$ is $\temparams$-well-tempered, and $\|\bhatM - \bstM\|_{\fro} \le \delspace\sigst_{k_\ell}/40$.  
Then, there exists an   orthogonal matrix $\bR \in \bbO(d)$  compatible with $\partshort$ such that
\begin{align*}
\max_{i \in[\ell]} E_i(\bR)   &\le 400\cproc \|\bhatM - \bstM\|_{\fro}^2 \lesssim \|\bhatM - \bstM\|_{\fro}^2. 
\end{align*}
Therefore, by \Cref{lem:well_tempered_intermediate}, this orthogonal $\bR$ and index $k_\ell$ satisfy   
\begin{subequations}
\begin{align}
\ErrTerm_0(\bR,k_\ell) &\lesssim \mucond\muspace  \|\bhatM - \bstM\|_{\fro}^2, \label{eq:ErrZeroAlmostFinal}\\ 
\ErrTerm_1(\bR,k_\ell) &\lesssim \frac{\muspec}{\delspace^2}\|\bhatM - \bstM\|_{\fro}^2 + \sqrt{r}\|\bstM - \bhatM\|_{\fro} + \tail_1(\bstM;k_\ell). \label{eq:ErrOneAlmostFinal}
\end{align}
\end{subequations}
\end{prop}
\begin{proof} If $\|\bhatM - \bstM\|_{\fro} \le \delspace\sigst_{k_\ell}/40$, then also $\|\bhatM - \bstM\|_{\op} \le \sigst_{k_\ell}\delspace/10$, so $\epstilfro \le 10 \|\bhatM - \bstM\|_{\fro}$. Hence, $\epstilop \le \epstilfro \le 10 \|\bhatM - \bstM\|_{\fro} \le \delspace\sigst_{k_\ell}/4$. Thus, by \Cref{lem:application_of_procrustes} followed by \Cref{cor:one_svd},  there exists an  orthogonal matrix $\bR \in \bbO(d)$ compatible with $\partshort$ for which
\begin{align*}
\max_{i \in[\ell]} E_i(\bR)   &\le 4\cproc \epstilfro^2 \le 400\cproc \|\bhatM - \bstM\|_{\fro}^2 \lesssim \|\bhatM - \bstM\|_{\fro}^2. 
\end{align*}
\Cref{eq:ErrZeroAlmostFinal} now follows directly from \Cref{lem:well_tempered_intermediate}. To achieve \Cref{eq:ErrOneAlmostFinal}, we see that directly from \Cref{lem:well_tempered_intermediate}, 
\begin{align} 
\ErrTerm_1(\bR,k_\ell) &\lesssim \frac{\muspec}{\delspace^2}\|\bhatM - \bstM\|_{\fro}^2 + \sum_{i > k_\ell} \sigma_i(\bhatM).  \label{eq:ErrTermOne_Inter}
\end{align}
Using that $\rank(\bhatM) = r$, we have 
\begin{align*}
\sum_{i > k_\ell} \sigma_i(\bhatM) = \sum_{i = k_\ell+1}^r \sigma_i(\bhatM)&\le \sum_{i = k_\ell+1}^r \sigma_i(\bstM) + \sum_{i = k_\ell+1}^r |\sigma_i(\bstM) - \sigma_i(\bhatM)|\\
&\le \sum_{i = k_\ell+1}^r \sigma_i(\bstM) + \sqrt{r\sum_{i = k_\ell+1}^r  |\sigma_i(\bstM) - \sigma_i(\bhatM)|^2}\\
&\le \sum_{i = k_\ell+1}^r \sigma_i(\bstM) + \sqrt{r}\|\bstM - \bhatM\|_{\fro} \tag*{\Cref{lem:Mat_pert_one}}\\
&:= \tail_1(\bstM;k_\ell) +  \sqrt{r}\|\bstM - \bhatM\|_{\fro}. 
\end{align*}
The desired bound follows by combining with \Cref{eq:ErrTermOne_Inter}.  
\end{proof}


\subsection{Existence of well-tempered partition}

To conclude the proof, it suffices to demonstrate the existence of a well-tempered partition of the singular values of $\bstM$, for which $\muspace,1/\delspace,\mucond,\muspec$ are all of reasonable magnitude. To do so, we focus on the pivots. One important subtlety  is that, for any given $k \in \N$, $\delstk$ may be very small, indeed even equal to zero. 

Hence, to construct the well-tempered  partition, we take in a target rank $k_{i+1}$, and show that we can use a slightly smaller rank $k_i$ for which $\delst_{k_i} \ge 1/k_{i+1}$. We then argue that we can construct a sequence of pivots $k_1,k_2,\dots$ for which singular values within those pivots are within a constant factor (the $\mu$-parameter for well-temperedness), the $\delst_{k_i}$ parameters are lower bounded (hence lower bounding the $\delta$ parameter). In addition, this partition ensures that the singular values at the pivot points grow \emph{at least} geometrically. This is helpful to control $\muspace$ and $\muspec$. The following technical lemma is proved in Subsubsection \ref{sec:lem:ssv_space}.


\begin{restatable}[Singular Value Spacing]{lem}{lemssv}\label{lem:ssv_space}
Fix any $s \in \N$ and $\sigma \in [\sigst_{s},\sigst_{1}]$. Then, there exists integer  $\ell \in \N$, and an  increasing sequence $0 = k_1 < k_2 \dots < k_{\ell} < k_{\ell+1} = s$  such that the following is true:
\begin{itemize}
    \item[(a)] For $i \in [\ell]$, $\delst_{k_i} \ge 1/k_{i{+}1} \ge 1/s$. 
    \item[(b)] For $i =\ell$, $\sigst_{k_i+1} \le 2e\sigma$, and for $i \in [\ell-1]$, $\sigst_{k_{i}+1}\le 2e^2\sigst_{k_{i+1}}$. 
    \item[(c)] For $i =\ell$, $\sigst_{k_i} \ge \sigma$, and for $i \in [\ell-1]$,  $\sigst_{k_{i}} \ge e\sigst_{k_{i+1}}$. 
\end{itemize}
\end{restatable}
With this technical lemma in hand, we can demonstrate the existence of a well-tempered partition with a number of desirable properties. The following is proved in Subsubsection \ref{sec:prop_exist_well_temp}. 

\begin{prop}[Well-Tempered Partition] 
\label{prop:exist_well_tempered} Fix any $s\in \N$ and $\sigma \in [\sigst_{s},\sigst_{1}]$. There exists a partition $\partshort$ of $[s]$,  which is  $(\delspace,\mu)$-well-tempered  with parameters satisfying
\begin{itemize} 
    \item[(a)] $\delspace \ge 1/s$ and $\mu \le 2e^2$. 
    \item[(b)] $k_\ell < s$, $\sigst_{k_\ell} \ge \sigma$, and $\muspec \le \frac{(\sigma)^{-1}}{1 - e^{-1}}$.
    \item[(c)] $\muspace \le \barl \cdot s^2$, where $\barl := \min\{1 + \ceil{\log \frac{\|\bstM\|_{\op}}{\sigma}},\,s\}$. 
    \item[(d)] $\tail_1(\bstM;k_\ell) \le 2e s \sigma +\tail_1(\bstM;s) $ and $\tail_2(\bstM;k_\ell) \le 4e^2 s \sigma^2 +\tail_2(\bstM;s)$.  
\end{itemize}
\end{prop}

\subsection{Proof of \Cref{thm:main_matrix}}
We recall the theorem here for convenience.
\thmmatrix*
\begin{proof}[Proof of \Cref{thm:main_matrix}] Fix any $\kpick\in \N$. To tune the bound, we also fix a parameter $\sigma \in [\sigma_{\kpick}(\bstM),\opnorm{\bstM}]$. We shall tune $\sigma$ at the end of the proof such that the following inequality is satisfied
\begin{align}
\epsilon \le \frac{\sigma}{40 s}. \label{eq:eps_ineq}
\end{align}

\paragraph{Extracting the balanced partition.}
Consider the balanced partition that arises from applying  \Cref{prop:exist_well_tempered} with parameter $\kpick$ and singular value parameter $\sigma$, and let $k_\ell$  be the resulting last  pivot.  Note that $\sigst_{k_{\ell}} \ge \sigma$, $\delst_{k_{\ell}} \ge 1/\kpick$, and $k_{\ell}<s$; i.e. $k_{\ell} \in [s-1]$. We shall ultimately choose the promised $k$ in the main theorem to be $k_{\ell}$, but retain the $\ell$-subscript for clarity in the proof below.

 As a consequence of \Cref{eq:eps_ineq}, we have
\begin{align}
\epsilon \le \frac{\sigst_{k_{\ell}}}{40 s}.  \label{eq:eps_ineq_two}
\end{align}
Note that \Cref{eq:eps_ineq_two} implies $k_{\ell} \le r$, because
\begin{align*}
\sigma_{k_{\ell}}(\bhatM) \ge \sigst_{k_{\ell}} - \|\bhatM - \bstM\|_{\op} \ge \sigst_{k_{\ell}} - \epsilon > 0.
\end{align*}

Moreover, we have
\begin{align}
\partshort \text{ is } (\delta,\mu)\text{-well tempered for } \delta = 1/s, \quad \mu \le 2e^2 = \BigOh{1}. \label{eq:well_tempered_eq_thing}
\end{align}
In particular,
\begin{align}
\epsilon \le \frac{\delta\sigst_{k_{\ell}}}{40} \label{eq:eps_cond_sufficient}.
\end{align}
We shall use \Cref{eq:eps_cond_sufficient} as the sufficient condition to invoke \Cref{prop:almost_final_bound}.  In addition, due to Weyl's inequality and \Cref{eq:eps_ineq_two}, 
\begin{align}
\sigma_{k_{\ell}}(\bhatM) \ge \sigst_{k_{\ell}} - \|\bhatM - \bstM\|_{\op} \ge \frac{39}{40}\sigst_{k_{\ell}} \ge 39\epsilon > 0 \label{eq:sigk_lb}.
\end{align}
We shall use this lower bound to verify the positive semi-definite  domination of $\bstPk$ at the end of the proof. Finally, we can also check that $\sigst_{k_{\ell}} - \sigst_{k_{\ell+1}} \ge 40\epsilon/s$ via similar manipulations.

\paragraph{Applying the error bounds.} In addition, \Cref{eq:eps_cond_sufficient} allows us to apply \Cref{prop:almost_final_bound}. This means that there exists an orthogonal matrix $\bR \in \bbO(d)$ which is compatible with $\partshort$ such that the following holds 
\begin{subequations}
\begin{align*}
\ErrTerm_0(\bR,k_{\ell}) &\lesssim \epsilon^2 \cdot \mu \cdot\muspace \lesssim \epsilon^2 \kpick^2\barl \tag{$\mu \lesssim 1$, $\muspace \lesssim  \kpick^2\barl$ }\\
\ErrTerm_1(\bR,k_{\ell}) &\lesssim \sqrt{r}\epsilon + \delta^2\muspec \epsilon^2  + \tail_1(\bstM;k_{\ell}) \\
&\lesssim \sqrt{r}\epsilon + \frac{\epsilon^2\kpick^2}{\sigma}  + \kpick \sigma +  \tail_1(\bstM;\kpick) \tag{$\delta\geq 1/s$, $\muspec \lesssim 1/\sigma$, $\tail_1(\bstM;\kpick) \lesssim \kpick \sigma +  \tail_1(\bstM;\kpick)$ }\nonumber \\
&\lesssim \sqrt{r}\epsilon + \frac{\sigma^2\kpick^2}{\sigma\kpick^2}  + \kpick \sigma +  \tail_1(\bstM;\kpick)  \tag{$\epsilon \le \sigma/(40\kpick)$ due to \Cref{eq:eps_ineq}}\\
&\lesssim \sqrt{r}\epsilon + \kpick \sigma +  \tail_1(\bstM;\kpick) .
\end{align*}
\end{subequations} 
Above, we used \Cref{prop:exist_well_tempered} which affords $\mu \lesssim 1$, $\muspace \lesssim  \kpick^2\barl$, $\delta\geq 1/s$, $\muspec \lesssim 1/\sigma$, and $\tail_1(\bstM;\kpick) \lesssim \kpick \sigma +  \tail_1(\bstM;\kpick)$. To summarize,
\begin{align}\label{eq:err_term_summary}
\ErrTerm_0(\bR,k_{\ell}) \lesssim \epsilon^2 \cdot \mu \cdot\muspace \lesssim \epsilon^2 \kpick^2\barl , \quad \ErrTerm_1(\bR,k_{\ell}) \le \sqrt{r}\epsilon + \kpick \sigma +  \tail_1(\bstM;\kpick).  
\end{align}
In addition, note from \Cref{prop:exist_well_tempered} that
\begin{align*}
\tail_2(\bstM;k_{\ell}) \lesssim \kpick \sigma^2 +\tail_2(\bstM;\kpick). 
\end{align*}

\paragraph{Tuning parameter  $\sigma$.}
We choose
\begin{align}
\sigma = \max\{\sigst_{\kpick},40\kpick\epsilon\}. \label{eq:sigma_choice}
\end{align}
This ensures that two of our constraints on $\sigma$ are satisfied: i.e.  $\sigma \ge \sigst_{\kpick}$, and that $\epsilon \le \frac{\sigma}{40\kpick}$. For our third constraint,  $\sigma \le \sigst_1$, to hold, this requires that $\epsilon \le \sigst_1/40\kpick$, which is ensured by the condition of the theorem. 

\paragraph{Applying our choice of $\sigma$ to the error bounds.}   
 For this choice of $\sigma$, we have
\begin{align*}
\barl &:= \min\left\{1 + \ceil{\log \frac{\|\bstM\|_{\op}}{\sigma}},\,\kpick\right\} \\
&\le \min\left\{1 + \ceil{\log \underbrace{\frac{\|\bstM\|_{\op}}{40\kpick \epsilon}}_{\ge 1}},\,\kpick\right\}
\lesssim  \min\left\{1 + \log \frac{\|\bstM\|_{\op}}{40\kpick \epsilon},\,\kpick\right\} := \lstargs. 
\end{align*}
Thus, by \Cref{eq:err_term_summary}
\begin{align*}
\ErrTerm_0(\bR,k_{\ell}) &\lesssim \epsilon^2 \cdot \kpick^2\lstargs.
\end{align*}
Similarly, \Cref{eq:err_term_summary},
\begin{align*}
\ErrTerm_1(\bR,k_{\ell}) &\lesssim \sqrt{r}\epsilon + \kpick \max\{\sigst_{\kpick},40\kpick\epsilon\} +  \tail_1(\bstM;\kpick)\\
&\lesssim (\sqrt{r}+\kpick^2)\epsilon + \kpick \sigst_{\kpick}  +  \tail_1(\bstM;\kpick) .
\end{align*}

Finally, using $\sigma = \max\{\sigst_{\kpick},40\kpick\epsilon\}$, we bound
\begin{align*}
\tail_2(\bstM;k_{\ell}) \lesssim \kpick \sigma^2 + \tail_2(\bstM;\kpick) \lesssim \kpick^3 \epsilon^2 + \kpick (\sigst_{\kpick})^2 +  \tail_2(\bstM;\kpick). 
\end{align*} 
We conclude by setting $k = k_{\ell}$.

\paragraph{Checking PSD domination.}
Lastly, we check the relevant PSD relation. Recall our  choice $k= k_{\ell}$. Let $\cV_k$ denote the range of the projection $\bstPk$, which is the span of the first $k$ basis vectors   under \Cref{eq:Aform}. Let $\bv = (\bv_1 + \bv_2) \in \R^p$ be such that $\bv_1 \in \cV_k$, and $\bv_2$ is supported on the remaining  $p-k$ basis vectors. Then, since $\bR$ is orthogonal and compatible with $\partshort$ and since $\bigcup_{i=1}^{\ell} \cK_{i} = [k]$, we have 
\begin{align*}
\bR \cdot(\bv_1 + \bv_2) = (\bw_1 + \bw_2)
\end{align*}
where $\|\bw_1\| = \|\bv_1\|$, $\|\bw_2\| = \|\bv_2\|$, and again $(\bw_1,\bw_2)$  decomposes into the first $k$ and remaining $p-k$ coordinates. Using \Cref{eq:Aform}, moreover, $\bhatA^\top \bhatA = \bhatSigma$. Hence, 
 \begin{align*}
 \bv^\top \bR^\top \bhatA^\top \bhatA \bR\bv &=  (\bw_1 + \bw_2)^\top \bhatSigma (\bw_1 + \bw_2)\\
 & \ge \sigma_k(\bhatSigma) \|\bw_1\|^2 = \sigma_{k_{\ell}}(\bhatSigma) \|\bv_1\|^2 \\
 &= \sigma_k(\bhatSigma) \bv_1^\top \bstPk\bv_1 = \sigma_k(\bhatSigma) \bv^\top \bstPk\bv,
 \end{align*}
 as $\bv_1$ is the projection of $\bv$ onto $\bstPk$. Lastly, as $\bhatA$ is balanced, $\sigma_{k}(\bhatSigma) = \sigma_{k}(\bhatM) := \sigma_{k_{\ell}}(\bhatM) \ge 39 \epsilon$ due to \Cref{eq:sigk_lb}. 
 This completes the proof of our \Cref{thm:main_matrix}.  
\end{proof}

\newcommand{\gstp}{\gst_p}
\newcommand{\fstp}{\fst_p} 

\newcommand{\outpr}{\fauxtimes 2}
\newcommand{\distxbar}{\bar{\cD}_{\xspac}}
\newcommand{\distybar}{\bar{\cD}_{\yspace}}
\newcommand{\phixtau}{\phi_{\xspac;\tau}}
\newcommand{\phiytau}{\phi_{\yspace;\tau}}
\newcommand{\cdxtau}{\cD_{\xspac;\tau}}
\newcommand{\cdytau}{\cD_{\yspace;\tau}}
\newcommand{\cdxtaun}{\cD_{\xspac;\tau_n}}
\newcommand{\cdytaun}{\cD_{\yspace;\tau_n}}
\newcommand{\cdtau}{\cD_{\tau}}
\newcommand{\cdtaun}{\cD_{\tau_n}}
\newcommand{\distime}{\cD_{\otimes}}
\newcommand{\distimes}{\cD_{\otimes}}

\newcommand{\Ltwo}{\cL_{2}}
\newcommand{\Ltwoto}[1]{\overset{\Ltwo(#1)}{\to}}

\newcommand{\sfR}{\bmsf{R}}
\newcommand{\epstilone}{\tilde{\epsilon}_{1\otimes 1}}
\newcommand{\fbarst}{\bar{f}^\star}
\newcommand{\gbarst}{\bar{g}^\star}
\newcommand{\Sigbarst}{\bar{\bmsf{\Sigma}}^\star}

\section{From Matrix Factorization to Bilinear Embeddings}\label{sec:proof_from_matrix_to_feature}

This section gives the limiting arguments that proceed from results about matrices to results about Hilbert-space embeddings under potentially non-discrete distributions. Specifically, we prove the following

 \begin{restatable}[Error on $\cD_{1\otimes 1}$]{theorem}{oneblock}\label{thm:one_block} Suppose $(\fhat,\ghat)$ are $\R^r$-embeddings. Then, for any $s\in \N$ and error bound $\epsilon > 0$ such that   
 \iftoggle{colt}
 {
 (i) $\epsilon^2 \ge \inf_{s' \ge s-1} \Riskr[s'](\fhat,\ghat;\cdone)$ and (ii) $s < \frac{\|\Sigst\|_{\op}}{40\epsilon}$
 }
 {
\begin{align*}
\epsilon^2 \ge \inf_{s' \ge s-1} \Riskr[s'](\fhat,\ghat;\cdone), \quad s < \frac{\|\Sigst\|_{\op}}{40\epsilon}, 
\end{align*}, 
}, then we have: (a) if $(\fhat,\ghat)$ are \emph{full-rank}, then 
there exists an index $k \in [\min\{r,s-1\}]$ and functions $f:\xspac \to \hilspace$ and $g:\yspace \to \hilspace$ such that  $(f,g)$ are aligned $k$-proxies and the error terms are bounded by 
\begin{subequations}
\begin{align}
   &\mathrm{(weighted~error)}~~~~~ \ErrTermHil_0(f,g,k) + \tailsf_2(k)  \lesssim \kpick^3 \epsilon^2 + \kpick (\bsigst_{\kpick})^2  +  \tailsf_2(\kpick),\label{eq:term_zero_thm}\\
    &\mathrm{(unweighted~error)}~~~~~\ErrTermHil_1(f,g,k) \lesssim (\sqrt{r}+\kpick^2)\epsilon + \kpick\bsigst_{\kpick}  +  \tailsf_1(\kpick);\label{eq:term_one_theorem}
    \end{align}
\end{subequations}
and (b) if $\epsone^2 \le (1-\alpha^{-1})(\bsigst_r)^2$ for some $\alpha\geq 1$, then $(\fhat,\ghat)$ are necessarily full-rank, and $\sigma_r(\fhat,\ghat)^2 \ge (\bsigst_r)^2/\alpha$, where we recall the definition of $\sigma_r(\fhat,\ghat)$ in \Cref{eq:sigi_def}. 
\end{restatable}
\begin{rem}
A few remarks are in order. The condition $\epsilon^2 \ge \inf_{s' \ge s-1} \Riskr[s'](\fhat,\ghat;\cdone)$ is for technical convenience; for intuition, one should think of $\epsilon^2 = \Risk(\fhat,\ghat;\cdone)$ as the risk on the ``top-block''.  
Next, we observe the differences in scaling: due to the weighting, $\ErrTermHil_0(f,g,k) + \tailsf_2(k) $ scales with $\epsilon^2$, and with the \emph{squares} of singular values,  whereas  $\ErrTermHil_1(f,g,k)$ scales with $\epsilon$, and $\ell_1$-sums of singular values. This is essential, because it means that the term $ \frac{1}{\sigma^2}(\tailsf_2(k) + \ErrTermHil_0(f,g,k) + \errtrain)^2$ in \Cref{prop:final_error_decomp_simple} can decay to zero. Lastly, our theorem gives us sufficient conditions on which $\sigma_r(\fhat,\ghat)$, which appears in the aforementioned  \Cref{prop:final_error_decomp_simple}, is indeed lower bounded.
\end{rem}

We begin in \Cref{sec:factor_recovery} by stating an intermediate guarantee for \Cref{thm:one_block}, \Cref{thm:one_block_main}, to whose proof the majority of this appendix is devoted, and provide preliminaries   and review proof-specific notation in \Cref{sec:limiting_preliminaries}.

To prove \Cref{thm:one_block_main}, we adopt the standard technique of approximation by so-called \emph{simple functions}:
\begin{defn}[Simple Functions]\label{defn:simple_function} Let $\scrZ$ be an abstract domain. We say that a function $\psi: \scrZ \to \R^p$ is simple if its image $\psi(\scrZ)$ is a set of finite cardinality. 
\end{defn}
In \Cref{sec:simple_func}, we show that our factorization theorem for matrices (i.e. \Cref{thm:main_matrix})  directly implies \Cref{thm:one_block_main}. Subsequently, \Cref{sec:finite_dim_hib} extends the guarantees to arbitrary (possible non-simple) functions, but with the restriction that they have a finite-dimensional range. The idea is to approximate our actual functions $f,\fst,g,\gst$ as the limit of simple functions. The steps in this section are mostly routine, but some care must be taken to ensure all the  simple functions can be \emph{balanced} under $\cdone$ in the sense of \Cref{defn:balanced}; recall that $(f,g)$ are balanced (under $\cdone$) if
\begin{align*}
\Exp_{\cdxone}[ff^\top] = \Exp_{\cdyone}[gg^\top].
\end{align*} 
To facilitate this, we  show that when all the embeddings have we can  approximate $f,\fst,g,\gst$ by simple functions whose range is ``smaller'' than the limiting function they approximate. Care must also be taken to handle the rotation matrices which align the functions $(\fst,\gst)$ with their estimates $(f,g)$. 

Finally, \Cref{sec:proof:thm:one_block_main} removes the restriction of a finite dimensional range, thereby concluding the proof of \Cref{thm:one_block_main}.  Subsubsection \ref{sec:lem:reduction_to_discrete} contains the proof of all supporting claims. Lastly, \Cref{thm:svd_general_thing} provides the generalization of our main SVD perturbation lemma, \Cref{thm:svd_pert}, to general distributions.

\subsection{Factor recovery for one block, \Cref{thm:one_block}}\label{sec:factor_recovery}

We first prove a variant of \Cref{thm:one_block}, from which that theorem can be readily derived.

\begin{theorem} 	\label{thm:one_block_main} Suppose that $(\fhat,\ghat)$ embeddings $ \Exp_{\disone}[(\langle \fhat,\ghat\rangle - \langle \fst,\gst \rangle)^2]\le \epsilon^2$, and pick any positive $\kpick \in \N$ and $\kpick>1$ such that $\kpick < \frac{\|\Sigst\|_{\op}}{40\epsilon}$. Then, $(\fhat,\ghat)$ are full-rank, there exists a $k \in[s-1]$ and functions $f:\xspac \to \hilspace$ and $g:\yspace \to \hilspace$ such that 
\begin{itemize} 
    \item[(a)]  $\langle f(x),g(y)\rangle = \langle \fhat(x),\ghat(y)\rangle $ for all $(x,y)$.
\item[(b)] The functions $(f,g)$ are valid proxies for $\fhat,\ghat$ in the sense of \Cref{defn:valid_proxies_cont}.
\item[(c)] The following error terms
\begin{align*}
\ErrTerm_0(f,g,k) &:= \max\left\{\Exp_{\cD_{1\otimes 1}}\left[\hilprodd{\fstk , \gstk - g}^2\right], \,\Exp_{\cD_{1\otimes 1}}\left[\hilprodd{\fstk - f , \gstk}^2\right]\right\}  \\
\ErrTerm_1(f,g,k) &:= \max\left\{\Exp_{\cdx{1}}\hilnormm{\fstk - f}^2,\,\Exp_{\cdy{1}}\hilnormm{\gstk - g}^2\right\}  
\end{align*} 
are bounded by 
\begin{align*}
\ErrTerm_0(f,g,k) + \tail_2(\Sigst;k) &\lesssim \kpick^3 \epsilon^2 + \kpick (\bsigst_{\kpick})^2 +  \tail_2(\Sigst;\kpick)\\
\ErrTerm_1(f,g,k)  &\lesssim (\sqrt{r}+\kpick^2)\epsilon + \kpick \bsigst_{\kpick}  +  \tail_1(\Sigst;\kpick),
\end{align*}
\item[(d)] For any $j$, $\sigma_j(\fhat,\ghat) \ge \sigma_j(\bstSigma) - \epsilon$. 
\end{itemize}
Moreover, if instead of assuming $(\fhat,\ghat)$ are full-rank, but in addition we assume that $\epsilon < \sigma_{r}(\Sigst)$, then $(\fhat,\ghat)$ are guaranteed to be full-rank so that the conclusion of the above theorem holds.
\end{theorem}

\newcommand{\Sigbarone}{\bar{\bSigma}_{1\otimes 1}}
Let us now prove \Cref{thm:one_block}.
\begin{proof}[Proof of \Cref{thm:one_block}] Fix any $s$. Consider any $s' \ge s-1$, and define $\fbarst := \fst_{s'}$ and $\gbarst := \gst_{s'}$. Define $\bar\ErrTerm_0(f,g,k)$, $\bar\ErrTerm_1(f,g,k)$ and $\Sigbarone$analogously, with $\fst,\gst$ replaced by $\fbarst,\gbarst$.  Finally, set $\epsilon^2 := \Exp_{\disone}[(\langle f,g\rangle - \langle \fbarst,\gbarst \rangle)^2]$. Applying \Cref{thm:one_block_main} with $\fst,\gst \gets \fbarst,\gbarst$, we find the existence of $f,g$ satisfying points (a), (b), as well as well as
\begin{align*}
\bar\ErrTerm_1(f,g,k) &\lesssim (\sqrt{r}+\kpick^2)\epsilon + \kpick \sigma_{\kpick}(\Sigbarone)  +  \tail_1(\Sigbarone;\kpick)\\ 
\bar\ErrTerm_0(f,g,k) + \tail_2(\Sigbarone;k) &\lesssim \kpick^3 \epsilon^2 + \kpick \sigma_{\kpick}(\Sigbarone)^2 +  \tail_2(\Sigbarone;\kpick),
\end{align*}
where above we bounded $\lstargs \le s$. 
Observe that $k \in [s-1]$, it holds that $\fbarst_k = \fst_k$ and $\gbarst_k = \gst_k$. Moreover, since $\Sigbarone \preceq \Sigst$ (since the former is an SVD approximation of the latter), and since $ \sigma_{\kpick}(\Sigbarone)^2 \le  \sigma_{\kpick}(\Sigst)^2$,
\begin{align*}
\bar\ErrTerm_1(f,g,k) &\lesssim (\sqrt{r}+\kpick^2)\epsilon + \kpick \bsigst_{\kpick}  +  \tail_1(\Sigst;\kpick)\\
\bar\ErrTerm_0(f,g,k)+\tail_2(\Sigbarone;k) &\lesssim \kpick^3 \epsilon^2 + \kpick (\bsigst_{\kpick})^2 +  \tail_2(\Sigst;\kpick) + \epsilon^2 \cdot \kpick^3. 
\end{align*}
Lastly, notice that 
\begin{align*}
\tail_2(\Sigst;k) &= \tail_2(\Sigbarone;k) + \sum_{i > s'}(\bsigst_i)^2\\
&\le \tail_2(\Sigbarone;k) + (\bsigst_s)^2 + \tail_2(\Sigst;s) \\
&\lesssim  \kpick^3 \epsilon^2 + \kpick (\bsigst_{\kpick})^2 +  \tail_2(\Sigst;\kpick), 
\end{align*}
where above we use $s' \ge s-1$. Hence, these differences get absorbed by the above bound on $\tail_2(\Sigbarst;k) $, yielding 
\begin{align*}
\ErrTerm_0(f,g,k)  + \tail_2(\Sigbarst;k) &\lesssim \kpick^3 \epsilon^2 + \kpick (\bsigst_{\kpick})^2 +  \tail_2(\Sigst;\kpick). 
\end{align*}
Since the above was true for any $s' \ge s-1$, we can replace $\epsilon$ with any $\epsilon$ satisfying
\begin{align*}
\epsilon^2 \ge \inf_{s' \ge s-1} \Exp_{\disone}[(\langle f,g\rangle - \langle \fst_{s'},\gst_{s'} \rangle)^2],
\end{align*}
as needed. Finally, the last part of \Cref{thm:one_block} is directly implied by \Cref{thm:one_block_main}(d).
\end{proof}

\subsection{Proof preliminaries}\label{sec:limiting_preliminaries}

For the majority of the proof, we assume that $\hilspace = \R^p$; that is,  the embeddings are finite dimensional (recall that all finite dimensional Hilbert spaces are isomorphic). This restriction is the simplest to remove, so we save removing it till the end of the argument. We also study balanced functions $f,g$ directly, and remove the balancing requirement at the end. 

\paragraph{Setup.} Let $\distx$ and $\disty$ be distributions over $\xspac$ and $\yspace$ which have finite support, and let $\distime:= \distx\otimes \disty$ denote the product measure. We consider functions $f,\fst:\xspac \to \R^p$ and $g,\gst:\yspace \to \R^p$ whose inner products have squared error $\epspred^2$:
\begin{align*}
\epspred^2 = \Exp_{\distimes}[(\langle f,g\rangle - \langle \fst,\gst\rangle )^2], \quad \distimes := \distx \otimes \disty. 
\end{align*}

\newcommand{\epspredk}{\epsilon_{\mathrm{pred},k}}

\paragraph{Key objects.} When reasoning about functions of random variables, we no longer have finite matrices whose singular values we can reason about. Instead, it is more convenient to describe spectral via expected outer-products. The following objects are central to our consideration: 
\begin{align*}
\bSigma &:= \Exp_{\distx}[ff^\top] = \Exp_{\disty}[gg^\top ]\\
\bstSigma &:= \Exp_{\distx}[(\fst)(\fst)^\top] = \Exp_{\disty}[(\gst)(\gst)^\top]\\
\tail_q(\bSigma;k) &:= \sum_{i>k}\sigma_i(\bSigma)^q, \quad q \ge 1,~~\bSigma \in \R^{p \times p}\\
    \fstk&:= \bstP_k f, \quad \gstk:= \bstP_k g\\
    f_k&:= \bP_k f, \quad g_k:= \bP_k g, 
\end{align*}
where $\bstP_k$ is the projection onto any top-$k$ eigenspace of $\bstSigma$ (unique when $\sigma_k(\bstSigma) > \sigma_{k+1}(\bstSigma)$), and $\bhatP_k$ is the projection onto any top-$k$ eigenspace of $\bSigma$.

We consider the following error terms:
    \begin{align*}
    \ErrTerm_0(\bR,k) &= \Exp_{\distime}[\langle \fstk - \bR f, \gstk \rangle^2] \vee \Exp_{\distime}[\langle \fstk, \bR g-\gstk \rangle^2]\\
    \ErrTerm_1(\bR,k) &= \Exp_{\distx}[\|\fstk - \bR f\|^2] \vee \Exp_{\disty}[\|\bR g-\gstk\|^2]\\
     \epsilon^2 &\ge \Exp_{\disone}[(\langle f,g\rangle - \langle \fst,\gst \rangle)^2] \label{eq:some_eps_def} \numberthis\\
     \epspredk^2 &:= \Exp_{\disone}[(\langle f_k,g_k\rangle - \langle \fst_k,\gst_k \rangle)^2]. \label{eq:some_eps_k_def}
    \end{align*}

\paragraph{Outer product notation.} To reduce notational clutter, we introduce a compact notation for vector outer products. Given a vector $v \in \R^p$, or more generally, functions $f:\xspac \to \R^p$ and $g:\yspace \to \R^p$, we let $v^{\outpr} :=vv^\top$, $f^{\outpr }:=ff^\top$, $g^{\outpr}:=gg^\top$. Notice that the typesetting of $\outpr$ differs from the standard tensor product $\otimes$ so as to avoid confusion with tensor-products of distributions, as in $\distx \otimes \disty$. 

\subsection{Guarantee for simple functions}\label{sec:simple_func}
    
    For simple functions, \Cref{thm:one_block_main} items (a)-(d) translate to the following guarantees. 
    
    \begin{prop}\label{cor:discrete_dist} 
        Suppose that  $(f,g)$ and $(\fst,\gst)$ are simple functions, and balanced under $\distime = \distx \otimes \disty$. Further, suppose $\epsilon$ as in \Cref{eq:some_eps_def} and $\kpick \in \N$ satisfies $\epsilon \le \frac{\|\bstSigma\|_{\op}}{40s}$.  Then, there exists an index $k \in [\kpick-1]$ and an  orthogonal matrix $\bR \in \bbO(p)$ such that 
            \begin{subequations}
            \begin{align}
            \ErrTerm_0(\bR,k) &\lesssim \epsilon^2 \cdot \kpick^2\lstargs\\
            \ErrTerm_1(\bR,k) &\lesssim (\sqrt{r}+\kpick^2)\epsilon + \kpick \sigma_{\kpick}(\bstSigma)  +  \tail_1(\bstSigma;\kpick) ,
            \end{align}
            \end{subequations}
            where we define $\lst(\epsilon,\kpick) := \min\left\{1 + \log \frac{\|\bstSigma\|_{\op}}{40\kpick \epsilon},\,\kpick\right\}$. Second, the index $k$ satisfies
            \begin{align*}
            \tail_2(\bstSigma;k) \lesssim \kpick^3 \epsilon^2 + \kpick (\sigma_{\kpick}(\bstSigma))^2 +  \tail_2(\bstSigma;\kpick).
            \end{align*}
            Third, $\bR$ and $k$ satisfy
            \begin{align*}
        \bR\bSigma\bR^\top \succeq \epsilon \bstP_k, \quad \sigma_k(\bstSigma) - \sigma_{k+1}(\bstSigma) \ge 40\epsilon/\kpick,
        \end{align*}
        and lastly $\max_{j}|\sigma_j(\bstSigma) - \sigma_j(\bSigma)| \le \epsilon$. 
    \end{prop}
    
    The key property of simple functions we use is that their expectations can be reduced to those over finitely-supported distributions. The following is proved  in Subsubsection 
\ref{sec:lem:reduction_to_discrete}. 

    \begin{lem}\label{lem:reduction_to_discrete} 
        Let $f_1,\dots,f_a: \xspac \to \R^p$ and $g_1,\dots,g_b:\yspace \to \R^p$ be simple functions, and let $\distx$ and $\disty$ be measures over $\xspac$ and $\yspace$, respectively. Then, there exist  finitely-supported distributions $\distxbar$ and $\distybar$ such that, for all functions $\Psi: \R^{p(a+b)} \to \scrV_{\Psi}$ mapping to some Euclidean space $\scrV_{\Psi}$ (possibly different for each $\Psi$), we have
        \begin{align*}
        \Exp_{\distx \otimes \disty}[\Psi(f_1(x),\dots,f_a(x),g_1(y),\dots,g_b(y))] = \Exp_{\distxbar \otimes \distybar}[\Psi(f_1(x),\dots,f_a(x),g_1(y),\dots,g_b(y))]. 
        \end{align*}
    \end{lem}

    We now turn to the proof of \Cref{cor:discrete_dist}. 
    \begin{proof}[Proof of \Cref{cor:discrete_dist}] 

        By \Cref{lem:reduction_to_discrete}, we may assume without loss of generality that $\distx$ and $\disty$ are distributions with finite support; indeed, by appropriate choices of $\Psi$, the discretization preserves expected outer-products (e.g. $\bstSigma$), balancing, the projection $\bstP_k$, and $\ErrTerm_0,\ErrTerm_1$.

        Continuing, assume $n = |\support(\distx)|$ and $m = |\support(\disty)|$. By augumenting the support with probability-zero points, we may assume without loss of generality that $p \le \min\{n,m\}$. Let $x_1,\dots,x_n$ and $y_1,\dots,y_m$ denote the elements of $\support(\distx)$ and $\support(\disty)$. 
        For $i \in [n]$ and $j \in [m]$, define $\sfp_i := \Pr_{x\sim \distx}[x= x_i]$ and $\sfq_j = \Pr_{y \sim \disty}[y = y_j]$.  We define the matrices $\bstM,\bhatM \in \R^{n \times m}$ via 
        \begin{align*}
        \bstM_{ij} = \sqrt{\sfp_i\sfq_j}\cdot \langle \fst(x_i),\gst(y_j)\rangle, \quad \bhatM_{ij} = \sqrt{\sfp_i\sfq_j}\cdot\hilprodd{f(x_i),g(y_j)}. 
        \end{align*}
        Further, define matrices $\bstA,\bhatA \in \R^{n \times p}$ and $\bstB,\bhatB \in \R^{m \times p}$ via their rows:
        \begin{align*}
        \bstA_{(i,:)} = \sqrt{\sfp_i}\fst(x_i)^\top , \quad \, \bstB_{(j,:)} = \sqrt{\sfq_j}\gst(y_j)^\top , \quad \bhatA_{(i,:)} = \sqrt{\sfp_i}f(x_i)^\top, \quad \bhatB_{(j,:)} = \sqrt{\sfq_j}g(y_j)^\top. 
        \end{align*}
        We readily check that 
        \begin{align*}
        \bstM = \bstA(\bstB)^\top \quad \bhatM = \bhatA\bhatB^\top. 
        \end{align*}
        \Cref{cor:discrete_dist} follows directly from \Cref{thm:main_matrix}, after invoking the substitutions invoked by the following lemma (and taking $\bR \gets \bR^\top$): 
        \begin{lem}\label{lem:matrix_discrete_correspondence} The following identities  hold. 
        \begin{itemize}
            \item[(a)] $\sigma_i(\bstM) =\sigma_i(\bstSigma)$ and $\sigma_i(\bhatM) = \sigma_i(\bSigma)$.
            \item[(b)] $\fronorm{\bstM - \bhatM}^2 = \Exp_{\distx\otimes \disty}[(\langle \fst(x),\gst(y)\rangle - \hilprodd{f(x),g(y)})^2] := \epspred^2 \le \epsilon^2$.  Consequently, by Weyl's inequality, $|\sigma_i(\bstSigma) -  \sigma_i(\bSigma)| \le \opnorm{\bstM - \bhatM} \le \fronorm{\bstM - \bhatM} \le \epsilon$.
            \item[(c)] $(\bstA)^\top \bstA = \Exp_{\distx}[(\fst)(\fst)^{\top} ]$, $ (\bstB)^\top \bstB = \Exp_{\disty}[(\gst)(\gst)^{\top}]$, so that $(\bstA)^\top \bstA = (\bstB)^\top \bstB$. Similarly, $\bhatA^\top \bhatA = \Exp_{\distx}[ff^\top]$ and $\bhatB^\top \bhatB = \Exp_{\disty}[gg^\top]$, so that $\bhatA^\top \bhatA = \bhatB^\top \bhatB. $
            \item[(d)] Using SVD approximations  in the sense of \Cref{eq:rank_k_svd}, we have that  $\bstA_k$'s  $i$-th row is  $\sqrt{\sfp_i}\cdot\fstk(x_i)^\top$ and $\bstB_k$'s    $j$-th row is  $\sqrt{\sfq_j}\cdot\gstk(y_j)^\top$ (notice, the $k$ is in the subscript). Similarly, we have that  $\bhatA_k$'s    $i$-th row  is  $\sqrt{\sfp_i}\cdot f_k(x_i)^\top$ and $\bhatB_k$'s  $j$-th row  is  $\sqrt{\sfq_j}\cdot g_k(y_j)^\top$. 
            \item[(e)] $\ErrTerm_0(\bR,k) = \|(\bstAk - \bhatA \bR^\top)(\bstBk)^\top\|_{\fro}^2 \vee \|\bstAk(\bstBk - \bhatB \bR^\top)^\top\|_{\fro}^2$. 
            \item[(f)] $\ErrTerm_1[\bR,k] = \|\bstAk - \bhatA\bR^\top\|_{\fro}^2 \vee \|\bstBk - \bhatB\bR^\top \|_{\fro}^2 $. 
        \end{itemize}
        \end{lem}
        \begin{proof} The proof of point (a) relies on point (c), namely $(\bstA)^\top \bstA = (\bstB)^\top \bstB$ (the argument is not circular, because the proof of point (c) does not rely on point (a)). Using this, we see $(\bstA)^\top \bstA = (\bstB)^\top \bstB$. Thus, from \Cref{lem:balance_construction}, $\sigma_i(\bstM) = \sigma_i((\bstA)^\top \bstA)$. Invoking point (c) again, we find $\sigma_i(\bstM)  = \sigma_i(\Exp_{\distx}[(\fst)^{\outpr} ]) := \sigma_i(\bstSigma)$. A similar argument applies  to showing $\sigma_i(\bhatM) = \sigma_i(\bSigma)$.

        The proof of points (b)-(f) rely on the same sorts of computations. We prove point (b) as an illustration. 
        \begin{align*}
        \fronorm{\bstM - \bhatM}^2 &= \sum_{ij} (\sqrt{\sfp_i\sfq_j}\cdot \langle \fst(x_i),\gst(y_j)\rangle - \sqrt{\sfp_i\sfq_j}\hilprodd{f(x_i),g(y_j)})^2\\
        &= \sum_{ij} \sfp_i\sfq_j( \langle \fst(x_i),\gst(y_j)\rangle - \hilprodd{f(x_i),g(y_j)})^2 \\
        &= \Exp_{\distx\otimes \disty}(\langle \fst(x),\gst(y)\rangle - \hilprodd{f(x),g(y)})^2. 
        \end{align*}
        The remaining points can be proved analogously. 
        \end{proof}
        This concludes the proof of \Cref{cor:discrete_dist}.
    \end{proof}

\subsection{Extension beyond simple functions}\label{sec:finite_dim_hib}
\newcommand{\Vstp}{\scrV_p^\star}
\newcommand{\Fbarp}{\bar{\scrF}_p}
\newcommand{\Gbarp}{\bar{\scrG}_p}

\newcommand{\Ttau}{\bT_{[\tau]}}
\newcommand{\Tsttau}{\bstT_{[\tau]}}

\newcommand{\fpartau}{f_{(\tau)}}
\newcommand{\gpartau}{g_{(\tau)}}
\newcommand{\gstpartau}{\gst_{(\tau)}}
\newcommand{\fstpartau}{\fst_{(\tau)}}
\newcommand{\Sigpartauf}{\bSigma_{(\tau),f}}
\newcommand{\Sigpartaug}{\bSigma_{(\tau),g}}
\newcommand{\Sigstpartaug}{\bstSigma_{(\tau),g}}
\newcommand{\Sigstpartauf}{\bstSigma_{(\tau),f}}
\newcommand{\limtau}{\lim_{\tau \to \infty}}

We now extend the guarantees of the previous section to the case  beyond simple functions. The analogue of \Cref{cor:discrete_dist} is as follows:

\begin{prop}\label{cor:non_finite_dist}
Suppose that $f,g,\fst,\gst$ map to $\R^{p}$,  and are balanced under $\distime = \distx \otimes \disty$, \emph{but  are not necessarily simple functions}. Further, suppose $\epsilon $ as in \Cref{eq:some_eps_def} and $\kpick \in \N$ satisfies 
\begin{align}
 \epsilon < \frac{\|\bstSigma\|_{\op}}{40s} \quad \text{(strict inequality)}.  \label{eq:epspred_cond}
\end{align}
Then, there exists an index $k \in [\kpick-1]$ and an orthogonal matrix $\bR \in \bbO(p)$ such that 
    \begin{subequations}
    \begin{align}
    \ErrTerm_0(\bR,k) &\lesssim \epsilon^2 \cdot \kpick^3\\
    \ErrTerm_1(\bR,k) &\lesssim (\sqrt{r}+\kpick^2)\epsilon + \kpick \sigma_{\kpick}(\bstSigma)  +  \tail_1(\bstSigma;\kpick).
    \end{align} 
    \end{subequations}
    Second, the index $k$ satisfies
    \begin{align*}
    \tail_2(\bstSigma;k) \lesssim \kpick^3 \epsilon^2 + \kpick (\sigma_{\kpick}(\bstSigma))^2 +  \tail_2(\bstSigma;\kpick).
    \end{align*}
    Third, $\bR$ and $k$ satisfy
    \begin{align*}
\bR^{\top}\bSigma\bR \succeq \epsilon \bstP_k, \quad \bSigma := \Exp_{\cdone}[ff^\top],
\end{align*}
and lastly $\max_{j \in [p]}|\sigma_j(\bstSigma) - \sigma_j(\bSigma)| \le \epsilon$. 
\end{prop}

    Before proving the above two propositions, we review some facts about $\Ltwo$ convergence, and some basic results for measure-theoretic probability theory which can be found in any standard reference (e.g. \cite{ccinlar2011probability}).

    \paragraph{$\Ltwo$ convergence.} We first review the definition of $\Ltwo$ convergence.

        \begin{defn}[$\Ltwo$ Convergence] Let $\cD$ be a measure on $\scrZ$. We say that
        $\psi: \scrZ \to \R^p$ is in $\Ltwo(\cD)$ if $\Exp_{\cD}[\|\psi\|^2] < \infty$.  Let $(\psi_\tau)_{\tau \ge 1}$ be a sequence of functions in $\Ltwo(\cD)$,  $\psi \in \Ltwo(\cD)$, and let $\cD$ be a measure on $\scrZ$. We say that $\psi_{\tau}$ converges to $\psi$ in $\Ltwo(\cD)$, denoted
        \begin{align*}
        \psi_\tau \Ltwoto{\cD} \psi,
        \end{align*}
        if $\lim_{\tau \to \infty} \Exp_{\cD}\|\psi_{\tau}-\psi\|^2 = 0$. 
        \end{defn}
        The following lemma is standard in probability theory (again, see e.g.,  \cite[Section 2]{ccinlar2011probability}). 
        
        \begin{lem}\label{lem:lim:simple} Let $\cD$ be a measure on $\scrZ$. Given any $\psi \in \Ltwo(\cD)$,  there exists a sequence of simple functions $\psi_{\tau} \in \Ltwo(\cD)$ such that $\psi_\tau \Ltwoto{\cD} \psi$.
        \end{lem}
        We shall often use the following lemma, which is easy to check.
        \begin{lem}\label{lem:outpr_convergence} If $\psi_{\tau} \Ltwoto{\cD} \psi$, then $\limtau\Exp[\psi_{\tau}^{\outpr}] = \Exp[\psi^{\outpr}]$. 
        \end{lem}
        The following fact is also useful.
        \begin{lem}\label{lem:outpr_rowspace}  If $\psi_{\tau} \Ltwoto{\cD} \psi$ and if $\range(\Exp[\psi_{\tau}^{\outpr}]) \subseteq  \range(\Exp[\psi^{\outpr}])$ for all $\tau$, then there exists some $\tau_0$ such that, for all $\tau \ge \tau_0$, $\range(\Exp[\psi_{\tau}^{\outpr}]) = \range(\Exp[\psi^{\outpr}])$.
        \end{lem}
        Proofs of \Cref{lem:outpr_convergence,lem:outpr_rowspace} are given in Subsubsection \ref{sec:lem:outpr}. 

\paragraph{Approximation by simple functions.}
Using the machinery introduced above, we approximate $f,g,\fst,\gst$ by a sequence of simple functions. Our approximation preserves an important property regarding the ranges of their covariances. 
\begin{lem}\label{lem:simple_rowspace} There exists a sequence of simple functions $f_{(\tau)},g_{(\tau)}$ $\fst_{(\tau)},\gst_{(\tau)}$ such that
\begin{align*}
f_{(\tau)} \Ltwoto \distx f, \quad g_{(\tau)} \Ltwoto \disty g, \quad \fst_{(\tau)} \Ltwoto \distx \fst, \quad \gst_{(\tau)} \Ltwoto\disty \gst,
\end{align*}
and, the covariances 
\begin{align*}
\Sigpartauf &:= \Exp[(\fpartau)^{\outpr}], \quad \Sigpartaug := \Exp[(\gpartau)^{\outpr}], \quad \Sigstpartauf := \Exp[(\fstpartau)^{\outpr}], \quad \Sigstpartaug := \Exp[(\gstpartau)^{\outpr}],
\end{align*}
satisfy $\range(\Sigpartauf) \cup \range(\Sigpartaug) \subseteq \range(\bSigma)$ and  $\range(\Sigstpartauf) \cup \range(\Sigstpartaug) \subseteq \range(\bstSigma)$. 
\end{lem}
The above lemma is proved in Subsubsection \ref{sec:lem:simple_row}. 
\paragraph{Constructing the balanced functions. } 

    \newcommand{\fstktau}{\fst_{k,[\tau]}}
    \newcommand{\gstktau}{\gst_{k,[\tau]}}
    \newcommand{\fktau}{f_{k,[\tau]}}
    \newcommand{\gktau}{g_{k,[\tau]}}
    \newcommand{\ftau}{f_{[\tau]}}
    \newcommand{\gtau}{g_{[\tau]}}
    \newcommand{\ftaun}{f_{[\tau_n]}}
    \newcommand{\gtaun}{g_{[\tau_n]}}
    \newcommand{\fsttau}{\fst_{[\tau]}}
    \newcommand{\gsttau}{\gst_{[\tau]}}
    \newcommand{\epstau}{\epsilon_{[\tau]}}
    \newcommand{\epsktau}{\epsilon_{k,[\tau]}}

    \newcommand{\Sigsttau}{\bstSigma_{[\tau]}}
    \newcommand{\Sigtau}{\bSigma_{[\tau]}}
    \newcommand{\Sigtaun}{\bSigma_{[\tau_n]}}
    \newcommand{\Pstktau}{\bstP_{k,[\tau]}}
    \newcommand{\Pktau}{\bP_{k,[\tau]}}

    \newcommand{\Sigsttaun}{\bstSigma_{[\tau_n]}}
    \newcommand{\Pstktaun}{\bstP_{k,[\tau_n]}}

    We cannot invoke \Cref{cor:discrete_dist} directly on the simple functions constructed above because they are not \emph{balanced}. Below we show that we can balance them, and that the matrices which achieve this converge to the identity.

    \begin{lem}\label{lem:balancing_sequence} There exists a sequence of invertible $p\times p$-matrices $(\Ttau)_{\tau \ge 1}$ and $(\Tsttau)_{\tau \ge 1}$ such that
    \begin{itemize}
    \item[(a)] For all $\tau$ sufficiently large, $\Exp_{\distx}[ (\Ttau \fpartau)^{\outpr})]  = \Exp_{\disty}[ (\Ttau^{-\top} \gpartau)^{\outpr})]$ and $\Exp_{\distx}[ (\Tsttau \fstpartau)^{\outpr})] = \Exp_{\disty}[ ((\Tsttau)^{-\top} \gstpartau)^{\outpr})]$. 
    \item[(b)] $\lim_{\tau\to \infty}\Ttau = \lim_{\tau\to \infty}\Tsttau = \eye_p$. 
    \end{itemize}
    \end{lem}
    The above lemma is proved  in Subsubsection \ref{sec:lem:balancing}. With these balancing matrices, we devise a new sequence of balanced functions and associated quantities:
\begin{align*}
&\ftau = \Ttau \fpartau, \quad \gtau = \Ttau^{-\top} \gpartau, \quad \Sigtau :=\Exp_{\distx}[ (\ftau)^{\outpr})]\\
& \fsttau = \Tsttau\fstpartau, \quad \gsttau = (\Tsttau)^{-\top}\gstpartau,\quad \Sigsttau :=\Exp_{\distx}[ (\fsttau)^{\outpr})], 
\end{align*}
and, letting $\bstP_{k,[\tau]}$ project onto the top $k$ singular values of $\bstSigma_{[\tau]}$ and defining $\Pktau$ analogously, we set
\begin{align*}
\fstktau &= \Pstktau\fsttau, \quad \gstktau = \Pstktau \fsttau\\
\fktau &= \Pktau\ftau, \quad \gktau = \Pktau \gtau.
\end{align*}
We also define the errors
\begin{align*}
\epstau^2 := \Exp_{\distimes}[(\langle \ftau,\gtau\rangle - \langle \fsttau,\gsttau\rangle )^2], \quad \epsktau^2 := \Exp_{\distimes}[(\langle \ftau,\gtau\rangle - \langle \fsttau,\gsttau\rangle )^2].
\end{align*}
Lastly, we define
\begin{align*}
\ErrTerm_{0,[\tau]}(\bR,k) &= \Exp_{\distimes}[\langle \fstktau - \bR \ftau, \gstktau \rangle^2] \vee \Exp_{\distimes}[\langle \fstktau, \bR \gtau-\gstktau \rangle^2]\\
\ErrTerm_{1,[\tau]}(\bR,k) &= \Exp_{\distx}[\|\fstktau - \bR \ftau\|^2] \vee \Exp_{\disty}[\|\gstktau - \bR \gtau\|^2],
\end{align*}
and recall 
\begin{align*}
    \ErrTerm_0(\bR,k) &= \Exp_{\distime}[\langle \fstk - \bR f, \gstk \rangle^2] \vee \Exp_{\distime}[\langle \fstk, \bR g-\gstk \rangle^2],\\
    \ErrTerm_1(\bR,k) &= \Exp_{\distx}[\|\fstk - \bR f\|^2] \vee \Exp_{\disty}[\|\bR g-\gstk\|^2].
    \end{align*}

\paragraph{Analyzing the balanced functions.}
In order to conclude the proof, we establish numerous useful properties of the balanced function sequence. The following lemma is proved  in Subsubsection \ref{sec:claim:limiting_application_claim}.
\begin{lem}\label{claim:limiting_application_claim} The followings are true:
\begin{itemize}
    \item[(a)] The sequences of  balanced functions converge to their targets in $\Ltwo$:
    \begin{align*}
    \ftau \Ltwoto{\distx} f, \quad \gtau \Ltwoto{\disty} g,\quad\fsttau \Ltwoto{\distx} \fst,\quad\gsttau \Ltwoto{\disty} \gst.
    \end{align*}
     More generally, if $\bR_{\tau_n}$  is a convergent subsequence  converging to $\bR$, then $ \bR_{\tau_n}  \ftaun \Ltwoto{\distx} \bR f$ and $ \bR_{\tau_n}  \gtaun \Ltwoto{\disty} \bR g$ as $n\to\infty$. 
    \item[(b)] We have $\limtau \bstSigma_{[\tau]} = \bstSigma$. Hence, by Weyl's inequality, $\limtau \tail_q(\bstSigma_{[\tau]};k) = \tail_q(\bstSigma;k)$ for any $q,k \ge 1$ (note that we have assumed here finite-dimensional embeddings, so the covariance operators are matrices and thus the sense of convergence is unambiguous).
    \item[(c)] Similarly, $\limtau \Sigtau = \bSigma$. More generally, if $\bR_{\tau_n}$  is a convergent subsequence  converging to $\bR$, then $\lim_{n\to\infty} \bR_{\tau_n} \Sigtaun  \bR_{\tau_n}^\top= \bR\bSigma\bR^\top$.
    \item[(d)] For any $k$ for which $\sigma_k(\bstSigma) > \sigma_{k+1}(\bstSigma)$,  $\lim_{\tau \to \infty} \Pstktau = \bstP_k$, where $\bstP_k$ projects onto the top $k$-eigenspace of  $\bstSigma$. Similarly, for any $k$ for which $\sigma_k(\bSigma) > \sigma_{k+1}(\bSigma)$, $\lim_{\tau \to \infty} \Pktau = \bP_k$,  where $\bP_k$ projects onto the top $k$-eigenspace of  $\bSigma$. 
    \item[(e)] For any $k$ for which $\sigma_k(\bstSigma) > \sigma_{k+1}(\bstSigma)$, 
    \begin{align*}
    \fstktau \Ltwoto{\distx} \fst_k, \quad \gstktau \Ltwoto{\disty} \gst_k.
    \end{align*}
    Similarly, for any $k$ for which $\sigma_k(\bSigma) > \sigma_{k+1}(\bSigma)$,
    \begin{align*}
    \fktau \Ltwoto{\distx} f_k, \quad \gktau \Ltwoto{\disty} g_k.
    \end{align*}
    \item[(f)] For any $\bR \in \bbO(p)$, and $k$ for which $\sigma_k(\bstSigma) > \sigma_{k+1}(\bstSigma)$, $\lim_{\tau \to \infty} \ErrTerm_{0,[\tau]}(\bR,k) = \ErrTerm_{0}(\bR,k)$ and $\lim_{\tau \to \infty} \ErrTerm_{1,[\tau]}(\bR,k) = \ErrTerm_{1}(\bR,k)$.  More generally, if $\bR_{\tau_n}$  is a convergent subsequence  converging to $\bR$, then we have $\lim_{n \to \infty} \ErrTerm_{i,[\tau_n]}(\bR_{\tau_n},k) = \ErrTerm_{i}(\bR,k)$, $i \in \{0,1\}$. 
    \item[(g)] $\lim_{\tau \to \infty} \epstau^2 = \epspred^2 \le \epsilon^2$ and, for any $k$ satisfying both $\sigma_k(\bstSigma) > \sigma_{k+1}(\bstSigma)$  and $\sigma_k(\bSigma) > \sigma_{k+1}(\bSigma)$ (supposing such a $k$ exists), $\lim_{\tau \to \infty} \epsktau^2 = \epspredk^2$. 
    \item[(h)]  For some $\eta$ sufficiently small, and for $\epsilon$ chosen to satisfy \Cref{{eq:epspred_cond}} for some $s\in \N$ and $s>1$, there exists some $\tau_0$ such that, for all $\tau \ge \tau_0$ sufficiently large, 
    \begin{align*}
    \epstau^2 \le (1+\eta) \epsilon^2 \le 2\epsilon^2 \vee   \frac{\|\Sigsttau\|_{\op}^2}{40^2 \kpick^2 }.
    \end{align*}
\end{itemize}
\end{lem}

\paragraph{Concluding the proof.} 
    We are now in a position to complete the proofs of \Cref{cor:non_finite_dist,prop:nondiscrete_svd_sing_val}.

    \begin{proof}[Proof of \Cref{cor:non_finite_dist}] By applying \Cref{cor:discrete_dist} to the functions $\ftau,\gtau,\fsttau,\gsttau$ with $\epsilon^2 \gets (1+\eta)\epsilon^2 \ge \epstau^2$ and invoking \Cref{claim:limiting_application_claim} part (h), the following claim is immediate:
    \begin{claim}\label{claim:tau_err_terms} For all $\tau \ge \tau_0$, there exists a $\bR_{\tau}$ and $k_{\tau} \in [s-1]$ such that
    \begin{align}
    \ErrTerm_{0,[\tau]}(\bR_{\tau},k_{\tau}) &\lesssim \epsilon^2 \cdot \kpick^3\cdot \\
    \ErrTerm_{1,[\tau]}(\bR_{\tau},k_{\tau}) &\lesssim (\sqrt{r}+\kpick^2)\epsilon + \kpick \sigma_{\kpick}(\Sigsttau)  +  \tail_1(\Sigsttau;\kpick).
    \end{align} 
    \footnote{A literal invocation of \Cref{cor:discrete_dist} would take $\ell_{\star,[\tau]}(\epsilon,s) := \min\left\{1 + \log \tfrac{\|\Sigsttau\|_{\op}}{40(1+\eta)\kpick \epsilon},\,\kpick\right\}$. Here, we use $(1+\eta) \ge 1$.} Moreover, the index $k_{\tau}$ satisfies
    \begin{align*}
    \tail_2(\Sigsttau;k_{\tau}) \lesssim \kpick^3 \epsilon^2 + \kpick (\sigma_{\kpick}(\Sigsttau))^2 +  \tail_2(\Sigsttau;\kpick). 
    \end{align*}
    Above, we note $\lesssim$ hides universal constants independent of $\tau$. Morever, 
    \begin{align*}
    \bR_{\tau}\Sigtau\bR_{\tau}^\top \succeq \epsilon \Pstktau, \quad \sigma_{k_{\tau}}(\Sigsttau) - \sigma_{k_{\tau}+1}(\Sigsttau) \ge 40\epsilon/\kpick. 
    \end{align*}
    Lastly, $\max_{j \in [p]}|\sigma_j(\Sigsttau) - \bSigma_{[\tau]}| \le \epstau$.
    \end{claim}
    We may now conclude the proof of \Cref{cor:non_finite_dist}.  Since $[s-1]$ is a finite set, and $\bbO(p)$ is compact, there exists a subsequence $(\bR_{\tau_n},k_{\tau_n})$ so that $\tau_n \ge \tau_0$ for all $n$, $k_{\tau_n} = k$ for some fixed $k \in [s-1]$, and $\bR_{\tau_n} \to \bR$ for some fixed $\bR \in \bbO(p)$. By \Cref{claim:limiting_application_claim} part (b) and Weyl's inequality, it must be the case that this $k$ satisfies $\sigma_k(\bstSigma) - \sigma_{k+1}(\bstSigma) \ge 40\epsilon/\kpick > 0$. Hence,
    \begin{align*}
    \ErrTerm_{0}(\bR,k) &= \lim_{n \to \infty} \ErrTerm_{0,[\tau_n]}(\bR_{\tau_n},k_{\tau_n}) \tag{\Cref{claim:limiting_application_claim} part (f)}\\
    &\lesssim \epsilon^2 \cdot \kpick^2\cdot  \lim_{n \to \infty}\ell_{\star,[\tau_n]}(\epsilon,s) \tag{\Cref{claim:tau_err_terms}}\\
    &= \epsilon^2 \cdot \kpick^2\cdot  \lim_{n \to \infty} \min\left\{1 + \log \frac{\|\Sigsttaun\|_{\op}}{40\kpick \epsilon},\,\kpick\right\} \tag{see \Cref{claim:tau_err_terms}}\\
    &= \epsilon^2 \cdot \kpick^2\cdot   \underbrace{\min\left\{1 + \log \frac{\|\bstSigma\|_{\op}}{40\kpick \epsilon},\,\kpick\right\}}_{=\ell_{\star}(\epsilon,s)}, \tag{\Cref{claim:limiting_application_claim} part (b)}
    \end{align*}
    and,
    \begin{align*}
    \ErrTerm_{1}(\bR,k) &= \lim_{n \to \infty} \ErrTerm_{1,[\tau_n]}(\bR_{\tau_n},k_{\tau_n}) \tag{\Cref{claim:limiting_application_claim} part (f)} \\
    &\lesssim (\sqrt{r}+\kpick^2)\epsilon + \lim_{n \to \infty}\left(\kpick \sigma_{\kpick}(\Sigsttaun)  +  \tail_1(\Sigsttaun;\kpick)\right) \tag{\Cref{claim:tau_err_terms}}\\
    &\le  (\sqrt{r}+\kpick^2)\epsilon + \kpick \sigma_{\kpick}(\bstSigma)  +  \tail_1(\bstSigma;\kpick) \tag{\Cref{claim:limiting_application_claim} part (b)}.
    \end{align*}
    Second,
    \begin{align*}
     \tail_2(\bstSigma;k) &= \lim_{n \to \infty}\tail_2(\Sigsttaun;k_{\tau_n}) \tag{\Cref{claim:limiting_application_claim} part (b)}\\
     &\lesssim \kpick^3 \epsilon^2 + \lim_{n \to \infty} \left(\kpick (\sigma_{\kpick}(\Sigsttaun))^2 +  \tail_2(\Sigsttaun;\kpick)\right)\tag{\Cref{claim:tau_err_terms}} \\
     &\lesssim \kpick^3 \epsilon^2 +\kpick (\sigma_{\kpick}(\bstSigma))^2 +  \tail_2(\bstSigma;\kpick).  \tag{\Cref{claim:limiting_application_claim} part (b)}
    \end{align*}
    Third,
    \begin{align*}
    \bR\bSigma\bR^\top &= \lim_{n\to \infty} \bR_{\tau_n} \Sigtaun \bR_{\tau_n}^\top \tag{\Cref{claim:limiting_application_claim} part (c)}\\
    &\succeq \epsilon\lim_{n\to \infty}\Pstktaun \tag{\Cref{claim:tau_err_terms}}\\
    &= \epsilon \bstP_k   \tag{\Cref{claim:limiting_application_claim} part (d)}.
    \end{align*}
    Finally, by parts $(b)$ and $(c)$, the fact that  $\max_{j \in [p]}|\sigma_j(\Sigsttau) - \bSigma_{[\tau]}| \le \epstau$ due to \Cref{claim:limiting_application_claim}, and Weyl's inequality. This concludes the proof of \Cref{cor:non_finite_dist}.
    \end{proof}

\subsection{From finite to infinite dimensional embeddings: Proof of \Cref{thm:one_block_main}}\label{sec:proof:thm:one_block_main}

We give the proof of \Cref{thm:one_block_main} from \Cref{cor:non_finite_dist}.  Fix any $\fhat: \xspac \to \R^r$ and $\ghat: \yspace \to \R^r$ that satisfy 
\begin{align}
  \Exp_{\cdone}[(\langle \fhat,\ghat \rangle_{\R^r} - \langle \fst,\gst \rangle_{\hilspace})^2]\le \epsilon^2 < \frac{\|\Sigst\|_{\op}^2}{(40s)^2}. \label{eq:epsone}
\end{align}
Now, fix a $p \in \N$, which we shall take sufficiently large. Since $\Sigst$ is trace class, we may assume without loss of generality that the space spanned by its top $p$ eigenvectors is unique.\footnote{If $\rank(\Sigst)$ is finite, let $p$ equal the rank. Otherwise, the eigenvectors must have decay so for any $p$, there exists some $p' \ge p$ which has eigengap.} Let $\Vstp$ denote the eigenspace spanned by these eigenvectors, and note that $\Vstp$ is isomorphic to $\R^p$. Finally, let let $\iota: \R^r \to \Vstp \subset \hilspace$ by an isometric inclusion in the sense of \Cref{defn:valid_proxies_cont}. 

Suppose first that $(\fhat,\ghat)$ are full-rank, and let $\bT$ be the balancing operator guaranteed by \Cref{lem:balancing_simple}. As $\iota$ is an isometry, \Cref{eq:epsone} implies  
\begin{align*}
\Exp_{\cdone}[(\langle \iota(\bT\fhat),\iota(\bT^{-1}\ghat) \rangle_{\hilspace} - \langle \fst,\gst \rangle_{\hilspace})^2]\le \epsilon^2. 
\end{align*}
as well as the 
following equality, which can be checked by evaluating the induced quadratic forms.
\begin{align}
\bSigma_{\iota} := \Exp_{\cdxone}[\iota(\fhat)\iota(\fhat)^\top] = \Exp_{\cdyone}[\iota(\ghat)\iota(\ghat)^\top], \label{eq:Sigdef_iota}
\end{align}
so that $\iota(\fhat)$ and $\iota(\ghat)$ are balanced and take values in $\Vstp$. 
A standard expansion and Cauchy-Schwartz inequality imply that 
\begin{align*}
&\Exp_{\cdone}[(\langle \iota(\fhat),\iota(\ghat) \rangle_{\hilspace} - \langle \fstp,\gstp \rangle_{\hilspace})^2]\\
&=\Exp_{\cdone}[(\langle \fhat,\ghat \rangle_{\R^r} - \langle \fstp,\gstp \rangle_{\hilspace})^2]\\
&\le \Exp_{\cdone}[(\langle \fhat,\ghat \rangle_{\R^r} - \langle \fst,\gst \rangle_{\hilspace})^2] + \Exp_{\cdone}[(\langle \fst,\gst \rangle_{\hilspace} - \langle \fstp,\gstp \rangle_{\hilspace})^2] \\
&\quad+ 2\sqrt{\Exp_{\cdone}[(\langle \fhat,\ghat \rangle_{\R^r} - \langle \fst,\gst \rangle_{\hilspace})^2] \Exp_{\cdone}[(\langle \fst,\gst \rangle_{\hilspace} - \langle \fstp,\gstp \rangle_{\hilspace})^2]}\\
&= \epsilon^2 + 2\tailsf_2(p)^{1/2}\epsilon + \tailsf_2(p) \tag{\Cref{lem:tailsf_two}}\\
&= \epsilon^2 + 2\tailsf_2(p)^{1/2}\epsilon + \tailsf_2(p) := \epsilon_{[p]}^2 \numberthis \label{eq:eps_p_thing_zero}.
\end{align*}
Recall that $\epsilon^2 < \frac{\|\Sigst\|_{\op}^2}{(40s)^2}$. Hence, by choosing some $p \ge s$ sufficiently large, we can ensure $\tailsf_2(p)$ is small enough that
\begin{align}
\epsilon_{[p]} < \frac{\|\Sigst\|_{\op}}{40\epsilon} , \quad \epsilon_{[p]}^2 \le 2\epsilon^2.\label{eq:eps_p_thing_one}
\end{align} As $\|\Sigst\|_{\op}=\|\bstSigma_p\|_{\op}$ due to \Cref{eq:sigsvals}, we then have
\begin{align}
\epsilon_{[p]} < \frac{\|\bstSigma_p\|_{\op}}{40\epsilon}. \label{eq:eps_p_thing_two}
\end{align}

Continuing the proof, let $\bstSigma_{[p]}$ denote $\Exp_{\cdone}[\fstp(\fstp)^\top]$, viewed as an operator on $\Vstp$, and notice that $\bstSigma = \Exp_{\cdone}[\gstp(\gstp)^\top]$ by \Cref{asm:bal}, and that
\begin{align}
\sigma_i(\bstSigma_{[p]}) = \begin{cases} \sigma_i(\Sigst) &i \in [p]\\
0 & i > p
\end{cases} \label{eq:sigsvals}. 
\end{align}

 Viewing $\Vstp$ as isomorphic to $\R^p$, define the error terms consider by \Cref{cor:non_finite_dist}:
\begin{align*}
    \ErrTerm_0(\bR,k) &= \Exp_{\distime}[\langle \fstk - \bR \iota(\bT\fhat), \gstk \rangle_{\Vstp}^2] \vee \Exp_{\distime}[\langle \fstk, \bR \iota(\bT^{-1}\ghat)-\gstk \rangle_{\Vstp}^2]\\
    \ErrTerm_1(\bR,k) &= \Exp_{\distx}[\|\fstk - \bR \iota(\bT\fhat)\|_{\Vstp}^2] \vee \Exp_{\disty}[\|\bR \iota(\bT^{-1}\ghat)-\gstk\|_{\Vstp}^2],
\end{align*}
where above the norms and inner products are the standard Euclidean inner product on $\Vstp$, and where $\bR$ is an orthogonal transformation of $\Vstp$. From \Cref{eq:eps_p_thing_one,eq:eps_p_thing_two}, we can apply \Cref{cor:non_finite_dist} to find that there exists an orthogonal operator $\bR$ for which 
\begin{align*}
    \ErrTerm_0(\bR,k)  &\lesssim \epsilon_{[p]}^2 \cdot \kpick^3\\
    \ErrTerm_1(\bR,k) &\lesssim (\sqrt{r}+\kpick^2)\epsilon_{[p]} + \kpick \sigma_{\kpick}(\bstSigma_{[p]})  +  \tail_1(\bstSigma_{[p]};\kpick) .
    \end{align*}
    Moreover, the index $k$ satisfies
    \begin{align*}
    \tail_2(\bstSigma_{[p]};k) &\lesssim \kpick^3 \epsilon_{[p]}^2 + \kpick (\sigma_{\kpick}(\bstSigma_{[p]}))^2 +  \tail_2(\bstSigma_{[p]};\kpick). 
    \end{align*}
    Lastly, $\bR$ and $k$ satisfy
    \begin{align}
\bR\bSigma_{\iota}\bR^{\top} \succeq \epsilon_{[p]} \bstP_{[k]},  \label{eq:pst_lb}
\end{align}
where $\bstP_{[k]}$ is the projection onto the top-$k$ singular space of $\bstSigma_{[p]}$, namely $\Vstp$,  and were $\bSigma_{\iota}$ is as in \Cref{eq:Sigdef_iota}. We now argue that for this transformation $\bR$, the embeddings $(f,g)$ defined by\footnote{under the natural inclusion of $\bR \iota(\fhat)$ from $\Vstp$ to $\hilspace$} 
    \begin{align*}
    f:= \bR \iota(\bT\fhat), \quad g:= \bR \iota(\bT^{-1}\ghat),
    \end{align*}
satisfy the conclusion of \Cref{thm:one_block_main}. 

\emph{Proof of part (a).} This follows from direct computation, as $\iota$ and multiplication by $\bR$ are isometries.

\emph{Proof of part (b).}  This follows from  the definition of $f$ and $g$, and from \Cref{eq:pst_lb}. 

\emph{Proof of part (c).} By \Cref{eq:sigsvals}, we have
\begin{align*}
\tail_j(\bstSigma_{[p]};k) \le \tailsf_j(k), \quad i \in \{1,2\} .
\end{align*}
Combining this  and the facts that $p \ge s$ and that $ \epsilon_{[p]} \lesssim \epsilon$ (see \Cref{eq:eps_p_thing_zero}),
\begin{align}
     \ErrTerm_1(\bR,k) &\lesssim (\sqrt{r}+\kpick^2)\epsilon + \kpick \sigma_{\kpick}(\Sigst)  +  \tailsf_1(\kpick), \label{eq:term_final_one}
    \end{align}
as well as $ \ErrTerm_0(\bR,k) \lesssim \epsilon^2 \cdot \kpick^3$. Similarly, the index $k$ satisfies
    \begin{align*} 
    \tail_2(\bstSigma_{[p]};k) \lesssim \kpick^3 \epsilon^2 + \kpick (\sigma_{\kpick}(\Sigst))^2 +  \tailsf_2(\kpick).
    \end{align*}
    Moreover, from \Cref{eq:sigsvals}, we see $\tailsf_2(k) = \tail_2(\bstSigma_{[p]};k) + \tailsf_2(p)$. 
    Since $p \ge s$,  
    \begin{align*}
    \tailsf_2(k) &\lesssim \kpick^3 \epsilon^2 + \kpick (\sigma_{\kpick}(\Sigst))^2 +  \tailsf_2(\kpick) + \tailsf_2(p) \\
    &\lesssim\kpick^3 \epsilon^2 + \kpick (\sigma_{\kpick}(\Sigst))^2 +  \tailsf_2(\kpick).  \numberthis{\label{eq:tail_sf_eq}}
    \end{align*}
    Thus, 
    \begin{align}
    \tailsf_2(k)  + \ErrTerm_0(\bR,k) &\lesssim \kpick^3 \epsilon^2 + \kpick (\sigma_{\kpick}(\Sigst))^2 +  \tailsf_2(\kpick). \label{eq:term_final_zero}
    \end{align}
    To conclude
    \begin{align*} 
\ErrTerm_0(\bR,k) &= \Exp_{\distime}[\langle \fstk - f, \gstk \rangle_{\Vstp}^2] \vee \Exp_{\distime}[\langle \fstk, g-\gstk \rangle_{\Vstp}^2],\numberthis{\label{equ:immediate_delta_def_0}}\\
\ErrTerm_1(\bR,k) &= \Exp_{\distx}[\|\fstk - f\|_{\Vstp}^2] \vee \Exp_{\disty}[\|\gstk - g\|_{\Vstp}^2].\numberthis{\label{equ:immediate_delta_def}}
\end{align*}
Hence, part (c) of \Cref{thm:one_block_main} follows from the above identification, and \Cref{eq:term_final_zero,eq:term_final_one}.

\emph{Proof of part (d).} Applying the last part of \Cref{cor:non_finite_dist} and using $p \ge r$ implies that
\begin{align*}
\sigma_j(\Exp_{\cdxone}[\iota(\ftil) \iota(\ftil)^\top]) > \sigma_j(\bSigma^{\star}_{[p]})  - \epsilon = \sigma_j(\bstSigma) -\epsilon.
\end{align*}

\paragraph{Removing the full-rank assumption $\fhat,\ghat$.} To replace the assumption that $(\fhat,\ghat)$ is full-rank with the assumption that $\epsilon < \sigma_r(\bstSigma)$, apply \Cref{lem:exist_balance_not_simple} to show that there exist  $\ftil, \gtil$ such that  (a) $(\ftil,\gtil)$ is full-rank if and only if $(\fhat,\ghat)$ is, (b) $\langle \fhat,\ghat \rangle = \langle \ftil, \gtil \rangle$ almost surely on $\cdxone$, and (c) 
\begin{align*}
\Exp_{\cdxone}[\ftil \ftil^\top] = \Exp_{\cdyone}[\gtil \gtil^\top], \quad \sigma_{r}(\Exp_{\cdxone}[\ftil\ftil^\top]) = \sighat.
\end{align*}
It suffices to show that $\ftil,\gtil$ is full-rank. To this end, let $\iota$ be an isometric embedding of $\R^r \to \Vstp$ for $p \ge r$, and note we have that $\Exp_{\cdone}[(\langle \fhat,\ghat \rangle_{\R^r} - \langle \fst,\gst \rangle_{\hilspace})^2]\le \epsilon^2$, and since $\iota$ is an isometry, 
\begin{align*}
\Exp_{\cdxone}[\iota(\ftil) \iota(\ftil)^\top] = \Exp_{\cdyone}[\iota(\gtil) \iota(\gtil)^\top]. 
\end{align*}
Applying the last part of \Cref{cor:non_finite_dist} and using $p \ge r$ implies that
\begin{align*}
\sigma_r(\Exp_{\cdxone}[\iota(\ftil) \iota(\ftil)^\top]) > \sigma_r(\bSigma^{\star}_{[p]})  - \epsilon = \sigma_r(\bstSigma) -\epsilon,
\end{align*}
which is striclty positive for $\epsilon < \sigma_r(\bstSigma)$. Since $\iota$ is an isometry, we have shown that $(\ftil,\gtil)$ are full-rank, which implies that $(\fhat,\ghat)$ are also full-rank.

\subsection{Proof of supporting claims}\label{sec:supporting_claims}
\subsubsection{Proof of \Cref{lem:reduction_to_discrete} \label{sec:lem:reduction_to_discrete}}

    \begin{proof}[] Define subsets of $\R^p$ by $\cU_f = \bigcup_{i=1}^a \{f_{i}(\scrX)\}$ and $\cW_g = \bigcup_{j=1}^b \{g_{j}(\scrY)\}$. Since $f_i$ and $g_j$ are simple, $\cU_f$ and $\cW_f$ are finite. Define the  sets 
    \begin{align*}
    \scrX_{\vec{\bu}} &:= \bigcap_{i=1}^a f_i^{-1}(\vec{\bu}^{(i)}), \quad \vec{\bu} = (\vec{\bu}^{(1)},\dots,\vec{\bu}^{(a)}) \in (\cU_f)^a\\
    \scrY_{\vec{\bw}} &:= \bigcap_{j=1}^b g_j^{-1}(\vec{\bw}^{(j)}), \quad \vec{\bw} = (\vec{\bw}^{(1)},\dots,\vec{\bw}^{(b)}) \in (\cW_g)^b. 
    \end{align*} 
    Let $\cvecU_f := \{\vec{\bu} \in (\cU_f)^a: \scrX_{\vec{\bu}}  \ne \emptyset\}$ and $\cvecW_g := \{\vec{\bw} \in (\cW_g)^b: \scrY_{\vec{\bw}}  \ne \emptyset\}$. Note that $\cvecU_f$ and $\cvecW_g$ are finite sets (since $\cU_f$ and $\cW_f$ are). By construction, for each $\vec{\bu} \in \cvecU_f$ (resp. $\vec \bw \in \cvecW_g$), there exists an $x_{\vec{\bu}} \in \xspac$ (resp. $y_{\vec{\bw}} \in \yspace$) such that for all $i \in [a]$ and $j\in[b]$, 
    \begin{align*}
    f_i(x_{\vec{\bu}}) = {\vec{\bu}}^{(i)}, \quad g_j(y_{\vec{\bw}}) = {\vec{\bw}}^{(j)}. 
    \end{align*}
    By construction, we also see that the sets $\scrX_{\vec{\bu}}$ and $\scrY_{\vec{\bw}}$ indexed by $\vec{\bu} \in \cvecU_f$  and  $\vec \bw \in \cvecW_g$ form a partition of $\xspac$ and $\yspace$, so we may define functions $\phi_{\xspac}$ and $\phi_{\yspace}$ by
    \begin{align*} 
     \phi_{\xspac}(x)&:= x_{\vec{\bu}}, ~x \in\scrX_{\vec{\bu}}, \quad \vec{\bu} \in \cvecU_f\\
    \phi_{\yspace}(y) &:= y_{\vec{\bw}},~ y \in\scrY_{\vec{\bw}}, \quad \vec \bw \in \cvecW_g.
    \end{align*}
    Note that since $x_{\vec{\bu}} \in \scrX_{\vec{\bu}}$, $\phi_{\xspac}$ is idempotent: $\phi_{\xspac} = \phi_{\xspac}\circ \phi_{\xspac}$; similarly, $\phi_{\yspace} = \phi_{\yspace}\circ \phi_{\yspace}$. 

    By definition of $\scrX_{\vec{\bu}}$ and $\scrY_{\vec{\bw}}$, it holds for all $x \in \scrX_{\vec{\bu}}$ and $y \in \scrY_{\vec{\bw}}$ that 
    \begin{align*}
    f_i(x) = f_i \circ \phi_{\xspac}(x), \quad g_j(y) = g_j \circ \phi_{\yspace}(y). 
    \end{align*}
    Hence, for any $\Psi$, $\Psi(f_1(x),\dots,f_a(x),g_1(y),\dots,g_b(y))$ can be written as some function $\tilde \Psi(\phi_{\xspac}(x), \phi_{\yspace}(y))$. To conclude, let $\distxbar$ denote the distribution of $\phi_{\xspac}(x)$ under $\distx$ and $\distybar$ denote the distribution of $\phi_{\yspace}(y)$ under $\disty$. Then,
    \begin{align*}
    &\Exp_{\distx \otimes \disty}[\Psi(f_1(x),\dots,f_a(x),g_1(y),\dots,g_b(y))] \\
    &\quad= \Exp_{\distx \otimes \disty}[\tilde \Psi(\phi_{\xspac}(x), \phi_{\yspace}(y))]\\
    &\quad= \Exp_{\distx \otimes \disty}[\tilde \Psi(\phi_{\xspac} \circ \phi_{\xspac}(x), \phi_{\yspace} \circ \phi_{\yspace}(y))] \tag{Idempotence of $\phi_{\xspac},\phi_{\yspace}$ }\\
    &\quad= \Exp_{\distxbar \otimes \distybar}[\tilde \Psi(\phi_{\xspac} (x), \phi_{\yspace}(y))]\\
    &\quad= \Exp_{\distxbar \otimes \distybar}[\Psi(f_1(x),\dots,f_a(x),g_1(y),\dots,g_b(y))].
    \end{align*}
    This completes the proof. 
    \end{proof}
\subsubsection{Proof of \Cref{lem:outpr_rowspace,lem:outpr_convergence} \label{sec:lem:outpr}}

    \begin{proof}[Proof of \Cref{lem:outpr_convergence}] We bound
    \begin{align*}
    &\|\Exp[\psi_{\tau}^{\outpr}] - \Exp[\psi^{\outpr}]\|_{\fro}^2 \\
    &= \left\|\Exp[\psi(\psi_{\tau} - \psi)^\top] + \Exp[(\psi_{\tau} - \psi)\psi^\top] + \Exp[(\psi_{\tau} - \psi)(\psi_{\tau} - \psi)^\top] \right\|_{\fro}^2\\
     &\le 3 \|\Exp[\psi(\psi_{\tau} - \psi)^\top]\|_{\fro}^2 + 3\|\Exp[(\psi_{\tau} - \psi)\psi^\top]\|_{\fro}^2  +3\|\Exp[(\psi_{\tau} - \psi)(\psi_{\tau} - \psi)^\top]\|_{\fro}^2 \\
     &\le 3 \Exp[\|\psi(\psi_{\tau} - \psi)^\top\|_{\fro}^2] + 3\Exp[\|(\psi_{\tau} - \psi)\psi^\top\|_{\fro}^2]  +3\Exp[\|(\psi_{\tau} - \psi)(\psi_{\tau} - \psi)\|_{\fro}^2] \tag{Jensen's Inequality}.\\
     &\leq  6 \Exp[\|\psi\|^2\|\psi_{\tau} - \psi\|^2] + 3\Exp[\|\psi_{\tau} - \psi\|^4]\\
      &\leq  6 \Exp[\|\psi\|^2] \Exp[\|\psi_{\tau} - \psi\|^2] + 3\Exp[\|\psi_{\tau} - \psi\|^4] \tag{Cauchy Schwartz}.
    \end{align*}
    This last term goes to $0$ as $\tau \to \infty$ by definition of $\Ltwo$ convergence. 
    \end{proof}
    \begin{proof}[Proof of \Cref{lem:outpr_rowspace}]
  Let $r := \rank(\Exp[\psi^{\outpr}])$, so that $\sigma_r(\Exp[\psi^{\outpr}]) > 0$. By  \Cref{lem:outpr_convergence}, there exists some $\tau_0$ so that  for all $\tau \ge \tau_0$ sufficiently large, $\Exp[\psi_{\tau}^{\outpr}] \succeq \Exp[\psi^{\outpr}] - \frac{\sigma_r}{2}\eye_p$. For all such $\tau \ge \tau_0$, and any $\bv \in \range(\Exp[\psi^{\outpr}]) \setminus \{0\}$, we have
    \begin{align*}
    \bv^\top \Exp[\psi_{\tau}^{\outpr}] \bv \ge \bv^\top \Exp[\psi_{}^{\outpr}]\bv - \sigma_r \|\bv\|^2/2 = \sigma_r\|\bv\|^2/2 >0.
    \end{align*} 
    Thus, $\dim(\nullspace(\Exp[\psi_{\tau}^{\outpr}])) \le  \dim(\nullspace(\Exp[\psi^{\outpr}]))$. Hence,  
    \begin{align*}\dim(\range(\Exp[\psi_{\tau}^{\outpr}])) \ge  \dim(\range(\Exp[\psi^{\outpr}])).
    \end{align*}  On the other hand, by assumption, $\range(\Exp[\psi_{\tau}^{\outpr}] ) \subseteq  \range(\Exp[\psi^{\outpr}])$. Hence, $\range(\Exp[\psi_{\tau}^{\outpr}]) = \range(\Exp[\psi^{\outpr}])$. 
    \end{proof}

\subsubsection{Proof of \Cref{lem:simple_rowspace}\label{sec:lem:simple_row}}

    \begin{proof} We give the construction of the sequence $\fpartau$, the others are similar. By \Cref{lem:lim:simple}, there exists  a sequence of functions $f_{\tau}$ such that $f_{\tau} \Ltwoto \distx f$, i.e. $\lim_{\tau \to \infty}\Exp_{\distx}\|f_{\tau} - f\|^2 = 0$. Let $\bP$ denote the orthogonal projection onto $\range(\bSigma)$. Then, $\bP f= f$ $\distx$-almost surely by \Cref{lem:in_range_as}.
    Hence, 
    \begin{align*}
    \lim_{\tau \to \infty}\Exp_{\distx}\|\bP f_{\tau} - f\|^2 &=\lim_{\tau \to \infty}\Exp_{\distx}\|\bP (f_{\tau} -  f)\|^2\\
    &\le \lim_{\tau \to \infty}\Exp_{\distx}\| f_{\tau} - f\|^2 \tag{$\bP$ is an orthogonal projection}\\
    &=0 \tag{$f_{\tau} \Ltwoto \distx f$}.
    \end{align*}
    Hence, let  $f_{(\tau)} := \bP f_{\tau} \Ltwoto \distx f$. Moreover, $\Sigpartauf := \Exp[(\fpartau)^{\outpr}] = \bP\Exp[(f_{\tau})^{\outpr}]\bP^\top$,  which ensures the inclusion of the rowspaces. 
    \end{proof}

\subsubsection{Proof of \Cref{lem:balancing_sequence}\label{sec:lem:balancing}}


    We demonstrate the existence of $\Ttau$, and note that the existence of $\Tsttau$ is similar. Recall from \Cref{lem:simple_rowspace} that $\Sigpartauf := \Exp[(\fpartau)^{\outpr}]$, and $\Sigpartaug := \Exp[(\gpartau)^{\outpr}]$, and
 $\range(\Sigpartauf) \cup \range(\Sigpartaug) \subseteq \range(\bSigma)$. By \Cref{lem:simple_rowspace} and \Cref{lem:outpr_convergence}, it follows that
 \begin{align*}
 \limtau \Sigpartauf = \limtau \Sigpartaug = \bSigma
 \end{align*}
 and for some $\tau_{0}$, $\range(\Sigpartauf )  = \range(\Sigpartaug )  = \range(\bSigma)$ (taking $\tau_0$ to be the maximum of the two $\tau_{0,f}$ and $\tau_{0,g}$ required for $\Sigpartauf$ and $\Sigpartaug$ individually ). Let $r = \dim(\range(\bSigma))$. By inflating $\tau_0$ if necessary, we can ensures $\sigma_r(\Sigpartauf )  = \sigma_r(\Sigpartaug ) > \sigma_r(\bSigma)/2$ for all $\tau \ge \tau_0$. Hence, by \Cref{lem:exist_transform}, for all $\tau \ge \tau_0$, there exist some $\Ttau \in \pd{p}$ such that
 \begin{align*}
    \Ttau\Sigpartauf\Ttau = \Ttau^{-1}\Sigpartaug\Ttau^{-1},
    \end{align*}
    satisfying
    \begin{align*}
    \max\{\|\Ttau\|_\op,\|\Ttau^{-1}\|_\op\} &\le (1+\Delta)^{1/4}, \quad  \text{where } \Delta := \frac{\|\Sigpartauf - \Sigpartaug\|_{\op}}{2\sigma_r(\bSigma)} \overset{\tau \to \infty}{\to 0}.
    \end{align*}
    Notice that $\Ttau = \Ttau^{\top}$ ($\pd{p}$ contains only symmetric matrices), we also have
    \begin{align*}
    \Exp[(\Ttau \fpartau)^{\outpr}] = \Ttau\Sigpartauf\Ttau^\top = \Ttau^{-\top}\Sigpartaug\Ttau^{-1}  \Exp[(\Ttau^{-\top} \gpartau)^{\outpr}],
    \end{align*}
    and that since $\limtau \max\{\|\Ttau\|_\op,\|\Ttau^{-1}\|_\op\} = 1$, $\limtau \Ttau = \eye_p$.

\subsubsection{Proof of \Cref{claim:limiting_application_claim}}\label{sec:claim:limiting_application_claim}
     \textbf{Part (a).} We have
    \begin{align*}
    \Exp_{\distx}\|\ftau - f\|^2 &= \Exp_{\distx}\|\Ttau\fpartau - f\|^2\\
    &\le 2\Exp_{\distx}\|(\Ttau - \eye_p)\fpartau\|^2 + 2\Exp_{\distx}\|\fpartau- f\|^2\\
    &\le 2\opnorm{\Ttau - \eye_p}^2\cdot\Exp_{\distx}\|\fpartau\|^2 + 2\Exp_{\distx}\|\fpartau- f\|^2.
    \end{align*}
    Since $\fpartau \Ltwoto{\distx} f$, $\sup_{\tau}\Exp_{\distx}\|\fpartau\|^2\le M$ for some $M < \infty$.  Hence, by \Cref{lem:balancing_sequence,lem:simple_rowspace},
    \begin{align*}
    \limtau\Exp_{\distx}\|\ftau - f\|^2\le \limtau 2\opnorm{\Ttau - \eye_p}^2 M + \limtau2 \Exp_{\distx}\|\fpartau- f\|^2 = 0.
    \end{align*}
    The proofs of the other guarantees are similar.  

    \paragraph{Parts (b) and (c).} These follow from part (a) and \Cref{lem:outpr_convergence}. 

    \paragraph{Part (d).} Set $\zeta :=  \sigma_{k}(\bstSigma) - \sigma_{k+1}(\bstSigma)$, and assume $\zeta > 0$. By part (b), $\limtau \bstSigma_{[\tau]} = \bstSigma$. Hence, by Weyl's inequality, there exists some $\tau_0$ such that for all $\tau \ge \tau_0$, $\sigma_{k+1}(\bstSigma_{[\tau]}) < \sigma_{k}(\bstSigma) - \zeta/2$. The convergence then follows by Wedin's Theorem (see e.g. \Cref{lem:gap_free_Wedin}). The convergence of $\Pktau$ to $\bP_k$ is analogous. 

    \paragraph{Part (e). } This follows from parts (a) and (d).

    \paragraph{Part (f).} This can be checked by using parts (a) and (e), together with standard applications of Cauchy Schwartz and/or Jensen's inequality.  

    \paragraph{Part (g).} The first statement can be checked by using part (a), together with standard applications of Cauchy Schwartz and/or Jensen's inequality. The second uses part (e) instead of part (a). 

    \paragraph{Part (h).} Since $\limtau \frac{\|\Sigsttau\|_{\op}^2}{40^2 \kpick^2 } = \frac{\|\bstSigma\|_{\op}^2}{40^2 \kpick^2 }$ by part (b) and Weyl's inequality, part (h) follows from part (g) and \Cref{eq:epspred_cond}. 
This completes the proof. \hfill $\blacksquare$

\subsection{SVD perturbation for distribution embeddings}
In this section, we reiterate the limiting analysis to establish an embedding analogue of our main perturbation result for the singular-value decomposition (\Cref{thm:svd_pert}). Specifically, the main result of this section is:
\begin{theorem}\label{thm:svd_general_thing} Let $(f,g)$ be balanced embeddings, with $\bSigma = \Exp_{\cdxone}[ff^\top] = \Exp_{\cdyone}[gg^\top]$, and let $\Exp_{\disone}[(\langle \fhat,\ghat\rangle - \langle \fst,\gst \rangle)^2]\le \epsilon^2$. Then, 
\begin{itemize}
    \item[(e)] Let $\bSigma = \Exp[ff^\top]$. Then, $ \sum_{i \ge 1}|\sigma_i(\bSigma) - \bsigst_i|^2 \le \epsilon^2$
    \item[(f)] Fix $k \in \N$, and let $\bP_k$ denote the projection onto the top $k$ eigenvectors of $\bSigma$.\footnote{Under the conditions of this statement, it holds that $\bP_k$ is unique.} Set $\bdelst_k := 1 - \frac{\bsigst_{k+1}}{\bsigst_k}$, and suppose that $\bsigst_{k} > 0$. Then, if $\epsilon \le \eta \bsigst_k\bdelst_k$ for a given $\eta \in [0,1)$. Then
        \begin{align*}
        \Exp_{\disone}[(\langle \bP_k f, \bP_k g\rangle - \langle \fstk,\gstk \rangle^2)] \le  \frac{81\epsilon^2}{(\bdelst_k(1-\eta))^2}. 
        \end{align*}

\end{itemize} %
\end{theorem}
Our proof follows by approximation to simple functions. 

\begin{lem}\label{prop:discrete_svd_sing_val} Ssuppose that $(f,g)$ and $(\fst,\gst)$ are simple functions embedding into $\R^p$, and balanced under $\distime = \distx \otimes \disty$, and that $\epsilon$ is as in \Cref{eq:some_eps_def}.  Then 
    \begin{itemize}
        \item[(a)] It holds that $\sum_{i \ge 1} |\sigma_i(\bSigma) - \sigma_i(\bstSigma)|^2 \le \epsilon^2$. 
        \item[(b)] Suppose $\sigma_k(\bstSigma) > 0$, and set $\delst_k := 1 - \frac{\sigma_{k+1}(\bstSigma)}{\sigma_k(\bstSigma)}$ and suppose that $\epsilon \le \eta \sigma_k(\bstSigma)\delst_k$, where $\eta \in [0,1)$. Then
        \begin{align*}
        \epspredk^2 \le  \frac{81\epsilon^2}{(\delst_k(1-\eta))^2}. 
        \end{align*}
    \end{itemize}
    \end{lem}
\begin{proof}[Proof of \Cref{prop:discrete_svd_sing_val}] For the first point, we have
\begin{align*}
\sum_{i \ge 1} |\sigma_i(\bSigma) - \sigma_i(\bstSigma)|^2 &= \sum_{i \ge 1} |\sigma_i(\bhatM) - \sigma_i(\bstM)|^2 \tag{\Cref{lem:matrix_discrete_correspondence}, part (a)}\\
&\le \|\bhatM - \bstM\|_{\fro}^2 \tag{\Cref{lem:Mat_pert_one}}\\
&\leq  \epsilon^2  \tag{\Cref{lem:matrix_discrete_correspondence}, part (b)}.
\end{align*}
For the second point, we can verify from \Cref{lem:matrix_discrete_correspondence} part (c) and the same computation as in the proof of part (b) that
\begin{align}
\Exp_{\distx\otimes \disty}[(\langle \fst_k(x),\gst_k(y)\rangle - \hilprodd{f_k(x),g_k(y)})^2] = \|\bstM_{[k]} - \bhatM_{[k]}\|_{\fro}^2, \label{eq:corres_svd_k_disc}
\end{align}
where $\bstM_{[k]}$ and $\bhatM_{[k]}$ denote the rank-$k$ SVD approximation of $\bstM$ and $\bhatM$, respectively. We now invoke our main SVD perturbation bound, \Cref{thm:svd_pert}. This states that if $\sigma_k(\bstM) > 0$ and $\delta =  1 - \sigma_{k+1}(\bstM)/\sigma_k(\bstM) > 0$, and     if $\|\bstM - \bhatM\|_{\op} \le \eta \sigma_{k}(\bstM)\delta$ for some $\eta \in (0,1)$, then the Frobenius norm error between the rank-$k$ SVD's of $\bstM$ and $\bhatM$ is bounded by
\begin{align*}
\|\bhatM_{[k]} - \bstM_{[k]}\|_{\fro}^2 &\le  \frac{81\|\bhatM - \bstM\|_{\fro}^2}{\delta^2(1-\eta)^2}. 
\end{align*}
Using the correspondences in \Cref{lem:matrix_discrete_correspondence}, we can take $\delta = \delst_k := 1 - \frac{\sigma_{k+1}(\bstSigma)}{\sigst_k(\bstSigma)}$, and that it is sufficient that $\epsilon \le \eta \sigma_k(\bstSigma)\delst_k$ (since $\sigma_k(\bstM) = \sigma_k(\bstSigma)$ and $\epsilon = \|\bstM - \bhatM\|_{\fro} \ge \|\bstM - \bhatM\|_{\op}$). Using \Cref{eq:corres_svd_k_disc} and \Cref{lem:matrix_discrete_correspondence} part (b), we conclude that for $\epsilon \le \eta \sigst_k\delst_k$,
\begin{align*}
\Exp_{\distx\otimes \disty}[(\langle \fst_k(x),\gst_k(y)\rangle - \hilprodd{f_k(x),g_k(y)})^2] = \|\bstM_{[k]} - \bhatM_{[k]}\|_{\fro}^2  &\le  \frac{81\epsilon^2}{(\delst_k(1-\eta))^2}. 
\end{align*}
This completes the proof. 
\end{proof}

Next, we remove the requirement of simple functions.
\begin{lem}\label{prop:nondiscrete_svd_sing_val} Suppose that  $(f,g)$ and $(\fst,\gst)$ are pairs of embeddings into $\R^p$, and are balanced under $\distime = \distx \otimes \disty$, \emph{but  are not necessarily simple functions}, and that $\epsilon$ is as in \Cref{eq:some_eps_def}.  Then
    \begin{itemize}
        \item[(a)] It holds that $\sum_{i \ge 1} |\sigma_i(\bSigma) - \sigma_i(\bstSigma)|^2 \le \epsilon^2$. 
        \item[(b)]  Suppose $\sigma_k(\bstSigma) > 0$, and set $\delst_k := 1 - \frac{\sigma_{k+1}(\bstSigma)}{\sigma_k(\bstSigma)}$ and suppose that $\epsilon \le \eta \sigma_k(\bstSigma)\delst_k$, where $\eta \in [0,1)$. Then
        \begin{align*}
        \epspredk^2 \le  \frac{81\epsilon^2}{(\delst_k(1-\eta))^2}. 
        \end{align*}
    \end{itemize}
    \end{lem}
    \begin{proof}[Proof of \Cref{prop:nondiscrete_svd_sing_val}] Let's start with part (a). By invoking \Cref{prop:discrete_svd_sing_val} part (a) for each $\tau$, 
    \begin{align*}
    \sum_{i \ge 1} |\sigma_i(\Sigsttau) - \sigma_i(\Sigtau)|^2 \le \epstau^2.
    \end{align*}
    Taking $\tau \to \infty$, \Cref{claim:limiting_application_claim} parts (b) and (c) ensure $\Sigsttau \to \bstSigma$, $\Sigtau \to \bSigma$. Thus, Weyl's inequality implies that  $ \lim_{\tau \to \infty} \sum_{i \ge 1} |\sigma_i(\Sigsttau) - \sigma_i(\Sigtau)|^2 =  \sum_{i \ge 1} |\sigma_i(\bstSigma) - \sigma_i(\bSigma)|^2$. \Cref{claim:limiting_application_claim} part (g) gives $\epstau^2 \to \epspred^2 \le \epsilon^2$, completing the proof of the statement.

    Next, let's turn to part (b). Fix a $k$ for which $\sigma_k(\bstSigma) > 0$. From part (a) and the condition that $\epsilon < \sigma_k(\bstSigma)$, it also follows that $\sigma_k(\bSigma) > 0$. Define $\delst_{k,[\tau]} := 1 - \frac{\sigma_{k+1}(\Sigsttau)}{\sigma_k(\Sigsttau)}$. \Cref{claim:limiting_application_claim} part (b) ensures $\Sigsttau \to \bstSigma$, so that (again using Weyl's inequality), $\delst_{k,[\tau]}$ is well defined for all $\tau$ sufficiently large, and converges to $\delst_{k} := 1 - \frac{\sigma_{k+1}(\bstSigma)}{\sigma_k(\bstSigma)}$. Using \Cref{claim:limiting_application_claim} again,  the assumption that $\epsilon \le \eta \sigma_k(\bstSigma)\delst_{k}$ implies that there is a sequence of $\eta_{k,[\tau]} \downarrow \eta$ such that $\epsktau \le \eta_{k,[\tau]} \sigma_k(\Sigsttau)\delst_{k,[\tau]}$ for all $\tau$ sufficiently large. Invoking \Cref{prop:discrete_svd_sing_val} part (b) for these $\tau$, 
        \begin{align*}
        \epsktau^2 \le  \frac{81\epsilon^2}{(\delst_{k,[\tau]}(1-\eta_{k,[\tau]}))^2}. 
        \end{align*}
        Taking limits $\tau \to \infty$ and again calling \Cref{claim:limiting_application_claim} concludes the proof. 
    \end{proof}

The proof of \Cref{thm:svd_general_thing} follows by extending \Cref{prop:nondiscrete_svd_sing_val} to infinite dimensional embeddings along the lines of \Cref{sec:proof:thm:one_block_main}. Details are similar (though considerably simpler) and are omitted for brevity. \hfill $\blacksquare$


\section{Supporting Linear Algebraic Proofs}\label{sec:lin_alg_proof}


\subsection{Balancing without loss of generality \label{sec:bal_wlog}}

\begin{proof}[Proof of \Cref{lem:balance_construction}] Let $\bA = \bU_A \bSigma_A \bV_A^\top$ and $\bB = \bU_B \bSigma_B \bV_B^\top$ where $\bSigma_A,\bSigma_B$ are diagonal matrices with elements ranked  in descending order, and $\bV_A,\bV_B \in \R^{d\times d}$. Since the $\bSigma_A,\bSigma_B$ and $\bV_A,\bV_B$ can be constructed from any  eigen-decomposition of $\bA^\top \bA$ and $\bB^\top\bB$, and since both are equal, we have $\bSigma_A= \bSigma_B=\bSigma$ and we can choose the basis $\bV_B$ such that $\bV_A = \bV_B$. Moreover, $\bV_A = \bV_B = \bR$ for some $\bR \in \bbO(d)$, since they are $d\times d$ matrices with orthonormal columns. Hence, $\bA = \bU_A \bSigma \bR$ and $\bB = \bV_B \bSigma \bR$ for some $\bR\in\bbO(d)$. 

To justify $\bA_{[k]} = \bU\bSigma_{[k]}^{\frac{1}{2}}\bR$, set $\bar{\bA} := \bA \bR^{\top} = \bU\bSigma$; defined $\bB_{[k]}$ and $\bbarB_{[k]}$ similarly. Then, it can be checked that $\bar{\bP}_{[k]}$ projects onto the top $k$ eigenspace of $\bbarA^\top \bbarA$, so that $\bP_{[k]} = \bR^{\top} \bar{\bP}_{[k]} \bR$ projects onto the top $k$ eigenspace of $\bA^\top \bA$. 
Hence,  
\begin{align*}
\bA_{[k]} = \bA \bP_{[k]} = \bbarA\bR \cdot \bR^{\top} \bar{\bP}_{[k]} \bR = \bU\bSigma_{[k]}^{\frac{1}{2}} \bR, 
\end{align*}
as needed. Similar argument holds for $\bB_{[k]}$. 
\end{proof}
\begin{proof}[Proof of \Cref{lem:simple_form_wlog}] From \Cref{lem:balance_construction}, there exists a $\bhatR\in \bbO(d)$ for which $\bhatApr = \bhatA\bhatR,\bhatBpr = \bhatB \bhatR$, and a $\bstR\in\bbO(d)$ for which $\bstApr = \bstA\bstR,\bstBpr = \bstB \bstR$, and for which we can take $\bstApr_{[k]} = \bstA_{[k]}\bstR,\bstBpr_{[k]} = \bstB_{[k]} \bstR$. Given $\bR'\in \bbO(d)$, we choose our orthogonal matrix to apply to the non-primed terms as $\bR = \bhatR\bR'(\bstR)^{-1} \in \bbO(d)$. We compute
\begin{align*}
\|(\bstAk - \bhatA \bR)(\bstBk)^\top\|_{\fro}^2&= \|(\bstAk - \bhatA \bhatR\bR'(\bstR)^{-1})(\bstBk)^\top\|_{\fro}^2\\
&= \|(\bstAk - \bhatA' \bR'(\bstR)^{-1})(\bstBk)^\top\|_{\fro}^2\\
&= \|(\bstAkpr (\bstR)^{-1} - \bhatA' \bR'(\bstR)^{-1})(\bstBkpr (\bstR)^{-1})^\top\|_{\fro}^2\\ 
&= \|(\bstAkpr  - \bhatA' \bR')(\bstR)^{-1}(\bstR)^{-\top}(\bstBkpr)^\top\|_{\fro}^2\\
&= \|(\bstAkpr - \bhatA' \bR')(\bstBkpr)^\top\|_{\fro}^2,
\end{align*}
where the second-last line uses that $(\bstR)^{-1}(\bstR)^{-\top} = \eye_p$ for $\bstR \in \bbO(d)$. The equality $\|\bstAk(\bstBk - \bhatB \bR)^\top\|_{\fro}^2 = \|\bstAkpr(\bstBkpr - \bhatB' \bR')^\top\|_{\fro}^2$ can be verified similarly. 

Moreover,
\begin{align*}
 \|\bstAk - \bhatA\bR\|_{\fro}^2  &=  \|\bstAkpr(\bstR)^{-1} - \bhatApr\bhatR^{-1}\cdot \bhatR\bR' (\bstR)^{-1}\|_{\fro}^2 \\
 &=  \|(\bstAkpr - \bhatApr\bR')(\bstR)^{-1}\|_{\fro}^2 = \|\bstAkpr - \bhatApr\bR'\|_{\fro}^2,
 \end{align*}
 where the last line uses the unitary invariant property of the Frobenius norm. The equality $\|\bstBk - \bhatB\bR \|_{\fro}^2  = \|\bstBkpr - \bhatBpr\bR' \|_{\fro}^2 $ follows similarly.

 Lastly, since $\bR = \bhatR\bR'(\bstR)^{-1}$, we have $\bR' = \bhatR^{-1}\bR\bstR$. Then

 \begin{align*}
 \rowspace\left((\bhatApr_{[k]} \bR')^\top (\bhatApr_{[k]} \bR')\right) &=  \rowspace\left((\bhatApr_{[k]} \bhatR^{-1}\bR\bstR)^\top (\bhatApr_{[k]} \bhatR^{-1}\bR\bstR)\right)\\
 &=  \rowspace\left((\bhatA_{[k]}\bR\bstR)^\top (\bhatA_{[k]}\bR\bstR)\right)\\
 &=  \rowspace\left((\bstR)^\top(\bhatA_{[k]}\bR)^\top (\bhatA_{[k]}\bR)\bstR\right)\\
 &\overset{(i)}{\supseteq}  \rowspace\left((\bstR)^\top(\bstA_{[k]})^\top (\bstA_{[k]})\bstR\right)\\
 &= \rowspace\left((\bstA_{[k]}\bstR)^\top (\bstA_{[k]}\bstR)\right)\\
  &= \rowspace\left((\bstApr_{[k]})^\top (\bstApr_{[k]})\right),
 \end{align*}
 where in $(i)$, we use that $ \rowspace\left((\bhatA_{[k]}\bR)^\top (\bhatA_{[k]}\bR)\right) \supseteq  \rowspace\left((\bstA_{[k]})^\top (\bstA_{[k]})\right)$, and that $\bR$ is a rotation matrix.  
\end{proof}

 \subsection{Supporting proofs for error decomposition} 


\subsubsection{Proof of \Cref{lem:decomp_by_partition}} \label{sec:proof:decomp_by_partition} 

We first state the following facts. 
    \begin{fact} Let $\scrK := \partshort$ be a {monotone} partition of $[d]$ for some $d\leq \min\{n,m\}$, and consider $\bA = \bU\bSigma$  and $\bB' = \bV'\bSigma'$. Then, 
    \begin{itemize} 
        \item[(a)] Any diagonal matrix  $\bSigma \in \R^{d \times d}$ is compatible with  $\scrK$.  Hence, $\bA = \sum_{i=1}^\ell \bA_{\cK_i}$, and similar for $\bB'$.
        \item[(b)] If $\bR \in \R^{d\times d}$ is compatibile with  $\scrK$, then $\bA \bR = \sum_{i=1}^\ell \bA_{\cK_i}\bR_{\cK_i}$. 
        \item[(c)] For any $1 \le i \ne j \le \ell$, $\bSigma_{\cK_i}\bSigma'_{\cK_{j}}= 0$. Hence,  $\bA_{\cK_i}(\bB_{\cK_j}')^\top = 0$, and also $(\bA_{\cK_i}\bR_{\cK_i})(\bB_{\cK_j}')^\top = 0$.
        \item[(d)]   $\bA_{\cK_i}(\bB')^\top = (\bA_{\cK_i})(\bB_{\cK_i}')^\top$ for any $1 \le i \le \ell$.
    \end{itemize}
    If in addition $\partshort$ is a monotone partition with pivots  $0 = k_1 < k_2 \dots < k_{\ell} < k_{\ell+1} = p$, then  $\bA_{[k_{\ell}]} = \sum_{i=1}^{\ell-1} \bA_{\cK_i}$, 
and similarly for $\bB'$.
    \end{fact}
    \begin{proof}
    We use the facts above to prove \Cref{lem:decomp_by_partition}:   
    \begin{align*}
    (\bstAkl - \bhatA \bR)(\bstBkl)^\top &=  \left((\sum_{i=1}^{\ell-1}\bstA_{\cK_i} )- (\sum_{i=1}^{\ell}\bhatA_{\cK_i}\bR_{\cK_i} )\right)(\bstBkl)^\top\\
    &=  \sum_{i=1}^{\ell-1} \left(\bstA_{\cK_i}  - \bhatA_{\cK_i}\bR_{\cK_i}\right)(\bstBkl)^{\top}  - (\bhatA_{\cK_\ell}\bR_{\cK_\ell})(\bstBkl)^\top\\
    &=  \sum_{j=1}^{\ell-1}\sum_{i=1}^{\ell-1} \left(\bstA_{\cK_i}  - \bhatA_{\cK_i}\bR_{\cK_i}\right)(\bstB_{\cK_j})^\top  - \underbrace{\sum_{j=1}^{\ell-1}(\bhatA_{\cK_\ell}\bR_{\cK_\ell})(\bstB_{\cK_j})^\top}_{=0}\\
    &=  \sum_{i=1}^{\ell-1} \left(\bstA_{\cK_i}  - \bhatA_{\cK_i}\bR_{\cK_i}\right)(\bstB_{\cK_i})^\top. 
    \end{align*}
    Proofs for the compositions of $(\bstA_{{[k_\ell]}} - \bhatA \bR)$, $(\bstBkl - \bhatB \bR)$ and $\bstAkl(\bstBkl - \bhatB \bR)^\top$ follow analogously. 
    \end{proof}

\subsubsection{Proof of \Cref{lem:well_tempered_intermediate} \label{sec:lem:well_tempered_inter}}
    \begin{proof}[Proof of \Cref{lem:well_tempered_intermediate}] For simplicity, we abbreviate $E_i = E_i(\bR)$. Let us prove the bound on $\ErrTerm_0(\bR,k_\ell)$ first. Recall $\ErrTerm_0(\bR,k_\ell) = \|(\bstAkl - \bhatA \bR)(\bstBkl)^\top\|_{\fro}^2 \vee \|\bstAkl(\bstBkl - \bhatB \bR)^\top\|_{\fro}^2$. We explicitly bound the first term $\|(\bstAkl - \bhatA \bR)(\bstBkl)^\top\|_{\fro}^2$, and note that a similar argument bounds the second term. 
    
    Invoking \Cref{eq:A_decomp_weighted}
    \begin{align*}
    \|(\bstAkl - \bhatA \bR)(\bstBkl)^\top\|_{\fro}^2  &= \left\|\sum_{i=1}^{\ell-1} \left(\bstA_{\cK_i}  - \bhatA_{\cK_i}\bR_{\cK_i}\right)(\bstB_{\cK_i})^\top\right\|_{\fro}^2\\
    &= \sum_{i=1}^{\ell-1} \left\|\left(\bstA_{\cK_i}  - \bhatA_{\cK_i}\bR_{\cK_i}\right)(\bstB_{\cK_i})^\top\right\|_{\fro}^2\\
    &\qquad+ \sum_{i,j=1,i\ne j}^{\ell-1} \underbrace{\left\langle\left(\bstA_{\cK_i}  - \bhatA_{\cK_i}\bR_{\cK_i}\right)(\bstB_{\cK_i})^\top,\, \left(\bstA_{\cK_j}  - \bhatA_{\cK_j}\bR_{\cK_j}\right)(\bstB_{\cK_j})^\top\right\rangle}_{=0}.
    \end{align*}
    Here, the second term vanishes because $(\bstB_{\cK_i})^\top \bstB_{\cK_j} = 0 = (\bstSigma_{\cK_i})^{\half}(\bstV)^\top \bstV(\bstSigma_{\cK_j})^{\half} = (\bstSigma_{\cK_i})^{\half}(\bstSigma_{\cK_j})^{\half} = 0$, since $\cK_i \cap \cK_j = \emptyset$ for $i \ne j$. Continuing,
    \begin{align*}
    \sum_{i=1}^{\ell-1} \left\|\left(\bstA_{\cK_i}  - \bhatA_{\cK_i}\bR_{\cK_i}\right)(\bstB_{\cK_i})^\top\right\|_{\fro}^2 &\le \sum_{i=1}^{\ell-1} \|\bstB_{\cK_i}\|_{\op}^2 \|\bstA_{\cK_i}  - \bhatA_{\cK_i}\bR_{\cK_i}\|^2\\
    &= \sum_{i=1}^{\ell-1} \max\{\sigst_{k'} : k' \in \cK_i\}  \cdot\|\bstA_{\cK_i}  - \bhatA_{\cK_i}\bR_{\cK_i}\|^2\\
    &\le \sum_{i=1}^{\ell-1} (\mucond\sigst_{k_{i+1}})\cdot\|\bstA_{\cK_i}  - \bhatA_{\cK_i}\bR_{\cK_i}\|^2 \tag{$(\delspace,\mucond)$-well-tempered}\\ 
    &=\mucond\sum_{i=1}^{\ell-1}  (\delst_{k_i} \wedge \delst_{k_{i+1}})^{-2}  \cdot\underbrace{(\delst_{k_i} \wedge \delst_{k_{i+1}})^2(\sigst_{k_{\kzedit{i+1}}})  \|\bstA_{\cK_i}  - \bhatA_{\cK_i}\bR_{\cK_i}\|^2}_{\le E_i} \\
    &\le \mucond  (\sum_{i=1}^{\ell-1}  (\delst_{k_i} \wedge \delst_{k_{i+1}})^{-2}) \cdot \max_{i \in [\ell-1]} E_i \\
    &\le \mucond ( \sum_{i=1}^{\ell-1}  (\delst_{k_i})^{-2} + (\delst_{k_{i+1}})^{-2})\cdot\max_{i \in [\ell-1]} E_i \\
    &= \mucond(\sum_{i=1}^{\ell-1}  (\delst_{k_i})^{-2} + \sum_{i=2}^{\ell}(\delst_{k_i})^{-2})\cdot\max_{i \in [\ell-1]} E_i. 
    \end{align*}
    
    By convention, $\delst_{0} = 1 \ge \delst_{k_i}$ (all relative gaps are at most $1$). Thus, $\sum_{i=2}^{\ell}(\delst_{k_i})^{-2} \le \sum_{i=1}^{\ell}(\delst_{k_i})^{-2}$, and the above is at most $2\mucond(\sum_{i=1}^{\ell}  (\delst_{k_i})^{-2})\max_{i \in [\ell]} E_i = 2\mucond \muspace \max_{i \in [\ell]} E_i$. This completes the proof of the first argument. 

    Let's now turn to $\ErrTerm_1(\bR,k_\ell) = \|\bstAkl - \bhatA\bR\|_{\fro}^2 \vee \|\bstBkl - \bhatB\bR\|_{\fro}^2$. Again, we focus on $\|\bstAkl - \bhatA\bR\|_{\fro}^2$. From \Cref{eq:A_decomp_unweighted}
    \begin{align*}
    \|\bstA_{[k_{\ell}]}- \bhatA \bR\|_{\fro}^2 &= \left\|\sum_{i=1}^{\ell-1} \left(\bstA_{\cK_i}  - \bhatA_{\cK_i}\bR_{\cK_i}\right) - \bhatA_{>k_\ell}\bR_{>k_\ell}\right\|_{\fro}^2\\
    &= \sum_{i=1}^{\ell-1} \left\|\bstA_{\cK_i}  - \bhatA_{\cK_i}\bR_{\cK_i}\right\|_{\fro}^2 + \|\bhatA_{>k_\ell}\bR_{>k_\ell}\|_{\fro}^2\\
    &\qquad+ \sum_{i\ne j}^{\ell-1} \underbrace{\left\langle \bstA_{\cK_i}  - \bhatA_{\cK_i}\bR_{\cK_i},  \bstA_{\cK_j} - \bhatA_{\cK_j}\bR_{\cK_j}\right\rangle}_{=0} - \sum_{i =1}^{\ell-1} \underbrace{\left\langle \bstA_{\cK_i}  - \bhatA_{\cK_i}\bR_{\cK_i},  \bhatA_{>k_\ell}\bR_{>k_\ell}\right\rangle}_{=0},
    \end{align*}
    where the terms on the second line vanish because they involve inner products of matrices whose columns have disjoint support. We bound 
    \begin{align*}
    \sum_{i=1}^{\ell-1} \left\|\bstA_{\cK_i}  - \bhatA_{\cK_i}\bR_{\cK_i}\right\|_{\fro}^2 &=  \sum_{i=1}^{\ell-1}  \frac{1}{(\delst_{k_i} \wedge \delst_{k_{i+1}})^2(\sigst_{k_{{i+1}}})}  \cdot \underbrace{{(\sigst_{k_{{i+1}}})}(\delst_{k_i} \wedge \delst_{k_{i+1}})^2 \left\|\bstA_{\cK_i}  - \bhatA_{\cK_i}\bR_{\cK_i}\right\|_{\fro}^2}_{\le E_i}\\
    &\le  \delspace^{-2}\sum_{i=1}^{\ell-1} \frac{1}{\sigst_{k_{{i+1}}}} E_i \tag{partition is $(\delspace,\mucond)$-well-tempered}\\
    &\le  \delspace^{-2}\max_{i \in [\ell-1]}E_i \cdot\sum_{i=1}^{\ell-1} \frac{1}{\sigst_{k_{{i+1}}}} =   \frac{\muspec}{\delspace^2}\max_{i \in [\ell-1]}E_i.
    \end{align*}

    Finally, using $\bR \in \bbO(d)$ and $\bhatA_{>k_\ell} = \bhatU\bhatSigma_{>k_\ell}^{\frac{1}{2}}$ for $\|\bhatU\|_{\op} = 1$, we have that 
    \begin{align*}
    \|\bhatA_{>k_\ell}\bR_{>k_\ell}\|_{\fro}^2 &\le \|\bhatA_{>k_\ell}\|_{\fro}^2 \le \|\bhatSigma_{>k_\ell}^{\frac{1}{2}}\|_{\fro}^2 = \sum_{i > k_\ell}\|\bhatSigma_{i}^{\frac{1}{2}}\|_{\fro}^2 = \sum_{i>k_\ell}\sigma_i(\bhatM).
    \end{align*} 
    In sum, 
    \begin{align*}
    \|\bstA_{[k_{\ell}]}- \bhatA \bR\|_{\fro}^2 &= \left\|\sum_{i=1}^{\ell-1} \left(\bstA_{\cK_i}  - \bhatA_{\cK_i}\bR_{\cK_i}\right) - \bhatA_{>k_\ell}\bR_{>k_\ell}\right\|_{\fro}^2\\
    &= \sum_{i=1}^{\ell-1} \left\|\bstA_{\cK_i}  - \bhatA_{\cK_i}\bR_{\cK_i}\right\|_{\fro}^2 + \|\bhatA_{>k_\ell}\bR_{>k_\ell}\|_{\fro}^2\\
    &\le \frac{\muspec}{\delspace^2}\max_{i \in [\ell-1]}E_i + \sum_{i>k_\ell}\sigma_i(\bhatM),
    \end{align*}
    as needed. 
    \end{proof}

\subsubsection{Proof of \Cref{lem:application_of_procrustes} \label{sec:lem:app_procrustes}}    
    \begin{proof}[Proof of \Cref{lem:application_of_procrustes}] We observe that for any monotone (in particular, well-tempered) partition, 
    \begin{align*}
    \bhatM_{\cK_i} =  \bhatM_{[k_{i+1}]}-\bhatM_{[k_i]},  \qquad \bstM_{\cK_i} = \bstM_{[k_{i+1}]}-\bstM_{[k_i]}, 
    \end{align*}
    where we let $\bstM_{[k_1]}= \bstM_{[0]} = 0$.
    Thus, $\circ = \{\op,\fro\}$, 
    \begin{align*}
    \max_{i \in [\ell]}(\delstki\wedge\delstkipl)\|\bhatM_{\cK_i} - \bstM_{\cK_i}\|_{\circ} &\le \max_{i \in [\ell]} \delstki\|\bhatM_{[k_i]} - \bstM_{[k_i]}\|_{\circ} + \delstkipl\|\bhatM_{[k_{i+1}]} - \bstM_{[k_{i+1}]}\|_{\circ}\\
    &\le 2\max_{i \in [\ell+1]} \delstki\|\bhatM_{[k_i]} - \bstM_{[k_i]}\|_{\circ} =: 2\tilde{\epsilon}_{\circ}.  
    \end{align*}
    Next, for a matrix of the form $\bA_{\cK_i} \in \R^{n \times d}$, let $\bA_{\langle\cK_i\rangle} \in \R^{n \times |\cK_i|}$ denote its canonical compact  representation.  We observe then that
    \begin{align*}
    \bhatA_{\langle\cK_i\rangle}\bhatB_{\langle\cK_i\rangle}^\top  = \bhatM_{\cK_i}, \qquad \bstA_{\langle\cK_i\rangle}(\bstB_{\langle\cK_i\rangle})^\top = \bstM_{\cK_i}. 
    \end{align*}
    Further, observe that $\sigma_{|\cK_i|}(\bstM_{\cK_i}) = \min\{\sigst_{k'} : k \in \cK_i\} = \sigst_{k_{{i+1}}}$. 
    Hence, \Cref{lem:proc} implies the following: for a given $i \in [\ell]$, if $2\epstilop \le \delspace\frac{\sigst_{k_\ell}}{2} \le \frac{(\delstki\wedge\delstkipl)\sigst_{k_{i+1}}}{2}$, then there exists an  orthogonal matrix $\bO_i \in \bbO(|\cK_i|)$ such that 
    \begin{align*}
    \|\bhatA_{\langle\cK_i\rangle}\bO_i - \bstA_{\langle\cK_i\rangle}  \|_{\fro}^2 + \|\bhatB_{\langle\cK_i\rangle}\bO_i - \bstB_{\langle\cK_i\rangle}  \|_{\fro}^2  &\le \frac{\cproc}{\sigma_{|\cK_i|}(\bstM_{\cK_i})} \| \bhatM_{\cK_i}  - \bstM_{\cK_i}\|_{\fro}^2\\
     &= \frac{\cproc}{\sigst_{k_{{i+1}}}} \| \bhatM_{\cK_i}  - \bstM_{\cK_i}\|_{\fro}^2.
    \end{align*}
    Multiplying both sides of the above inequality by $(\delstki\wedge\delstkipl)^2\sigst_{k_{{i+1}}}$,  there exists a $\bO_i\in \bbO(|\cK_i|)$ such that  
    \begin{align*}
    (\delstki\wedge\delstkipl)^2\sigst_{{k_{{i+1}}}} \left(\|\bhatA_{\langle\cK_i\rangle}\bO_i - \bstA_{\langle\cK_i\rangle}  \|_{\fro}^2 \vee \|\bhatB_{\langle\cK_i\rangle}\bO_i - \bstB_{\langle\cK_i\rangle}  \|_{\fro}^2\right)  &\le \cproc (\delstki\wedge\delstkipl)^2 \| \bhatM_{\cK_i}  - \bstM_{\cK_i}\|_{\fro}^2\\
    &\le 4\cproc \epstilfro^2.
    \end{align*}

    Now, let $\bR$ be the block matrix compatible with $\partshort$, such that  
    $\bR_{\langle \cK_{i} \rangle} = \bO_i$ (that is, the block of $\bR$ corresponding to the set $\cK_i$ is the matrix $\bO_i$).   Since $\bR$ is a block-orthogonal matrix, it is orthogonal. Moreover, it is straightforward that
    \begin{align*}
    \|\bhatA_{\langle\cK_i\rangle}\bO_i - \bstA_{\langle\cK_i\rangle}  \|_{\fro} = \|\bhatA_{\cK_i}\bR_{\cK_i} - \bstA_{\cK_i}  \|_{\fro},
    \end{align*}
    and analogously for the ``$\bB$''-factors. Hence, for all $i \in [\ell]$ 
    \begin{align*}
    E_i = (\delstki\wedge\delstkipl)^2\sigst_{{k_{{i+1}}}} \left(\|\bhatA_{\cK_i}\bR_{\cK_i} - \bstA_{\cK_i} \|_{\fro}^2 \vee \|\bhatB_{\cK_i}\bR_{\cK_i} - \bstB_{\cK_i}  \|_{\fro}^2\right)  &\le 4\cproc \epstilfro^2.
    \end{align*}
    This completes the proof. 
    \end{proof}

\subsection{Existence of well-tempered partition} 

\subsubsection{Proof of \Cref{lem:ssv_space} \label{sec:lem:ssv_space}}

We first restate the lemma as below. 
    \lemssv*
     \newcommand{\tmpk}{\tilde k} 
    Define
    \begin{align*}
    \tmpk_{1} = \begin{cases} \max\left\{ k' <s: \delst_{k'} \ge \frac{1}{s}, \sigst_{k'} \ge \sigma\right\} & \text{ if such a $k' \ge 0$ exists}\\
    0 & \text{otherwise}. 
    \end{cases}
    \end{align*}
    And for $i \ge 1$, define 
    \begin{align*}
    \tmpk_{i+1} = \begin{cases} \max\left\{ k' < \tmpk_{i}: \delst_{k'} \ge \frac{1}{\tmpk_{i}}, \sigst_{k'} \ge e\sigst_{\tmpk_{i}}\right\} & \text{ if such a $k' \ge 0$ exists}\\ 
    0 & \text{otherwise}. 
    \end{cases}
    \end{align*}
    We terminate this recursive definition of $\tmpk_i$ the first time there is some $\tmpk_i=0$. Thus, 
    let $\ell:=\{i\geq 1:\tmpk_i=0\}$ (which is a unique index). Finally, for $i\in[\ell]$, choose  $k_i=\tmpk_{\ell+1-i}$.

    We now verify this sequence satisfies the desired properties. Items \textbf{(a)}  and \textbf{(c)} clear from the definition. 

    \paragraph{Item (b). } We mainly prove the argument that for  $i \in [\ell-1]$, $\sigst_{k_{i}+1}\le 2e^2\sigst_{k_{i+1}}$. 

    We may assume $k_{i+1} \ne 1$, since otherwise $k_i=0=k_1$, i.e. $i=1$, and $k_{i}+1 = 0 + 1 = 1 = k_{i+1}$, and the bound is vacuous. Continuing, fix an index $i$, and let $\bar{k}_i:= \max\{k' < k_{i+1}:  \sigst_{k'} \ge e\sigst_{k_{i+1}}\}$. Notice  that in particular $\sigst_{\bar{k}_i+1} < e\sigst_{k_{i+1}}$. Then, we  can equivalently express
    \begin{align*}
    k_{i} = \begin{cases} \max\left\{ k' < \bar{k}_i+1: \delst_{k'} \ge \frac{1}{k_{i+1}}\right\} &  \text{ if such a $k' \ge 0$ exists}\\
    0 & \text{otherwise}.
    \end{cases}
    \end{align*} 
    In particular, if $k_{i}+1>\bar{k}_{i}$, then $k_{i}=\bar{k}_{i}$, and thus $\sigst_{k_{i}+1}=\sigst_{\bar{k}_{i}+1}<e\sigst_{k_{i+1}}\leq 2e^2 \sigst_{k_{i+1}}$, we are finished. 
    
    Otherwise, if  $k_{i}+1\le \bar{k}_{i}$, we know that 
    for all 
    \begin{align}
    \forall j \in [k_{i}+1,\bar{k}_{i}],~~~ \delst_{j} \le \frac{1}{k_{i+1}}. \quad \label{eq:del_inter_bound}
    \end{align} 
     Hence, 
    \begin{align*}
    \sigst_{k_{i}+1} &= \sigst_{\bar k_i+1}\left(\prod_{j = k_{i}+1}^{\bar k_{i}} \frac{\sigst_{j}}{\sigst_{j+1}}\right)\\
    &= \sigst_{\bar k_i+1}\left(\prod_{j = k_{i}+1}^{\bar k_{i}} \frac{1}{1-\delst_{j}}\right) \tag{ $\frac{\sigst_{j+1}}{\sigst_{j}} = 1-\delst_{j}$ }\\
    &\le \sigst_{\bar k_i+1}\left(\prod_{j = k_{i}+1}^{\bar k_{i}} \frac{1}{1-1/k_{i+1}}\right) \tag{\Cref{eq:del_inter_bound}}\\
    &\le \sigst_{\bar k_i+1}\cdot (1 - 1/k_{i+1})^{-\left(\bar{k}_i - (k_{i}+1)\right) }\\ 
    &\le \sigst_{\bar k_i+1}\cdot (1 - 1/k_{i+1})^{-k_{i+1}} \le e\sigst_{k_{i+1}}\cdot (1 - 1/k_{i+1})^{-k_{i+1}} \tag{$\sigst_{\bar k_i+1} \le e\sigst_{k_{i+1}}$}.
    \end{align*} 
    Using the elementary inequality $(1 - \frac{1}{n})^{n} \ge e^{-1}(1 - \frac{1}{n})$ for $n \ge 1$, and the fact that $k_{i+1} \ge 2$, we obtain that $(1 - 1/k_{i+1})^{-k_{i+1}} \le 2e$. Hence, $\sigst_{k_{i}+1}  \le 2e^2\sigst_{k_{i+1}}$.
    
    Proof for the argument that $\sigst_{k_\ell+1} \le 2e\sigma$ is nearly identical, by  introducing notation   $\bar{k}_\ell$, defined as  $\bar{k}_\ell:= \max\{k' < s:  \sigst_{k'} \ge \sigma\}$ and noticing that  ${k_\ell}$ can be equivalently expressed as 
    \begin{align*}
    k_{\ell} = \begin{cases} \max\left\{ k' < \bar{k}_\ell+1: \delst_{k'} \ge \frac{1}{s}\right\} &  \text{if such a $k' \ge 0$ exists}\\
    0 & \text{otherwise},
    \end{cases}
    \end{align*}  
    and the rest of the proof follows the same argument. 
\hfill $\blacksquare$

\subsubsection{Proof of \Cref{prop:exist_well_tempered}\label{sec:prop_exist_well_temp}}

    We let $\partshort$ denote the partition whose pivots are given by the points in \Cref{lem:ssv_space}. 
    \paragraph{Item (a).} From item's (a) and (b) of \Cref{lem:ssv_space}, the partition is $(\delspace,\mu)$-well-tempered for $\delspace \ge 1/s$ and $\mu \le 2e^2$. 

    \paragraph{Item (b).} $\sigst_{k_\ell} \ge \sigma$  follows from   \Cref{lem:ssv_space}, part (c). From that same lemma, we also see that for $i\in[\ell]$,  $\sigst_{k_i} \ge e^{\ell-i}\sigst_{k_\ell} \ge e^{\ell-i}\sigma$. Hence
    \begin{align*}
    \muspec &= \sum_{i=1}^{\ell} (\sigst_{k_i})^{-1} \le \sum_{i=1}^{\ell} e^{-(\ell-i)}\sigma^{-1}  \le \sigma^{-1}\sum_{i\ge 0} e^{-i} = \frac{\sigma^{-1}}{1 - e^{-1}}.
    \end{align*}
    \paragraph{Item (c).} Finally, we develop bounds on $\muspace$.  We bound
    \begin{align*}
    \muspace = \sum_{i=1}^{\ell} (\delst_{k_i})^{-2} \le \ell \max_{i \in [\ell]}(\delst_{k_i})^{-2} \le \ell s^2.
    \end{align*}
    Clearly $\ell \le s$. Moreover, from \Cref{lem:ssv_space}, part (c), since $\sigma_{k_i}$ grow geometrically by factors of $e$, we must have that $\ell \le 1 + \ceil{\log \frac{\|\bstM\|_{\op}}{\sigma}}$. Hence, $\ell \le  \barl := \min\{1 + \ceil{\log \frac{\|\bstM\|_{\op}}{\sigma}},\,s\}$.

    \paragraph{Item (d).} We bound
    \begin{align*}
    \tail_2(\bstM;k_\ell) &= \sum_{j> k_\ell} (\sigst_j)^2 =  \sum_{j=k_\ell+1}^{s}(\sigst_j)^2 + \sum_{j > s}(\sigst_j) =  \sum_{j=k_\ell+1}^{s}(\sigst_j)^2 + \tail_2(\bstM;s)\\ 
    &\le s(\sigst_{k_\ell+1})^2+ \tail_2(\bstM;s)\\
    &\le 4e^2s\sigma^2 + \tail_2(\bstM;s), 
    \end{align*}
    where in the last line, we used \Cref{lem:ssv_space}, part (b). The bound on $\tail_1(\bstM;k_\ell)$ is analogous. 
\hfill $\blacksquare$

\end{proof}

\subsection{Useful linear  algebra facts}

We conclude the section by several useful facts about the linear algebra. 

    \begin{lem}[Eq. (1), \cite{li2020elementary}]\label{lem:li_elementary} Let $\bM, \bM' \in  \R^{n \times m}$ where $\rank(\bM') = r$. Then, 
    \begin{align*}
    \forall i \in \{1,\dots,\min\{n,m\} - r\}, \quad \sigma_{i}(\bM - \bM') \ge \sigma_{i+r}(\bM). 
    \end{align*}
    \end{lem}

    \begin{lem}[Theorem A.14, \cite{bai2010spectral}]\label{lem:von_neuman_one} Let $\bM = \bA \bB^\top$ have rank (at most) $r$. Then, $\sum_{i=1}^r \sigma_i(\bM) \le \sum_{i=1}^r \sigma_i(\bA)  \sigma_i(\bB)$. 
    \end{lem}

    \begin{lem}[Theorem A.37 (ii),  \cite{bai2010spectral}]\label{lem:Mat_pert_one} For any $\bM,\bM' \in \R^{n\times m}$, 
    \begin{align*}
    \sum_{i=1}^{\nu} (\sigma_i(\bM)-\sigma_i(\bM'))^2 \le \|\bM - \bM'\|_{\fro}^2,
    \end{align*}
    where the above holds for $\nu = \min(n,m)$, and thus, also holds for any $1 \le \nu \le \min(n,m)$. 
    \end{lem}

\section{The Balancing Operator}\label{sec:balancing_operator}


\subsection{Properties of the balancing operator}
     
    \balop* 
    
    The uniqueness of $\bW=\Psibal(\bY; \bX) $ (and hence the well-definedness of the map $\Psibal$) is a consequence of the following lemma.
    
    \begin{lem}\label{lem:bal_properties} Let $\bX,\bY \in \pd{p}$. The balancing operator has the following properties:
    \begin{enumerate}[label=(\roman*)]
        \item\label{item:bal_unique}  \textbf{Uniqueness: } There is a unique $\bW = \Psibal(\bY; \bX)$ is the unique positive definite matrix satisfying $\bX = \bW \bY \bW$, so that $\Psibal$ is well-defined.  
        \item\label{item:positive_sacling} \textbf{Positive scaling:} $\Psibal(\alpha \bY;\bX) = \alpha^{-\frac{1}{2}}\Psibal(\bY;\bX)$.
        \item\label{item:anti_monotone} \textbf{Anti-monotonicity:}  If $\bY \succeq \bY'$, then $\Psibal(\bY; \bX) \preceq \Psibal(\bY'; \bX)$.
        \item\label{item:bal_comparison} \textbf{Comparison with $\bX$:} If $\bY \succeq \tau\bX$, then $\Psibal(\bY; \bX) \preceq \tau^{-\half}\eye_p$, similarly, if  $\bY \preceq \tau\bX$, $\Psibal(\bY; \bX) \succeq \tau^{-\half}\eye_p$. 
        \item\label{item:id_comparison} \textbf{Comparison with identity:} If $\bY \succeq \tau\eye_p$, then $\Psibal(\bY; \bX) \preceq \tau^{-\half}\bX^{\frac{1}{2}}$; similarly, if  $\bY \preceq \tau\eye_p$, $\Psibal(\bY; \bX) \succeq \tau^{-\half}\bX^{\frac{1}{2}}$.
        \item\label{item:inverse_symmetry} \textbf{Inverse symmetry:} $\Psibal(\bY;\bX) = \Psibal(\bX;\bY)^{-1}$.
        \item\label{item:bal_WZ}  Let $\bZ = \bW^{\half}\bY \bW^{\half} = \bW^{-\half} \bX \bW^{-\half}$. Then, there exist orthogonal matrices $\bO_1,\bO_2\in \bbO(p)$ such that $\bZ \preceq \frac{1}{2}(\bO_1 \bX \bO_1^\top + \bO_2 \bY \bO_2^\top)$. Moreover, $\lambda_i(\bZ) = \sigma_i(\bX^{\half}\bY^{\half})$.
        \end{enumerate}
    \end{lem}
    \begin{proof} \textbf{\Cref{item:bal_unique}.} One can directly check that $\bW = \Psibal(\bY;\bX)$  satisfies $\bX = \bW \bY \bW$. For uniqueness,  $\bW$ satisfying $\bX = \bW \bY \bW$ satisfies $\eye_p = \bW'(\bX^{\half}\bY \bX^{\half})\bW'$, where $\bW' := \bX^{-\half} \bW \bX^{-\half}$. Thus $(\bW')^{-2} = \bX^{\half}\bY \bX^{\half}$, so that $(\bW')^2 = (\bX^{\half}\bY \bX^{\half})^{-1}$. Note that $\bX^{\half}\bY \bX^{\half} \succ 0$, and since we stipulate $\bW \succ 0$, $\bW' \succ 0$. Thus, by \cite[Theorem 7.2.6]{horn2012matrix}, it follows that $\bW' = (\bX^{\half}\bY \bX^{\half})^{-\half}$ is the unique positive definite square root of $ (\bX^{\half}\bY \bX^{\half})^{-1}$. Solving for $\bW = \bX^{\half}\bW' \bX^{\half}$, we see $\bW = \bX^{\half}(\bX^{\half}\bY \bX^{\half})^{-\half}\bX^{\half}$. 

    \paragraph{\Cref{item:positive_sacling}. } This is a straightforward computation.

    \paragraph{\Cref{item:anti_monotone}.} Let $\bY \succeq \bY'$. Then, $\bX^{\half}\bY \bX^{\half} \succeq \bX^{\half}\bY' \bX^{\half}$. The mapping $\bZ \mapsto \bZ^{-\half}$ is operator anti-monotone on $\psd{p}$ (\cite[Corollary 7.7.4]{horn2012matrix}). Thus, 
    \begin{align*}
    (\bX^{\half}\bY \bX^{\half})^{-\half} &\preceq (\bX^{\half}\bY' \bX^{\half})^{-\half}.
    \end{align*} 
    Therefore, 
    \begin{align*}
    \Psibal(\bY;\bX) = \bX^{\half}(\bX^{\half}\bY \bX^{\half})^{-\half}\bX^{\half} \preceq \bX^{\half}(\bX^{\half}\bY' \bX^{\half})^{-\half}\bX^{\half} = \Psibal(\bY';\bX).
    \end{align*}

    \paragraph{\Cref{item:bal_comparison,item:id_comparison} } Fix a $r \in \R$. Then
    \begin{align*}
    \Psibal(\tau \bX^r;\bX) = \bX^{\half}(\bX^{\half} \cdot (\tau \bX^r) \bX^{\half})^{-\half}\bX^{\half} = \bX^\half (\tau \bX^{r+1})^{-\half} \bX^\half = \tau^{-\half}\bX^{\frac{1-r}{2}}
    \end{align*}
    In particular, if $r = 1$, $\Psibal(\tau \bX^r;\bX) = \tau^{-\half}\eye_p$, whereas if $r = 0$, $\Psibal(\tau \eye_p;\bX) = \tau^{-\half}\bX^{\frac{1}{2}}$. The conclusion follows from monotonicity. 

    \paragraph{\Cref{item:inverse_symmetry}} If $\bW$ satisfies $\bX = \bW \bY \bW$, then $\bW' = \bW^{-1}$ satisfies $\bY = \bW' \bX \bW'$. The result follows from the uniqueness  of $\Psibal$.
 
    \paragraph{\Cref{item:bal_WZ}} We start with the following claim:
    \begin{claim} Consider a PSD matrix $\bLambda = \bL\bL^\top \in \pd{p}$ with $\bL \in \R^{p \times p}$, we have $\bLambda^{\half} = \bO^\top \bL^\top = \bL\bO$ for some $\bO \in \bbO(p)$.
    \end{claim}
    \begin{proof} Let $\bL = \bU\bSigma \bV^\top$ be an SVD of $\bL$. Then, $\bLambda = \bU\bSigma^2\bU^\top$, $\bLambda^{\half} = \bU\bSigma\bU^\top = \bU\bSigma\bV^\top \bV\bU^\top = \bL (\bV\bU^\top)$. Similarly, $\bLambda^{\half} = \bU\bV^\top\bL^\top = \bO^\top \bL^\top$. 
    \end{proof}
    Now, set $\bZ = \bW^{-\half}\bX\bW^{-\half} = \bW^{\half}\bY\bW^{\half}$. Then, by the above claim there exist orthogonal matrices  $\bO_1,\bO_2$ such that $\bO_1\bX^{\frac{1}{2}}\bW^{-\half} = (\bW^{-\half}\bX\bW^{-\half})^{\half} = \bZ^{\half}$ and $\bW^{\half}\bY^{\half}\bO_2 = (\bW^{\half}\bY\bW^{\half})^{\half} = \bZ^{\half}$. Hence, 
    \begin{align*}
    \bZ = \bO_1\bX^{\half}\bW^{-\half}\bW^{\half}\bY^{\half}\bO_2 = \bO_1 \bX^{\half}\bY^{\half}\bO_2. 
    \end{align*}
    Thus, for any $\bv \in \R^p$,
    \begin{align*}
    \bv^\top \bZ \bv &= \bv^\top\bO_1 \bX^{\half}\bY^{\half}\bO_2\bv\\
    &\le \|\bv^\top\bO_1 \bX^{\half}\| \cdot \|\bY^{\half}\bO_2\bv\|\\
    &\le \frac{1}{2}\left(\|\bv^\top\bO_1 \bX^{\half}\|^2 +\|\bY^{\half}\bO_2\bv\|^2\right)\\
    &= \frac{1}{2}\bv^\top\left(\bO_1 \bX \bO_1^\top + \bO_2 \bY \bO_2^\top \right)\bv.
    \end{align*}
    Moreover, since $\bZ \in \pd{p}$, we have $\lambda_i(\bZ) =  \sigma_i(\bZ ) = \sigma_i(\bO_1\bX^{\half}\bY^{\half}\bO_2) =  \sigma_i(\bX^{\half}\bY^{\half})$. 
    \end{proof}

\subsection{Balancing ``close'' covariances}

\begin{lem}\label{lem:exist_transform} Let $\bSigma,\bSigma' \in \psd{p}$ be two matrices with $\range(\bSigma) = \range(\bSigma')$. Then, there exists a transformation $\bT \in \pd{p}$ such that  
\begin{align*}
\bT\bSigma\bT = \bT^{-1}\bSigma'\bT^{-1}, \quad \text{ and, since $\bT = \bT^{\top}$,~~}  \bT\bSigma\bT^\top = \bT^{-1}\bSigma'\bT^{-\top}.
\end{align*}
Moreover, this transformation satisfies, for $r = \rank(\bSigma)$,
\begin{align*}
\max\{\|\bT\|_\op,\|\bT^{-1}\|_\op\} &\le (1+\Delta)^{1/4}, \quad  \text{where } \Delta := \frac{\|\bSigma - \bSigma'\|_{\op}}{\lambda_r(\bSigma) \wedge  \lambda_r(\bSigma')},\\
\sigma_i(\bT\bSigma\bT) = \sigma_i(\bT\bSigma\bT^\top)  &= \sigma_i(\bSigma^{\half}\bSigma^{'\half}) \numberthis \label{eq_sig_thing_itwo}
\end{align*}
Lastly, if $\rank(\bSigma) = \rank(\bSigma') = p$, then $\bT$ is unique and given by
\begin{align*}
\bT = \Psibal(\bSigma';\bSigma)^{\half} =\left(\bSigma^{\half}(\bSigma^{\half}\bSigma' \bSigma^{\half})^{-\half}\bSigma^{\half}\right)^{\half}
\end{align*}
\end{lem}
\begin{proof} The last part of the theorem, when $\rank(\bSigma) = \rank(\bSigma') = p$, is a direct consequence of \Cref{lem:bal_properties}. We now handle the case when $\rank(\bSigma) = \rank(\bSigma') < p$. Let $\bU \in \R^{p \times r}$ consist of columns which form an orthonormal basis for $\range(\bSigma) = \range(\bSigma')$. Set $\bX = \bU^\top \bSigma \bU$ and $\bY = \bU^\top \bSigma' \bU$. Then, 
\begin{align*}
\|\bX^{-\half}\bY\bX^{-\half}\|_{\op} \le 1 +  \frac{\|\bSigma - \bSigma'\|_{\op}}{\lambda_r(\bSigma)}, \quad \|\bY^{-\half}\bX\bY^{-\half}\|_{\op} \le 1 +  \frac{\|\bSigma - \bSigma'\|_{\op}}{\lambda_r(\bSigma')}.
\end{align*}
Thus, setting $\Delta := \frac{\|\bSigma - \bSigma'\|_{\op}}{\lambda_r(\bSigma) \wedge  \lambda_r(\bSigma')}$, we have
\begin{align*}
\bY \preceq (1+\Delta)\bX, \quad \bX \preceq (1+\Delta) \bY.
\end{align*}
Let $\bW := \Psibal(\bY;\bX)$. Then, from \Cref{lem:bal_properties}.\ref{item:bal_comparison},
\begin{align}
\max\{\|\bW\|_\op,\|\bW^{-1}\|_\op\} \le \sqrt{1+\Delta}. \label{eq:Delta_bW}
\end{align}
Moreover, from \Cref{lem:bal_properties}.\ref{item:bal_WZ},
\begin{align*}
\sigma_i(\bW^{\half}\bY\bW^{\half}) = \sigma_i(\bW^{-\half}\bX\bW^{-\half}) &= \sigma_i(\bX^{\half}\bY^{\half}) \\
&= \sigma_i((\bU^\top\bSigma \bU)^{\half}(\bU^\top\bSigma' \bU)^{\half})\\
&= \sigma_i(\bSigma^{\half}\bSigma^{'\half}), \quad i \in[r]\numberthis\label{eq:sig_i_last_line}
\end{align*}
where the last equality can be verified by a diagonalization argument, and using  the fact that $\bU$ is a basis for the row space of $\bSigma,\bSigma'$.

To construct the transformation $\bT$, set
\begin{align*}
\bT = \bU\bW^{-\half}\bU^\top + (\eye_p - \bU\bU^\top), \text{ so that } \bT^{-1} = \bU\bW^{\half}\bU^\top + (\eye_p - \bU\bU^\top).
\end{align*}
Note that $\bT \in \pd{p}$, since $\bW \in \pd{r}$ and $\bU$ is orthonormal. Since $(\eye_p - \bU\bU^\top)\bSigma = \bSigma(\eye_p - \bU\bU^\top) = 0$ (and similarly with $\bSigma'$),
\begin{align*}
\bT\bSigma \bT &= \bU\bW^{-\half}\bU^\top\bSigma \bU\bW^{-\half}\bU^\top\\
&= \bU\bW^{-\half}\bX\bW^{-\half}\bU^\top \\
&= \bU\bW^{\half}\bY\bW^{\half}\bU^\top \numberthis \label{eq:intermediate_some_thing}\\
&= \bU\bW^{\half}\bU^\top \bSigma' \bU\bW^{\half}\bU^\top \\
&= \bT^{-1} \bSigma' \bT^{-1}. 
\end{align*}
Moreover, by \Cref{eq:Delta_bW},
\begin{align*}
\max\{\|\bT\|_{\op},\|\bT^{-1}\|_{\op}\} = \max\{1,\|\bU\bW^{-\frac{1}{2}}\bU^\top\|_{\op},\|\bU\bW^{\frac{1}{2}}\bU^\top\|_{\op}\} \le (1+\Delta)^{1/4}. 
\end{align*}
Finally, by \Cref{eq:intermediate_some_thing} followed by \Cref{eq:sig_i_last_line},
\begin{align*}
 \sigma_i(\bT\bSigma \bT) = \sigma_i(\bU\bW^{-\half}\bX\bW^{-\half}\bU^\top) =   \sigma_i(\bW^{-\half}\bX\bW^{-\half}) = \sigma_i(\bSigma^{\half}\bSigma^{'\half}), i \in [r],
\end{align*}
whereas, for $i > r$, we verify that $\bW^{-\half}\bX\bW^{-\half} = 0$. Since $\bSigma,\bSigma'$ share the same range and have rank $r$, we have $\sigma_i(\bSigma^{\half}\bSigma^{'\half}) = 0$ for $i > r$. 
\end{proof}

\subsubsection{Perturbation of the balancing operator}

\begin{lem}[Perturbations of $\Psibal$, Relative Error]\label{lem:Psibal_rel} Fix $\epsilon \in (0,1)$. Then,
\begin{itemize} 
\item[(a)] Let $\bX,\bY,\bY' \in \pd{p}$, with $(1-\epsilon)\bY \preceq \bY' \preceq (1+\epsilon)\bY$. Then,
\begin{align*}
(1+\epsilon)^{-\frac{1}{2}} \Psibal(\bY; \bX) \preceq \Psibal(\bY';\bX)  \preceq (1-\epsilon)^{-\frac{1}{2}} \Psibal(\bY; \bX).
\end{align*}
\item[(b)] Similarly, if $\bX,\bX',\bY \in \pd{p}$, with $(1-\epsilon)\bX \preceq \bX' \preceq (1+\epsilon)\bX$. Then,
\begin{align*}
(1-\epsilon)^{\frac{1}{2}} \Psibal(\bY; \bX) \preceq \Psibal(\bY;\bX')  \preceq (1+\epsilon)^{\frac{1}{2}} \Psibal(\bY; \bX).
\end{align*}
\item[(c)] Finally, $\bX,\bX',\bY,\bY'\in \pd{p}$, with $(1-\epsilon)\bX \preceq \bX' \preceq (1+\epsilon)\bX$ and $(1-\epsilon)\bY \preceq \bY' \preceq (1+\epsilon)\bY$, then
\begin{align*}
\left(1-2\epsilon\right) \Psibal(\bY;\bX) \preceq \Psibal(\bY';\bX') \preceq \left(1 + \frac{2\epsilon}{1-\epsilon}\right) \Psibal(\bY;\bX).
\end{align*}
\end{itemize}
\end{lem}
\begin{proof}[Proof of \Cref{lem:Psibal_rel}] By anti-monotonicity  of $\Psibal(\cdot;\bX)$ and the explicit formula for $\Psibal$,
\begin{align*}
\Psibal(\bY';\bX) \succeq \Psibal((1+\epsilon)\bY; \bX) = (1+\epsilon)^{-\frac{1}{2}} \Psibal(\bY; \bX).
\end{align*}
By the same token, 
\begin{align*}
(1+\epsilon)^{-\frac{1}{2}} \Psibal(\bY; \bX) \preceq \Psibal(\bY';\bX)  \preceq (1-\epsilon)^{-\frac{1}{2}} \Psibal(\bY; \bX).
\end{align*}
Hence, the result follows from the inverse symmetry of $\Psibal$ (\Cref{lem:bal_properties}.\ref{item:inverse_symmetry}).

Finally, combining the first two parts of the lemma, we have
\begin{align*}
\Psibal(\bY';\bX')  \succeq (1+\epsilon)^{-\frac{1}{2}} \Psibal(\bY; \bX') \succeq (1-\epsilon)^{\frac{1}{2}}(1+\epsilon)^{-\frac{1}{2}}\Psibal(\bY; \bX), 
\end{align*}
and 
\begin{align*}
\Psibal(\bY';\bX')  \preceq (1-\epsilon)^{-\frac{1}{2}} \Psibal(\bY; \bX') \preceq (1-\epsilon)^{-\frac{1}{2}}(1+\epsilon)^{\frac{1}{2}}\Psibal(\bY; \bX).
\end{align*}
To conclude, we bound 
\begin{align*}
(1-\epsilon)^{\frac{1}{2}}(1+\epsilon)^{-\frac{1}{2}} = \sqrt{ \frac{1-\epsilon}{1+\epsilon}} = \sqrt{1 - \frac{2\epsilon}{1+\epsilon} } \ge \sqrt{1 - 2 \epsilon} \ge 1 - 2\epsilon. 
\end{align*}
and 
\begin{align*}
(1+\epsilon)^{\frac{1}{2}}(1-\epsilon)^{-\frac{1}{2}} = \sqrt{ \frac{1+\epsilon}{1-\epsilon}} = \sqrt{1 + \frac{2\epsilon}{1-\epsilon} } \le 1 + \frac{2\epsilon}{1-\epsilon}. 
\end{align*}
concluding the proof. 
\end{proof}

\begin{lem}[Additive Perturbation of Balancing Operator]\label{lem:additive_pert_psibal} Let $\bX, \bX',\bY,\bY' \in \pd{p}$ be matrices such that $\mu \eye_p \preceq \bX,\bX',\bY,\bY' \preceq M\eye_p$, and  $\|\bX - \bX'\|_{\op},\|\bY - \bY'\|_{\op} \le \Delta \le \mu/3$ for some $\Delta>0$. Then,
\begin{align*}
\|\Psibal(\bY';\bX') - \Psibal(\bY;\bX)\|_{\op} \le  3\Delta \cdot \frac{\sqrt{M/\mu}}{\mu}. 
\end{align*}
Moreover, $\|\Psibal(\bY';\bX')\|_{\op},\|\Psibal(\bY;\bX)\|_{\op} \le \sqrt{M/\mu}$ and $\Psibal(\bY';\bX'), \Psibal(\bY;\bX) \succeq \sqrt{\mu/M} \eye_p$.
\end{lem}
\begin{proof}[Proof of \Cref{lem:additive_pert_psibal}] Under the above conditions, it holds that $(1- \mu^{-1}\Delta) \bX \preceq \bX' \preceq (1+\mu^{-1}\Delta)\bX$, and similarly for $\bY$ and $\bY'$. Applying \Cref{lem:Psibal_rel} with $\epsilon = \Delta/\mu \le 1/3$, we have
\begin{align*}
(1-2\epsilon) \Psibal(\bY; \bX) \preceq \Psibal(\bY';\bX')  \preceq (1+3\epsilon) \Psibal(\bY; \bX).
\end{align*}
This gives
\begin{align*}
\|\Psibal(\bY';\bX') - \Psibal(\bY;\bX)\|_{\op} \le 3\frac{\Delta}{\mu}\|\Psibal(\bY; \bX)\|_{\op}.
\end{align*}
Lastly, since $\bY \succeq \mu/M \bX$ (as $\bY \succeq \mu \eye_p$ and $\bX \preceq M \eye_p$), it holds $\|\Psibal(\bY; \bX)\|_{\op} \le \sqrt{M/\mu}$ by \Cref{lem:bal_properties}.\ref{item:bal_comparison}. Thus, 
\begin{align*}
\|\Psibal(\bY';\bX') - \Psibal(\bY;\bX)\|_{\op} \le 3\Delta \cdot \frac{\sqrt{M/\mu}}{\mu}. 
\end{align*}
A similar computation also shows $\|\Psibal(\bY'; \bX')\|_{\op} \le \sqrt{M/\mu}$, and $\Psibal(\bY';\bX'), \Psibal(\bY;\bX) \succeq \sqrt{\mu/M} \eye_p$.
\end{proof}

We recall the definition of the balanced covariance.

\balcov*

\begin{rem}[Symmetry of $\covbal$]\label{rem:covbal_sym}
Note that, from definition of $\Psibal$, we also have $\covbal(\bX,\bY) = \Psibal(\bY;\bX)^{\frac{1}{2}}\bY\Psibal(\bY;\bX)^{\frac{1}{2}} = \Psibal(\bX;\bY)^{-\frac{1}{2}}\bY\Psibal(\bX;\bY)^{-\frac{1}{2}}  = \covbal(\bY,\bX)$.
\end{rem}

\begin{lem}[Perturbation of Balanced Covariance]\label{lem:pert_bal_cov}  Let $\bX, \bX',\bY,\bY' \in \pd{p}$ be the matrices such that $\mu \eye_p \preceq \bX,\bX',\bY,\bY' \preceq M\eye_p$, and  $\|\bX - \bX'\|_{\op},\|\bY - \bY'\|_{\op} \le \Delta \le \mu/3$ for some $\Delta>0$. Then,
\begin{align*}
\|\covbal(\bX',\bY') - \covbal(\bX,\bY)\|_{\op} \le 4(M/\mu)^2 \Delta.
\end{align*}
Moreover, we have 
\begin{align*}
\|\Psibal(\bY';\bX')^{-\frac{1}{2}} - \Psibal(\bY;\bX)^{-\frac{1}{2}}\|_{\op} \vee \|\Psibal(\bY';\bX')^{\frac{1}{2}} - \Psibal(\bY;\bX)^{\frac{1}{2}}\|_{\op} \le \frac{3}{2 \mu}(M/\mu)^{3/4} 
\Delta.\end{align*}
\end{lem}
\begin{proof}[Proof of \Cref{lem:pert_bal_cov}] We have
\begin{align*}
&\|\covbal(\bX',\bY') -\covbal(\bX,\bY)\|_{\op} \\
&= \|\Psibal(\bY';\bX')^{\frac{1}{2}}\bY'\Psibal(\bY';\bX')^{\frac{1}{2}} - \Psibal(\bY;\bX)^{\frac{1}{2}}\bY\Psibal(\bY;\bX)^{\frac{1}{2}}\|_{\op}\\
&\le \|\Psibal(\bY';\bX')^{\frac{1}{2}}(\bY'-\bY)\Psibal(\bY';\bX')^{\frac{1}{2}}\|_{\op} \\
&\quad + \|\Psibal(\bY';\bX')^{\frac{1}{2}} -  \Psibal(\bY;\bX)^{\frac{1}{2}}\|_{\op}\|\bY\|_{\op}(\|\Psibal(\bY';\bX')^{\frac{1}{2}}\|_{\op} +\|\Psibal(\bY;\bX)^{\frac{1}{2}}\|_{\op})\\
&\le \Delta \sqrt{M/\mu} + 2 M (M/\mu)^{1/4}\|\Psibal(\bY';\bX')^{\frac{1}{2}} - \Psibal(\bY;\bX)^{\frac{1}{2}}\|_{\op}.
\end{align*}
We now require following perturbation inequality for the matrix square root.
\begin{lem}[Perturbation of Matrix Square Root, Lemma 2.2. in \cite{schmitt1992perturbation}]\label{lem:mat_sqrt} Let $\bA_1,\bA_2 \in \pd{p}$ satisfy $\bA_1,\bA_2 \succeq \gamma \eye_p$. Then, $\|\bA_1^{\frac{1}{2}} - \bA_2^{\frac{1}{2}}\|_{\op} \le \frac{1}{2\sqrt{\gamma}}\|\bA_1 - \bA_2\|_{\op}$.
\end{lem}
Using $\Psibal(\bY';\bX')^{\frac{1}{2}},\Psibal(\bY;\bX)^{\frac{1}{2}} \succeq \sqrt{\mu/M}\eye_p$, \Cref{lem:mat_sqrt} followed by \Cref{lem:additive_pert_psibal} implies
\begin{align*}
\|\Psibal(\bY';\bX')^{\frac{1}{2}} - \Psibal(\bY;\bX)^{\frac{1}{2}}\|_{\op}& \le \frac{1}{2}(M/\mu)^{1/4}\|\Psibal(\bY';\bX') - \Psibal(\bY;\bX)\|_{\op}\\
\le \frac{3}{2 \mu}(M/\mu)^{3/4} \Delta. \numberthis\label{eq:inter_Psibal_thing_bound}
\end{align*}
Thus, we conclude the first part of the lemma:
\begin{align*}
\|\covbal(\bX',\bY') -\covbal(\bX,\bY)\|_{\op} \le \Delta \sqrt{M/\mu} + 3(M/\mu)^2\Delta \le 4(M/\mu)^2 \Delta.
\end{align*}
The second bound in the lemma was derived above, and  the bound on $\|\Psibal(\bY';\bX')^{\frac{1}{2}} - \Psibal(\bY;\bX)^{\frac{1}{2}}\|$ is precisely \Cref{eq:inter_Psibal_thing_bound}. The bound $\|\Psibal(\bY';\bX')^{\frac{-1}{2}} - \Psibal(\bY;\bX)^{\frac{-1}{2}}\|$ follows from the inverse symmetry of the balancing operator (\Cref{lem:bal_properties}.\ref{item:inverse_symmetry}). 
\end{proof}

\subsection{Balancing of finite-dimensional embeddings }\label{sec:balance_other_rank}

\begin{lem}\label{lem:in_range_as} Let $\distx$ be a distribution over $\xspac$, let $\bSigma = \Exp_{\distx}[ff^\top]$, and let $\bP$ be the orthogonal projection on $\range(\bSigma)$. Then $\bP f= f$ $\distx$-almost surely; that is, $\Pr_{\distx}[f(x) \in \range(\bSigma)] = 1$. 
\end{lem}
\begin{proof}
It suffices to show $\Exp[\|(\eye_p- \bP)f\|^2] = 0$. As $\bP \bSigma = \bSigma \bP$, we have
    \begin{align*}
    \Exp[\|(\eye_p- \bP)f\|^2]  &= \trace[\Exp[((\eye_p- \bP)f)((\eye_p- \bP)f)^\top]] = \trace(\bSigma - \bP \bSigma - \bSigma \bP^\top + \bP\bSigma \bP^\top) = 0\\
    &= \trace(\bSigma - \bSigma - \bSigma + \bSigma) = 0.
    \end{align*} 
\end{proof}

\begin{lem}\label{lem:exist_balance_not_simple} For any pair of embeddings $\fhat : \xspac \to \R^r$ and $\ghat: \yspace \to \R^r$, there exists embeddings $\tilde{f}: \xspac \to \R^r$ and $\tilde{g}: \yspace \to \R^r$ such that 
\begin{itemize} 
    \item[(a)] $\langle \fhat,\ghat \rangle = \langle \tilde f, \tilde g \rangle$ almost surely, and 
\begin{align*}
\Exp_{\cdxone}[\tilde f\tilde{f}^\top] = \Exp_{\cdyone}[\tilde{g}\tilde{g}^\top].
\end{align*}
\item[(b)] For all $i \in \N$, $\sigma_i(\fhat,\ghat) = \sigma_i(\Exp_{\cdxone}[\tilde f\tilde{f}^\top]) = \sigma_i(\Exp_{\cdxone}[\tilde f\tilde{f}^\top])$, where we recall $\sigma_i(\cdot,\cdot)$ defined in \Cref{eq:sigi_def}
\item[(c)] $(\tilde f,\tilde g)$ is full-rank if and only if $(\hat{f},\hat{g})$ is, and in this case $\bT$ is uniqely given by
\begin{align}
\bT = \Psibal(\bSigma_g;\bSigma_f)^{\half} =\left(\bSigma_f^{\half}(\bSigma_f^{\half}\bSigma_g \bSigma_f^{\half})^{-\half}\bSigma_f^{\half}\right)^{\half}, \label{eq:T_explicit}
\end{align}
where $\bSigma_f = \Exp[\fhat\fhat^\top]$ and $\bSigma_g = \Exp[\ghat\ghat^\top]$.
\item[(d)] There exists linear transformations $\bL_f,\bL_g$ such that $\tilde f = \bL_f \hat f,\tilde g = \bL_g \tilde g$. In particular, if $(\tilde f,\tilde g)$ satisfy an analogue of \Cref{asm:cov} for $\kappa > 0$,
\begin{align*}
&\Exp_{x\sim \cdx{2}}[\fhat(x)\fhat(x)^\top] \preceq \kappa\cdot\Exp_{x\sim \cdx{1}}[\fhat(x)\fhat(x)^\top] \quad \text{ and } \\
&\Exp_{y\sim \cdy{2}}[\ghat(y)\ghat(y)^\top] \preceq \kappa\cdot\Exp_{y\sim \cdy{1}}[\ghat(y)\ghat(y)^\top],
\end{align*}
then $\tilde{f},\tilde{g}$ satisfy the same inequality.
\end{itemize}
\end{lem}
\begin{proof} Given $(\fhat,\ghat)$, let us construct a sequence of embeddings $(\fhat_{i},\ghat_{i})_{i \ge 0}$, with covariances $\bSigma_{f,i} := \Exp_{\cdxone}[\fhat_i(\fhat_i)^\top]$ and $\bSigma_{g,i} := \Exp_{\cdyone}[\ghat_i(\ghat_i)^\top]$, and minimum rank
\begin{align*}
r_i := \min\{\rank(\bSigma_{f,i},\bSigma_{g,i}\}
\end{align*}
Lastly, set $\bP_{f,i}$ to be the orthogonal projection on the range of $\bSigma_{f,i}$ and $\bP_{g,i}$ the same for $\bSigma_{g,i}$. We define
\begin{align*}
(\fhat_0,\ghat_0) = (\fhat,\ghat), \quad (\fhat_{i+1},\ghat_{i+1}) = \begin{cases}
(\bP_{g,i}\fhat_i, \ghat_i) &\rank(\bSigma_{f,i}) \ge \rank(\bSigma_{g,i})\\
(\fhat_i, \bP_{g,i}\ghat_i)  & \text{otherwise}
\end{cases}
\end{align*}
We establish three claims.
\begin{claim}\label{claim:full_rank_iff} For any $n$, $(\fhat_n,\ghat_n)$ is full-rank if and only if $(\fhat,\ghat)$ is, which is true if and only if $\fhat_n = \fhat$ and $\ghat_n = g$. 
\end{claim}
\begin{proof} We argue by induction that $(\fhat_n,\ghat_{n})$ is full-rank if and only if $(\fhat_{n+1},\ghat_{n+1})$. The ``if'' follows since  $\rank(\bSigma_{\cdot,n}) \le \rank(\bSigma_{\cdot,n+1})$. The ``only if'' follows since if $(\fhat_n,\ghat_n)$ is full-rank, $\bP_{f,n} = \bP_{g,n}$ are the identity, and thus, $\fhat_{n+1} = \fhat_n$, $\ghat_{n+1} = \ghat_n$. 
\end{proof}
\begin{claim}\label{claim:same_inner_on_n} For any $n$, let holds that $\langle \fhat_n,\ghat_n \rangle = \langle \fhat,\ghat \rangle$ almost-surely under $\cdone$.
\end{claim}
\begin{proof} We prove by induction on $n$. The base case $n = 0$ is immediate. Assume now that  $\langle \fhat_n,\ghat_n \rangle = \langle \fhat,\ghat \rangle$  holds almost-surely under $\cdone$.  Assume that without los of generality $\rank(\bSigma_{f,n}) \ge \rank(\bSigma_{g,n})$, so that $(\fhat_{n+1},\ghat_{n+1}) = (\bP_{g,n}\fhat_n, \ghat_n)$. Then, by symmetry of the projection $\bP_{g,n}$, we have 
\begin{align*}
\langle \fhat_{n+1},\ghat_{n+1}\rangle = \langle \bP_{g,n}\fhat_n, \ghat_n \rangle =  \langle \fhat_n, \bP_{g,n} \ghat_n \rangle.
\end{align*}
By \Cref{lem:in_range_as}, $\bP_{g,n} \ghat_n = \ghat_n$ almost surely, and the result follows.
\end{proof}
\begin{claim}\label{claim:n_last_claim}
Let 
\begin{align*}
n := \inf\{i \in \N: \rank(\bSigma_{f,n})= \rank(\bSigma_{f,n+1}) \text{ and } \rank(\bSigma_{g,n})= \rank(\bSigma_{g,n+1})\}
\end{align*} 
Then $n$ is finite, and $\range(\bSigma_{f,n}) = \range(\bSigma_{g,n})$.
\end{claim}
\begin{proof} That $n$ is finite follows since the ranks of the covariances $\rank(\bSigma_{\cdot,i+1})\le  \rank(\bSigma_{\cdot,i})\}$ are non-increasing. Next, without loss of generality, assume that $\rank(\bSigma_{f,n}) \ge \rank(\bSigma_{g,n})$, so that $(\fhat_{n+1},\ghat_{n+1}) = (\bP_{g,n}\fhat_n, \ghat_n)$. Then, $\bSigma_{g,n} = \bSigma_{g,n+1}$, and
\begin{align*}
\bSigma_{f,n+1} = \bP_{g,n}\bSigma_{f,n}\bP_{g,n}, \text{ so } \range(\bSigma_{f,n+1}) \subset \range(\bP_{g,n}) = \range(\bSigma_{g,n}).
\end{align*}
On the other hand, 
\begin{align*}
\rank(\bSigma_{g,n}) \le \rank(\bSigma_{f,n}) = \rank(\bSigma_{f,n+1}) = \rank(\bP_{g,n}\bSigma_{f,n}\bP_{g,n}),
\end{align*}
which implies that $\range(\bSigma_{f,n+1}) \subset \range(\bSigma_{g,n}) = \range(\bSigma_{g,n+1})$. 
\end{proof}
Hence, let $\bT$ denote the (symmetric) positive definite transformation assured by applying  \Cref{lem:exist_transform} to 
\begin{align}
\bSigma \gets \bSigma_{f,n+1}, \quad\bSigma' \gets \bSigma_{g,n+1}; \label{eq:sigsigpr}
\end{align} these matrices have the same range by the above claim. Take  $\tilde{f} := \bT \fhat_{n+1}$ and $\tilde g := \bT^{-1} \ghat_{n+1}$. We show all desired properties holds.

\paragraph{Part (a).} The transformation $\bT$ ensures that 
\begin{align*}
\Exp_{\cdxone}[\tilde f\tilde{f}^\top] = \bT\bSigma_{f,n+1}\bT =  \bT^{-1}\bSigma_{g,n+1}\bT^{-1} = \Exp_{\cdyone}[\tilde{g}\tilde{g}^\top],
\end{align*}
Moreover, symmetry of $\bT$ and \Cref{claim:same_inner_on_n} imply that,  almost surely, 
\begin{align*}
\langle \tilde f , \tilde g \rangle = \langle \bT \fhat_{n+1}, \bT^{-1} \ghat_{n+1}\rangle = \langle  \fhat_{n+1}, \bT^\top\bT^{-1} \ghat_{n+1}\rangle = \langle  \fhat_{n+1},  \ghat_{n+1}\rangle =  \langle \fhat, \ghat \rangle,
\end{align*}

\paragraph{Part (b).} This is a consequence of \Cref{eq_sig_thing_itwo} in \Cref{lem:exist_transform}, noting that for $\bSigma,\bSigma'$ defined in \Cref{eq:sigsigpr} that $\sigma_r(\fhat,\ghat) := \sigma_r(\bSigma^{1/2}(\bSigma')^{1/2}))$.

\paragraph{Part (c). } Note that if $\fhat,\ghat$ are full-rank, $\bSigma_{f,n+1} = \bSigma_{f}$ and $\bSigma_{g,n+1} = \bSigma_g$, so that $\bT$ is uniquely given by \Cref{eq:T_explicit} due to \Cref{lem:bal_properties}. Note that $(\tilde{f},\tilde g)$ is full-rank if and only if $(\fhat_{n+1},\ghat_{n+1})$, which by \Cref{claim:full_rank_iff} is full-rank if and only if $(\fhat,\ghat)$ is. 
\end{proof}

\subsection{Analysis of separation rank}

\seprankdef*

\begin{lem}[Properties of Separated Rank]\label{lem:prop_sep_rank} Given $r_0 \in [p]$ and $\sigma \in [\sigma_{r_0}(\bSigma), \opnorm{\bSigma}/e]$, $\seprank(r_0,\sigma;\bSigma)$ enjoys the following properties:
\begin{itemize}
    \item[(a)] $\seprank(r_0,\sigma;\bSigma)$ is well-defined: i.e. for some $r\in[r_0]$, it holds that  $\sigma_r(\bSigma) - \sigma_{r+1}(\bSigma) \ge \frac{\sigma_r(\bSigma)}{r_0}$ and $\sigma_r(\bSigma) \ge \sigma$. 
    \item[(b)] For $r = \seprank(r_0,\sigma;\bSigma)$, we have $\sigma_{r+1}(\bSigma) \le e \sigma$. 
    \end{itemize}
\end{lem}
\begin{proof} To prove part (a), we observe that since $\sigma \ge \sigma_{r_0}(\bSigma)$, there must exist some maximal $r_{\max} \in [r_0]$ for which $\sigma_{r_{\max}}(\bSigma) \le \sigma$. Now suppose that, for the sake of contradiction, for all $r \le r_{\max}$, it holds that  $\sigma_r(\bSigma) - \sigma_{r+1}(\bSigma) <\frac{\sigma_r(\bSigma)}{r_0}$. Then, $\|\bSigma\|_{\op} = \sigma_1(\bSigma) \le (1+1/r_0)^{r_{\max}}\sigma_{r_{\max}+1}(\bSigma) \le (1+1/r_0)^{r_{\max}}\sigma \le e\sigma$. This contradicts our condition that $\sigma \le \|\bSigma\|_{\op}/e$.

To prove part (b), again let $r_{\max} \le r_0$ be as in the proof of part (a). We must have that $r = \seprank(r_0,\sigma;\bSigma) \le r_{\max}$. If $r = r_{\max}$, then $\sigma_{r+1}(\bSigma) \le \sigma$. Otherwise, for any $r' \in \{r+1,r+2,\dots,r_{\max}\}$, it holds that  $\sigma_{r'}(\bSigma) \le (1+1/r_0)\sigma_{r'+1}(\bSigma)$. Hence, $\sigma_{r+1} \le (1+1/r_0)^{r_{\max} - r}\sigma_{r_{\max} + 1} \le e \sigma_{r_{\max}+1} \le e\sigma$.
\end{proof}

\subsection{Proof of \Cref{prop:balproj}}\label{sec:proof_prop_balproj}
\begin{lem}\label{lem:proj_dif} Fix $\bSigma,\bSigma'$, $r_0 \in [p]$ and suppose $\|\bSigma - \bSigma'\|_{\op} \le  \sigma/4r_0$. Lastly, assume that one of the two hold
\begin{itemize}
    \item[(i)] $\sigma\in [\frac{4}{3}\sigma_{r_0}(\bSigma), \frac{4}{5e}\opnorm{\bSigma}]$.
    \item[(ii)] There exists positive numbers $\sigbar_{r_0}$ and $\sigbar_{1}$ satisfying $\max\{|\sigbar_{r_0} - \sigma_{r_0}(\bSigma)|, |\sigbar_1 - \sigma_1(\bSigma)|\} \le \sigma/4$ for which $\sigma\in [2 \sigbar_{r_0}, \frac{2}{3e}\sigbar_{1}]$. 
\end{itemize}
Let $r = \seprank(r_0,\sigma;\bSigma')$, and let $\bP_r$ and $\bP_{r}'$ denote the projections onto the top-$r$ singular spaces of $\bSigma$ and $\bSigma'$. Then, for any Schatten $p$-norm $\|\cdot\|_{\circ}$,  $\bP_r$ and $\bP_{r'}$ are unique, and
\begin{align*}
\|\bP_r - \bP_r'\|_{\circ} \le 4r_0\frac{\|\bSigma - \bSigma'\|_{\circ}}{\sigma}.
\end{align*}
The lemma also holds under the following more general condition: 
\end{lem}
\begin{proof}  Set $\Delta = \|\bSigma - \bSigma'\|_{\op}$. We shall show that under both conditions of the lemma, it holds that
\begin{align*}
\sigma \in [ \sigma_{r_0}(\bSigma'),\opnorm{\bSigma'}/e],
\end{align*}
so that the conditions of \Cref{lem:prop_sep_rank} are met.

\emph{Condition (i).} By Weyl's inequality and our assumption on $\Delta$ and the first assumption on $\sigma$,
\begin{align*}
\sigma - \sigma_{r_0}(\bSigma') &\ge \sigma (1- \tfrac{1}{4r_0}) - \sigma_{r_0}(\bSigma) \ge \tfrac{3\sigma}{4}-  \sigma_{r_0}(\bSigma) \ge 0\\
\opnorm{\bSigma'}/e - \sigma &\ge \opnorm{\bSigma}/e  - (1+\tfrac{1}{4r_0})\sigma \ge 0,
\end{align*}
so that $\sigma\in [\sigma_{r_0}(\bSigma'), \frac{1}{e}\opnorm{\bSigma'}]$.

\emph{Condition (ii).} We are given $\sigbar_{r_0}$ and $\sigbar_{1}$ satisfying $\max\{|\sigbar_{r_0} - \sigma_{r_0}(\bSigma)|, |\sigbar_1 - \sigma_1(\bSigma)|\} \le \sigma/4$ for which $\sigma\in [2 \sigbar_{r_0}, \frac{2}{3e}\sigbar_{1}]$. Thus,
\begin{align*}
\sigma - \sigma_{r_0}(\bSigma') &\ge \sigma (1- \tfrac{1}{4r_0} - \tfrac{1}{4}) - \sigbar_{r_0} (\bSigma_{\mathrm{ref}}) \ge \sigma/2 -  \sigbar_{r_0} \ge 0\\
\opnorm{\bSigma'}/e - \sigma &\ge \sigbar_1  - (1+\tfrac{1}{4r_0} + \tfrac{1}{2})\sigma \ge 0.
\end{align*}

Next, set $\mu = \sigma_{r}(\bSigma') - \Delta$. Then $\sigma_r(\bSigma') \ge \mu$, and by Weyl's  inequality, $\sigma_r(\bSigma) \ge \sigma_{r}(\bSigma') - \Delta = \mu$. Moreover, the definition of $\seprank$ ensures that $\sigma_{r}(\bSigma') \ge \sigma$ as well as
\begin{align*}
\sigma_{r+1}(\bSigma') \le \sigma_{r}(\bSigma')(1 - 1/r_0) = \mu  + \Delta  - \sigma_{r}(\bSigma')/r_0 \le \mu + \Delta - \sigma/r_0.  
\end{align*}
Again, by Weyl's inequality, $\sigma_{r+1}(\bSigma)  \le \mu + 2\Delta - \sigma/r_0.
$  
Thus, for $\Delta \le \sigma/4r_0$, $\max\{\sigma_{r+1}(\bSigma'),\sigma_{r+1}(\bSigma)\} \le \mu - \tau,$ where $\tau = \sigma/2r_0$. It follows from \Cref{lem:gap_free_DK} that if $\bU_{r}'$ is an orthonormal basis corresponding to the top $r$ eigenvalues of $\bSigma'$, and $\bU'_{>r}$ is an orthonormal basis corresponding to eigenvalues $r+1,\dots,p$, and defining $\bU_{r},\bU_{>r}$ analogously for $\bSigma$. Then,  one has that, for any Schatten-$p$ norm $\|\cdot\|_{\circ}$ 
\begin{align*}
\max\{\|(\bU_r')^\top\bU_{>r}\|_{\circ},\,\|(\bU_{>r}')^\top\bU_{r}\|_{\circ}\} \le 2r_0\frac{\|\bSigma - \bSigma'\|_{\circ}}{\sigma}.
\end{align*}
On the other hand, if  $\bP_r,\bP_r'$ denote the projections onto the top-$r$ eigenspaces of $\bSigma,\bSigma'$, we have
\begin{align*}
\|\bP_r' - \bP_r\|_{\circ} &\le \|\bP_r'(\bP_r' - \bP_r)\|_{\circ} +  \|(\eye_p - \bP_r')(\bP_r' - \bP_r)\|_{\circ}\\
&= \|\bP_r' - \bP_r'\bP_r\|_{\circ} +  \|(\eye_p - \bP_r')\bP_r\|_{\circ}\\
&= \|\bP_r'(\eye_p - \bP_r)\|_{\circ} +  \|(\eye_p - \bP_r')\bP_r\|_{\circ}\\
&= \|(\bU_r')^\top\bU_{>r}\|_{\circ} + \|(\bU_{>r}')^\top\bU_{r}\|_{\circ}\\
&\le 4r_0\frac{\|\bSigma - \bSigma'\|_{\circ}}{\sigma}.
\end{align*}
\end{proof}

\propbalproj*
\begin{proof}[Proof of \Cref{prop:balproj}] Throughout, we also set $\bSigma = \covbal(\bX,\bY)$, $\bW = \Psibal(\bY;\bX)$, and $\bW' =\Psibal(\bY';\bX')$. We further let $\bP_r$ and $\bP_r'$ denote the projections onto the top-$r$ singular spaces of $\bSigma$ and $\bSigma'$, respectively. 

By   \Cref{lem:pert_bal_cov} and that $\Delta \le \frac{\mu}{32r_0} (M/\mu)^2$ implies  $\Delta \le \mu/3$, we have that
\begin{align*}
\|\bSigma - \bSigma'\|_{\op} \le 4 (M/\mu)^2 \Delta.
\end{align*}
Hence, for $\Delta \le \frac{\mu}{32r_0} (M/\mu)^2$, it holds that 
\begin{align}
\|\bSigma - \bSigma'\|_{\op} \le \frac{\mu}{8r_0} \le \frac{\sigma}{8r_0}. \label{eq:Sigdiff_cond}
\end{align}
Moreover, condition (d) of the present proposition matches condition (ii) of  \Cref{lem:proj_dif} (recall that lemma only requires one of conditions (i) or (ii) to be met). Thus, \Cref{lem:proj_dif} implies
\begin{align*}
\|\bP_r - \bP_r'\|_{\op} \le 4r_0\frac{\|\bSigma - \bSigma'\|_{\op}}{\sigma} \le \frac{16r_0 (M/\mu)^2 \Delta}{\sigma} \le \frac{16r_0 (M/\mu)^2 \Delta}{\mu}.
\end{align*}
Thus, combining the above bound with \Cref{lem:pert_bal_cov} and the norm bounds on $\bW,\bW'$, as well as  their inverses due to \Cref{lem:additive_pert_psibal}, it follows that
\begin{align*}
&\|\bQ' - \bQ\|_{\op}\\
&=\|(\bW')^{-\frac{1}{2}}\bP_r'(\bW')^{\frac{1}{2}} - \bW^{-\frac{1}{2}}\bP_r\bW^{\frac{1}{2}}\|_{\op} \\
&\le \|(\bW')^{-\frac{1}{2}}(\bP_r'- \bP_r)(\bW')^{\frac{1}{2}}\|_{\op} + \|(\bW')^{\frac{1}{2}}\bP_r((\bW')^{\frac{1}{2}} - (\bW)^{\frac{1}{2}})\|_{\op} + \|((\bW')^{\frac{1}{2}}-\bW^{\frac{1}{2}})\bP_r (\bW)^{\frac{1}{2}}\|_{\op} \\
&\le \|\bP_r'- \bP_r\|_{\op}\|\bW'\|_{\op} + 2\max\{\|\bW\|_{\op},\|\bW'\|_{\op}\}^{1/2}\cdot\|(\bW')^{\frac{1}{2}} - (\bW)^{\frac{1}{2}}\|_{\op}\\
&\le \frac{16r_0 (M/\mu)^2 \Delta}{\mu} \cdot (M/\mu)^{1/2} +  2 \cdot (M/\mu)^{1/4} \frac{3}{2 \mu}(M/\mu)^{3/4}\Delta\\
&\le \frac{19r_0 (M/\mu)^{5/2} \Delta}{\mu}. 
\end{align*}
Second, using $\|\bW\|_{\op} \vee \|\bW^{-1}\|_{\op} \le\sqrt{M/\mu}$ from \Cref{lem:additive_pert_psibal}, 
\begin{align*}
\|\bQ\|_{\op} = \|\bW^{-\frac{1}{2}}\bP_r\bW^{\frac{1}{2}}\|_{\op} &\le \sqrt{\|\bW\|_{\op}\|\bW^{-1}\|_{\op}}\|\bP_r\| \le \sqrt{M/\mu},
\end{align*}
and similarly for $\|\bQ'\|_{\op}$. 
Finally, we have from the assumption on $\sigst_r$ and Weyl's inequality and \Cref{eq:Sigdiff_cond} that 
\begin{align*}
\sigst_{r} &\ge \sigma_r(\bSigma) - \sigma/4 \ge \sigma_r(\bSigma') - \|\bSigma - \bSigma'\|_{\op} - \frac{\sigma}{8r_0} \ge \sigma_{r}(\bSigma') - \frac{\sigma}{4r_0}\\
\sigst_{r+1} &\le \sigma_{r+1}(\bSigma) + \sigma/4 \le \sigma_{r+1}(\bSigma') + \|\bSigma - \bSigma'\|_{\op} + \frac{\sigma}{8r_0} \le \sigma_{r+1}(\bSigma') + \frac{\sigma}{4r_0}.
\end{align*}
From \Cref{lem:prop_sep_rank}, we have $\sigma_{r}(\bSigma') \ge \sigma$ and $\sigma_{r+1}(\bSigma') \le e \sigma$, so using $r_0 \ge 1$, we have $\sigst_{r} \ge 3\sigma/4$ and $\sigst_{r+1} \le (e+\frac{1}{4})\sigma \le 3\sigma$. Finally, from the definition of $\seprank$ (\Cref{defn:seprank}),
\begin{align*}
(1- \frac{1}{r_0})\sigma_{r}(\bSigma') - \sigma_{r+1}(\bSigma')  \ge 0.
\end{align*}
Using the previous display this implies  
\begin{align*}
(1- \tfrac{1}{r_0})\sigst_r - \sigst_{r+1} \ge - (1- \tfrac{1}{r_0})\frac{\sigma}{4r_0} - \frac{\sigma}{4r_0} \ge - \frac{\sigma}{2r_0},
\end{align*}
so rearranging, and using $\sigma \le 4\sigst_r/3$ as derived above,
\begin{align*}
\sigst_r - \sigst_{r+1} \ge \frac{\sigst_r}{r_0} - \frac{\sigma}{2r_0} \ge \frac{\sigst_r}{r_0} - \frac{2\sigst_r}{3r_0} = \frac{\sigst_r}{3r_0}.
\end{align*}
This completes the proof.
\end{proof}

\newcommand{\etatest}{\eta_{\mathrm{tst}}}
\newcommand{\etatrain}{\eta_{\mathrm{trn}}}
\newcommand{\etacov}{\eta_{\mathrm{cov}}}

\newcommand{\cdtwo}{\cD_{2\otimes 2}}
\section{Supporting Proofs for Error Decomposition}\label{sec:proof_main_results_error_decomp}


In this section, we prove a slightly more specific statement of \Cref{prop:final_error_decomp_simple}, which makes dependencies on the problem parameters explicit. We also state and prove an error decomposition result under more general assumptions which allow for additive slack, as  described below. 

The remainder of the section is structured as follows. 
\begin{itemize}
	\item[(a)] \Cref{sec:main_decomposition_results} states our main results, both under  \Cref{asm:cov,asm:density}, as well as under more general assumptions (\Cref{asm:density_b,asm:B_bounded,asm:cov_b}) which allow for additive slack. 
	\item[(b)] \Cref{sec:proof_err_decomp_a}  sketches the main steps of the proof. The proofs of the constituent lemmas are deferred to subsequent sections. This section focuses on \Cref{prop:final_error_decomp}, and mentions the modifications for \Cref{prop:final_error_decomp_b} at its end in Subsubsection \ref{ssec:err_decomp_additive_slack}.  
	\item[(c)] \Cref{sec:key_com_lemmas} outlines helpful lemmas we refer to as ``change of measure'' lemmas. One key lemma uses the covariance-relation, \Cref{asm:cov}, to relate certain expectation  under $\cD_{i\otimes 2}$ and $\cD_{2 \otimes j}$ to those under $\cD_{i \otimes 1}$ and $\cD_{1 \otimes j}$.  
	\item[(d)]  \Cref{sec:lem:risk_decomp_one,sec:err_two_bound,ssec:proof_of_Dtwo_two_decomp}, prove the various lemmas given in \Cref{sec:proof_err_decomp_a}.
	\item[(e)]  \Cref{proof:generic_risk_bound} derives \Cref{thm:generic_risk_bound} from  \Cref{prop:final_error_decomp_simple} and \Cref{thm:one_block}.
\end{itemize}

\subsection{Main  results}\label{sec:main_decomposition_results}  We now give risk  decomposition results which make  dependencies on problem parameters explicit. 
Our granular guarantee under \Cref{asm:cov,asm:density}  is as follows. 

\begin{propmod}{prop:final_error_decomp_simple}{a}[Final Error Decomposition, Explicit Dependence]\label{prop:final_error_decomp} Suppose \Cref{asm:cov,asm:density} hold. For any $k \le r$ with some fixed integer  $r>0$, and any aligned $k$-proxies $(f,g)$ of the $\R^r$-embeddings   $(\fhat,\ghat)$, denote  $\ErrTermHil_0=\ErrTermHil_0(f,g,k)$ and $\ErrTermHil_1=\ErrTermHil_1(f,g,k)$. Let $\sigma\le \sigma_r(\fhat,\ghat)$ be a lower bound on $\sigma_r(\fhat,\ghat)$ as defined in \Cref{eq:sigi_def}, which satisfies $\sigma^2 \in (0,\tailsf_2(k) + \ErrTermHil_0 + \errtrain]$.   Then, 
\begin{align*}
\Risk(f,g;\Dtest) &\lesssim \kaptest \kapcov^2 \left( (\ErrTermHil_1)^2  + \frac{1}{\sigma^2}\left(\errapx + \ErrTermHil_0 + \kapcov \kaptrain\errtrain\right)^2   \right).
\end{align*}
In particular, suppressing polynomial dependence on $\kaptrain,\kapcov$, we recover
\begin{align*}
\Risk(f,g;\Dtest) \lesssimst   (\ErrTermHil_1)^2  + \frac{1}{\sigma^2}(\tailsf_2(k) + \ErrTermHil_0 + \errtrain)^2.
\end{align*} 
The same bound holds more generally when $\tailsf_2(k) + \ErrTermHil_0 + \errtrain$ is replaced with an upper bound $M$, and when  $\sigma^2$ need only satisfy $\sigma^2  \le M$. Moreover, it also holds that
\begin{align*}
\Risk(f,g;\cdone) \le \kaptrain\Risk(f,g;\Dtrain).
\end{align*}
\end{propmod}

\subsubsection{Error decomposition under additive remainders}
The above error decomposition holds under slightly more general  conditions, which allow for additive \emph{additive remainders} in the multiplicative approximations in \Cref{asm:cov,asm:bal}.

\newcommand{\weight}{\kappa}
\begin{asmmod}{asm:density}{b}[Coverage Decomposition with Additive Slack]\label{asm:density_b} 
There exists $\kaptest,\kaptrain > 0$ and $\etatest,\etatrain \in (0,1]$ such that $\Dtrain$ covers all pairs $\cD_{i \otimes j}$ with $(i,j)=(1,1)$, $(i,j)=(1,2)$, and $(i,j)=(2,1)$, and $\Dtest$ is continuous with respect to the mixture of all pairs $\cD_{i \otimes j}$. Formally, for all $(x,y) \in \xspac \times \yspace$,
\begin{align*}
&\Pr_{\disij}\left[\frac{\rmd \cD_{i \otimes j}(x,y)}{\rmd \Dtrain(x,y)} > \kaptrain\right] \le \etatrain, \quad (i,j) \in \{(1,1),(1,2),(2,1)\} \tag{Train Coverage}\\
&\Pr_{\Dtest}\left[\frac{\rmd \Dtest(x,y)}{\sum_{i,j\in\{1,2\}}  \rmd\cD_{i\otimes j}(x,y)} > \kaptest\right] \le \etatest. \tag{Test Coverage}
\end{align*} 
\end{asmmod}

\begin{asmmod}{asm:cov}{b}[Change of Covariance with Additive Slack]\label{asm:cov_b} There exists a $\kapcov \ge 1$ and $\etacov \ge 0$ such that, for any $v \in \hilspace$,
\begin{align*}
\Exp_{x\sim \cdx{2}}[\hilprod{\fst(x),v}^2] &\le \kapcov\cdot\Exp_{x\sim \cdx{1}}[\hilprod{\fst(x),v}^2] + \etacov\hilnorm{v}^2 \\
\Exp_{y\sim \cdy{2}}[\hilprod{\gst(y),v}^2] &\le \kapcov\cdot\Exp_{y\sim \cdy{1}}[\hilprod{\gst(y),v}^2] + \etacov\hilnorm{v}^2.
\end{align*}
\end{asmmod}

For additive slack, we further require uniform boundedness of the embeddings (rather than just their inner products.) 

\begin{assumption}[Boundedness]\label{asm:B_bounded} There exists an upper bound $B > 0$ such that 
\begin{align*}
\max\left\{\sup_{x\in \xspac} \hilnorm{f(x)}\vee\hilnorm{\fst(x)},\, \sup_{y \in \yspace}\hilnorm{g(y)}\vee \hilnorm{\gst(y)}\right\} \le B.
\end{align*}
\end{assumption}

We now state the general analogue of our error decomposition,  \Cref{prop:final_error_decomp_b},  which allows for additive slack terms. For simplicity, we assume $\etacov,\etatest,\etatrain \le 1$.

\begin{propmod}{prop:final_error_decomp_simple}{b}[Final Error Decomposition with Additive Slack]\label{prop:final_error_decomp_b}  For any $k \le r$ with some fixed integer  $r>0$, and any aligned $k$-proxies $(f,g)$ of the $\R^r$-embeddings   $(\fhat,\ghat)$, denote  $\ErrTermHil_0=\ErrTermHil_0(f,g,k)$ and $\ErrTermHil_1=\ErrTermHil_1(f,g,k)$. Let $\sigma = \sigma_r(\fhat,\ghat)$ (as in \Cref{eq:sigi_def}), and let $\sigma$ satisfy $\sigma^2 \in (0,\tailsf_2(k) + \ErrTermHil_0]$. Then, under \Cref{asm:density_b,asm:cov_b,asm:B_bounded}, we have

\begin{align*}
\Risk(f,g;\Dtest) &\lesssim \kaptest \kapcov^2 \left( (\ErrTermHil_1)^2  + \frac{1}{\sigma^2}\left(\errapx + \ErrTermHil_0 + \kapcov \kaptrain\errtrain\right)^2   \right)\\
&\qquad+ B^4 \poly(\kapcov,\kaptest,\kaptrain)\cdot(\etatest + \etatrain + \etacov).
\end{align*}
\end{propmod}

\subsection{Overview of proof} \label{sec:proof_err_decomp_a} 

	For now, we focus on \Cref{prop:final_error_decomp}, which assumes \Cref{asm:cov,asm:density}. The modification of  \Cref{prop:final_error_decomp_b}, under \Cref{asm:cov_b,asm:density_b,asm:B_bounded} are described at the end. Each of the lemmas in this section is proved under these more general conditions. 

	Fix any embeddings $f:\xspac \to \hilspace$ and $g: \yspace \to \hilspace$. We shall ultimately enforce that $f,g$ are aligned $k$-proxies (\Cref{defn:valid_proxies_cont}) for some $(\fhat,\ghat)$, though this  is only necessary for one step of the proof. We begin by recalling the error terms from \Cref{defn:key_err_terms_simple}, and introducing a few other terms in our analysis.
	
	\begin{defn}[Key Error Terms]\label{defn:err_terms_key} Given functions $f:\xspac \to \hilspace$ and $g:\xspac \to \hilspace$ and $k \in \N$, define
	\begin{align*}
	\ErrTermHil_0(f,g,k) &:= \max\left\{\Exp_{\cD_{1\otimes 1}}\left[\hilprodd{\fstk , \gstk - g}^2\right], \,\Exp_{\cD_{1\otimes 1}}\left[\hilprodd{\fstk - f , \gstk}^2\right]\right\} \tag{weighted error}\\
	\ErrTermHil_1(f,g,k) &:= \max\left\{\Exp_{\cdx{1}}\hilnormm{\fstk - f}^2,\,\Exp_{\cdy{1}}\hilnormm{\gstk - g}^2\right\} \tag{unweighted error}\\
	\ErrTermHil_2(f,g,k) &:= \max\left\{\Exp_{\cdx{2}}\hilnormm{\fstk - f},\,\Exp_{\cdy{2}}\hilnormm{\gstk - g}^2\right\} \tag{$\cdtwo$-recovery error}\\
	\errapx(k) &:= \Risk(\fstk,\gstk;\cD_{1\otimes 1}) \tag{approximation error}.
	\end{align*}
	When it is clear from the context, we will use the shorthand notation $\ErrTermHil_0,\ErrTermHil_1,\ErrTermHil_2$ and $\errapx$, respectively, for convenience. 
	\end{defn}

	Above, $\ErrTermHil_0(f,g,k)$ captures differences $\gst_k -g$ (resp $\fst_k -f$) weighted by $\fst_k$ (resp. $\gst_k$) under $\disone$. $\ErrTermHil_1$ captures the unweighted differences (i.e. in $\hilnormm{\cdot}$) under $\disone$, and $\ErrTermHil_2$ does the same under $\distwo$. Since $\fst$ and $\gst$ have spectral decay, we expect the unweighted errors $\ErrTermHil_1,\ErrTermHil_2$ to be larger than the weighted one  $\ErrTermHil_0$.

	\paragraph{Bounding $\Dtest$-risk with $\cdtwo$-risk.} We begin the proof with a lemma which bounds the risk under $\Dtest$ by the risk under the bottom-right block $\cdtwo$, plus the risk under   $\Dtrain$ (i.e. $\errtrain$). The following is proved in \Cref{sec:lem:risk_decomp_one}.
	\begin{lem}[Error Decomposition on $\Dtest$]\label{lem:risk_decomp_one} Under \Cref{asm:density}, the following holds for any $f:\xspac \to \hilspace$ and $g:\yspace \to \hilspace$: 
	\begin{align*}
\Risk(f,g;\Dtest)  &\le \kaptest \left(\Risk(f,g;\distwo) +3 \kaptrain\Risk(f,g;\Dtrain) \right). 
	\end{align*} 
	\end{lem}

	\paragraph{Bounding the  $\cdtwo$-risk.} 
	The difficulty is now in handling the error on the bottom-right  block $\cdtwo$. Our analysis reveals that the leading order term is precisely the weighted error $\ErrTermHil_0$, with the unweighted errors $\ErrTermHil_1$ and $\ErrTermHil_2$ entering only in a quadratic way (i.e. at most second order) into the error. The following is proved in \Cref{ssec:proof_of_Dtwo_two_decomp}.
	\begin{lem}[Error Decomposition on $\cD_{2\otimes 2}$]\label{lem:decomp_two_two} Under \Cref{asm:density,asm:cov}, for any $f:\Fmaps$, $g:\Gmaps$, and $k \in \N$
	\begin{align*}
	\Risk(f,g;\distwo) 
	&\lesssim  \kapcov^2 (\ErrTermHil_0  +  (\ErrTermHil_1)^2 +\errapx) +  (\ErrTermHil_2)^2 +  \kapcov \kaptrain\errtrain  ,
	\end{align*}
	where above we suppress error term dependence on $f,g,k$. 
	\end{lem}

	\paragraph{Bounding $\ErrTermHil_2$.} We now turn to bounding $\ErrTermHil_2$. This requires making full use of the assumption that $(f,g)$ are aligned $k$-proxies of $(\fhat,\ghat)$.  Going forward, recall from \Cref{eq:sigi_def} and the construction in \Cref{defn:valid_proxies_cont} that
	\begin{align*}
	\sigma_r(\Exp_{\cdxone}[ff^\top]) = \sigma_r(\fhat,\ghat) > 0,
	\end{align*}
	where the positivity is a consequence of the assumption that $(\fhat,\ghat)$ are full-rank. We may now bound $\ErrTermHil_2$. The following is proved  in \Cref{sec:err_two_bound}.

	\begin{lem}[Decomposition of $\ErrTermHil_2$]\label{lem:ErrTwoDecomp}Suppose $(f,g)$ are aligned $k$-proxies of  $(\fhat,\ghat)$.  Then, we have 
	\begin{align*}
	\sigma_r(\fhat,\ghat)\ErrTermHil_2 \lesssim \left(\frac{\errtrain}{\omegaoff} \right) + \kapcov(\ErrTermHil_0 + \errapx(k)) 	\end{align*}
	where above we suppress dependence on $f,g,k$ in all error terms. 
	\end{lem}

\paragraph{Concluding the proof of \Cref{prop:final_error_decomp}.}
	Lastly, we observe that the rank-$k$  approximation error under $\cD_{1 \otimes 1}$ is precisely the tail term  $\tailsf_2(k)$.
	\begin{lem}\label{lem:tailsf_two} We have $\errapx(k) = \Risk(\fstk,\gstk;\cD_{1\otimes 1}) = \tailsf_2(k)$. 
	\end{lem}
	\begin{proof} Using that projection matrices are self-adjoint and idempotent,
	\begin{align*}
	\Risk(\fstk,\gstk;\cD_{1\otimes 1}) &= \Exp_{\cdone}[ (\langle \fstk(x),\gstk(y) \rangle - \langle \fst(x),\gst(y) \rangle)^2 ]\\
	&= \Exp_{\cdone}[ (\langle \projopstk \fst(x),\projopstk \gst(y) \rangle - \langle \fst(x),\gst(y) \rangle)^2 ]\\
	&= \Exp_{\cdone}[ (\langle \projopstk \fst(x),\gst(y) \rangle - \langle \fst(x),\gst(y) \rangle)^2 ]\\
	&= \Exp_{\cdone}[ \langle (\eyeop-\projopstk) \fst(x),\gst(y) \rangle^2 ]\\
	&= \trace[\Sigst (\eyeop-\projopstk) \Sigst (\eyeop-\projopstk)] = \sum_{j > k}\lambda_{j}(\Sigst)^2 := \tailsf_{2}(k),
	\end{align*}
    where in the last line, we use that $\projopstk$ projects onto the top $k$ eigenvalues of $\Sigst$.  
	\end{proof}
	Putting these terms together reveals our final error decomposition result. 
	\begin{proof}[Proof of  \Cref{prop:final_error_decomp}] Let $\sigma \le \sigma_r(\fhat,\ghat)$. 
	We write
	\begin{align*}
	&\Risk(f,g;\Dtest) \\
	&\overset{(i)}{\lesssim} \kaptest \left(\Risk(f,g;\cD_{2\otimes 2}) + \kaptrain\errtrain\right)\\
	&\overset{(ii)}{\lesssim} \kaptest\left(\kapcov^2 (\ErrTermHil_0  + \errapx + (\ErrTermHil_1)^2)+     (\ErrTermHil_2)^2  +     \kapcov\kaptrain \errtrain\right) + \kaptrain\kapdens \errtrain\\
	&\lesssim \kaptest\left(\kapcov^2 (\ErrTermHil_0  + \errapx + (\ErrTermHil_1)^2)+     (\ErrTermHil_2)^2  +     \kapcov\kaptrain \errtrain\right),
	\end{align*}
	where in the last line, we absorb terms using $\kapcov \ge 1$. Continuing the string of inequalities, 
	\begin{align*}
	&\overset{(iii)}{\le} \kaptest\left(\kapcov^2 (\ErrTermHil_0  +  \errapx   + (\ErrTermHil_1)^2 ) + \frac{\kapcov^2}{\sigma^2}(\errapx + \ErrTermHil_0 + \kapcov\kaptrain\errtrain)^2   +  \kapcov\kaptrain \errtrain\right)\\
	&\overset{(iv)}{\lesssim} \kaptest\left(\kapcov^2 (\ErrTermHil_1)^2  + \frac{\kapcov^2}{\sigma^2}\left(\errapx + \ErrTermHil_0 + \kapcov \kaptrain\errtrain\right)^2   \right) \\
	&= \kaptest \kapcov^2 \left( (\ErrTermHil_1)^2  + \frac{1}{\sigma^2}\left(\errapx + \ErrTermHil_0 + \kapcov \kaptrain\errtrain\right)^2   \right), 
	\end{align*}
	where $(i)$ uses \Cref{lem:risk_decomp_one}, $(ii)$ uses \Cref{lem:decomp_two_two}, $(iii)$ invokes \Cref{lem:ErrTwoDecomp}, and $(iv)$ applies the assumption $\sigma^2 \le \errapx + \ErrTermHil_0 +   \errtrain$ and $\kapcov,\kaptrain \ge 1$ to absorb the term 
	\begin{align*}
	\kapcov^2(\ErrTermHil_0  +  \errapx)+    \kapcov\kaptrain \errtrain \le \frac{\kapcov^2}{\sigma^2}\left(\errapx + \ErrTermHil_0 +    \kapcov\kaptrain \errtrain\right)^2,
	\end{align*}
	loosing at most a constant factor of $2$. The last inequality uses $\kapcov \ge 1$.  Note that the simplification also holds when likewise $\errapx + \ErrTermHil_0 +   \errtrain$ by an upper bound $M$, such that $\sigma^2 \le M$. 

	Lastly, the final statement of the proposition, namely the bound $\Risk(f,g;\cdone) \le \kaptrain \Risk(f,g;\Dtrain)$, is a direct consequence of \Cref{asm:density}.
	\end{proof}

\subsubsection{Modifications for additive slack}\label{ssec:err_decomp_additive_slack}

Like their analogues in proving  \Cref{prop:final_error_decomp}, the following lemmas are proved in \Cref{sec:lem:risk_decomp_one,sec:err_two_bound,ssec:proof_of_Dtwo_two_decomp}, respectively.

\begin{lemmod}{lem:risk_decomp_one}{b}[Error Decomposition on $\Dtest$ with Additive Slack]\label{lem:risk_decomp_one_b}Under \Cref{asm:density_b,asm:B_bounded},
\begin{align*}
\Risk(f,g;\Dtest)  &\le \kaptest \left(\Risk(f,g;\distwo) +3 \kaptrain\Risk(f,g;\Dtrain) \right) + 4B^4(\etatest + 3\kaptest \etatrain). 
\end{align*}
\end{lemmod}

\begin{lemmod}{lem:decomp_two_two}{b}[Error Decomposition on $\cD_{2\otimes 2}$ with Additive Slack]\label{lem:decomp_two_two_b}  Under \Cref{asm:density_b,asm:cov_b,asm:B_bounded}, for any $k \in \N$,
\begin{align*}
\Risk(f,g;\distwo) &\lesssim \kapcov^2 (\ErrTermHil_0  +  (\ErrTermHil_1)^2 +\errapx) +  (\ErrTermHil_2)^2 +  \kapcov \kaptrain\errtrain  +   B^4 \kapcov(\etacov +   \etatrain),
\end{align*}
where above we suppress error term dependence on $f,g,k$.
\end{lemmod}

\begin{lemmod}{lem:ErrTwoDecomp}{b}[Decomposition of $\ErrTermHil_2$ with Additive Slack]\label{lem:ErrTwoDecomp_b} Suppose $k \le r$, and $(f,g)$ are aligned $k$-proxies for $(\fhat_r,\ghat_r)$ with $\projopst_k$. Then for $k \le r$,
\begin{align*}
\sigma_r(\fhat,\ghat)\ErrTermHil_2 \lesssim \left(\frac{\errtrain}{\omegaoff} \right) + \kapcov(\ErrTermHil_0 + \errapx(k)) + B^2(\etacov + \etatest), 
\end{align*}
where above we suppress dependence on $f,g,k$ in all error terms.
\end{lemmod}
Deriving \Cref{prop:final_error_decomp_b} from the previous lemmas follows in much the same way as \Cref{prop:final_error_decomp}, and is omitted for brevity.

\subsection{Key change-of-measure lemmas}\label{sec:key_com_lemmas}
	We begin by establishing some important change-of-measure results. 
	
	\begin{lem}[Change of Covariance]\label{lem:change_cov} Under either the boundedness assumption \Cref{asm:B_bounded}, or assumping $\etacov = 0$, the following holds   for any $i,j \in \{1,2\}$ and any $(\tilde{f},\tilde{g})$, under \Cref{asm:cov_b}, 
	\begin{itemize}
	\item $\Exp_{\cD_{i \otimes 2}}[\langle \tilde{f},\gstk\rangle^2] \le \kapcov \Exp_{\cD_{i \otimes 1}}[\langle \tilde{f},\gstk\rangle^2] +  B^2 \etacov$
	\item $\Exp_{\cD_{2 \otimes j}}[\langle \fstk, \tilde{g}\rangle^2] \le \kapcov  \Exp_{\cD_{1 \otimes j}}[\langle \fstk,\tilde{g}\rangle^2] +  B^2 \etacov$. 
	\end{itemize}
	The same holds if $\gstk $ (resp. $\fstk$) are replaced by $\gstgk := \gst - \gstk$ (resp. $\fstgk := \fst-\fstk$) or $\gst$ (resp. $\fst$), and under \Cref{asm:cov}, the above holds with $\etacov = 0$. 
	\end{lem}
	\begin{proof} Since \Cref{asm:cov} is stronger than  \Cref{asm:cov_b}, we focus on the proofs under \Cref{asm:cov_b}. Let's begin by proving the first item under \Cref{asm:cov_b}; the extension to the second item is similar. We have
	\begin{align*}
	\Exp_{\cD_{i \otimes 2}}[\langle \tilde{f},\gstk\rangle^2] &= \Exp_{\cdxi}[\Exp_{\cdy{2}}[\langle \tilde{f},\gstk\rangle^2]] \tag{Fubini}\\
	&= \Exp_{\cdxi}[\Exp_{\cdy{2}}[\langle \tilde{f},\projopst_k\gst\rangle^2]] \\
	&\le \Exp_{\cdxi}[\kapcov\Exp_{\cdy{1}}[\langle \projopst_k\tilde{f},\gst\rangle^2] + \etacov \hilnormm{\projopst_k\tilde{f}}^2 ] \tag{\Cref{asm:cov_b}}\\
	&\le \Exp_{\cdxi}[\kapcov\Exp_{\cdy{1}}[\langle \projopst_k\tilde{f},\gst\rangle^2] + \etacov \hilnormm{\tilde{f}}^2 ] \tag{$\projopst_k$ is a projection}\\
	&\le \Exp_{\cdxi}[\kapcov\Exp_{\cdy{1}}[\langle \projopst_k\tilde{f},\gst\rangle^2] ] + {B}^2\etacov  \tag{\Cref{asm:B_bounded}}\\
	&= \kapcov\Exp_{\cdxi \otimes \cdy{1}}[\langle \projopst_k\tilde{f},\gst\rangle^2] + {B}^2\etacov  \tag{Fubini}\\
	&= \kapcov\Exp_{\cdxi \otimes \cdy{1}}[\langle \tilde{f},\gstk\rangle^2] + {B}^2\etacov.
	\end{align*}
	As mentioned, the second item is similar. To derive the similar  bounds for $\gst-\gstk$, we use that $\gst - \gstk = (\eyeop - \projopst_k)\gst$,  and $\eyeop - \projopst_k$ is also a projection operator; the bound for $\fst - \fstk$ can be derived  similarly. Finally, the bounds for $\fst,\gst$ are slightly simpler to establish, because we need not commute the projection operator.  
	\end{proof}

	\begin{lem}[Change of Risk]\label{lem:change_of_risk} The following bounds hold:
	\begin{itemize} 
	\item The risk on the ``off-diagonal'' product distribution is bounded by 
	\begin{align*}
	\Risk(\fst_k,\gst_k; \cD_{1 \otimes 2}) \vee\Risk(\fst_k,\gst_k; \cD_{2 \otimes 1}) &\le \kapcov \errapx(k) + \etacov B^2.
	\end{align*}
	\item The risk on the ``bottom-right'' product distribution is bounded by
	\begin{align*}
	\Risk(\fst_k,\gst_k; \cD_{2 \otimes 2})  &\le \kapcov^2 \errapx(k) + 2\kapcov \etacov B^2.
	\end{align*}
\end{itemize}
	\end{lem}
	\begin{proof}
	Introduce the shorthand $\fstgk := \fst - \fstk$ and $\gstgk := \gst- \gstk$. Note that 
	\begin{align*}
	\fstgk = (\bsfeye - \projopstk) \fst, \quad \gstgk = (\bsfeye - \projopstk) \gst,
	\end{align*}
	and hence $\fstgk,\gstgk$ are $B$-bounded under \Cref{asm:B_bounded}.

	Observe that since $\projopstk$ is an orthogonal projection, so is $\bsfeye - \projopstk$. Since orthogonal projections are self-adjoint and idempotent,
	\begin{equation} \label{eq:k_trunc_err_comp}
	\begin{aligned}
	\hst - \hilprodd{\fst_k,\gst_k} &= \hilprodd{\fst,\gst} - \hilprodd{\projopstk\fst,\gst} = \hilprodd{(\bsfeye - \projopstk)\fst,\gst} \\
	&=  \hilprodd{(\bsfeye - \projopstk)^{\adj}(\bsfeye - \projopstk)\fst,\gst} =  \hilprodd{(\bsfeye - \projopstk)\fst,(\bsfeye - \projopstk)\gst} = \hilprodd{\fstgk, \gstgk}.
	\end{aligned}
	\end{equation} 
	Thus, by \Cref{lem:change_cov} and the fact that $\fstgk$ is $B$-bounded,  we have 
	\begin{align*}
	\Risk(\fst_k,\gst_k; \cD_{1 \otimes 2}) = \Exp_{\cD_{1\otimes 2}}(\hst - \hilprodd{\fst_k,\gst_k})^2 &= \Exp_{\cD_{1\otimes 2}} \hilprodd{\fstgk, \gstgk}^2 \le \kapcov \Exp_{\cD_{1\otimes 1}} \hilprodd{\fstgk, \gstgk}^2 + \etacov B^2. 
	\end{align*}
	
	Similarly, $\Risk(\fst_k,\gst_k; \cD_{2 \otimes 1}) \le \kapcov \Exp_{\cD_{1\otimes 1}} \hilprodd{\fstgk, \gstgk}^2 + \etacov B^2$. Finally, by two applications of \Cref{lem:change_cov}, we have 
	\begin{align*}
	\Risk(\fst_k,\gst_k; \cD_{2 \otimes 2}) = \Exp_{\cD_{2\otimes 2}}(\hst - \hilprodd{\fst_k,\gst_k})^2 &= \Exp_{\cD_{2\otimes 2}} \hilprodd{\fstgk, \gstgk}^2\\
	&\le \kapcov \Exp_{\cD_{1\otimes 2}} \hilprodd{\fstgk, \gstgk}^2 + \etacov B^2\\
	&\le \kapcov^2 \Exp_{\cD_{1\otimes 1}} \hilprodd{\fstgk, \gstgk}^2 + (\kapcov \etacov + \etacov) B^2\\
	&\le \kapcov^2 \Exp_{\cD_{1\otimes 1}} \hilprodd{\fstgk, \gstgk}^2 + 2\kapcov \etacov B^2 \\
	&= \kapcov^2 \Exp_{\cD_{1\otimes 1}} (\hst - \hilprodd{\fst_k,\gst_k})^2 + 2\kapcov \etacov B^2  \tag{\Cref{eq:k_trunc_err_comp}}\\
	&= \kapcov^2 \Risk(\fstk,\gstk;\cdone) + 2\kapcov \etacov B^2. 
	\end{align*}
	This completes the proof. 
	\end{proof}

	\begin{lem}\label{claim:ij_to_train} Suppose the  boundedness assumption, i.e.  \Cref{asm:B_bounded} holds for some $(f,g)$. Then, for any $(i,j)\in \{(1,1),(1,2),(2,1)\}$,
	\begin{align*}
	\Risk(f,g;\disij) \le  4B^4\epstrain + \kaptrain\Risk(f,g;\Dtrain). 
	\end{align*} 
	The same also holds without \Cref{asm:B_bounded}, with  $\epstrain = 0$ (ignoring the $4B^2\epstrain$ term). 
	\end{lem}
	\begin{proof}
	Define the event
	\newcommand{\Etrainij}{\cE_{\mathrm{train},i\otimes j}}
	\begin{align*}
	\Etrainij &:= \left\{\frac{\rmd \cD_{i \otimes j}(x,y)}{\rmd \Dtrain(x,y)} \le \kaptrain\right\}.
	\end{align*}
	We first consider the case where \Cref{asm:B_bounded} holds. To this end, 
	we consider any function $F:\xspac \times \yspace \to [0,M]$. We then have
	\begin{align*}
	\Exp_{\disij}[F(x,y)] &\le  M\Pr_{\disij}[\neg \Etrainij] +\Exp_{\disij}[F(x,y)\I\{\Etrainij\}]\\
	&=  M\Pr_{\disij}[\neg \Etrainij] + \Exp_{\Dtrain}[F(x,y)\I\{\Etrainij\} \cdot \frac{\rmd \disij(x,y)}{\rmd\Dtrain(x,y)}]\\
	&\le  M\Pr_{\disij}[\neg \Etrainij] + \Exp_{\Dtrain}[F(x,y)\frac{1}{\kaptrain}]\\
	&= \kaptrain M\Pr_{\disij}[\neg \Etrainij] +\Exp_{\Dtrain}[F(x,y)] \le M\epstrain + \kaptrain\Exp_{\Dtrain}[F(x,y)].
	\end{align*}

	The result follows by setting $F(x,y) = (\hilprodd{f(x),g(y)} - \hst(x,y))^2$, which lies in $[0,4B^4]$ by \Cref{lem:B_bounded}, stated just below. 
	
	For the case \Cref{asm:B_bounded} does not hold but with $\epstrain=0$, the first term on the right-hand side above does not appear, which completes the proof. 
	\end{proof}

	\begin{lem}\label{lem:B_bounded} Under \Cref{asm:B_bounded}, given any $B$-bounded functions $f,g$, the function $F(x,y) = (\hilprodd{f(x),g(y)} - \hst(x,y))^2$ satisfies $0 \le F(x,y) \le 4B^4$.   
	\end{lem}
	\begin{proof} Since $f,g$ are $B$-bounded $|\hilprodd{f,g}| \le \hilnormm{f}\hilnormm{g} \le B^2$.  Similarly, $|\hst| \le \hilnormm{\fst}\hilnormm{\gst} \le B^2$ The bound follows.
	\end{proof}

	\subsection{Proof of \Cref{lem:risk_decomp_one,lem:risk_decomp_one_b}}\label{sec:lem:risk_decomp_one}

		\newcommand{\Etest}{\cE_{\mathrm{test}}}

		\begin{proof} We prove the more general statement under \Cref{asm:density_b}, and explain the modification to \Cref{asm:density} afterward. Define the event 
		\begin{align*}
		\Etest &:=\left\{\frac{\rmd \Dtest(x,y)}{\sum_{i,j\in\{1,2\}} \rmd\cD_{i\otimes j}(x,y)} \le \kaptest\right\}.
		\end{align*}
		Then, for any bounded, nonnegative function $F(x,y): \xspac \times \yspace \to [0,M]$, we have
		\begin{align*}
		&\Exp_{\Dtest}[F(x,y)] \\
		&\le M\Pr_{\Dtest}[\neg \Etest] +  \Exp_{\Dtest}[\I\{\Etest\}F(x,y)] \\
		&= M\Pr_{\Dtest}[\neg \Etest] +  \int_{(x,y)} \left(F(x,y) \cdot \Bigg(\sum_{i,j\in\{1,2\}}  \rmd\cD_{i\otimes j}(x,y)\Bigg) \cdot \frac{\rmd \Dtest(x,y)}{\sum_{i,j\in\{1,2\}}  \rmd\cD_{i\otimes j}(x,y)}  \I\{\Etest\}\right) \\
		&\le M\Pr_{\Dtest}[\neg \Etest] +  \kaptest\int_{(x,y)} F(x,y) \cdot \Bigg(\sum_{i,j\in\{1,2\}}  \rmd\cD_{i\otimes j}(x,y)\Bigg) \\
		&= M\Pr_{\Dtest}[\neg \Etest] + \kaptest \sum_{i,j = 1}^2 \Exp_{ \disij}[F(x,y)]\\
		&\le M\etatest + \kaptest \sum_{i,j = 1}^2 \Exp_{ \disij}[F(x,y)].
		\end{align*}
		Taking $F(x,y) = (\hilprodd{f(x),g(y)} - \hst(x,y))$, which takes values in $[0,4B^4]$ by \Cref{lem:B_bounded}, we find
		\begin{align*}
		\Risk(f,g;\Dtest) \le  4B^4\etatest + \kaptest \sum_{i,j = 1}^2 \Exp_{ \disij}[F(x,y)].
		\end{align*}
	 	By \Cref{claim:ij_to_train}, we bound 
	 	\begin{align*}
	 	\sum_{i,j \ne (2,2)} \Risk(f,g;\disij) &\le   12B^4 \etatrain   + 3 \kaptrain\Risk(f,g;\Dtrain). 
	 	\end{align*}
	 	Therefore, 
		\begin{align*}
		\Risk(f,g;\Dtest) &\le \kaptest \left(\Risk(f,g;\distwo) +3 \kaptrain\Risk(f,g;\Dtrain) \right) + 4B^4(\etatest + 3\kaptest \etatrain).
		\end{align*}
		The bound follows.  To obtain the simpler statement with \Cref{asm:density},   under which we can take $\etatest = \etatrain = 0$, and complete the proof.  
		\end{proof}

	\subsection{Proof of \Cref{lem:decomp_two_two,lem:decomp_two_two_b}}\label{ssec:proof_of_Dtwo_two_decomp}
		We begin with an elementary algebraic lemma which helps us expand the risk  $\Risk(f,g;\cdtwo)$.
		\begin{lem}\label{lem:dtwo_expansion} For any $\hst:\xspac \times \yspace\to \R$,  $f_1,f_2:\xspac \to \hilspace$, and  $g_1,g_2:\yspace \to \hilspace$, we have
		\begin{align*}
		&(\hilprodd{f_1,g_1} - \hst)^2 \le 2(\hilprodd{f_2,g_2} - \hst)^2 + 6\hilprodd{ f_1 - f_2, g_2}^2 + 6\hilprodd{ f_2, g_1 - g_2}^2 +  6\hilnormm{f_1 - f_2}^2 \hilnormm{g_1 - g_2}^2. 
		\end{align*}
		\end{lem}
		\begin{proof}[Proof of \Cref{lem:dtwo_expansion}] Set $h_1 = \langle f_1, g_1 \rangle$ and $h_2 = \langle f_2,g_2 \rangle$. Then, 
		\begin{align*}
		(h_1 - \hst)^2 - (h_2 - \hst)^2 &= (h_1 - \hst + h_2 - \hst)(h_1 - h_2)\\
		&= (h_1 - h_2)^2 + 2(h_2 - \hst)(h_1 - h_2)\\
		&\le 2(h_1 - h_2)^2 + (h_2 - \hst)^2. 
		\end{align*}
		Hence, we have $(h_1 - \hst)^2 \le 2(h_1 - h_2)^2 + 2(h_2 - \hst)^2.$ To conclude, we bound 
		\begin{align*}
		(h_1 - h_2)^2 &= (\hilprodd{ f_1, g_1} - \hilprodd{f_2, g_2})^2\\
		&= (\hilprodd{ f_1 - f_2, g_2} + \hilprodd{f_2, g_1 - g_2} +  \hilprodd{ f_1 - f_2, g_1 - g_2})^2\\
		&\le 3\hilprodd{ f_1 - f_2, g_2}^2 + 3\hilprodd{ f_2, g_1 - g_2}^2 +  3\hilprodd{ f_1 - f_2, g_1 - g_2}^2\\
		&\le 3\hilprodd{ f_1 - f_2, g_2}^2 + 3\hilprodd{ f_2, g_1 - g_2}^2 +  3\hilnormm{f_1 - f_2}^2 \hilnormm{g_1 - g_2}^2.
		\end{align*}
		Combining the two displays completes the proof. 
		\end{proof}

		\paragraph{Step 1. Change of covariance under $\distwo$.} Taking $f_1 = f$, $g_1 = g$, $f_2 = \fstk$ and $g_2 = \gstk$, \Cref{lem:dtwo_expansion} implies 
		\begin{align*}
		&\Exp_{\cD_{2\otimes 2}}[(\hilprodd{f,g}-\hst)^2] - 2\Exp_{\cD_{2\otimes 2}}[(\hilprodd{\fstk,\gstk}-\hst)^2]    \\
		&\le 6\left(\Exp_{\cD_{2\otimes 2}}\hilprodd{ f - \fstk, \gstk}^2 +  \Exp_{\cD_{2\otimes 2}}\hilprodd{\fstk, g - \gstk}^2 +  \Exp_{\cD_{2\otimes 2}}[\hilnormm{f - \fstk}^2 \hilnormm{g - \gstk}^2]\right)\\
		&\overset{(i)}{\le} 6\kapcov\left(\Exp_{\cD_{2\otimes 1}}\hilprodd{ f - \fstk, \gstk}^2 +   \Exp_{\cD_{1\otimes 2}}\hilprodd{\fstk, g - \gstk}^2\right) +  6\Exp_{\cD_{2\otimes 2}}\hilnormm{f - \gstk}^2 \hilnormm{g - \gstk}^2 + 48 B^2\etacov \\
		&\overset{(ii)}{=} 6\kapcov\left(\Exp_{\cD_{2\otimes 1}}\hilprodd{ f - \fstk, \gstk}^2 +   \Exp_{\cD_{1\otimes 2}}\hilprodd{\fstk, g - \gstk}^2\right) +  6\Exp_{\cdx{2}}\hilnormm{f - \fstk}^2 \cdot \Exp_{\cdy{2}}\hilnormm{f - \fstk}^2 + 48 B^4\etacov\\
		&\overset{(iii)}{\le} 6\kapcov\left(\Exp_{\cD_{2\otimes 1}}\hilprodd{ f - \fstk, \gstk}^2 +   \Exp_{\cD_{1\otimes 2}}\hilprodd{\fstk, g - \gstk}^2\right) +  6(\ErrTermHil_2)^2 + 48 B^4\etacov, \numberthis \label{eq:dtwo_exp_last_line}
		\end{align*}
		where in $(i)$ we apply \Cref{lem:change_cov} to the terms $\Exp_{\cD_{2\otimes 2}}\hilprodd{ f - \fstk, \gstk}^2$ and $\Exp_{\cD_{2\otimes 2}}\hilprodd{\fstk, g - \gstk}^2$, for with $\tilde{B} = 2B$, in $(ii)$ we use that $\distwo = \cdx{2}\otimes \cdy{2}$ is a product measure, and in $(iii)$ we recall the definition of $\ErrTermHil_2 = \ErrTermHil_2(f,g,k)$.

		\paragraph{Step 2. Expansion of $\cD_{1 \otimes 2}$ and $\cD_{2 \otimes 1}$.} Next, we expand the first two terms in \Cref{eq:dtwo_exp_last_line}. First, $\hilprodd{ f - \fstk, \gstk} = \hilprodd{ f , \gstk} -\hstk = \hilprodd{f,g} - \hstk + \hilprodd{f,\gstk - g} = \hilprodd{f,g} - \hstk + \hilprodd{f - \fstk,\gstk -  g} + \hilprodd{\fstk , \gstk - g}$. Hence, 
		\begin{align*}
		&\Exp_{\cD_{2\otimes 1}}\hilprodd{ f - \fstk, \gstk}^2  \\
		&\le 3\Exp_{\cD_{2\otimes 1}}[(\hilprodd{f,g} - \hstk)^2] + 3\Exp_{\cD_{2\otimes 1}}\hilprodd{\fstk , \gstk - g}^2 +  3\Exp_{\cD_{2\otimes 1}}\hilprodd{f-\fstk, \gstk - g}^2\\
		&\overset{(i)}{\le} 3\Exp_{\cD_{2\otimes 1}}[(\hilprodd{f,g} - \hstk)^2] + 3\kapcov \Exp_{\cD_{1\otimes 1}}\hilprodd{\fstk , \gstk - g}^2 +  3\Exp_{\cD_{2\otimes 1}}\hilnormm{\fstk - f}^2\hilnormm{\gstk - g}^2 + 12B^4 \etacov,
		\end{align*}
		where in $(i)$ we again apply \Cref{lem:change_cov}. 
		We can further expand
		\begin{align*}
		 \Exp_{\cD_{2\otimes 1}}\hilnormm{\fstk - f}^2\hilnormm{\gstk - g}^2 &= \Exp_{\cdx{2}}\hilnormm{\fstk - f}^2\cdot \Exp_{\cdy{1}}\hilnormm{\gstk - g}^2\\
		 &\le  \frac{1}{2\kapcov}(\Exp_{\cdx{2}}\hilnormm{\fstk - f}^2)^2 + \frac{\kapcov}{2}(\Exp_{\cdy{1}}\hilnormm{\gstk - g}^2)^2\\
		 &\le \frac{1}{2\kapcov}(\ErrTermHil_2)^2 + \frac{\kapcov}{2}(\ErrTermHil_1)^2,
		\end{align*}
		where again, we recall the definition of $\ErrTermHil_2$ and $\ErrTermHil_1$ in \Cref{defn:err_terms_key}.  In sum, we find 
		\begin{align*}
		&\Exp_{\cD_{2\otimes 1}}\hilprodd{ f - \fstk, \gstk}^2  \\
		&\le 3\Exp_{\cD_{2\otimes 1}}[(\hilprodd{f,g} - \hstk)^2] + 3\kapcov \Big(\Exp_{\cD_{1\otimes 1}}\hilprodd{\fstk , \gstk - g}^2 + \frac{1}{2}(\ErrTermHil_1)^2\Big)+  \frac{3}{2\kapcov}(\ErrTermHil_2)^2 + 12B^4 \etacov. 
		\end{align*}
		A similar analysis bounds
		\begin{align*}
		&\Exp_{\cD_{1\otimes 2}}\hilprodd{ \fstk, g - \gstk}^2  \\
		&\overset{(i)}{\le} 3\Exp_{\cD_{1\otimes 2}}[(\hilprodd{f,g} - \hstk)^2] + 3\kapcov (\Exp_{\cD_{1\otimes 1}}\hilprodd{f-\fstk , \gstk }^2 + \frac{1}{2}(\ErrTermHil_1)^2)+  \frac{3}{2\kapcov}(\ErrTermHil_2)^2 + 12B^4 \etacov.
		\end{align*}
		Thus, defining
		\begin{align*}
		\ErrTermOff = \Exp_{\cD_{1\otimes 2}}[(\hilprodd{f,g} - \hstk)^2] + \Exp_{\cD_{2\otimes 1}}[(\hilprodd{f,g} - \hstk)^2],
		\end{align*}
		we have
		\begin{align*}
		&\Exp_{\cD_{2\otimes 1}}\hilprodd{ f - \fstk, \gstk}^2  + \Exp_{\cD_{1\otimes 2}}\hilprodd{ \fstk, g - \gstk}^2  \\
		&\le 3\Exp_{\cD_{1\otimes 2}}[(\hilprodd{f,g} - \hstk)^2] + 3\Exp_{\cD_{2\otimes 1}}[(\hilprodd{f,g} - \hstk)^2]\\
		&\quad + 3\kapcov (\Exp_{\cD_{1\otimes 1}}\hilprodd{f-\fstk , \gstk }^2 + \Exp_{\cD_{1\otimes 1}}\hilprodd{\fstk ,g- \gstk }^2  + (\ErrTermHil_1)^2)+  \frac{3}{\kapcov}(\ErrTermHil_2)^2 + 24B^4 \etacov\\
		&= 3\ErrTermOff + 3\kapcov (2\ErrTermHil_0  + (\ErrTermHil_1)^2)+  \frac{3}{\kapcov}(\ErrTermHil_2)^2 + 24B^4 \etacov. \numberthis \label{eq:off_term_expansion}
		\end{align*}

		\paragraph{Step 3. Intermediate simplification.} 
		Combining \Cref{eq:off_term_expansion,eq:dtwo_exp_last_line}, we find
		\begin{align*}
		&\Exp_{\cD_{2\otimes 2}}[(\hilprodd{f,g}-\hst)^2] - 2\Exp_{\cD_{2\otimes 2}}[(\hilprodd{\fstk,\gstk}-\hst)^2] \\
		&\le 6\kapcov\left(\Exp_{\cD_{2\otimes 1}}\hilprodd{ f - \fstk, \gstk}^2 +   \Exp_{\cD_{1\otimes 2}}\hilprodd{\fstk, g - \gstk}^2\right) +  6(\ErrTermHil_2)^2 + 48 B^4\etacov\\
		&\le 18\kapcov \ErrTermOff + 18\kapcov^2 (2\ErrTermHil_0  + (\ErrTermHil_1)^2)+  24(\ErrTermHil_2)^2 + (144\kapcov + 48) B^4 \etacov. 
		\end{align*}
		That is, by rearranging
		\begin{align*}
		\Exp_{\cD_{2\otimes 2}}[(\hilprodd{f,g}-\hst)^2]   
		&\le   18\kapcov^2 (2\ErrTermHil_0  + (\ErrTermHil_1)^2)+  24(\ErrTermHil_2)^2 + (144\kapcov + 48) B^4 \etacov\\
		&\qquad + 18\kapcov \ErrTermOff +  2\Exp_{\cD_{2\otimes 2}}[(\hilprodd{\fstk,\gstk}-\hst)^2]. \numberthis \label{eq:inter_last_line}
		\end{align*}

		\paragraph{Step 4. Concluding the proof.} To conclude, we upper bound \Cref{eq:inter_last_line}. We begin by noting that, by \Cref{lem:change_of_risk},
		\begin{align*}
		2\Exp_{\cD_{2\otimes 2}}[(\hilprodd{\fstk,\gstk}-\hst)^2] = 2\Risk(\fstk,\gstk;\cD_{2\otimes 2}) \le  2\kapcov^2\Risk(\fstk,\gstk;\cD_{1\otimes 1}) + 4\kapcov \etacov B^4. 
		\end{align*}
		
		Similarly, again by \Cref{lem:change_of_risk}
		\begin{align*}
		\ErrTermOff &:= \Exp_{\cD_{1\otimes 2}}[(\hilprodd{f,g} - \hstk)^2]+ \Exp_{\cD_{2\otimes 1}}[(\hilprodd{f,g} - \hstk)^2]\\
		&\le 2\Exp_{\cD_{1\otimes 2}}[(\hilprodd{f,g} - \hst)^2] + 2\Exp_{\cD_{2\otimes 1}}[(\hilprodd{f,g} - \hst)^2]\\
		&\qquad + 2\underbrace{\Exp_{\cD_{1\otimes 2}}[(\hstk - \hst)^2]}_{=\Risk(\fstk,\gstk;\cD_{1\otimes 2})} + 2\underbrace{\Exp_{\cD_{2\otimes 1}}[(\hstk - \hst)^2]}_{=\Risk(\fstk,\gstk;\cD_{2\otimes 1})}\\
		&\le 2\Exp_{\cD_{1\otimes 2}}[(\hilprodd{f,g} - \hst)^2] + 2\Exp_{\cD_{2\otimes 1}}[(\hilprodd{f,g} - \hst)^2]\\
		&\qquad + 4\kapcov\Risk(\fstk,\gstk;\cD_{1\otimes 1}) + 4 \etacov B^4.
		\end{align*}
		Hence, 
		\begin{align*}
		&18\kapcov \ErrTermOff +  2\Exp_{\cD_{2\otimes 2}}[(\hilprodd{\fstk,\gstk}-\hst)^2]\\
		&\le (4\cdot 18 + 2) \kapcov^2\Risk(\fstk,\gstk;\cD_{1\otimes 1}) + (4\cdot 18 + 4) \kapcov\etacov B^5\\
		&\qquad + (2\cdot 18) \kapcov \left(\Exp_{\cD_{1\otimes 2}}[(\hilprodd{f,g} - \hst)^2] + \Exp_{\cD_{2\otimes 1}}[(\hilprodd{f,g} - \hst)^2]\right).\\
		&=74\kapcov^2\Risk(\fstk,\gstk;\cD_{1\otimes 1}) + 76  \kapcov\etacov B^5 + 36 \kapcov \left(\Risk(f,g;\cD_{1\otimes 2}) + \Risk(f,g;\cD_{2\otimes 1})\right). 
		\end{align*}

		By \Cref{claim:ij_to_train}, and using that $0 \le (\hilprodd{f,g} - \hst)^2 \le 4B^4$ 
		\begin{align*}
		\left(\Risk(f,g;\cD_{1\otimes 2}) + \Risk(f,g;\cD_{2\otimes 1})\right) &\le 2\kaptrain\Risk(f,g;\Dtrain) + 8B^4 \etatrain,
		\end{align*}
		Thus, 
		\begin{align*}
		&18\kapcov \ErrTermOff +  2\Exp_{\cD_{2\otimes 2}}[(\hilprodd{\fstk,\gstk}-\hst)^2]\\
		&\le  74\kapcov^2\Risk(\fstk,\gstk;\cD_{1\otimes 1}) + 72 \kapcov \kaptrain \Risk(f,g;\Dtrain) +  (8\cdot 36)\etatrain \kapcov B^4  +  76  \kapcov\etacov B^4.  
		\end{align*}
		In sum 
		\begin{align*}
		&\Exp_{\cD_{2\otimes 2}}[(\hilprodd{f,g}-\hst)^2]   \\ 
		&\le   18\kapcov^2 (2\ErrTermHil_0  + (\ErrTermHil_1)^2)+  24(\ErrTermHil_2)^2 + (144\kapcov + 48) B^4 \etacov\\
		&\qquad + 18\kapcov \ErrTermOff +  2\Exp_{\cD_{2\otimes 2}}[(\hilprodd{\fstk,\gstk}-\hst)^2] \\
		&\le   18\kapcov^2 (2\ErrTermHil_0  + (\ErrTermHil_1)^2)+  24(\ErrTermHil_2)^2 + 72 \kapcov\kaptrain \Risk(f,g;\Dtrain)  \\
		&\qquad+ 74\kapcov^2\Risk(\fstk,\gstk;\cD_{1\otimes 1}) + (144\kapcov + 48) B^4 \etacov +   288\etatrain \kapcov B^4  +  76  \kapcov\etacov B^4 \\
		&\le   18\kapcov^2 (2\ErrTermHil_0  +  (\ErrTermHil_1)^2)+  24(\ErrTermHil_2)^2 + 72 \kapcov \kaptrain\underbrace{\Risk(f,g;\Dtrain)}_{=\errtrain}   \\
		&\qquad+ 74\kapcov^2\underbrace{\Risk(\fstk,\gstk;\cD_{1\otimes 1})}_{=\errapx}+ 268 \kapcov B^4 \etacov +   288\etatrain \kapcov B^4, 
		\end{align*}
		where in the last line, we used $\kapcov \ge 1$ and $144 + 48 + 76 = 268$. Dropping constants and simplifying,
		\begin{align*}
		\Risk(f,g;\cdtwo) &= \Exp_{\cD_{2\otimes 2}}[(\hilprodd{f,g}-\hst)^2] \\
		&\lesssim \kapcov^2 (\ErrTermHil_0  +  (\ErrTermHil_1)^2 +\errapx) +  (\ErrTermHil_2)^2 +  \kapcov \kaptrain\errtrain  +   B^4 \kapcov(\etacov +   \etatrain).  
		\end{align*}
		 The proof for \Cref{lem:decomp_two_two} follows by setting $\etatrain=\etacov=0$. \hfill $\blacksquare$

	\subsection{Proof of \Cref{lem:ErrTwoDecomp,lem:ErrTwoDecomp_b}}\label{sec:err_two_bound}
	Recall the definitions 
		\begin{align*}
		\ErrTermHil_0 &:= \max\left\{\Exp_{\cD_{1\otimes 1}}\left[\hilprodd{\fstk , \gstk - g}^2, \,\Exp_{\cD_{1\otimes 1}}\hilprodd{\fstk , \gstk - g}^2\right]\right\}\\
		\ErrTermHil_1 &:= \max\left\{\Exp_{\cdx{1}}\hilnormm{\fstk - f}^2,\,\Exp_{\cdy{1}}\hilnormm{\gstk - g}^2\right\}\\
		\ErrTermHil_2 &:= \max\left\{\Exp_{\cdx{2}}\hilnormm{\fstk - f}^2,\,\Exp_{\cdy{2}}\hilnormm{\gstk - g}^2\right\}\\
		\errtrain(f,g) &:= \Risk(f,g;\Dtrain)\\
		\errapx(k) &:= \Risk(\fstk,\gstk;\cD_{1\otimes 1})
		\end{align*}
		where the dependence of $f,g,k$ is suppressed in all $\ErrTermHil_{(\cdot)}$ terms.  Our aim is to bound $\ErrTermHil_2$. We focus on bounding $\Exp_{\cdy{2}}\hilnormm{\gstk - g}^2$, for the bound on $\Exp_{\cdx{2}}\hilnormm{\fstk - f}^2$ is analogous. 

		Further, let us recall what it measn for $(f,g)$ to be aligned $k$-proxies. This means that
		(a) $f = (\iota_r \circ \bT^{-1})\fhat$, $g = (\iota_r \circ \bT) \ghat$, where $\iota_r :\R^r \to \hilspace$ is an isometric inclusion, and $\bT$ is the balancing operator of \Cref{lem:balancing_simple}, and (b) for $\projopstk$ projection onto the top $k$-eigenvectors of $\Sigst$, we have 
	\begin{align}\label{eq:alignment}
	\range(\projopstk) \subseteq  \range(\Exp_{\cdxone}[ff^\top]).
	\end{align}
	In particular, let $\scrV := \range(\Exp_{\cdxone}[ff^\top])$. Since $\fhat,\ghat$ are full-rank, $\scrV = \range(\iota_r) = \range(\Exp_{\cdyone}[gg^\top])$. Moreover, $\range(\Exp_{\cdyone}[\gstk(y)\gstk(y)^\top]) = \range(\projopstk) \subseteq  \scrV_r $. Hence, 
	By \Cref{lem:in_range_as} , $g(x)$ $\gstk(y) \in \scrV$ almost surely, and thus, $\gstk(y) - g(y) \in \scrV$ with probability one. In addition, since $\scrV = \range(\iota_r)$ has dimension $r$, it follows that for any $\bmsf{v} \in \scrV$, and since $\sigma_r(\Exp_{\cdxone}[ff^\top]) = \sigma_r(\fhat,\ghat)$ in view of the construction in \Cref{defn:valid_proxies_cont},
	\begin{align*}
	\bmsf{v}^\top \Exp_{\cdxone}[ff^\top] \bmsf{v} \ge \|\bmsf{v}\|^2 \cdot\sigma_r(\fhat,\ghat).
	\end{align*}
	Therefore, 
		\begin{align*}
		\sigma_r(\fhat,\ghat)\Exp_{\cdy{2}}\hilnormm{\gstk - g}^2 &\le \Exp_{\cdy{2}}\left[\frac{1}{\sigma_r(\fhat,\ghat)}\Exp_{\cdx{1}}\hilprodd{f,\gstk - g}^2\right]
		\\
		&= \frac{1}{\sigma_r(\fhat,\ghat)}\Exp_{\cD_{1 \otimes 2}}\left[\hilprodd{f,g - \gstk}^2\right]. 
		\end{align*}
		In other words, we bound $\Exp_{\cdy{2}}\hilnormm{\gstk - g}^2$ by relating an expectation involving  $\cdx{1}$. Now, we can further expand
		\begin{align*}
		\hilprodd{f,g - \gstk} &= \hilprodd{f, g} - \hilprodd{f, \gstk}= \hilprodd{f, g} -  \hilprodd{\fstk, \gstk}  - \hilprodd{f-\fstk,\gstk}\\
		&= (\hilprodd{f, g} - \hst) - (\hilprodd{\fstk, \gstk}  - \hst) - \hilprodd{f-\fstk,\gstk}.
		\end{align*}
		Hence, 
		\begin{align*}
		\sigma_r(\fhat,\ghat)\Exp_{\cdy{2}}\hilnormm{\gstk - g}^2 &\le 3\Exp_{\cD_{1\otimes2}}(\hilprodd{f, g} - \hst)^2 + 3\Exp_{\cD_{1\otimes2}}(\hilprodd{\fstk, \gstk} - \hst)^2 +  3\Exp_{\cD_{1\otimes2}}\hilprodd{f-\fstk,\gstk}^2\\
		&\le 3\Risk(f,g;\cD_{1\otimes 2}) + 3\Risk(\fstk,\gstk;\cD_{1\otimes 2}) +  3\Exp_{\cD_{1\otimes2}}\hilprodd{f-\fstk,\gstk}^2. 
		\end{align*}
		By \Cref{lem:change_cov} and the fact that $f - \fstk$ is $2B$-bounded, 
		\begin{align*}
		\Exp_{\cD_{1\otimes2}}\hilprodd{f-\fstk,\gstk}^2 \le  \kapcov\Exp_{\cD_{1\otimes1}}\hilprodd{f-\fstk,\gstk}^2 + 4B^4 \etacov = \kapcov \ErrTermHil_0 + 4B^4 \etacov.  
		\end{align*}
		By \Cref{lem:change_of_risk}, 
		\begin{align*}
		\Risk(\fstk,\gstk;\cD_{1\otimes 2}) \le \kapcov \Risk(\fst_k,\gst_k; \cD_{1 \otimes 1}) + \etacov B^4 = \kapcov \errapx  + \etacov B^4.
		\end{align*}
		Finally, by applying \Cref{claim:ij_to_train},
		\begin{align*}
		\Risk(f,g;\cD_{1\otimes 2}) \le  4B^4\etatest + \kaptrain \Risk(f,g;\Dtrain)   = 4B^4\etatest + \kaptrain\errtrain.
		\end{align*}
		Thus, 
		\begin{align*}
		\sigma_r(\fhat,\ghat)\Exp_{\cdy{2}}\hilnormm{\gstk - g}^2 \le 3\left(\frac{\errtrain}{\omegaoff} \right) + 3\kapcov(\ErrTermHil_0 + \errapx(k)) + 12B^2(\etacov + \etatest).  
		\end{align*}
		This completes the proof. 
	\hfill$\blacksquare$

\subsection{Proof of \Cref{thm:generic_risk_bound} }\label{proof:generic_risk_bound}

Set $\epsilon^2 = \epsone^2$. For any $s\in \N,\epsilon > 0$ satisfying  $s < \|\Sigst\|_{\op}/40\epsilon$ and 
\begin{align*}
\epsilon^2 \ge \inf_{s' \ge s-1} \Exp_{\disone}[(\langle f,g\rangle - \langle \fst_{s'},\gst_{s'} \rangle^2)],
\end{align*}
   by \Cref{thm:one_block}, we can always find a $k$ for which
\begin{equation}
\begin{aligned}\label{eq:estimates_thing}
(\ErrTermHil_0(f,g,k) + \tailsf_2(k) + \errtrain(f,g))^2 &\lesssim \kpick^6 \epsilon^4 + \kpick^2 (\bsigst_{\kpick})^4 +  \tailsf_2(\kpick)^2 + \epstrain^4\\
    \ErrTermHil_1(f,g,k)^2 &\lesssim ({r}+s^4)\epsilon^2 + \kpick^2 (\bsigst_{\kpick})^2  +  \tailsf_1(\kpick)^2.
\end{aligned}
\end{equation}

\Cref{prop:final_error_decomp} ensures that, with the choice of $s = r+1$, 
\begin{align*}
&\Risk(f,g;\Dtest)  \\
&\lesssim_{\star}   (\ErrTermHil_1)^2  + \frac{1}{\sighat_r^2}(\tailsf_2(k) + \ErrTermHil_0 + \errtrain)^2    \\
&\le   (\ErrTermHil_1)^2  + \frac{\alpha}{(\bsigst_r)^2}(\tailsf_2(k) + \ErrTermHil_0 + \errtrain)^2    \\
&\lesssim  \left((r+\kpick^4)\epsilon^2 + \kpick^2 (\bsigst_{\kpick})^2  +  \tailsf_1(\kpick)^2\right)  + \alpha\cdot \frac{\kpick^6 \epsilon^4 + \epstrain^4 + \kpick^2 (\bsigst_{\kpick})^4 +  \tailsf_2(\kpick)^2 }{(\bsigst_r)^2} \tag{$s = r+1 \lesssim r$}\\
&\lesssim  \left(r^4\epsilon^2 + r^2 (\bsigst_{r+1})^2  +  \tailsf_1(r+1)^2\right)  + \alpha\cdot \frac{r^6 \epsilon^4 + \epstrain^4 + r^2 (\bsigst_{r+1})^4 +  \tailsf_2(r+1)^2 }{(\bsigst_r)^2} \\
&\le  \left(r^4\epsilon^2 + (1+\alpha) r^2 (\bsigst_{r+1})^2  +  \tailsf_1(r+1)^2\right)  + \alpha\cdot \frac{r^6 \epsilon^4 + \epstrain^4 +  \tailsf_2(r+1)^2 }{(\bsigst_r)^2} \\
&\le  \left(r^4\epsilon^2 + (1+\alpha) r^2 (\bsigst_{r+1})^2  +  \tailsf_1(r)^2\right)  + \alpha\cdot \frac{r^6 \epsilon^4 + \epstrain^4 +  \tailsf_2(r)^2 }{(\bsigst_r)^2} \tag{Monoticity of $\tailsf_q$}\\
&\lesssim  \left(r^4\epsilon^2 + \alpha r^2 (\bsigst_{r+1})^2  +  \tailsf_1(r)^2\right)  + \alpha\cdot \frac{r^6 \epsilon^4 + \epstrain^4 +  \tailsf_2(r)^2 }{(\bsigst_r)^2} \tag{$\alpha \ge 1$}.
%
\end{align*}
The last statement of the theorem - upper bounding $\alpha \le 2$, is precisely the last statement of \Cref{thm:one_block}. \hfill $\blacksquare$

\end{document}